%% file: BAIUV.tex
\documentclass[final, 12pt]{alt2023} 

\title[Dealing with Unknown Variances in Best-Arm Identification]{Dealing with Unknown Variances in Best-Arm Identification}

\usepackage{times}

 \altauthor{
 	\Name{Marc Jourdan} \Email{marc.jourdan@inria.fr}\\
  \Name{R\'emy Degenne} \Email{remy.degenne@inria.fr}\\
  \Name{Emilie Kaufmann} \Email{emilie.kaufmann@univ-lille.fr}\\
  \addr{Univ. Lille, CNRS, Inria, Centrale Lille, UMR 9198-CRIStAL, F-59000 Lille, France}
}

\usepackage{stmaryrd}
\usepackage{enumitem}
\usepackage{mathabx}
\usepackage{booktabs}       

\renewcommand{\ln}{\log}

\newcommand{\taud}{\tau_{\delta}}
\newcommand{\simplex}{\triangle_{K}}
\newcommand{\cA}{\mathcal{A}}

\newcommand{\cM}{\mathcal{M}}
\newcommand{\cN}{\mathcal{N}}
\newcommand{\cF}{\mathcal{F}}

\newcommand{\cO}{\mathcal{O}}
\newcommand{\cS}{\mathcal{S}}
\newcommand{\cT}{\mathcal{T}}
\newcommand{\cH}{\mathcal{H}}

\newcommand{\Real}{\mathbb{R}}
\newcommand{\Natural}{\mathbb{N}}

\newcommand{\1}{\mathbf{1}}
\newcommand{\probability}{\mathbb{P}}
\newcommand{\bP}{\mathbb{P}}
\DeclareMathOperator*{\expectedvalue}{\mathbb{E}}

\DeclareMathOperator*{\argmin}{arg\,min}
\DeclareMathOperator*{\argmax}{arg\,max}

\newcommand{\eqdef}{\buildrel \text{def}\over =}

\DeclareMathOperator{\KL}{KL}
\DeclareMathOperator{\GLR}{GLR}

\renewcommand{\epsilon}{\varepsilon}
\renewcommand{\ln}{\log}

\PassOptionsToPackage{vlined}{algorithm2e}
\RestyleAlgo{ruled}

\begin{document}

\maketitle

\begin{abstract}%
The problem of identifying the best arm among a collection of items having Gaussian rewards distribution is well understood when the variances are known.
Despite its practical relevance for many applications, few works studied it for unknown variances.
In this paper we introduce and analyze two approaches to deal with unknown variances, either by plugging in the empirical variance or by adapting the transportation costs.
In order to calibrate our two stopping rules, we derive new time-uniform concentration inequalities, which are of independent interest.
Then, we illustrate the theoretical and empirical performances of our two sampling rule wrappers on Track-and-Stop and on a Top Two algorithm.
Moreover, by quantifying the impact on the sample complexity of not knowing the variances, we reveal that it is rather small.
\end{abstract}


\input{sections/introduction}

\input{sections/lower_bound_and_glr}

\input{sections/calibration_stopping_rules}

\input{sections/sampling}

\section{Conclusion}

In this paper we provided two approaches to deal with unknown variances, either by plugging in the empirical variance or by adapting the transportation costs.
New time-uniform concentration results were derived to calibrate our two stopping rules.
Then, we showed theoretical guarantees and competitive empirical performance of our two sampling rule wrappers on two existing algorithms.

While the literature abounds with designs of sampling rule, the optimal calibration of stopping rules is a most pressing issue as it leads to lower empirical stopping time.
While calibrated thresholds have been derived with (near) optimal dependency in $\delta$, those thresholds are known to be too conservative in the moderate confidence regime where their empirical error rate is orders of magnitude lower than $\delta$.

Finally, in the fixed-budget setting, characterizing the impact of not knowing the variances on the probability of misidentifying the best-arm is still an open problem.
While similar approaches might be used to deal with the unknown variances, the resulting algorithms might not enjoy similar theoretical guarantees and empirical performance.


\acks{Experiments presented in this paper were carried out using the Grid'5000 testbed, supported by a scientific interest group hosted by Inria and including CNRS, RENATER and several Universities as well as other organizations (see https://www.grid5000.fr).  This work has been partially supported by the THIA ANR program ``AI\_PhD@Lille''. The authors acknowledge the funding of the French National Research Agency under the projects BOLD (ANR-19-CE23-0026-04) and FATE (ANR-22-CE23-0016-01).}

\bibliographystyle{abbrvnat}
\bibliography{BAIUV}

\appendix

\section{Outline} \label{app:outline}

\begin{itemize}
	\item Notation are summarized in Appendix~\ref{app:notations}.
	\item In Appendix~\ref{app:characteristic_times}, we study the characteristic times $T^\star_{\sigma^2}(\mu)$, $T^\star(\mu, \sigma^2)$ and how to compute the optimal allocation oracles $w^\star_{\sigma^2}(\mu)$, $w^\star(\mu, \sigma^2)$.
	\item Properties on the GLR and EV-GLR statistics are detailed in Appendix~\ref{app:glr}.
	\item In Appendix~\ref{app:box_concentration}, we prove time-uniform and fixed time upper and lower tail concentrations for sub-exponential process and the empirical mean and variance of Gaussian observations.
	\item For $d$-dimensional exponential families, a time-uniform concentration inequality on the sum of KL is shown in Appendix~\ref{app:kl_concentration}.
	\item In Appendix~\ref{app:thresholds}, we show how to calibrate the GLR and EV-GLR stopping rules and derive several family of thresholds.
	\item Asymptotic guarantees on the expected sample complexity (Theorem~\ref{thm:upper_bound_sample_complexity_algorithm}) and the impossibility results (Theorem~\ref{thm:impossibility_result}) are proved in	Appendix~\ref{app:expected_sample_complexity}.
	\item The functions $(\overline{W}_{i})_{i \in \{-1,0\}}$ based on Lambert's branches and their properties are presented in Appendix~\ref{app:lambert_W_functions}.
	\item Implementation details and additional experiments are available in Appendix~\ref{app:additional_experiments}.
\end{itemize}

\begin{table}[h]
\caption{Notation for the setting.}
\label{tab:notation_table_setting}
\begin{center}
\begin{tabular}{c c l}
\toprule
Notation & Type & Description \\
\midrule
$K$ & $\mathbb{N}$ & Number of arms \\
$\mathcal D$ & & Set of Gaussian distributions \\
$\mathcal D_{v}$ & & Set of Gaussian distributions with variance $v$ \\
$\nu$ & $ \mathcal D^K$ & Vector of distributions, $\nu = (\nu_a)_{a \in [K]}$ \\
$\mu$ & $\mathbb{R}^K$ & Vector of means, $\mu = (\mu_a)_{a \in [K]}$ \\
$\sigma^2$ & $(\mathbb{R}^\star_+)^K$ & Vector of variances, $\sigma^2 = (\sigma^2_a)_{a \in [K]}$ \\
$\nu_{\mu, \sigma^2}$ & $\mathcal D^K$ & Problem in which each arm has distribution $\mathcal N(\mu_a, \sigma^2_a)$ \\
$\theta_a$ & $\mathbb{R} \times \mathbb{R}^\star_-$ & Natural parameter of the distribution of arm $a \in [K]$ \\
$\theta$ & $(\mathbb{R} \times \mathbb{R}^\star_-)^K$ & Vector of natural parameters, $\theta = (\theta_a)_{a \in [K]}$ \\
$F$ & $\mathbb{R}\to \mathbb{R} \times \mathbb{R}_+$ & Sufficient statistic \\
$\phi$ & $(\mathbb{R} \times \mathbb{R}^\star_-) \to \mathbb{R}$ & Log-partition function \\
$T^\star_{\sigma^2}(\mu), w^\star_{\sigma^2}(\mu)$ & & Characteristic time and optimal allocation, known $\sigma^2$ \\
$T^\star_{\sigma^2, \beta}(\mu), w^\star_{\sigma^2, \beta}(\mu)$ & & Characteristic time and $\beta$-optimal allocation, known $\sigma^2$ \\
$T^\star(\mu, \sigma^2), w^\star(\mu, \sigma^2)$ & & Characteristic time and optimal allocation\\
$T_{\beta}^\star(\mu, \sigma^2), w_{\beta}^\star(\mu, \sigma^2)$ & & Characteristic time and $\beta$-optimal allocation\\
	\bottomrule
\end{tabular}
\end{center}
\end{table}

\section{Notation} \label{app:notations}

\begin{table}
\caption{Notation for the algorithms.}
\label{tab:notation_table_algorithms}
\begin{center}
\begin{tabular}{c c l}
\toprule
Notation & Type & Description \\
\midrule
$\delta$ &  $(0,1)$ & Confidence parameter \\
$a_t$ & $[K]$ & Arm sampled at time $t$ \\
$X_{t, a_t}$ & $\mathbb{R}$ & Observation at time $t$, $X_{t, a_t} \sim \nu_{a_t}$ \\
$\cF_t$ &   & History up to time $t$, $\sigma(a_1, X_{1,a_1}, \cdots, a_{t}, X_{t,a_t})$ \\
$N_{t,a}$ & $\mathbb{N}$ & Empirical count, $N_{t,a} = \sum_{s=1}^t \1\{a_s = a\}$ \\
$\mu_{t,a}$ & $\mathbb{R}$ & Empirical mean, $\mu_{t,a} = \frac{1}{N_{t,a}}\sum_{s=1}^t X_{s,a_s} \1\{a_s = a\}$ \\
$\sigma_{t,a}^2$ & $\mathbb{R}^\star_+$ & Empirical variance, $\sigma_{t,a}^2 = \frac{1}{N_{t,a}}\sum_{s=1}^t (X_{s,a_s} - \hat{\mu}_{t,a})^2 \1\{a_s = a\}$ \\
$\taud$, $\taud^{\text{EV}}$ & $\mathbb{N}$ & Stopping times (sample complexity)  \\
$\hat{a}_{t}$ & $[K]$ & Candidate arm at time $t$, $\hat{a}_{t} \in \argmax_{a \in [K]} \mu_{t,a}$ \\
$Z_{a}(t)$, $Z^{\text{EV}}_{a}(t)$ & $\mathbb{R}^\star_+$ & GLR and EV-GLR statistic of arm $a$ at time $t$ \\
$c_{a,b}(N_t,\delta)$ & $\mathbb{R}$ & Stopping threshold at time $t$ for the arm pair $(a,b)$ \\
$F_{t,a}$ & $\mathbb{R} \times \mathbb{R}_+$ & Averaged statistic, $\frac{1}{N_{t,a}} \sum_{s=1}^t F(X_{s,a_s}) \1\{a_s = a\} $ \\
$\theta_{t,a}$ & $\mathbb{R} \times \mathbb{R}^\star_-$ & Maximum likelihood estimator of $\theta_a$, $\nabla\phi^{-1}(F_{t,a})$ for $N_{t,a} \ge 2$ \\
\bottomrule
\end{tabular}
\end{center}
\end{table}

We recall some commonly used notation:
the set of integers $[K] \eqdef \{1, \cdots, K\}$,
the complement $X^{\complement}$ of a set $X$,
the $(K-1)$-dimensional probability simplex $\simplex \eqdef \left\{w \in \Real_{+}^{K} \mid  \sum_{a \in [K]} w_a = 1 \right\}$,
the Gaussian distribution $\mathcal N(x, v)$ with mean $x$ and variance $v$,
the Kullback-Leibler (KL) divergence $\KL((x_1, \sigma_1^2),(x_2, \sigma_2^2))$ between two distributions $\cN(x_1, \sigma_1^2)$ and $\cN(x_2, \sigma_2^2)$,
Landau's notation $o$ and $\cO$ and
the two main branches $W_{-1}$ (negative) and $W_0$ (positive) of the Lambert $W$ function which is implicitly defined by the equation $W(x)e^{W(x)} = x$.
Problem-specific notation are grouped in Table~\ref{tab:notation_table_setting}.
Table~\ref{tab:notation_table_algorithms} gathers notation for the algorithms.

\input{sections/appendix_lower_bounds}

\input{sections/appendix_glrt}

\input{sections/appendix_box_concentration}

\input{sections/appendix_kl_concentration}

\input{sections/appendix_thresholds}

\input{sections/appendix_tas_optimality}

\input{sections/appendix_lambert_W}

\input{sections/appendix_additional_experiments}

\end{document}

%% file: sections/introduction.tex

\section{Introduction}

In a bandit model, an agent sequentially collects samples from unknown probability distributions, called arms.
These samples may be viewed as rewards that the agent seeks to maximize, or equivalently minimize its regret \citep{Bubeck12survey}.
In this paper our focus is instead on a Best Arm Identification (BAI) problem in which the agent should identify the arm that has the largest expected reward using as few samples as possible, without incentive on maximizing rewards.

We are interested in a Gaussian bandit model in which the variances of the arms are \emph{unknown}.
Quite surprisingly, and despite its practical relevance, this problem has received little attention in the bandit literature.
Gaussian distributions could indeed be used to model the revenue generated by different versions of a website in the context of A/B testing, or some biological indicator of the efficiency of a treatment in the context of an adaptive clinical trial comparing several treatments.
In both case, assuming known variances is a limitation.
Formally, we consider a bandit model with $K$ arms $\nu_{1},\dots,\nu_{K}$ in which $\nu_a$ is a Gaussian distribution with mean $\mu_a$ and variance $\sigma_a$.
The best arm (assumed unique) is defined as the arm with largest mean  $a^\star(\mu) \eqdef \argmax_{a \in [K]} \mu_a$.
We consider the fixed confidence setting, in which the parameter $\delta \in (0, 1)$ is an upper bound on the probability that the algorithm makes an error.

A fixed confidence BAI algorithm is made of a {sampling rule} and a {stopping and recommendation rule}. In each time $t \in \mathbb{N}$, an arm $a_t \in [K]$ is chosen by the \emph{sampling rule}, then an observation $X_{t,a_t} \sim \nu_{a_t}$ is received.
The choice of $a_t$ may depend on a random variable $U_{t-1}$, independent of everything else, which models internal randomization.
The $\sigma$-algebra generated by $(U_0, X_{1,a_1}, \ldots, U_{t-1}, X_{t,a_t}, U_t)$ is denoted by $\mathcal F_t$. $a_t$ is then an $\mathcal F_{t-1}$-measurable random variable, and $X_t$ is independent of $\mathcal F_{t-1}$ conditionally on $a_t$.
The \emph{stopping rule} is a stopping time with respect to the filtration $(\mathcal F_t)_{t \in \mathbb{N}}$, denoted by $\taud$.
When the algorithm stops, it recommends an arm $\hat{a}_{\taud} \in [K]$, which is measurable with respect to $\taud$.
$\taud$ is called the sample complexity of the algorithm.

The goal of a fixed confidence best arm identification algorithm is to return the best arm with high probability while having low sample complexity.
The main requirement we impose on a fixed confidence identification method is $\delta$-\emph{correctness}.

\begin{definition}[$\delta$-correct]
Let $\mathcal D$ be a set of distributions on $\mathbb{R}$.
Given $\delta \in (0, 1)$, we say that an identification strategy is $\delta$-correct on the problem class $\mathcal D^K$ if for all $\nu = (\nu_a)_{a \in [K]} \in \mathcal D^K$,
$\probability_{\nu} \left( \taud < +\infty , \: \hat{a}_{\taud} \ne a^\star(\mu) \right) \le \delta$~.
\end{definition}

We follow the approach pioneered by \citet{garivier_2016_OptimalBestArm} and initially introduced for one-dimensional parametric models (e.g. Gaussian with known variance).
They derived lower bounds on the expected sample complexity of $\delta$-correct algorithms and introduced algorithms inspired by the maximization of those lower bounds.
Extending their lower bound to our two-parameters setting allows us to quantify the impact on the expected sample complexity of not knowing the variances, and reveals that this impact is rather small.
To leverage the stopping and sampling rules of existing algorithms, we propose two approaches to deal with unknown variances: plugging in the empirical variance or considering the transportation costs for unknown variance.

As it is common in previous work for the stopping rule, we will compare a Generalized Likelihood Ratio (GLR) to a well chosen threshold \citep{kaufmann_2018_MixtureMartingalesRevisited}.
Our two approaches yield the Empirical Variance GLR (EV-GLR) stopping rule, which plugs in the empirical variance in a GLR assuming known variance, and the GLR stopping rule, which corresponds to a GLR assuming unknown variance.
Our main technical contribution lies in the derivation of (near) optimal stopping thresholds which ensure the $\delta$-correctness of both the GLR and the EV-GLR stopping rules, regardless of the sampling rule.
These thresholds are based on new time-uniform concentration inequalities for Gaussian with unknown variances, which are of independent interest (Corollary~\ref{cor:uniform_time_upper_lower_tail_concentration_variance} and Theorem~\ref{thm:uniform_upper_tail_concentration_kl_exp_fam_gaussian}).

When considering the sampling rule, each approach yields a wrapper which is a simple procedure that can be applied to any BAI algorithm for known variances.
We illustrate each wrapper with Track-and-Stop \citep{garivier_2016_OptimalBestArm} and the Top Two algorithm $\beta$-EB-TCI \citep{jourdan_2022_TopTwoAlgorithms}.
By deriving upper bound on the expected sample complexity, we show that algorithms obtained by adapting the transportation costs enjoy stronger
theoretical guarantees than the ones plugging in the empirical variance.
In particular, we propose the first asymptotically optimal algorithms for Gaussian bandits with unknown variances.
Our experiments reveal that both wrappers have comparable performance when applied to several BAI algorithms including the ones above, DKM \citep{degenne_2019_NonAsymptoticPureExploration} and FWS \citep{wang_2021_FastPureExploration}.
This reinforces our finding that not knowing the variances has a small impact on the sample complexity.

\paragraph{Related work}
Algorithms based on GLR stopping rules and aimed at matching a sample complexity lower bound were either studied for one-parameter exponential families \citep{degenne_2019_NonAsymptoticPureExploration} or under generic heavy tails assumption \citep{Agrawal20GeneBAI}.
Other algorithms are either based on eliminations or on confidence intervals and have been mostly analyzed for sub-Gaussian distributions with a known variance proxy\footnote{A random variable $X$ with mean $\mu$ is $\sigma^2$ sub-Gaussian if $\mathbb{E}[\exp(\lambda(X-\mu))] \leq \frac{\lambda^2\sigma^2}{2}$ for all $\lambda\in \mathbb{R}$.} \citep{even_dar_2006_ActionEliminationStopping,kalyanakrishnan_2012_PACSubsetSelection,Jamiesonal14LILUCB}.
For the special case of bounded distributions, confidence intervals based on the empirical variance have been used \citep{gabillon_2012_BestArmIdentification,Lu21MEVariance} but the resulting algorithms cannot be applied to unbounded distributions as they rely on the empirical Bernstein inequality \citep{MaurerPontil09}.
In the fixed budget setting, in which the size of the exploration phase is fixed in advance, it is possible to upper bound the error probability of the Successive Reject algorithm of \citet{audibert_2010_BestArmIdentification} when the variances are unknown, as we only need to upper bound the probability that one empirical mean is smaller than another, see also \citet{Faella20FBVariance}.
However, in the fixed-confidence setting elimination thresholds, confidence intervals or GLR tests need to be calibrated in a data-dependent way, which calls for the development of new time-uniform concentration inequalities, that we provide in this work.

In the related literature on ranking and selection \citep{hong_2021_ReviewRankingSelection}, the problem of finding the Gaussian distribution with largest mean has been studied for unknown variances.
This literature mostly seek to design algorithm that are $\delta$-correct whenever the gap between the best and second best arm is larger than some specified indifference zone \citep{KimNelson01}.
However the work of \citet{Fan16RSUV} does not consider an indifference zone and their algorithm is therefore comparable to ours.
They propose an elimination strategy which features the empirical variances and whose calibration is done based on simulation arguments (resorting to continuous-time approximations) and justified in an asymptotic regime only (when $\delta$ goes to zero).
Our algorithms have better empirical performance and stronger theoretical guarantees.

%% file: sections/lower_bound_and_glr.tex

\section{Lower Bounds and GLR-based Stopping Rules}
\label{sec:lower_bounds_and_glrs}

First, we introduce the lower bounds characterizing the complexity of the setting in Section~\ref{ssec:lower_bounds}.
Then, we present the generalized log-likelihood ratios (GLR) stopping rules in Section~\ref{ssec:GLR_stopping_rule}.

\subsection{Lower Bounds}
\label{ssec:lower_bounds}

In the following, all the distributions are Gaussian denoted by $\nu_{x, \sigma^2} = \cN (x, \sigma^2)$.
The class of Gaussian distributions with known variance $\sigma^2$ is denoted by $\mathcal D_{\sigma^2}
= \{ \nu_{x, \sigma^2} \mid \exists x \in \mathbb{R} \}$, and the class of Gaussian distributions with unknown variance by $\mathcal D
= \bigcup_{\sigma^2 > 0} \mathcal D_{\sigma^2}$.
We denote the Kullback-Leibler (KL) divergence between $\nu_{x_1, \sigma_1^2}$ and $\nu_{x_2, \sigma_2^2}$ by $\KL((x_1, \sigma_1^2), (x_2, \sigma_2^2))$.

Let $(\mu,\sigma^2)  \in \cM = \mathbb{R}^K \times (\mathbb{R}^{\star}_+)^K$ such that $|a^\star(\mu)|=1$.
The alternative sets $\Lambda(\mu, \sigma^2) = \{(\lambda, \kappa^2) \in \cM \mid a^\star(\mu) \notin \argmax_{a } \lambda_a \}$ and $\Lambda_{\sigma^2}(\mu) = \{\lambda \mid (\lambda , \sigma^2) \in \cM, \: a^\star(\mu) \notin \argmax_{a} \lambda_a\}$ are the sets of parameter for which $a^\star(\mu)$ is not the best arm.
The $(K-1)$-dimensional probability simplex is denoted by $\triangle_K = \{w \in \Real^{K}_{+} \mid \sum_{a \in [K]} w_a = 1 \}$.

For Gaussian with unknown (resp. known) variances, Lemma~\ref{lem:lower_bound_sample_complexity} shows that $T^\star(\mu, \sigma^2)$ (resp. $T^\star_{\sigma^2}(\mu)$) is the asymptotic complexity of the BAI problem on the instance $\nu \eqdef (\nu_{\mu_a,\sigma_{a}^2})_{a }$, where
\begin{align*}
	T^\star(\mu, \sigma^2)^{-1}
	&= \sup_{w \in \triangle_K} \inf_{(\lambda, \kappa^2) \in \Lambda(\mu, \sigma^2)} \sum_{a \in [K]} w_a \KL((\mu_a, \sigma_{a}^2), (\lambda_a, \kappa_{a}^2))
	\: , \\
T^\star_{\sigma^2}(\mu)^{-1}
&= \sup_{w \in \triangle_K} \inf_{\lambda \in \Lambda_{\sigma^2}(\mu)} \sum_{a \in [K]} w_a \KL((\mu_a, \sigma_{a}^2), (\lambda_a, \sigma_a^2))
\: .
\end{align*}
The maximizer over the simplex $\triangle_K$ in these complexities is denoted by $w^\star(\mu, \sigma^2)$ and $w^\star_{\sigma^2}(\mu)$.
The rationale for the difference between the $T^\star(\mu, \sigma^2)$ and $T^\star_{\sigma^2}(\mu)$ is that when the variances are unknown, there exist instances of the form $(\lambda, \kappa^2)$ for $\kappa \neq \sigma$ that are harder to differentiate from $(\mu, \sigma^2)$ than instances of the form  $(\lambda, \sigma^2)$ with respect to an information criterion.

\begin{lemma}[\citet{garivier_2016_OptimalBestArm}] \label{lem:lower_bound_sample_complexity}
An algorithm which is $\delta$-correct on all problems in $\mathcal D_{\sigma^2}^K$ satisfies that for all $\mu \in \mathbb{R}^K$,
$
\mathbb{E}_{\nu}[\taud] \ge T^\star_{\sigma^2}(\mu) \ln(1/(2.4\delta)) \: .
$

An algorithm which is $\delta$-correct on all problems in $\mathcal D^K$ satisfies that for all $(\mu,\sigma^2)  \in \cM $,
$
\mathbb{E}_{\nu}[\taud] \ge T^\star(\mu, \sigma^2) \ln(1/(2.4\delta)) \: .
$
\end{lemma}

We say that an algorithm is asymptotically optimal on $\mathcal D^K$ if it is $\delta$-correct and its sample complexity matches that lower bound, i.e. $\liminf_{\delta \rightarrow 0} \expectedvalue_{\mu}[\taud] / \ln(1/\delta) \leq T^\star(\mu,\sigma^2)$.
A weaker notion of optimality is $\beta$-optimality \citep{Qin2017TTEI,Shang20TTTS}.
An algorithm is called asymptotically $\beta$-optimal on $\mathcal D^K$ if it satisfies $\liminf_{\delta \rightarrow 0} \expectedvalue_{\mu}[\taud] / \ln(1/\delta) \leq T_{\beta}^\star(\mu,\sigma^2)$ and is $\delta$-correct, for $T_{\beta}^\star(\mu,\sigma^2)$ defined as follows.
For $\beta \in (0,1)$, the definition of $T^\star_{\beta}(\mu,\sigma^2)$ is the same as $T^\star(\mu,\sigma^2)$ with the additional constraint on the outer maximization that $w_{a^\star} = \beta$, hence $T^\star(\mu,\sigma^2) = \min_{\beta \in (0,1)} T^\star_{\beta}(\mu,\sigma^2)$.

An asymptotically $\beta$-optimal algorithm is asymptotically minimizing the sample complexity among algorithms which allocate a $\beta$ fraction of samples to the best arm.
\citet{Russo2016TTTS} shows that an asymptotically $\beta$-optimal algorithm with $\beta = 1/2$ also has an expected sample complexity which is asymptotically optimal, up to a multiplicative factor $2$, i.e. $T^\star_{1/2}(\mu,\sigma^2) \le 2 T^\star(\mu,\sigma^2)$.
The $\beta$-optimality on $\mathcal D_{\sigma^2}^K$ involves $T^\star_{\sigma^2, \beta}(\mu)$, which is similarly related to $T^\star_{\sigma^2}(\mu)$.
While there is a rich literature on asymptotically ($\beta$-)optimal algorithms for Gaussian with known variance, we are the first to derive algorithms with those guarantees when the variances are unknown.

\subsection{Comparing the Complexities}

To compare $T^\star_{\sigma^2}(\mu)$ and $T^\star(\mu,\sigma^2)$, we first propose a more explicit expression of the infimum over the  alternative set featured in their expression, in terms of appropriate \emph{transportation costs}.

\begin{lemma}\label{lem:calcul}
	For every $\mu$ such that $a^\star(\mu) = \{a^\star\}$ and $w \in \mathbb{R}_+^{K}$,
	\[
	\inf_{(\lambda,\kappa^2) \in \Lambda(\mu,\sigma^2)} \sum_{a \in [K]} w_a \KL\left((\mu_a,\sigma_a^2), (\lambda_a, \kappa_a^2)\right)  = \min_{a \neq a^\star} C(a^\star, a ; w) \: ,
	\]
where the transportation cost from $a$ to $b$ given an allocation $w$ is defined by
\begin{align}
	C(a, b ; w)  & =   \1\{\mu_a > \mu_b\} \inf_{\substack{ \lambda_b \geq \lambda_{a} \\ \kappa_a \geq 0, \kappa_{b} \geq 0}} \sum_{c \in \{a, b\}} w_c \KL\left((\mu_c,\sigma_c^2) , (\lambda_c,\kappa_c^2)\right) \label{eq:KL_without_inf} \\
 &= \1\{\mu_a > \mu_b\} \inf_{\lambda \in (\mu_b,\mu_{a})} \sum_{c \in \{a, b\}} \frac{w_c}{2} \log\left(1 + \frac{(\mu_c - \lambda)^2}{\sigma_c^2}\right) \: . \nonumber
\end{align}
\end{lemma}
From the proof (Appendix~\ref{app:ss_explicit_formulas_inequalities}) we note that the minimizer in $\kappa$ is $\kappa_a = \sigma_a^2 + (\mu_a - \lambda)^2$, thus even if we want to identify the arm with largest mean, the closest alternatives have an increased variance.
When the variances are known, computing the infimum over the alternative $\lambda\in \Lambda_{\sigma^2}(\mu)$ yields the same expression but with a different transportation cost, which has a convenient closed form:
\[
 C_{\sigma^2}(a, b ; w) = \1\{\mu_a > \mu_b\} \inf_{\lambda \in (\mu_b,\mu_{a})} \sum_{c \in \{a, b\}} {w_c}\frac{(\mu_c - \lambda)^2}{2\sigma_c^2} = \1\{\mu_a > \mu_b\} \frac{1}{2} \frac{(\mu_{a} - \mu_b)^2}{\sigma_{a}^2/w_{a} + \sigma_{b}^2/w_{b}} \: .
\]
On the contrary, the infimum in the mean parameter $\lambda$ in the transportation cost for unknown variance has no simple analytic form (see Appendix~\ref{app:ss_optimal_allocation_oracles} for details on its computation).
Still, comparing the two types of transportation costs (and using properties of the mapping $x \mapsto \log(1+x)/x$) permits to establish a link between $T^\star_{\sigma^2}(\mu)$ and $T^\star(\mu,\sigma^2)$ (resp. $T^\star_{\sigma^2, \beta}(\mu)$ and $T^\star_{\beta}(\mu,\sigma^2)$), hence to quantify the impact of not knowing the variances.

\begin{lemma} \label{lem:complexity_inequalities}
Let $d(\mu, \sigma^2) = \underset{a \neq a^\star(\mu)}{\max}\frac{(\mu_{a^\star(\mu)} - \mu_{a})^2}{\min\{\sigma_{a}^2,\sigma_{a^\star(\mu)}^2\}}$.
Then,
\begin{equation} \label{eq:characteristic_times_inequalities}
	1 < \frac{T^\star(\mu, \sigma^2)}{T^\star_{\sigma^2}(\mu)} \leq \frac{d(\mu, \sigma^2)}{\ln\left(1+d(\mu, \sigma^2)\right)}   \quad \text{and} \quad 	1 < \frac{T^\star_{\beta}(\mu, \sigma^2)}{T^\star_{\sigma^2, \beta}(\mu)} \leq \frac{d(\mu, \sigma^2)}{\ln\left(1+d(\mu, \sigma^2)\right)}  \: .
\end{equation}
\end{lemma}

When $d(\mu, \sigma^2)$ is small, say $d(\mu, \sigma^2) \leq 1$, the two complexities are close since we then have $T^\star_{\sigma^2}(\mu)/T^\star(\mu, \sigma^2) \in [\log 2,1)$.
Observe that a small $d(\mu, \sigma^2)$ also implies that the BAI problem is hard: if $d(\mu, \sigma^2) \le c \in \mathbb{R}_+$ then for all $a \in [K]$, $\frac{\min\{\sigma_{a}^2,\sigma_{a^\star(\mu)}^2\}}{(\mu_{a^\star(\mu)} - \mu_{a})^2} \ge c^{-1}$.
Since that ratio is roughly the number of samples needed to distinguish the two arms, the problem is hard when it is large.
Still, there exist instances with an arbitrarily large complexity ratio ${T^\star(\mu, \sigma^2)}/{T^\star_{\sigma^2}(\mu)}$ (Lemma~\ref{lem:ratio_characteristic_time_large}).
We conjecture that they always correspond to easy problems, for which both $T^\star_{\sigma^2}(\mu)$ and $T^\star(\mu, \sigma^2)$ are small.
Lemma~\ref{lem:complexity_inequalities} is not sufficient to prove this conjecture as there exists hard instances with a large value of $d(\mu, \sigma^2)$ and instances for which the upper bound in~\eqref{eq:characteristic_times_inequalities} is not tight (Appendix~\ref{app:sss_exp_characteristic_times}).



\subsection{GLR Stopping Rules}
\label{ssec:GLR_stopping_rule}

Given any sampling rule, constructing a stopping and recommendation rule for the BAI problem may be viewed as a sequential testing problem with multiple hypotheses $ \left\{ \mu_{a} = \max_{b \in [K]} \mu_{b}\right\}$.
In one of the first papers on active hypothesis testing (in which the data collection process is further optimized), \citet{chernoff_1959_SequentialDesignExperiments} proposed to rely on Generalized Likelihood Ratio Tests (GLRT) for stopping.
This idea was later popularized by \citet{garivier_2016_OptimalBestArm} for the BAI problem.

For all $a \in [K]$, let $N_{t,a} = \sum_{s \in [t]} \1\{a_s = a\}$, $\mu_{t,a}$ and $\sigma_{t,a}^{2}$ be the empirical count, mean and variance of arm $a$ after time $t$, where
\begin{equation*} 
	\mu_{t,a} \eqdef \frac{1}{N_{t,a}} \sum_{s \in [t]} \1\{a_s = a\} X_{s,a}
	\quad \quad \text{and} \quad \quad
	\sigma_{t,a}^{2} \eqdef \frac{1}{N_{t,a}} \sum_{s \in [t]} \1\{a_s = a\}  \left( X_{s,a} - \mu_{t,a}\right)^2 \: .
\end{equation*}

For Gaussian with unknown variances, the GLR to reject $(\mu, \sigma^2) \in \Lambda$, with $\Lambda \subseteq \mathcal{D}$, is written as 
\begin{align} \label{eq:GLR_expression}
 \text{GLR}^{\mathcal{D}}_{t}(\Lambda)
=  \inf_{(\lambda, \kappa^2) \in \Lambda}  \sum_{a \in [K]} N_{t,a} \KL((\mu_{t,a}, \sigma_{t,a}^2), (\lambda_a, \kappa_{a}^2)) \: ,
\end{align}
which is reminiscent to the expression in the lower bound. We let $\hat a_t \eqdef a^\star(\mu_t)$ denote the empirical best arm (EB). Similar calculations as in the proof of Lemma~\ref{lem:calcul} yield that $\GLR^{\mathcal{D}}_{t}(\Lambda(\mu_t, \sigma_t^2)) = \min_{a \neq \hat{a}_t} Z_a(t)$ where the GLR statistic of arm $a \neq \hat a_t$ is defined as
\begin{equation*} 
	Z_a(t) = \GLR^{\mathcal{D}}_{t}(\{(\lambda, \kappa^2) \mid \lambda_a \ge \lambda_{\hat a_t}\}) = \inf_{ \lambda \in [\mu_{t,a},\mu_{t,\hat a_t}]}  \sum_{b \in \{a, \hat a_t\}} \frac{N_{t,b}}{2} \ln \left( 1 + \frac{(\mu_{t,b} - \lambda)^2}{\sigma_{t,b}^2}\right) \: ,
\end{equation*}
which we refer to as the empirical transportation cost between arm $\hat a_t$ and arm $a$.

\paragraph{GLR Stopping Rule}
In its general form, the GLR stopping rule triggers when $\GLR^{\mathcal{D}}_{t}(\Lambda(\mu_t, \sigma_t^2)) $ exceeds a threshold $c(t, \delta)$. Here we propose to further exploit the structure of the problem and use a family of thresholds $c_{a,b} : \mathbb{N}^K \times (0,1] \rightarrow \mathbb{R}_+$ for all $(a,b) \in [K]^2$, leading to the stopping rule
\begin{equation} \label{eq:def_stopping_rule_glrt}
	\taud
	\eqdef \inf \left\{t \in \mathbb{N} \mid \forall a \ne \hat a_t, \: Z_a(t) > c_{\hat a_t,a}(N_t,\delta) \right\}
	\: .
\end{equation}

\paragraph{EV-GLR Stopping Rule}
For known variances, we should consider $\text{GLR}^{\mathcal{D}_{\sigma^2}}(\{\lambda : \lambda_a \geq \lambda_{\hat a_t}\})$, which can be computed in closed-form and depends on the variance $\sigma^2$.
Replacing the variance vector $\sigma^2$ by its empirical estimate $\sigma^2_t$ yields the Empirical Variance GLR (EV-GLR) statistic
\begin{equation*} 
	Z^{\text{EV}}_a(t) =
	\inf_{ u \in [\mu_{t,a},\mu_{t,\hat a_t}]}   \sum_{b \in \{a, \hat{a}_t\}} N_{t,b} \frac{(\mu_{t,b} - u)^2}{2\sigma_{t,b}^2} = \frac{1}{2} \frac{(\mu_{t,a} - \mu_{t,\hat{a}_t})^2}{\sigma_{t,a}^2/N_{t,a} + \sigma_{t,\hat{a}_t}^2/N_{t,\hat{a}_t}} \: .
\end{equation*}
The EV-GLR stopping rule given a family of thresholds $(c_{a,b})_{(a,b) \in [K]^2}$ is defined as
\begin{equation} \label{eq:def_stopping_rule_evglrt}
	\taud^{\text{EV}} \eqdef \inf \left\{t \in \mathbb{N} \mid \forall a \ne \hat{a}_{t}, \: Z^{\text{EV}}_a(t) > c_{\hat{a}_{t},a}(N_t,\delta) \right\} \: .
\end{equation}

Given their proximity with the lower bound --see \eqref{eq:GLR_expression}--, GLR stopping rules are good candidates to match $T^\star$. Indeed, it is easy to prove that sampling arms from $w^\star$ and using the threshold $c_{a,b}(N,\delta) = \log(1/\delta)$, the lower bound would be matched. However, such a threshold is too good to be $\delta$-correct (Section~\ref{sec:calibration_stopping_rules}).
Moreover, $w^\star$ needs to be estimated since it is unknown (Section~\ref{sec:sampling_rules}).

%% file: sections/calibration_stopping_rules.tex

\section{Calibration of the Stopping Thresholds}
\label{sec:calibration_stopping_rules}

We present ways of calibrating the thresholds used by the GLR stopping rule, by leveraging concentration arguments.
Under any sampling rule, to obtain a $\delta$-correct GLR stopping rule it suffices to show that the family of thresholds is such that the following time-uniform concentration inequality holds for all $\nu \in \mathcal D^{K}$: with probability $1-\delta$, for all $t \in \Natural$ and for all $a \neq a^\star(\mu)$,
\begin{equation} \label{eq:log_based_concentration}
	 \sum_{b \in \{a, a^\star(\mu)\}}\frac{N_{t,b}}{2} \ln \left( 1 + \frac{(\mu_{t,b} - \mu_{b})^2}{\sigma_{t,b}^2}\right)
\le c_{a, a^\star(\mu)}(N_t,\delta)  \: .
\end{equation}

Aiming at matching the lower bound, we want to derive a family of thresholds satisfying $c_{a,b}(N, \delta) \sim_{\delta \to 0} \ln\left(1/\delta\right) $. As regards the time dependency, generalizations of the law of the iterated logarithm suggest we could achieve $\cO(\ln \ln t)$. Both dependencies are achieved for known variances \citep{kaufmann_2018_MixtureMartingalesRevisited}, and we are the first to show it for unknown variances (Theorem~\ref{thm:delta_correct_complex_threshold_glrt}).
While simple ideas yield $\delta$-correct thresholds (Section~\ref{ssec:simple_ideas}), obtaining the ideal dependency in $\delta$ requires sophisticated concentration arguments (Section~\ref{ssec:beyond_box}).

Similar arguments can be used to calibrate the thresholds used by the EV-GLR stopping rule (Appendix~\ref{app:thresholds}).
Moreover, $\delta$-correct thresholds for the EV-GLR stopping rule can be obtained by using the ones calibrated for GLR stopping rule, and vice-versa (Lemma~\ref{lem:glrt_evglrt_stopping_threshold_relationships}).

\subsection{Simple Ideas}
\label{ssec:simple_ideas}

As per-arm concentration results are easier to obtain, we first control each term of the sum in (\ref{eq:log_based_concentration}).

\paragraph{Student thresholds}
Since $(\mu_{t,a} - \mu_{a})/\sigma_{t,a} $ is an observation of the Student distribution $\cT_{N_{t,a}-1}$,
a first simple approach involves the quantiles of Student distributions with $n$ degrees of freedom.
A direct union bound over time and arms yield a $\delta$-correct family of thresholds (Lemma~\ref{lem:delta_correct_student_thresholds}).

\paragraph{Box thresholds}
As illustrated in Figure~\ref{fig:stopping_thresholds_evolutions_glrt}, the Student threshold suffers from a probably sub-optimal dependence in both $\log(1/\delta)$ and $t$.
This is why we propose an alternative method where the union bound is replaced by time-uniform concentration (which has proved useful to improve both dependencies in different contexts) and the Student concentration by concentration on the mean and the variance separately.
The resulting time-uniform upper and lower tail concentration inequalities for the empirical variance (Corollary~\ref{cor:uniform_time_upper_lower_tail_concentration_variance}) are of independent interest.
Thanks to these ``box'' confidence regions on $(\mu_t, \sigma_t^2)$, Lemma~\ref{lem:delta_correct_box_thresholds} yields a $\delta$-correct family of thresholds.

\begin{lemma} \label{lem:delta_correct_box_thresholds}
	Let $\eta_{0}>0$, $s> 1$, $\zeta$ be the Riemann $\zeta$ function and, for $i \in \{0,-1\}$, $\overline{W}_{i}(x) = -W_{i}(-e^{-x})$ for $x\geq1$ where $(W_{i})_{i \in \{0,-1\}}$ are the branches of the Lambert $W$ function. Define
	\begin{align*}
			&\varepsilon_\mu(t, \delta)  = \frac{1}{t}\overline{W}_{-1} \left( 1 + 2\ln \left( \frac{4(K-1)\zeta(s)}{\delta}\right) + 2s +  2s \ln\left(1 + \frac{\ln t}{2s}\right) \right) \: , \\
			&1- \varepsilon_{-,\sigma} (t, \delta)  = \overline{W}_{0} \left(1 +  \frac{2(1+\eta_{0})}{t}\left(\ln\left( \frac{4(K-1)\zeta(s)}{\delta} \right) + s\ln \left( 1+ \ln_{1+\eta_0}(t)\right) \right) \right) -\frac{1}{t} \:  .
	\end{align*}
	The family of thresholds $c_{a,b}^{\text{Box}}(N_t, \delta)$ with value $+ \infty$ if $t < \max_{c \in \{a, b \}} t_c^{\text{Box}}(\delta)$ and otherwise
	\begin{equation}  \label{eq:def_log_box_threshold_glrt}
		c_{a,b}^{\text{Box}}(N_t, \delta) = \sum_{c \in \{a, b \}} \frac{N_{t,c}}{2} \ln \left( 1 + \frac{\varepsilon_\mu(N_{t,c}, \delta)}{1 - \varepsilon_{-,\sigma} (N_{t,c}-1, \delta)} \right)
	\end{equation}
	yields a $\delta$-correct family of thresholds for the GLR stopping rule. The stochastic initial times are
\begin{equation} \label{eq:def_initial_time_box_thresholds}
	t_{a}^{\text{Box}}(\delta) = \inf \left\{ t \mid  N_{t,a} > 1+ e^{1+ W_{0} \left( \frac{2(1+\eta_{0})}{e}\left(\ln\left( \frac{4(K-1)\zeta(s)}{\delta} \right) + s\ln \left( 1+ \frac{\ln (N_{t,a}-1)}{\ln(1+\eta_{0})}\right) \right) -e^{-1}\right)}  \right\} \: .
\end{equation}
\end{lemma}

To derive the Box threshold, we leverage a lower bound on the empirical variance which is ensured to be strictly positive (hence informative) thanks to the initial time condition~\eqref{eq:def_initial_time_box_thresholds}.
As $W_{0}(x) \in [-1,+\infty)$, it also yields that $N_{t,a} > 2$.
Using that $W_0(x) \approx \ln(x)- \ln\ln(x)$ (Appendix~\ref{app:lambert_W_functions}), it is asymptotically equivalent to $\frac{2(1+\eta_{0})\ln (1/\delta)}{\ln \ln (1/\delta)}$.
Since the lower bound in Lemma~\ref{lem:lower_bound_sample_complexity} suggests that the stopping time is asymptotically equivalent to $T^\star(\mu, \sigma^2) \ln \left(1/\delta\right)$, the condition~\eqref{eq:def_initial_time_box_thresholds} has a vanishing influence compared to the stopping time.
For the parameters used in our simulations (see Section~\ref{ssec:simulations_stopping_thresholds}),~\eqref{eq:def_initial_time_box_thresholds} is empirically satisfied after sampling each arm $16$ times for $\delta = 0.1$ and $20$ times for $\delta = 0.001$.
Recall that $\overline{W}_{-1}(x) \approx x + \log x$ and $\overline{W}_{0}(x) \approx e^{-x + e^{-x}}$ (see Appendix~\ref{app:lambert_W_functions}).

\subsection{Beyond Box}
\label{ssec:beyond_box}

While being simpler to derive by controlling each arm independently, the above thresholds have a worse $\delta$ dependency than more sophisticated approach controlling directly the joint term~\eqref{eq:log_based_concentration}.
Since it is challenging to deal with~\eqref{eq:log_based_concentration}, we consider as a proxy the KL divergences for which is is easier to construct martingales, which can improve on the $\delta$ dependency.
To do so, we consider the formulation~\eqref{eq:KL_without_inf}, which removes the minimization step over variances, and apply the arguments used to obtain~\eqref{eq:log_based_concentration}.
Under any sampling rule, to obtain a $\delta$-correct GLR stopping rule it suffices to show that the family of thresholds is such that the following time-uniform concentration inequality holds for all $\nu \in \mathcal D^{K}$: with probability $1-\delta$, for all $t \in \Natural$ and for all $a \neq a^\star(\mu)$,
\begin{equation} \label{eq:kl_based_concentration}
	 \sum_{b \in \{a, a^\star(\mu)\}} N_{t,b}\KL((\mu_{t,b}, \sigma_{t,b}^2), (\mu_b, \sigma^2_b))
\le c_{a, a^\star(\mu)}(N_t,\delta)  \: .
\end{equation}

\paragraph{KL thresholds}
First, we derive time-uniform concentration results on the summation of KL divergences (Theorems~\ref{thm:uniform_upper_tail_concentration_kl_exp_fam} and~\ref{thm:uniform_upper_tail_concentration_kl_exp_fam_gaussian}), which are of independent interest.
Then, applied to our setting, it yields a $\delta$-correct family of thresholds (Theorem~\ref{thm:delta_correct_complex_threshold_glrt}).

\begin{theorem} \label{thm:delta_correct_complex_threshold_glrt}
	Let $\eta_{1}>0$, $\gamma,s >1$. Let $\varepsilon_\mu$, $\varepsilon_{-,\sigma}$ as in Lemma~\ref{lem:delta_correct_box_thresholds} with $\tilde \delta = \frac{\delta}{3}$ and $(t^{\text{Box}}_a)_a$ as in (\ref{eq:def_initial_time_box_thresholds}),
	\begin{align*}
			&1 + \varepsilon_{+,\sigma} (t, \delta)  = \overline{W}_{-1} \left(1 +  \frac{2(1+\eta_{1})}{t}\left(\ln\left( \frac{12(K-1)\zeta(s)}{\delta} \right) + s\ln \left( 1+ \ln_{1+\eta_{1}}(t)\right) \right) \right) -\frac{1}{t} \: .
	\end{align*}
For all $t$, define $i_{t,a} = \lfloor \log_\gamma N_{t,a} \rfloor$, $n_{t,a} = \gamma^{i_{t,a}}$, $ \bar{t}_a = \inf \left\{ t \mid N_{t,a} = n_{t,a} \right\}$,
	\begin{align*}
	& \mu_{++,t,a}^2 = \max_{\pm} \left(\mu_{ \bar{t}_a ,a} \pm 2\sigma_{ \bar{t}_a,a}\sqrt{ \frac{\varepsilon_\mu(n_{t,a}, \delta)}{1-\varepsilon_{-,\sigma} (n_{t,a}-1, \delta)} }  \right)^2 \: , \\
	&\sigma^2_{\pm,t,a} = \sigma_{ \bar{t}_a ,a}^2 \frac{1 \pm \varepsilon_{\pm,\sigma} (n_{t,a}-1, \delta)}{1 \mp \varepsilon_{\mp,\sigma} (n_{t,a}-1, \delta)} \quad \text{and} \quad	R_{ t,a}(\delta) = \frac{\sigma^3_{+,t,a} f_{+}\left( g(\sigma^2_{+,t,a}, \mu_{++,t,a}^2)\right)}{\sigma^3_{-,t,a} f_{-}\left( g(\sigma^2_{-,t,a},\mu_{++,t,a}^2)\right)} \: ,
	\end{align*}
 where $f_{\pm}(x) = \frac{1 \pm \sqrt{1 - x}}{\sqrt{x}}$ and $g(x,y) = \frac{2x}{(x+2y+\frac{1}{2})^2}$.
 The family of thresholds $c_{a,b}^{\KL}(N_t, \delta)$ with value $+ \infty$ if $t < \max_{c \in \{a, b \}} \max \{t^{\text{Box}}_c(\delta/3), t_{c}^{\text{m}}(\delta)\}$ and otherwise
	\begin{equation} \label{eq:def_complex_threshold_glrt}
		c_{a,b}^{\KL}(N_t, \delta) = 4\overline{W}_{-1}\left( 1 + \frac{\log\frac{2\zeta(s)^{2}}{\delta}}{4}
		+ \frac{s}{4} \sum_{c \in \{a,b\}} \log(1 + \log_\gamma N_{t,c})
		+ \frac{1}{2}\sum_{c \in \{a,b\}} \log\left(\gamma R_{t, c}(\delta)\right) \right)
\end{equation}
	yields a $\delta$-correct family of thresholds for the GLR stopping rule. The stochastic initial times are
	\begin{equation} \label{eq:def_initial_time_monotonicity}
	t_{a}^{\text{m}}(\delta) = \inf \left\{ t \mid  N_{t,a} > 1 + \max \left\{ \frac{e^{s / \ln\left( \frac{12(K-1)\zeta(s)}{\delta} \right)}  }{1+ \eta_0},  \frac{e^{s /\left( \ln\left( \frac{12(K-1)\zeta(s)}{\delta} \right) - \frac{1}{2(1+\eta_1)} \right)}  }{1+\eta_1}\right\}  \right\} \: .
	\end{equation}
\end{theorem}

As $\overline{W}_{-1}(x) \approx x + \ln(x)$, Theorem~\ref{thm:delta_correct_complex_threshold_glrt} proves that we can obtain $\delta$-correct threshold with the dependencies $c(t,\delta) \sim_{\delta \to  0} \ln\left(1/\delta\right) $ and $c(t,\delta) \sim_{t \to +\infty}  C \ln\ln(t)$, which are widely used in practice for BAI problems.
While this dependency was already motivated when the variances are known \citep{kaufmann_2018_MixtureMartingalesRevisited}, Theorem~\ref{thm:delta_correct_complex_threshold_glrt} legitimates its use for unknown variances.

To control the KL divergence between the true parameter and the MLE for Gaussian with unknown variances, our threshold combines two concentration results and is obtained by covering $\Natural$ with slices of times with geometrically increasing size to cover (referred to as the ``peeling'' method).
First, we use a crude per-arm concentration step to restrict the estimated parameters to a region around the true mean and variance.
Then, a second result uses the knowledge of the restriction to get a finer concentration on the weighted sum of KL.
It is proved for generic exponential families by approximating the KL divergence by a quadratic function on this crude confidence region.
In \eqref{eq:def_complex_threshold_glrt}, $R_{t, a}(\delta)$ represents the cost of this approximation, while $\log(1 + \log_\gamma N_{t,c})$ is the cost of time-uniform.
The initial time condition~\eqref{eq:def_initial_time_monotonicity} ensures the monotonicity of the preliminary concentration, and it is of the form $N_{t,a} > 1 + c_{0}(\delta)$ where $c_{0}(\delta) > 0$.
In our simulations (see Section~\ref{ssec:simulations_stopping_thresholds}),~\eqref{eq:def_initial_time_monotonicity} is empirically satisfied after sampling each arm twice for all considered $\delta$.

\citet{degenne_2019_ImpactStructureDesign} derives concentration on the KL divergence of sub-Gaussian $d$-dimensional exponential families defined on the natural parameter space $\Theta_{D} = \Real^d$. This doesn't include Gaussian with unknown variance, but our proof builds on his method.
The main challenge was to tackle $\Theta_{D} \neq \Real^d$, and we solved it by truncation on the sequence of crude confidence regions.
In generalized linear bandits, truncated Gaussians were also used to derive tail-inequalities for martingales ``re-normalized'' by their quadratic variation \citep{faury_2021_VarianceSensitiveConfidence}.
For general $d$-dimensional exponential families, \citet{Chowdhury_2022_BregmanDeviationsExpFam} derives concentrations on the KL divergence between the true parameter and a linear combination of the MLE and the true parameter.
As we are interested in the KL divergence between the true parameter and the MLE, we cannot leverage their result.

\paragraph{BoB thresholds}
While the KL thresholds reach the desired dependency in $(t, \delta)$, using \eqref{eq:kl_based_concentration} instead of \eqref{eq:log_based_concentration} yields larger thresholds due to additive constants.
To overcome this hurdle, we maximize \eqref{eq:log_based_concentration} under the per-arm box constraints (Lemma~\ref{lem:delta_correct_box_thresholds}) and the pairwise non-linear constraint (Theorem~\ref{thm:delta_correct_complex_threshold_glrt}).
The resulting family of thresholds is denoted by BoB (Best of Both) thresholds.
While the BoB thresholds have no closed-form solution, they can be approximated with non-linear solvers, e.g. Ipopt \citep{Wachter_2006_IpOpt}.

\begin{corollary} \label{cor:delta_correct_Kinf_thresholds}
	Let $f(x,y) = (1+y)x-1-\ln(x)$ for all $(x,y) \in (\Real_{+}^{\star})^2$.
	Let $(t^{\text{Box}}_a)_a$ and $(t^{\text{m}}_a)_a$ as in (\ref{eq:def_initial_time_box_thresholds},\ref{eq:def_initial_time_monotonicity}).
	Let $\epsilon_{\mu},\epsilon_{-,\sigma}$ as in Lemma~\ref{lem:delta_correct_box_thresholds} and $(c_{b,a}^{\KL})_{b,a \in [K]}$ as in (\ref{eq:def_complex_threshold_glrt}). The family of thresholds $c_{a,b}^{\text{BoB}}(N_t, \delta)$ with value $+ \infty$ if $t < \max_{c \in \{a, b \}} \max \{t^{\text{Box}}_c(\delta/6), t_{c}^{\text{m}}(\delta/2)\}$ and otherwise solution of the optimization problem
 \begin{align*}
 \text{maximize } &  \frac{1}{2} \sum_{c \in \{a, b\}} N_{t,c} \ln \left( 1 + y_c\right)
 \\
 \text{such that }\: & \forall c \in \{a, b\}, \quad y_c \geq 0, \: x_c y_c \leq \epsilon_{\mu}(N_{t,c}, \delta/2), \: x_c \geq 1 - \epsilon_{-,\sigma}(N_{t,c}-1, \delta/2) \: ,\\
 &\text{and} \quad  \frac{1}{2} \sum_{c \in \{a, b\}} N_{t,c} f \left(x_c, y_c\right) \leq c_{b,a}^{\KL}(N_t,\delta/2) \: ,
 \end{align*}
 yields a $\delta$-correct family of thresholds for the GLR stopping rule.
\end{corollary}

Since \eqref{eq:log_based_concentration} is smaller than \eqref{eq:kl_based_concentration}, the KL constraint is an upper bound on the BoB threshold.
Compared to the box threshold, the maximization underlying the BoB threshold has an additional constraint.
Therefore, we have $c_{a,b}^{\text{BoB}}(N_t, \delta) \le \min\{c_{b,a}^{\text{Box}}(N_t,\delta/2), c_{b,a}^{\KL}(N_t,\delta/2)\}$.
In particular, the BoB threshold combines the best of both thresholds in terms of $(t,\delta)$ dependencies.

\subsection{Simulations}
\label{ssec:simulations_stopping_thresholds}

We perform numerical simulations to compare the family of thresholds introduced above for the GLR stopping rule (see Appendix~
\ref{app:sss_exp_thresholds} for the EV-GLR stopping rule).
Taking $K=2$, we consider the instance $\mu = (0,-0.2)$ and $\sigma^2 = (1,0.5)$.
Since we are not interested in observing the influence of the sampling rule, the stream of data is uniform between both arms.
For the thresholds, we set the parameters to $s=2$, $\gamma=1.2$ and $\eta_0=\eta_1 = \ln \left( 1/\delta\right)^{-1}$.

\begin{figure}[ht]
	\centering
	\includegraphics[width=0.48\linewidth]{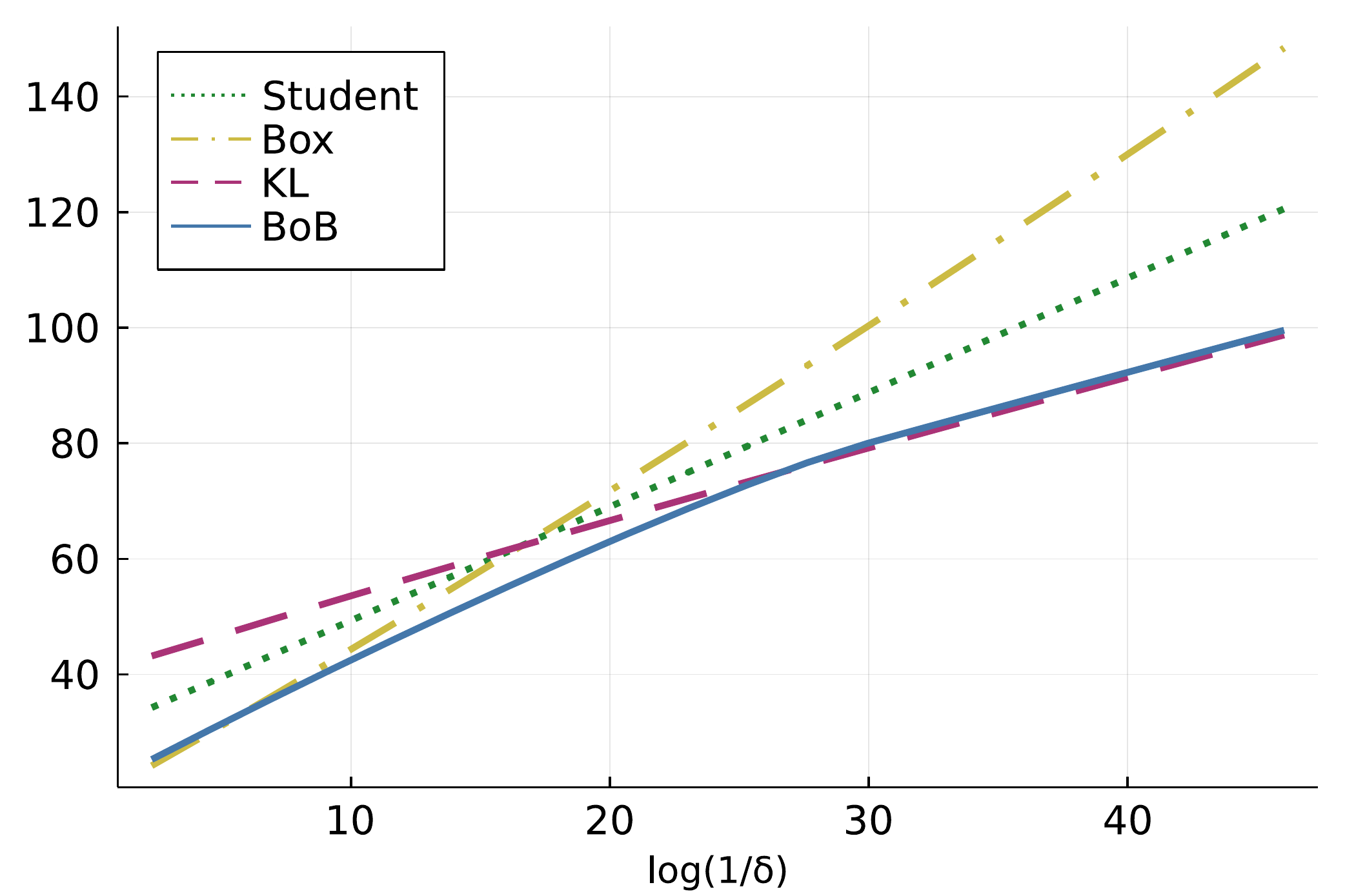}
	\includegraphics[width=0.48\linewidth]{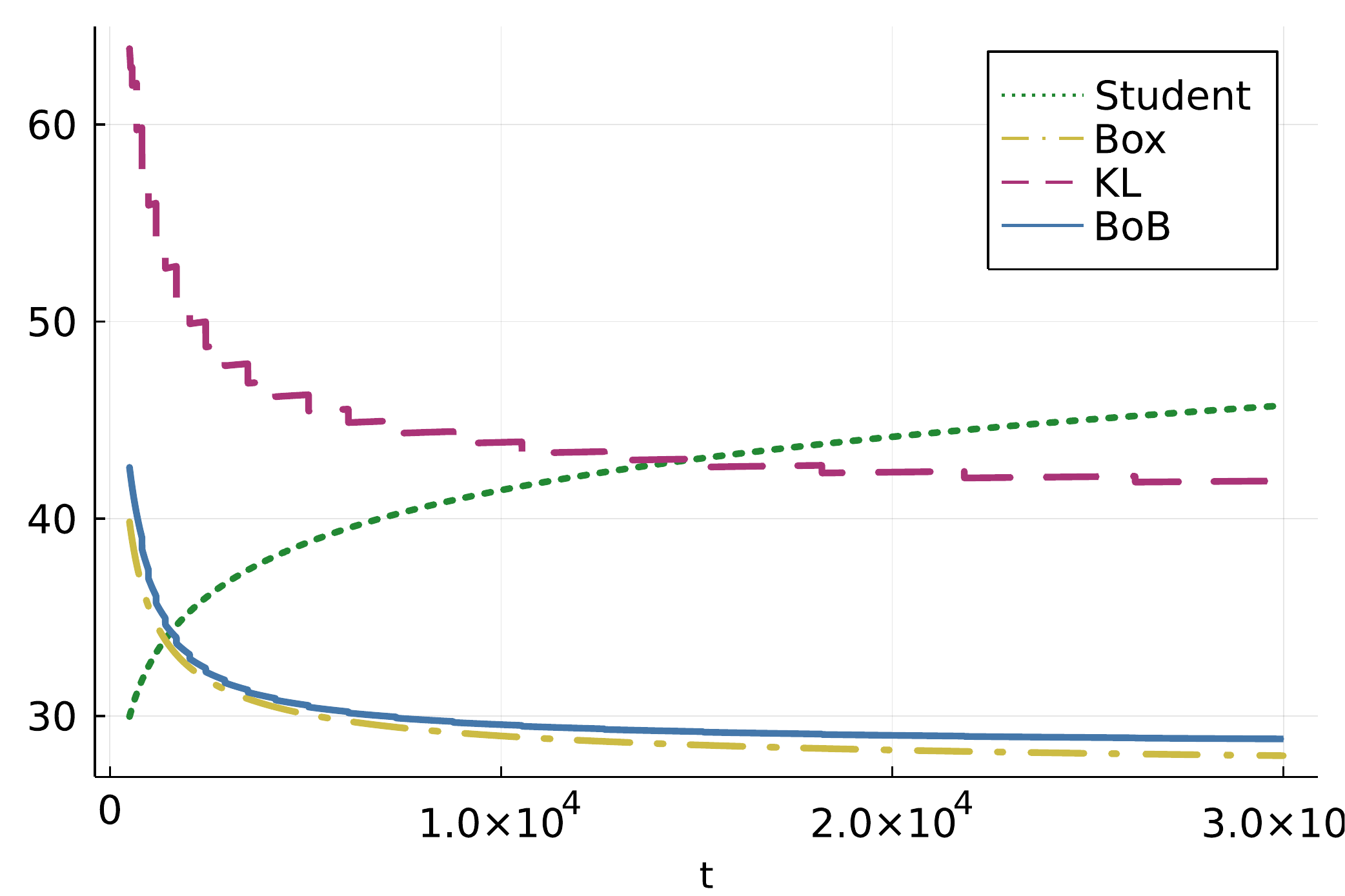}
	\caption{Thresholds for (\ref{eq:def_stopping_rule_glrt}) as a function of (a) $\ln\left(1/\delta\right)$ for $t = 5000$ and (b) $t$ for $\delta = 0.01$.}
	\label{fig:stopping_thresholds_evolutions_glrt}
\end{figure}

Figure~\ref{fig:stopping_thresholds_evolutions_glrt} plots the dependency of the thresholds in $\ln\left(1/\delta\right)$ and $t$.
In Figure~\ref{fig:stopping_thresholds_evolutions_glrt}(a), we are only interested by the slopes, and smaller slopes are equivalent to better dependency in $\ln\left(1/\delta\right)$.
As expected, Student thresholds have poor performance for both variables.
While box thresholds improve in $t$, they suffer from a worse dependency in $\ln \left(1/\delta\right)$.
KL thresholds circumvent this issue with the best dependency in $\ln \left(1/\delta\right)$ so far.
However, they incur a large constant cost making it worse than the box threshold in practice.
As hoped, BoB thresholds combine the good performance in $t$ of the box threshold and the asymptotic dependency in $\ln \left(1/\delta\right)$ of the KL threshold.

The improved theoretical dependency of the BoB threshold comes at the price of a higher computational cost: on average $400$, $600$ and $800$ times larger than the ones for the KL threshold, the Box threshold and the Student threshold respectively.
When the computational cost is a major concern, the Box threshold should be used since it has low computational cost and good empirical performance.
Alternatively, we could use the BoB threshold and evaluate the stopping rule only on a predefined geometric grid of times.
This ``lazy'' stopping rule is still $\delta$-correct.

%% file: sections/sampling.tex

\section{Sampling Rule Wrappers}
\label{sec:sampling_rules}

After calibrating the stopping threshold to ensure $\delta$-correctness, we need to design a sampling rule which requires few samples before stopping.
Given any BAI algorithm for Gaussian with known variances, we propose two wrappers that can adapt the algorithm to tackle unknown variances: plugging in the empirical variance or adapting the transportation cost.

When the variances are unknown, a natural idea is to plug in the empirical variances instead of using the true variances which are now unknown.
We can apply this wrapper to any BAI algorithm.

Section~\ref{sec:lower_bounds_and_glrs} discusses the differences and links between the transportation costs for known and unknown variances.
Leveraging this interplay, we can adapt a BAI algorithm to use the transportation costs for unknown variances instead of the ones for known variances.
We can apply this wrapper to any BAI algorithm relying on transportation costs.

We illustrate how to instantiate each wrapper (Section~\ref{ssec:instantiating_wrappers}), derive guarantees on their asymptotic expected sample complexity (Section~\ref{ssec:upper_bound_sample_complexity}), and assess their empirical performance (Section~\ref{ssec:experiments}).

\subsection{Instantiating the Wrappers}
\label{ssec:instantiating_wrappers}

As initialization, we start by pulling each arm $n_0 \ge 2$ times, and let $t_0 = n_0 K$.

\paragraph{Track-and-Stop}
The Track-and-Stop algorithm \citep{garivier_2016_OptimalBestArm} computes at each time $t > t_0$ the optimal allocation for the considered transportation costs, i.e. $w_t = w^\star_{\sigma^2}(\mu_t)$ for Gaussian with known variances.
Given the vector $w_t$ in the simplex, it uses a so-called tracking procedure to obtain an arm $a_{t+1}$ to sample.
We describe and use the one called C-tracking by \citet{garivier_2016_OptimalBestArm}.
On top of this tracking a forced exploration is used to enforce convergence towards the optimal allocation for the true unknown parameters.
Let $\epsilon \in (0,1/K]$ and $\triangle_{K}^{\epsilon} = \{w \in [\epsilon,1]^{K} \mid \sum_{a \in [K]}w_{a}=1\}$.
Defining $w_{t}^{\epsilon}$ the $L^{\infty}$ projection of $w_{t}$ on $\triangle_{K}^{\epsilon}$, C-Tracking pulls $a_{t+1} \in \argmax_{a\in [K]} \sum_{s = t_{0}}^{t} w^{\epsilon}_{s,a} - N_{t,a}$.

Plugging in the empirical variance yields the EV-TaS (Empirical Variance Track-and-Stop) algorithm which computes $w_t = w^\star_{\sigma_t^2}(\mu_t)$.
Adapting the transportation cost yields the TaS algorithm which uses $w_t = w^\star(\mu_t, \sigma_t^2)$.
Computing $w^\star_{\sigma_t^2}(\mu_t)$ and $w^\star(\mu_t, \sigma_t^2)$ can be done by solving an equivalent optimization problem with one bounded variable (Theorem~\ref{thm:equivalent_optimization_problem} in Appendix~\ref{app:ss_optimal_allocation_oracles}), which can itself be numerically approximated with binary search.

\paragraph{Top Two algorithm}
At each time $t > t_0$, the Top Two algorithm $\beta$-EB-TCI \citep{jourdan_2022_TopTwoAlgorithms} pulls the EB leader $B_{t+1}^{\text{EB}} = \hat a_t$ with probability $\beta$.
If $B_{t+1}^{\text{EB}}$ is not sampled, then it pulls the TCI challenger $A_{t+1}^{\text{TCI}} \in \argmin_{a \neq B_{t+1}^{\text{EB}}} C_{t}(B_{t+1}^{\text{EB}}, a) + \log N_{t,a}$ for the considered transportation costs $C_{t}(a,b)$,
i.e. $C_{t}(a,b) = \1\{\mu_{t,a} > \mu_{t,b}\} \frac{1}{2} \frac{(\mu_{t,a} - \mu_{t,b})^2}{\sigma_{a}^2/N_{t,a} + \sigma_{b}^2/N_{t,b}}$ for Gaussian with known variance.

Plugging in the empirical variance yields the $\beta$-EB-EVTCI algorithm which uses
\begin{equation*} 
	C^{\text{EV}}_{t}(a, b) \eqdef \1 \{\mu_{t,a} > \mu_{t,b} \} \frac{1}{2} \frac{(\mu_{t,a} - \mu_{t,b})^2}{\sigma_{t,a}^2/N_{t,a} + \sigma_{t,b}^2/N_{t,b}}   \: .
\end{equation*}
Adapting the transportation cost yields the $\beta$-EB-TCI algorithm which computes
\begin{equation*} 
	C_{t}(a, b) \eqdef
	\1 \{\mu_{t,a} > \mu_{t,b} \}  \inf_{ \lambda  \in \Real} \sum_{c \in \{a, b\}} \frac{N_{t,c}}{2} \ln \left( 1 + \frac{(\mu_{t,c} - \lambda)^2}{\sigma_{t,c}^2}\right)   \: .
\end{equation*}
Since $C_{t}(\hat a_t, a) = Z_{a}(t)$, we can re-use computations of the GLR stopping rule.

\subsection{Sample Complexity Upper Bound}
\label{ssec:upper_bound_sample_complexity}

Definition~\ref{def:asymptotically_tight_threshold} introduces the notion of asymptotically tight family threshold \citep{jourdan_2022_TopTwoAlgorithms}, which corresponds informally to $c_{a,b}(N,\delta) \sim_{\delta \to 0} \log(1/\delta)$.
As hinted in Figure~\ref{fig:stopping_thresholds_evolutions_glrt}(a), the KL and the BoB thresholds are asymptotically tight (Appendix~\ref{app:ss_asymptotically_tight_thresholds}), but not the Student and Box thresholds.

\begin{definition} \label{def:asymptotically_tight_threshold}
	A family of thresholds $(c_{a,b})_{(a,b) \in [K]^2}$ is said to be asymptotically tight if there exists $\alpha \in [0,1)$, $\delta_0 \in (0,1]$, functions $f,\bar{T} : (0,1] \to \mathbb{R}_+$ and $C$ independent of $\delta$ satisfying: (1) for all $(a,b) \in [K]^2$, $\delta \in  (0,\delta_0]$ and $N \in \Natural^{K}$ such that $\| N \|_1 \ge \bar{T}(\delta)$, then $c_{a,b}(N, \delta) \le f(\delta) + C \| N \|_1^\alpha$, (2) $\limsup_{\delta \to 0} f(\delta)/\log(1/\delta) \le 1$ and $\limsup_{\delta \to 0} \bar{T}(\delta)/\log(1/\delta) = 0$.
\end{definition}

When combined with the GLR stopping rule using the KL or the BoB thresholds, Theorem~\ref{thm:upper_bound_sample_complexity_algorithm} shows that TaS (resp. $\beta$-EB-TCI) is a $\delta$-correct and asymptotically (resp. $\beta$-)optimal algorithm.

\begin{theorem} \label{thm:upper_bound_sample_complexity_algorithm}
	Using the GLR stopping rule with an asymptotically tight family of thresholds, TaS (resp. $\beta$-EB-TCI with $n_0 \ge 4$) satisfies that, for all $\nu$ with $|a^\star(\mu)|=1$ (resp. $\min_{a \neq b}|\mu_a - \mu_b| > 0 $),
	\[
	 \limsup_{\delta \rightarrow 0} \frac{\mathbb{E}_{\nu}\left[\taud\right]}{\log (1 / \delta)} \le T^\star(\mu, \sigma^2)	\qquad \text{(resp. }T^\star_{\beta}(\mu, \sigma^2) {)} \: .
	\]
\end{theorem}

In Appendix~\ref{app:expected_sample_complexity}, we derive a similar result involving $T^\star_{\sigma^2}(\mu)$ (resp. $T^\star_{\sigma^2, \beta}(\mu)$) for EV-TaS (resp. $\beta$-EB-EVTCI with $n_0 \ge 6$) combined with the EV-GLR stopping rule using an asymptotically threshold.
However, since $T^\star_{\sigma^2}(\mu) < T^\star(\mu, \sigma^2)$ and $T^\star_{\sigma^2, \beta}(\mu) < T^\star_{\beta}(\mu, \sigma^2)$, neither of these algorithms can be $\delta$-correct.
Otherwise it would yield a contradiction with the lower bound in Lemma~\ref{lem:lower_bound_sample_complexity}.
Moreover, as there exist instances for which the ratios $T^\star(\mu, \sigma^2) / T^\star_{\sigma^2}(\mu)$ and $T^\star_{\beta}(\mu, \sigma^2) / T^\star_{\sigma^2, \beta}(\mu)$ are arbitrarily large (Lemma~\ref{lem:ratio_characteristic_time_large}), multiplying the thresholds by a problem independent constant is not sufficient either to obtain $\delta$-correctness, as expressed in Theorem~\ref{thm:impossibility_result}.

\begin{theorem} \label{thm:impossibility_result}
	There exists a sampling rule such that: for all asymptotically tight family of thresholds $(c_{a,b})_{(a,b) \in [K]^2}$ and problem independent constant $\alpha_0 > 0$, combining this sampling rule with the EV-GLR stopping rule using $(\alpha_0 c_{a,b})_{(a,b) \in [K]^2}$ yields an algorithm which is not $\delta$-correct.
\end{theorem}

Inspired by Section~\ref{sec:calibration_stopping_rules}, we propose families of thresholds (EV-Student, EV-Box and EV-BoB) which are $\delta$-correct for the EV-GLR stopping rule (see Appendix~\ref{app:thresholds}) but are not asymptotically tight (Theorem~\ref{thm:impossibility_result}).
Still, in our experiments the empirical proportion of error is lower than $\delta$ even when using a heuristic, asymptotically tight threshold.

Based on Theorems~\ref{thm:upper_bound_sample_complexity_algorithm} and~\ref{thm:impossibility_result}, algorithms obtained by adapting the transportation costs enjoy stronger theoretical guarantees than the ones plugging in the empirical variance.

\subsection{Experiments}
\label{ssec:experiments}

We compare the empirical performance of the two wrappers for different BAI algorithms in the moderate regime ($\delta=0.01$).
As benchmarks, we consider FHN$2$ (procedure $2$ in \citet{Fan16RSUV}, see Algorithm~\ref{algo:fhn2} in Appendix~\ref{app:additional_experiments}), uniform sampling and ``fixed'' sampling which is an oracle playing with proportions $w^\star(\mu, \sigma^2)$.
FHN$2$ is an elimination strategy which repeatedly samples all arms until only one arm is left.
Its elimination mechanism is calibrated by resorting to continuous-time approximations.
Therefore, FHN2 is only asymptotically $\delta$-correct and has no guaranties on the sample complexity.
Based on \citet{degenne_2019_NonAsymptoticPureExploration} and \citet{wang_2021_FastPureExploration}, plugging in the empirical variance yields EV-DKM and EV-FWS, while DKM and FWS refers to the algorithms using the transportation costs for unknown variances.
Even though those instances are not analyzed, we believe that similar guarantees on the sample complexity can be shown.

Algorithms obtained by plugging in the empirical variance uses the EV-GLR stopping rule, while the GLR stopping rule is used by the ones with adapted transportation cost and the uniform sampling.
We consider the stylized stopping threshold $c(t,\delta) = \log\left( (1 + \log t)/\delta\right)$, which was proposed in \cite{garivier_2016_OptimalBestArm}.
While it doesn't ensure $\delta$-correctness of the stopping threshold, it is asymptotically tight and yields an empirical error which is several order of magnitude lower than $\delta$.
Top Two algorithms use $\beta = 0.5$.

We assess the performance on $1000$ random instances with $K=10$ such that $(\mu_{1}, \sigma_{1}^2)=(0,1)$. For $a \neq 1$, we set $(\mu_{a}, \sigma_{a}^2) = (- \Delta_{a}, r_{ a})$ where $\Delta_{a} \sim \mathcal U ([0.2, 1.0])$ and $r_{ a} \sim \mathcal U ([0.1, 10])$.
To illustrate the two regimes for $T^\star(\mu, \sigma^2)/T^\star_{\sigma^2}(\mu)$, we consider a \textit{standard} instance ($T^\star(\mu, \sigma^2)/T^\star_{\sigma^2}(\mu) \approx 1.015$) and an \textit{easy} instance ($T^\star(\mu, \sigma^2)/T^\star_{\sigma^2}(\mu) \approx 1.384$).
We average over $5000$ runs.

In Figure~\ref{fig:random_instances_heuristic_threshold}, we observe that algorithms obtained by plugging in the empirical variance yield similar result as the ones using the adapted transportation cost, and slightly better performance on the easy instance.
Moreover, those wrapped BAI algorithms outperforms uniform sampling and are on par with ``fixed'' sampling.
On random instances FHN2 has similar performance to the wrapped BAI algorithms, but it wastes precious samples on easy instances.

\begin{figure}[ht]
	\centering
	\includegraphics[width=0.44\linewidth]{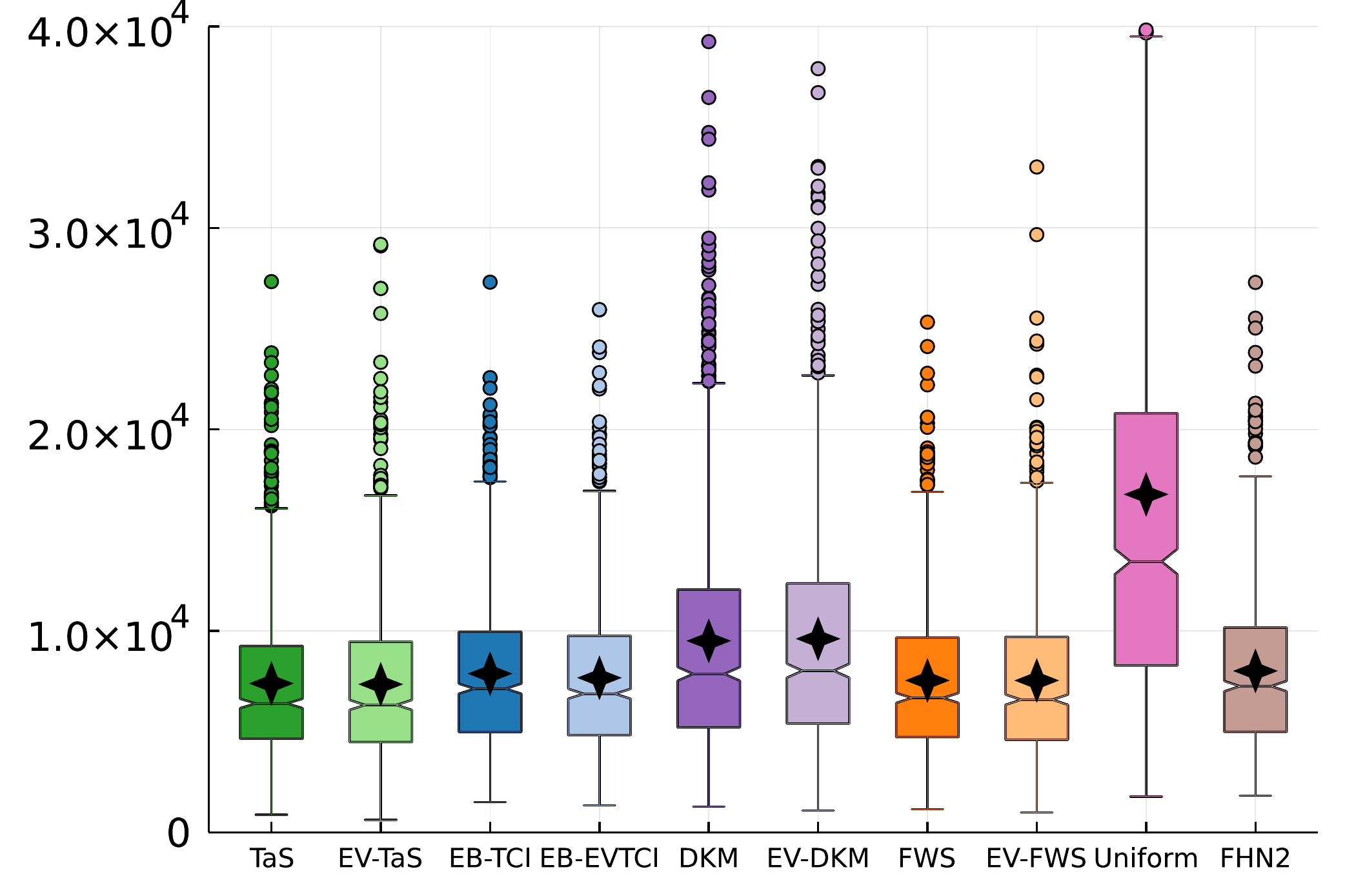} \\
	\includegraphics[width=0.44\linewidth]{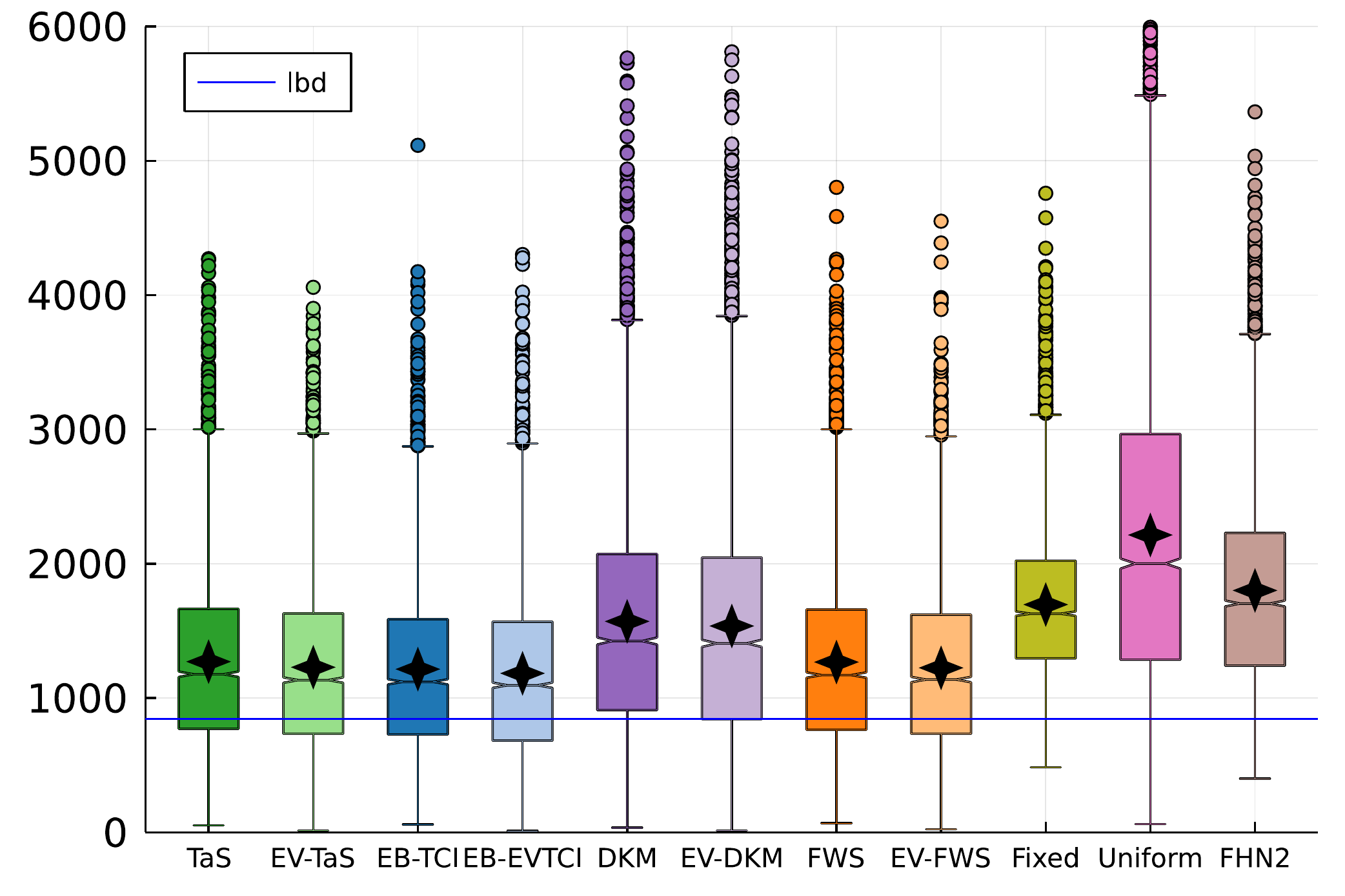}
	\includegraphics[width=0.44\linewidth]{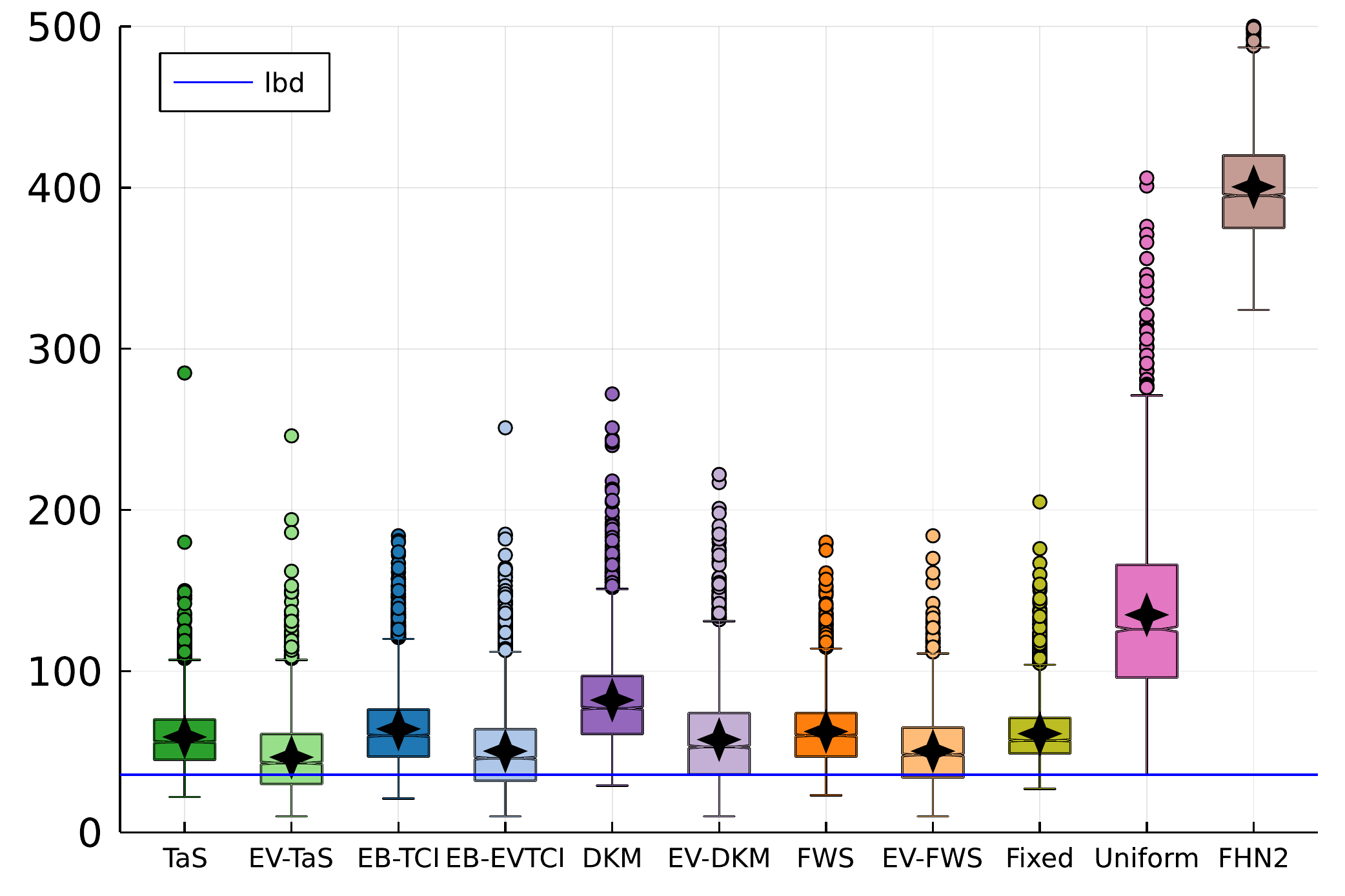}
	\caption{Empirical stopping time on Gaussian (top) random instances with $K=10$, (left) standard instance $(\mu, \sigma^2) = ((1.0, 0.85, 0.8, 0.7, 0.65)$, $(1.0, 0.6, 0.5, 0.4, 0.35))$ and (right) easy instance $(\mu,\sigma^2) = ((1.0, 0.2, 0.15, 0.1, 0.05)$, $(1.0, 0.05, 0.05, 0.05, 0.05))$. Lower bound is $T^\star(\mu) \log(1/\delta)$.}
	\label{fig:random_instances_heuristic_threshold}
\end{figure}

%% file: sections/appendix_lower_bounds.tex

\section{Characteristic Times} \label{app:characteristic_times}

In Appendix~\ref{app:ss_explicit_formulas_inequalities}, we show properties on the characteristic times $T^\star_{\sigma^2}(\mu)$ and $T^\star(\mu,\sigma^2)$, i.e. explicit formulas and relationships. In Appendix~\ref{app:ss_optimal_allocation_oracles}, we show those characteristic times can be obtained as solution of a simpler optimization problem, allowing to compute the associated optimal allocation.

Theorem 1 in \citet{garivier_2016_OptimalBestArm} yields the first part of Lemma~\ref{lem:lower_bound_sample_complexity} (Gaussian with known variances). Its second part (Gaussian with unknown variances) is a direct consequence of the arguments in \citet{garivier_2016_OptimalBestArm}, hence we omit the proof.

\subsection{Explicit Formulas and Inequalities} \label{app:ss_explicit_formulas_inequalities}

We can obtain slightly more explicit formulas for the characteristic times $T^\star_{\sigma^2}(\mu)$ and $T^\star(\mu,\sigma^2)$ (Lemma~\ref{lem:complexity_expressions}).
It is direct to see that similar explicit formulas can be shown for $T^\star_{\sigma^2, \beta}(\mu)$ and $T^\star_{\beta}(\mu, \sigma^2)$.

\begin{lemma} \label{lem:complexity_expressions}
Let $\cM = \mathbb{R}^K \times (\mathbb{R}^\star_+)^K$ and $a^\star(\mu) = a^\star$. Then,
\begin{align*}
&T^\star_{\sigma^2}(\mu)^{-1}
= \max_{w \in \triangle_K} \min_{a \ne a^\star} \inf_{y \in [\mu_a, \mu_{a^\star}]} \sum_{b \in \{a, a^\star\}} w_b \frac{(y - \mu_b)^2}{2\sigma^2_b} = \max_{w \in \triangle_K} \min_{a \ne a^\star} \frac{(\mu_a - \mu_{a^\star})^2}{2\left(\frac{\sigma^2_a}{w_a} + \frac{\sigma^2_{a^\star}}{w_{a^\star}} \right)}
\: , \\
&T^\star(\mu, \sigma^2)^{-1}
= \max_{w \in \triangle_K} \min_{a \ne a^\star} \inf_{y \in [\mu_a, \mu_{a^\star}]} \sum_{b \in \{a, a^\star\}} \frac{w_b}{2} \log \left( 1 + \frac{(y - \mu_b)^2}{\sigma^2_b}\right)
\: .
\end{align*}
\end{lemma}
\begin{proof}
While an explicit formula for $T^\star_{\sigma^2}(\mu)$ was already proven in \citet{garivier_2016_OptimalBestArm}, we derive the equivalent result for $T^\star(\mu,\sigma^2)$. For Gaussian with different variances, the KL has the following expression
	\begin{align*}
		\KL((\mu, \sigma^2), (\lambda, \kappa^2)) = \frac{1}{2} \left( \frac{(\mu - \lambda)^2}{ \kappa^2} +  \frac{\sigma^2}{\kappa^2} - 1 - \ln \left( \frac{\sigma^2}{\kappa^2} \right) \right) \: .
	\end{align*}

	Direct computations yield that
	\begin{align*}
		\inf_{\kappa^2>0} \KL((\mu, \sigma^2), (\lambda, \kappa^2)) &= \KL((\mu, \sigma^2), (\lambda, \sigma^2 + (\mu-\lambda)^2)) = \frac{1}{2} \ln\left( 1+ \frac{(\mu - \lambda)^2}{\sigma^2}\right) \:
	\end{align*}

The set $\Lambda(\mu, \sigma^2) = \{(\lambda, \kappa^2) \in \mathbb{R}^K \times (\mathbb{R}^{\star}_+)^K \mid a^\star \notin \argmax_{a\in [K]} \lambda^a \}$ can be rewritten as $\bigcup_{a\neq a^\star} \left\{ \lambda \in \Real^{K} \mid  \lambda^a > \lambda^{a^\star}\right\} \times (\Real_{+}^{\star})^{K}$.
Therefore, taking $\lambda_{b}=\mu_b$ and $\kappa_{b}^2=\sigma_{b}^2$ for $b\notin \{a,a^\star\}$, we obtain
	\begin{align*}
		T^\star(\mu,\sigma^2)^{-1}&= \max_{w\in \simplex} \min_{a\neq a^\star} \inf_{(\lambda,\kappa^2)\in \Real^2\times (\Real^{\star}_{+})^2: \lambda^a > \lambda^{a^\star}} \sum_{b\in \{a,a^\star\}} w_{b} \KL((\mu_b, \sigma_{b}^2), (\lambda_{b}, \kappa_{b}^2)) \\
		&= \max_{w\in \simplex} \min_{a\neq a^\star} \inf_{\lambda \in \Real^2: \lambda^a > \lambda^{a^\star}}   \sum_{b\in \{a,a^\star\}}  \frac{w_{b}}{2} \ln\left( 1+ \frac{(\mu_b - \lambda_{b})^2}{\sigma_{b}^2}\right) \\
		&= \max_{w\in \simplex} \min_{a\neq a^\star} \inf_{\lambda \in [\mu_a,\mu_{a^\star}]}   \sum_{b\in \{a,a^\star\}}  \frac{w_{b}}{2} \ln\left( 1+ \frac{(\mu_b - \lambda)^2}{\sigma_{b}^2}\right)
	\end{align*}
	where the last equality is obtained since at the infimum there is equality, i.e. $\lambda^a=\lambda^{a^\star}$, and Lemma~\ref{lem:characteristic_time_properties_to_show_equivalent}.
\end{proof}

Using Lemma~\ref{lem:complexity_expressions} and algebraic manipulation, we derive inequalities between $T^\star_{\sigma^2}(\mu)$ and $T^\star(\mu,\sigma^2)$ (Lemma~\ref{lem:complexity_inequalities}).
It is direct to see that the same proof would hold for $T^\star_{\sigma^2, \beta}(\mu)$ and $T^\star_{\beta}(\mu, \sigma^2)$.
\begin{proof}
For all $a \neq a^\star$, $\lambda \in [\mu_a,\mu_{a^\star}]$ and $\{a,a^\star\}$, we have
	\begin{align*}
		\frac{(\mu_b - \lambda)^2}{\sigma_{b}^2} \leq d(\mu, \sigma^2) \eqdef \max_{a \neq a^\star} \frac{(\mu_{a^\star} - \mu_{a})^2}{\min\{\sigma_{a^\star}^2, \sigma_{a}^2\}}
	\end{align*}
	Using that $x \mapsto \frac{\ln(1+x)}{x}$ is decreasing for $x\geq 0$, we obtain, for all $\lambda \in [\mu_a,\mu_{a^\star}]$ and $b \in \{a,a^\star\}$,
	\begin{align*}
	 \ln\left( 1+ \frac{(\mu_b - \lambda)^2}{\sigma_{b}^2}\right) \geq \frac{(\mu_b - \lambda)^2}{\sigma_{b}^2} \frac{\ln\left( 1+ d(\mu, \sigma^2)\right)}{d(\mu, \sigma^2)} \: .
	\end{align*}
	Since $T^\star_{\sigma^2}(\mu)^{-1}= \max_{w\in \simplex} \min_{a\neq a^\star} \inf_{\lambda \in [\mu_a,\mu_{a^\star}]}   \sum_{b\in \{a,a^\star\}}  w_{b}\frac{(\mu_b - \lambda)^2}{2\sigma_{b}^2}$, this yields
	\begin{align*}
		T^\star(\mu,\sigma^2)^{-1} \geq \frac{\ln\left( 1+ d(\mu, \sigma^2)\right)}{d(\mu, \sigma^2)} T^\star_{\sigma^2}(\mu)^{-1}
	\end{align*}

	The same arguments as above also yield the formulation
	\begin{align*}
		2T^\star(\mu,\sigma^2)^{-1} = \max_{w\in \simplex} \inf_{\lambda \in \Lambda_{\sigma^2}(\mu)}   \sum_{b \in [K]}  w_{b} \ln\left( 1+ \frac{(\mu_b - \lambda_{b})^2}{\sigma_{b}^2}\right) \: ,
	\end{align*}
	and a similar one for $T^\star_{\sigma^2}(\mu)$. Since $x \mapsto \ln(1+x)$ is concave, we obtain that
	\begin{align*}
		2T^\star(\mu,\sigma^2)^{-1} \leq \ln \left( 1+ 2 T^\star_{\sigma^2}(\mu)^{-1} \right) < 2T^\star_{\sigma^2}(\mu)^{-1}
	\end{align*}
	where the strict inequality uses that $T^\star_{\sigma^2}(\mu)^{-1} >0$.
\end{proof}

Using Lemma~\ref{lem:complexity_expressions}, Lemma~\ref{lem:ratio_characteristic_time_large} shows that $T^\star(\mu, \sigma^2) / T^\star_{\sigma^2}(\mu)$ can become arbitrarily large by taking instances with large gaps.
\begin{lemma} \label{lem:ratio_characteristic_time_large}
	For all $K \ge 2$, there exists a sequence of instances $(\nu_n)_{n \in \Natural}$ with $|a^\star(\nu_n)| = 1$ such that $\lim_{n \to + \infty} T^\star(\mu_n, \sigma_n^2) / T^\star_{\sigma_n^2}(\mu_n) = + \infty$.
\end{lemma}
\begin{proof}
	Let $K \ge 2$.
	We construct explicitly the sequence instances.
	Let $(\Delta_{n})_{n \in \Natural}$ such that $\Delta_{n} > 0$ for all $n \in \Natural$ and $\lim_{n \to + \infty} \Delta_{n} = + \infty$.
	For all $n \in \Natural$, we define $\nu_n = \nu_{\mu_n, \sigma^2_{n}}$ where $(\mu_{n,1}, \sigma_{n,1}^2) = (0,1)$ and $(\mu_{n,a}, \sigma_{n,a}^2) = (-\Delta_{n},1)$ for all $a \neq 1$.

	Using Lemma~\ref{lem:complexity_expressions}, we obtain
	\begin{align*}
		\frac{T^\star_{\sigma_n^2}(\mu_n) \Delta_{n}^2}{2}  &=  \min_{\beta \in (0,1)} \min_{w \in \triangle_{K-1}} \max_{a \ne 1} \frac{1}{(1-\beta)w_a} + \frac{1}{\beta} = \min_{\beta \in (0,1)} \frac{K-1}{1-\beta} + \frac{1}{\beta} = (1+ \sqrt{K-1})^2 \: ,
	\end{align*}
	and, taking $y = -\Delta_{n}/2$, we obtain
	\begin{align*}
		2T^\star(\mu_n, \sigma_n^2)^{-1}
		&= \max_{\beta \in (0,1)} \max_{w \in \triangle_{K-1}} \min_{a \ne 1} \inf_{y \in [-\Delta_{n}, 0]}  \left\{ \beta \log( 1 + y^2) + (1-\beta) w_a \log ( 1 +(\Delta_{n} + y)^2) \right\} \\
		&\le \log ( 1 +\Delta_{n}^2 / 4 ) \max_{\beta \in (0,1)}  \left\{ \beta  + (1-\beta) \max_{w \in \triangle_{K-1}} \min_{a \ne 1} w_a \right\} \\
		&= \log ( 1 +\Delta_{n}^2 / 4 )   \max_{\beta \in (0,1)}  \left\{ \beta  + \frac{1-\beta}{K-1} \right\} \le \log ( 1 +\Delta_{n}^2 / 4 ) \: .
	\end{align*}
	Therefore, we have
	\[
	\frac{T^\star(\mu_n, \sigma_n^2)}{T^\star_{\sigma_n^2}(\mu_n)} \ge \frac{\Delta_{n}^2}{(1+ \sqrt{K-1})^2\log ( 1 +\Delta_{n}^2 / 4 ) } \to_{n \to + \infty} + \infty \: ,
	\]
	since $\lim_{n \to + \infty} \Delta_{n} = + \infty$ and $\log(1+x) =_{x \to + \infty} o(x)$.
\end{proof}

It is slightly more technical to construct a sequence of instances $(\nu_n)_{n \in \Natural}$ with $\min_{a \neq b} |\mu_{a,n} - \mu_{b,n}| > 0$ such that $\lim_{n \to + \infty} T^\star_{\beta}(\mu_n, \sigma_n^2) / T^\star_{\sigma_n^2, \beta}(\mu_n) = + \infty$.
Since the means have to be distinct, it is not possible to use directly closed-form formulas.
However, the same type of construction with increasing gaps will allow to conclude the proof.
The main argument behind both proofs is that $\log(1+x) =_{x \to + \infty} o(x)$.

\subsection{Optimal Allocation Oracles} \label{app:ss_optimal_allocation_oracles}

In \citet{garivier_2016_OptimalBestArm}, they show that $w^\star_{\sigma^2}(\mu)$ can be computed as the solution of an optimization problem with one bounded variable.
By using similar arguments, Theorem~\ref{thm:equivalent_optimization_problem} gives a similar equivalent optimization problem for $w^\star(\mu, \sigma^2)$.
Without loss of generality we consider $a^\star(\mu)=1$ to ease the notations in the following arguments.

\begin{theorem} \label{thm:equivalent_optimization_problem}
For all $a \in [K]$ and $\lambda \in \Real$, let $d_{a}(\lambda) = \log \left( 1 + \frac{(\lambda - \mu_a)^2}{\sigma^2_a}\right)$. For all $a \neq 1$ and $x \in [0,+\infty)$, let $g_{a}(x) =  \min_{\lambda \in \Real} \left( d_{1}(\lambda) + x d_{a}(\lambda) \right)$, with $\lambda_{a}(x)$ being the minimizer realizing $g_{a}(x)$.
The functions $g_{a}$ are strictly increasing function with values on $[0,d_{1}(\mu_a))$ with inverse function $x_{a}(y) = g^{-1}(y)$.
Then, we have $w(\mu, \sigma^2)_{a} = x_{a}(y^\star)w(\mu, \sigma^2)_{1}$ for all $a\neq 1$,
\begin{align*}
	&2T^\star(\mu, \sigma^2)^{-1} = \frac{y^\star}{1+\sum_{a\neq 1}x_{a}(y^\star)} \quad \text{and} \quad w(\mu, \sigma^2)_{1} = \frac{1}{1+\sum_{a\neq 1}x_{a}(y^\star)}
\end{align*}
Moreover, $y^\star \in [0,\min_{a\neq 1}d_{1}(\mu_a))$ is a solution of the equation $F(y)=1$ where
\begin{align*}
	F(y) = \sum_{a = 2}^{K} \frac{d_{1}(\lambda_{a}(x_{a}(y)))}{d_{a}(\lambda_{a}(x_{a}(y)))}
\end{align*}
is an increasing function such that $F(0)=0$ and $\lim_{+ \infty }F(y) = + \infty$.
\end{theorem}
\begin{proof}
	When $w_1 = 0$, we have $\min_{a \neq 1} \left(w_1 d_{1}(\lambda) + w_a d_{a}(\lambda) \right) = 0$, hence this is not the maximum of a positive quantity, i.e. for all $w \in w^\star(\mu, \sigma^2)$, $w_1 > 0$.
	Therefore, by dividing by $w_1$ and using that $g_{a}(x) = \inf_{\lambda \in \Real} \left( d_{1}(\lambda) + x d_{a}(\lambda) \right)$, we obtain directly
	\begin{align*}
		2T(\mu, \sigma^2)^{-1} = \max_{w \in \simplex} \min_{a \neq 1} \inf_{\lambda \in \Real} \left(w_1 d_{1}(\lambda) + w_a d_{a}(\lambda) \right) = \max_{w \in \simplex} w_1 \min_{a \neq 1} g_{a}\left( \frac{w_a}{w_1}\right) \: .
	\end{align*}

	Let $w^\star \in w^\star(\mu, \sigma^2)$.
	Then, using the above result, we obtain
	\begin{align*}
		w^\star \in \argmax_{w \in \simplex} w_1 \min_{a \neq 1} g_{a}\left( \frac{w_a}{w_1}\right)
	\end{align*}
	Introducing $x_{a}^{\star}=\frac{w^\star_a}{w_{1}^{\star}}$ for all $a \neq 1$, using that $\sum_{a\in [K]}w_a^\star = 1$, one has
	\begin{align*}
	w_{1}^{\star}=\frac{1}{1+\sum_{a=2}^{K} x_{a}^{\star}} \quad \text { and, for } a \geq 2, w^\star_a=\frac{x_{a}^{\star}}{1+\sum_{a=2}^{K} x_{a}^{\star}} \: .
	\end{align*}
	Moreover, $\{x_{a}^{\star}\}_{a=2}^{K} \in \mathbb{R}^{K-1}$ belongs to
	\begin{equation} \label{eq:reformulation_optimization_with_xs}
		\argmax_{\{x_{a}\}_{a=2}^{K} \in \mathbb{R}^{K-1}} \frac{\min _{a \neq 1} g_{a}\left(x_{a}\right)}{1+\sum_{a=2}^{K} x_a}
	\end{equation}

Let $\cO =\left\{ b \in [K]\setminus \{1\} \mid g_{b}\left(x_{b}^{\star}\right)=\min _{a \neq 1} g_{a}\left(x_{a}^{\star}\right)\right\}$ and $\cA= [K]\setminus (\{1\}\cup \cO )$.
	Let's show that all the $g_{a}\left(x_{a}^{\star}\right)$ have to be equal.
	 Assume that $\cA \neq \emptyset$. For all $a \in \cA$ and $b \in \cO$, one has $g_{a}\left(x_{a}^{\star}\right)>g_{b}\left(x_{b}^{\star}\right)$.
		Using the continuity of the $g_a$ functions and the fact that they are strictly increasing (Lemma~\ref{lem:characteristic_time_properties_to_show_equivalent}), there exists $\epsilon>0$ such that
	\begin{align*}
	\forall a \in \cA, b \in \cO, \quad g_{a}\left(x_{a}^{\star}-\epsilon /|\cA|\right)>g_{b}\left(x_{b}^{\star}+\epsilon /|\cO|\right)>g_{b}\left(x_{b}^{\star}\right) \: .
	\end{align*}
	We introduce $\bar{x}_{a}=x_{a}^{\star}-\epsilon /|\cA|$ for all $a \in \cA$ and $\bar{x}_{b}=x_{b}^{\star}+\epsilon /|\cO|$ for all $b \in \cO$, hence $\sum_{a=2}^{K} \bar{x}_{a} = \sum_{a=2}^{K}x_{a}^{\star}$. There exists $b \in \cO$ such that $\min _{a \neq 1} g_{a}\left(\bar{x}_{a}\right) = g_{b}\left(x_{b}^{\star}+\epsilon /|\cO|\right)$, hence
	\begin{align*}
	\frac{\min _{a \neq 1} g_{a}\left(\bar{x}_{a}\right)}{1+\bar{x}_{2}+\ldots \bar{x}_{K}}=\frac{g_{b}\left(x_{b}^{\star}+\epsilon /|\cO|\right)}{1+x_{2}^{\star}+\cdots+x_{K}^{\star}}>\frac{g_{b}\left(x_{b}^{\star}\right)}{1+x_{2}^{\star}+\cdots+x_{K}^{\star}}=\frac{\min _{a \neq 1} g_{a}\left(x_{a}^{\star}\right)}{1+x_{2}^{\star}+\cdots+x_{K}^{\star}} \:,
	\end{align*}
	This is a contradiction with the fact that $x^{\star}$ belongs to (\ref{eq:reformulation_optimization_with_xs}).
	Hence $\cA = \emptyset$ and there exists $y^\star \in [0, \min_{a\neq 1}d_{1}(\mu_a))$ such that, for all $a \in [K]\setminus \{1\}$,
	\begin{align*}
		g_{a}(x_{a}^{\star}) = y^\star \iff x_{a}^{\star} = x_{a}(y^\star)
	\end{align*}
	with the function $x_{a}$ introduced in Lemma~\ref{lem:characteristic_time_properties_to_show_equivalent}. From (\ref{eq:reformulation_optimization_with_xs}), $y^\star$ belongs to
	\begin{align*}
		\argmax_{y \in [0, \min_{a\neq 1}d_{1}(\mu_a))} \frac{y}{1+\sum_{a=2}^{K} x_a(y)}
	\end{align*}
	Using Lemma~\ref{lem:characteristic_time_properties_to_show_equivalent}, we have that $y^\star$ is a solution of $F(y)=1$ with
	\begin{align*}
		F(y) = \sum_{a = 2}^{K} \frac{d_{1}(\lambda_{a}(x_{a}(y)))}{d_{a}(\lambda_{a}(x_{a}(y)))} \: .
	\end{align*}
\end{proof}

A key theoretical and computational difference between the oracle for $w^\star_{\sigma^2}(\mu)$ and the one for $w^\star(\mu,\sigma^2)$ is that $\lambda_{a}(x)$ is defined implicitly as one of the real solution of a third order polynomial equation (Lemma~\ref{lem:characteristic_time_properties_to_show_equivalent}).
Fortunately, there exists closed form solutions for the roots of a third order polynomial equation, namely Cardano's formula.
Therefore, we only need to compute three roots, among which at least one is real (two might be complex), and find the real one minimizing our original functions.

Lemma~\ref{lem:characteristic_time_properties_to_show_equivalent} gathers technical results used to prove Theorem~\ref{thm:equivalent_optimization_problem}.

\begin{lemma} \label{lem:characteristic_time_properties_to_show_equivalent}
	Let $d_a$, $g_a$ and $\lambda_a$ as in Theorem~\ref{thm:equivalent_optimization_problem}.

	\begin{enumerate}
		\item For all $x \in (0,+\infty)$, $\lambda_{a}(x) \in (\mu_{a},\mu_{1})$, $\lambda_{a}(0) = \mu_{1}$ and $\lim_{x \rightarrow + \infty}\lambda_{a}(x) = \mu_{a}$.
		Moreover,	the functions $\lambda_{a}(x)$ are among the (at least one) real solutions of $P_{a}(\lambda, x) = 0$ where
		\begin{align*}
			P_{a}(\lambda, x) &= \lambda^3 - \alpha_{a,2}(x) \lambda^2 + \alpha_{a,1}(x) \lambda - \alpha_{a,0}(x) \: ,\\
			\alpha_{a,2}(x) &= \mu_{a} + \mu_{1} +\frac{\mu_{a} + x \mu_{1}}{1+x} \: , \\
			\alpha_{a,1}(x) &= \frac{\sigma_{a}^2+x\sigma_{1}^2}{1+x} + \mu_{1} \mu_{a} + (\mu_{1}+\mu_{a})\frac{\mu_{a}+x\mu_{1}}{1+x} \: , \\
			\alpha_{a,0}(x) &=   \frac{\mu_{1} (\mu_{a}^2+ \sigma_{a}^2) + \mu_{a} (\mu_{1}^2+ \sigma_{1}^2) x}{1+x} \:.
		\end{align*}
		\item The function $g_{a}$ is a concave and strictly increasing one-to-one mapping from $[0,+ \infty)$ to $[0, d_1(\mu_a))$, such that $g_{a}'(x) = d_{a}(\lambda_{a}(x))$ and $g_{a}''(x) = \lambda_{a}'(x) d_{a}'(\lambda_{a}(x))$. In particular, $x \mapsto \lambda_{a}(x)$ is decreasing.
		Moreover, the function $x_{a}$ is strictly increasing and $x_{a}'(y) = \frac{1}{d_{a}(\lambda_{a}(x(y)))}$.
		\item Defining $G(y) = \frac{y}{1+ \sum_{a=2}^{K}x_{a}(y)}$ and $F(y) = \sum_{a = 2}^{K} \frac{d_{1}(\lambda_{a}(x_{a}(y)))}{d_{a}(\lambda_{a}(x_{a}(y)))}$ for $y \in [0,\min_{a\neq 1}d_{1}(\mu_{a}))$. Then $G'(y) = 0$ if and only if $F(y)=1$. The function $F$ is increasing such that $F(0) = 0$ and $\lim_{y \mapsto \min_{a\neq 1} d_{1}(\mu_{a})} F(y) = +\infty$.
	\end{enumerate}
\end{lemma}
\begin{proof}
	For all $a \in [K]$ and all $\lambda \in \Real$, let $d_{a}^{\text{KV}}(\lambda) = \frac{(\lambda - \mu_a)^2}{\sigma^2_a}$. Then, we have
	\[
	d_{a}'(\lambda) = \frac{(d^{\text{KV}}_{a})'(\lambda)}{1+d^{\text{KV}}_{a}(\lambda)} = \frac{2(\lambda - \mu_{a})}{\sigma_{a}^2 + (\lambda - \mu_{a})^2} \: .
\]

	\textbf{(1)}
	Let $h_{a}(x, \lambda) = d_{1}(\lambda) + x d_{a}(\lambda) $.
	Since $\frac{\partial h_{a}}{\partial \lambda}(x, \lambda) = d_{1}'(\lambda) + x d_{a}'(\lambda)$, we have $\frac{\partial h_{a}}{\partial \lambda}(x, \lambda) > 0$ for $\lambda > \mu_{1}$ and $\frac{\partial h_{a}}{\partial \lambda}(x, \lambda) < 0$ for $\lambda < \mu_{a}$. Therefore, the solution of $\frac{\partial h_{a}}{\partial \lambda}(x, \lambda) = 0$ is in $[\mu_{a},\mu_{1}]$ if it exists. As $h_{a}(x, \lambda)$ is continuous and bounded on $[\mu_{a},\mu_{1}]$, $\lambda_{a}(x)$ exists.

	The functions $d_{a}$ have a unique minimizer $\mu_{a}$ such that $d_{a}(\mu_{a}) = 0$.
	Using that $h_{a}(0, \lambda) = d_{1}(\lambda)$, we obtain $\lambda_{a}(0) = \mu_{1}$. Introducing $w = w_1 x$ such that $w_{1}+w \leq 1$, we have $\lambda_{a}(x)  = \tilde{\lambda}_{a}(w) =  \argmin_{\lambda \in [\mu_{a},\mu_{1}]} w_1 d_{1}(\lambda) + w d_{a}(\lambda)$.
	The same argument as above shows that $\tilde{\lambda}_{a}(1) = \mu_{a}$, as $w=1$ implies $w_{0}$, we have $\lim_{x \rightarrow + \infty}\lambda_{a}(x) = \tilde{\lambda}_{a}(1) = \mu_{a}$.

	Since $d_{a}'$ is continuous, we have $d_{a}'(\mu_a) = 0$. For all $x>0$ and $a \in [K]\setminus \{1\}$, we have $\frac{\partial h_{a}}{\partial \lambda}(x, \mu_{a}) = d_{1}'(\mu_{a}) + x d_{a}'(\mu_{a}) =  d_{1}'(\mu_{a}) < 0$ and $\frac{\partial h_{a}}{\partial \lambda}(x, \mu_{1}) = d_{1}'(\mu_{1}) + x d_{a}'(\mu_{1}) = x d_{a}'(\mu_{1}) > 0$. Therefore, we have $\lambda_{a}(x) \in (\mu_{a},\mu_{1})$ for all $x \in (0,+\infty)$.

	Writing the condition that $\lambda_{a}(x)$ is a minimum, rewrites as $d_{1}'(\lambda_{a}(x)) + x d_{a}'(\lambda_{a}(x)) = 0$ and $d_{1}''(\lambda_{a}(x)) + x d_{a}''(\lambda_{a}(x)) > 0$. Direct computations yield
		\begin{align*}
			&d_{1}'(\lambda_{a}(x)) + x d_{a}'(\lambda_{a}(x)) = 0  \\
			\iff & (\lambda - \mu_{1}) (\lambda - \mu_{a})\left( \lambda - \frac{\mu_a + x\mu_1}{1+x} \right)  + \frac{\sigma_{a}^2+x\sigma_{1}^2}{1+x} \left( \lambda - \lambda^{\text{KV}}_{a}(x) \right) = 0 \iff  P_{a}(\lambda, x) = 0
		\end{align*}
		where $P_{a}(\lambda, x)$ is defined in the statement of Lemma~\ref{lem:characteristic_time_properties_to_show_equivalent}.
		When $x \rightarrow + \infty$, we already know that $\lambda_{a}(x) = \mu_{a} + o(1)$.
		To prove $\lim_{x \rightarrow +\infty} x d_{a}(\lambda_{a}(x)) = 0$, we need a finer dependency in $x$.
		Writing the change of variable $y = \lambda - \mu_a$, we obtain $P_{a}(\lambda, x)= 0$ if and only if $Q_{a}(y,x) = 0$ where
		\begin{align*}
			Q_{a}(y,x) = y (y + \mu_{a} - \mu_{1}) \left( y - \frac{x(\mu_{1}-\mu_a)}{1+x} \right)  + \frac{\sigma_{a}^2+x\sigma_{1}^2}{1+x} \left(y - \frac{\sigma_{a}^2(\mu_{1} - \mu_{a})}{\sigma_{a}^2+x \sigma_{1}^2} \right)
		\end{align*}
		When $x \rightarrow + \infty$, we have $y(x) = o(1)$, $\frac{x(\mu_{1}-\mu_a)}{1+x} = (\mu_{1} -\mu_{a})(1+\cO(\frac{1}{x}))$, $\frac{\sigma_{a}^2(\mu_{1} - \mu_{a})}{\sigma_{a}^2+x \sigma_{1}^2} = \cO(\frac{1}{x})$ and $\frac{\sigma_{a}^2+x\sigma_{1}^2}{1+x} = \sigma_{1}^2 + \cO(\frac{1}{x})$, we obtain the following
		\begin{align*}
			Q_{a}(y(x),x) = 0 &\implies  y(x) \left( \sigma_{1}^2  + (\mu_{1} - \mu_{a})^2 + o(1) \right) +  \cO(\frac{1}{x}) = 0
		\end{align*}
		This yields that $y(x) = \cO(\frac{1}{x})$. Therefore, $d^{\text{KV}}_{a}(\lambda_{a}(x)) = \frac{y(x)^2}{\sigma_{a}^2} = \cO(\frac{1}{x^2})$ and
	\begin{align*}
		xd_{a}(\lambda_{a}(x)) = x\ln(1+d^{\text{KV}}_{a}(\lambda_{a}(x))) \sim_{\infty} x d^{\text{KV}}_{a}(\lambda_{a}(x)) \sim_{\infty} x  \frac{y(x)^2}{\sigma_{a}^2} = 0
	\end{align*}
	Therefore, we have shown $\lim_{x \rightarrow +\infty} x d_{a}(\lambda_{a}(x)) = 0$.

	\noindent\textbf{(2)}
	Since $g_{a}(x) =  \min_{\lambda \in [\mu_{a},\mu_{1}]}  d_{1}(\lambda) + x d_{a}(\lambda)$, $g_a$ is a concave function as infimum of an infinite number of linear functions. Since $g_{a}(x)  = h_{a}(x, \lambda_{a}(x))$, $\frac{\partial h_{a}}{\partial \lambda}(x, \lambda_{a}(x)) = 0$ (optimality condition) and $\frac{\partial h_{a}}{\partial x}(x, \lambda) = d_{a}(\lambda)$, we have
	\begin{align*}
		g_{a}'(x)  &= \frac{\partial h_{a}}{\partial x}(x, \lambda_{a}(x)) +  \lambda_{a}'(x) \frac{\partial h_{a}}{\partial \lambda}(x, \lambda_{a}(x)) = \frac{\partial h_{a}}{\partial x}(x, \lambda_{a}(x)) = d_{a}(\lambda_{a}(x)) > 0
	\end{align*}
	where the last inequality is strict since $\lambda_{a}(x) \in (\mu_{a},\mu_{1}]$ for $x \in [0,+\infty)$.
	Therefore, $g_{a}$ is a strictly increasing one-to-one mapping from $[0,+ \infty)$ to $[0, \lim_{+\infty} g_{a}(x))$.
	Since $\lim_{x \rightarrow + \infty}\lambda_{a}(x) =\mu_{a}$ and $\lim_{x \rightarrow +\infty} x d_{a}(\lambda_{a}(x)) = 0$, we obtain $\lim_{x \rightarrow +\infty} g_{a}(x) = d_{1}(\mu_{a})$.

	Directly computing the derivative of $d_{a}(\lambda_{a}(x))$, we obtain $g_{a}''(x) = \lambda_{a}'(x) d_{a}'(\lambda_{a}(x))$.
	Using that $\lambda_{a}(x) > \mu_{a}$ for all $x \in [0,+\infty)$, we have $d_{a}'(\lambda_{a}(x)) > 0 $.
	Since $g_a$ is concave, we know that $g_{a}''(x) \geq 0$, hence $\lambda_{a}'(x)\leq 0$, i.e. $x \mapsto \lambda_{a}(x)$ is decreasing.

 	Using the above result, $x_{a}(y)$ is well defined.
	 Using the derivative of the inverse function and $g_{a}'(x) = d_{a}(\lambda_{a}(x))$, we obtain $x_{a}'(y) = \frac{1}{g_{a}'(x_{a}(y))} = \frac{1}{d_{a}(\lambda_{a}(x_{a}(y)))}$. Since $d_{a}(\lambda_{a}(x)) > 0 $ for all $x \in [0,+\infty)$, we have $x_{a}'(y) > 0$ for all $y \in [0, d_{1}(\mu_{a}))$.

	\noindent\textbf{(3)} Using that $x_{a}'(y) = \frac{1}{d_{a}(\lambda_{a}(x(y)))}$ and $d_{1}(\lambda_{a}(x_{a}(y))) + x_{a}(y) d_{a}(\lambda_{a}(x_{a}(y))) = y $. Direct computations yield that $G'(y) = 0$ if and only if
	\begin{align*}
		&\frac{1}{1+ \sum_{a=2}^{K}x_{a}(y)} = \frac{y}{(1+ \sum_{a=2}^{K}x_{a}(y))^2} \sum_{a=2}^{K}x_{a}'(y) \\
		& \iff \sum_{a=2}^{K} \frac{y}{d_{a}(\lambda_{a}(x(y)))} = 1+ \sum_{a=2}^{K}x_{a}(y)  \iff \sum_{a=2}^{K} \frac{d_{1}(\lambda_{a}(x(y)))}{d_{a}(\lambda_{a}(x(y)))} = 1  \iff F(y) = 1
	\end{align*}
		Using that $\lambda_{a}(0) = \mu_1$ and $d_{1}(\mu_{1}) = 0$, we obtain $F(0) = 0$.
		Using that $\lim_{x \rightarrow + \infty} \lambda(x) = \mu_{a}$, $d_{1}(\mu_{a}) > 0$, $d_{a}(\mu_{a}) = 0$ and $\lim_{y \rightarrow d_1(\mu_a)}x_{a}(y) = + \infty$, we obtain that $\lim_{y \mapsto \min_{a\neq 1} d_1(\mu_a)} F(y) = +\infty$.	Let $H(y) = \sum_{a = 2}^{K} \frac{1}{d_{a}(\lambda_{a}(x_{a}(y)))}$ for $y \in [0,\min_{a\neq 1}d_{1}(\mu_{a}))$.
		Since $d_{1}'(\lambda_{a}(x)) + x d_{a}'(\lambda_{a}(x)) = 0$, $d_{1}(\lambda_{a}(x_{a}(y))) + x_{a}(y) d_{a}(\lambda_{a}(x_{a}(y))) = y $, direct computations yield
		\begin{align*}
			H'(y) & = -  \sum_{a = 2}^{K} x_{a}'(y) \lambda_{a}'(x_{a}(y))   \frac{d_{a}'(\lambda_{a}(x_{a}(y)))}{d_{a}(\lambda_{a}(x_{a}(y)))^2} \\
			F'(y) &= \sum_{a = 2}^{K} x_{a}'(y) \lambda_{a}'(x_{a}(y)) \frac{d_{1}'(\lambda_{a}(x_{a}(y))) d_{a}(\lambda_{a}(x_{a}(y))) - d_{1}(\lambda_{a}(x_{a}(y))) d_{a}'(\lambda_{a}(x_{a}(y)))}{d_{a}(\lambda_{a}(x_{a}(y)))^2}\\
			&= -\sum_{a = 2}^{K} x_{a}'(y) \lambda_{a}'(x_{a}(y)) d_{a}'(\lambda_{a}(x_{a}(y))) \frac{x_{a}(y) d_{a}(\lambda_{a}(x_{a}(y))) + d_{1}(\lambda_{a}(x_{a}(y)))}{d_{a}(\lambda_{a}(x_{a}(y)))^2} \\
			&= -y  \sum_{a = 2}^{K} x_{a}'(y) \lambda_{a}'(x_{a}(y))  \frac{d_{a}'(\lambda_{a}(x_{a}(y)))}{d_{a}(\lambda_{a}(x_{a}(y)))^2} = y H'(y)
		\end{align*}

		Note that for all $a \in [K]\setminus \{1\}$, $y \mapsto \frac{1}{d_{a}(\lambda_{a}(x_{a}(y)))}$ is an increasing function of $y$ since $d_{a}(\lambda)$ is strictly increasing on $(\mu_{a}, \mu_{1}]$, $\lambda_{a}(x)$ is decreasing on $[0,+\infty)$ and $x_{a}(y)$ is strictly increasing on $y \in [0,\min_{a\neq 1}\lim_{+\infty} g_{a}(x))$. As a summation of increasing functions, $H$ is an increasing function, i.e. $H'(y) \geq 0$. Since $F'(y) = y H'(y) \geq 0$, this yield that $F$ is an increasing function.
\end{proof}

%% file: sections/appendix_glrt.tex

\section{Generalized Log-Likelihood Ratios} \label{app:glr}

The GLR to reject the hypothesis $\cH_0 =\{(\mu,\sigma^{2}) \in \Lambda\}$ for a subset $\Lambda \subseteq \cM$ of the whole parameter space $\cM$ is given by
\begin{equation*}
\text{GLR}^{\cM}_{t}(\Lambda)
\eqdef \ln \frac{\sup_{(\mu, \sigma^2) \in \cM} d \nu_{\mu, \sigma^2}(X_{1,a_1}, \cdots, X_{t, a_t})}{\sup_{(\lambda, \kappa^2) \in \Lambda} d \nu_{\lambda, \kappa^2}(X_{1,a_1}, \cdots, X_{t,a_t})}
\end{equation*}
where $d\nu_{\mu, \sigma^2}(X_{1,a_1}, \cdots, X_{t,a_t})$ denotes the likelihood of the observations $X_{1,a_1}, \cdots, X_{t,a_t}$ for a Gaussian bandit with parameters $(\mu, \sigma^2)$. The empirical mean and variance are defined as
\begin{equation*}
	\mu_{t,a} = \frac{1}{N_{t,a}} \sum_{s \in [t]} \1\{a_s = a\} X_{s,a}
	\quad \quad \text{and} \quad \quad
	\sigma_{t,a}^{2} = \frac{1}{N_{t,a}} \sum_{s \in [t]} \1\{a_s = a\}  \left( X_{s,a} - \mu_{t,a}\right)^2 \: .
\end{equation*}

Using the formula of the KL for Gaussian bandits with different variances, the GLR rewrites as
\begin{align*}
	\text{GLR}^{\cM}_{t}(\Lambda) &= \inf_{(\lambda, \kappa^2) \in \Lambda}  \sum_{a \in [K]} \frac{N_{t,a}}{2}\left(  \ln\left(\frac{\kappa_{a}^2}{\sigma_{t,a}^2}\right)  + \frac{\sigma_{t,a}^{2} + (\mu_{t,a} - \lambda_a)^2}{\kappa_{a}^2}\right) \\
	&= \inf_{(\lambda, \kappa^2) \in \Lambda}  \sum_{a \in [K]} N_{t,a} \KL((\mu_{t,a}, \sigma_{t,a}^2), (\lambda_a, \kappa_{a}^2)) \: .
\end{align*}

When $\Lambda = \cM = \mathbb{R}^K \times (\mathbb{R}^\star_+)^K$, the minimizers of the log-likelihood are $(\mu_t, \sigma_t^2)$, i.e. the empirical estimators $(\mu_t, \sigma_t^2)$ are exactly the MLE.

\paragraph{Explicit formulas} The GLR has a more explicit formula when considering the alternative to the subset of interest $\Lambda(\mu_t, \sigma_t^2)$ (Lemma~\ref{lem:glrt_formula_gaussian}).

\begin{lemma} \label{lem:glrt_formula_gaussian}
Let $\cM = \mathbb{R}^K \times (\mathbb{R}^\star_+)^K$. Then, $\GLR^{\mathcal{D}}_{t}(\Lambda(\mu_t, \sigma_t^2)) = \min_{a \neq \hat{a}_t} Z_a(t)$ where
\begin{equation*}
	Z_a(t) = \GLR^{\mathcal{D}}_{t}(\{(\lambda, \kappa^2) \mid \lambda_a \ge \lambda_{\hat a_t}\}) = \inf_{ \lambda : \lambda_{a} \geq \lambda_{\hat a_t}}  \sum_{b \in \{a, \hat a_t\}} \frac{N_{t,b}}{2} \ln \left( 1 + \frac{(\mu_{t,b} - \lambda_{b})^2}{\sigma_{t,b}^2}\right) \: .
\end{equation*}
\end{lemma}
\begin{proof}
Since $a^\star(\mu_t) = \hat a_t$ and $\Lambda(\mu_t, \sigma_t^2) = \bigcup_{a \ne \hat a_t} \{(\lambda, \kappa^2) \mid \lambda_a \ge \lambda_{\hat a_t}\}$, we obtain
\begin{align*}
	&\text{GLR}^{\cM}_{t}(\Lambda(\mu_t, \sigma_t^2)) = \min_{a \neq \hat a_t}  Z_{a}(t) \quad \text{with} \quad Z_{a}(t) = \GLR^{\mathcal{D}}_{t}(\{(\lambda, \kappa^2) \mid \lambda_a \ge \lambda_{\hat a_t}\}) \: .
\end{align*}
For $b \notin \{a , \hat a_t\}$, we can take $(\lambda_b, \kappa_{b}^2) = (\mu_{t,b}, \sigma_{t,b}^2)$, hence we obtain
\begin{align*}
	Z_{a}(t) &= \inf_{ (\lambda,\kappa^2) : \lambda_{a} \geq \lambda_{\hat a_t}}  \sum_{b \in \{a, \hat a_t\}} N_{t,b} \KL \left( (\mu_{t,b}, \sigma_{t,b}^2), (\lambda_{b}, \kappa_{b}^2)\right) \\
	&=\inf_{ \lambda : \lambda_{a} \geq \lambda_{\hat a_t}}  \sum_{b \in \{a, \hat a_t\}} \frac{N_{t,b}}{2} \ln \left( 1 + \frac{(\mu_{t,b} - \lambda_{b})^2}{\sigma_{t,b}^2}\right)
\end{align*}
The second equality uses that the objectives and the constraints are separate and that
\[
\argmin_{\kappa_{b}^2}\KL \left( (\mu_{t,b}, \sigma_{t,b}^2), (\lambda_{b}, \kappa_{b}^2)\right) = (\mu_{t,b} - \lambda_b)^2 + \sigma_{t,b}^2 \: ,
\]
since $f_{a}(x) = \frac{a}{x} + \ln(x)$ has $f_{a}'(x) = \frac{1}{x} - \frac{a}{x^2}$ and $\argmin_{x > 0} f_{a}(x) = a$, for $a = (\mu_{t,b} - \lambda_b)^2 + \sigma_{t,b}^2$.
\end{proof}

\paragraph{Inequalities}
Lemma~\ref{lem:glrt_evglrt_inequalities} gives inequalities between $Z_a(t)$ and $Z^{\text{EV}}_a(t)$.

\begin{lemma} \label{lem:glrt_evglrt_inequalities}
For $a\neq \hat{a}_t$, let $C_{a}(\mu_{t}, \sigma^2_{t}) = \frac{(\mu_{t, \hat a_t} - \mu_{t, a})^2}{\min\{\sigma_{t, \hat a_t}^2,\sigma_{t,a}^2\}}$. The statistics $Z^{\text{EV}}_a(t) $ and $Z_a(t)$ satisfy
\begin{equation} \label{eq:glrt_evglrt_inequalities}
	Z^{\text{EV}}_a(t) \geq Z_a(t) \geq \frac{\ln \left( 1+ C_a(\mu_{t}, \sigma^2_{t})\right)}{C_a(\mu_{t}, \sigma^2_{t})}  Z^{\text{EV}}_a(t) \: .
\end{equation}
\end{lemma}

The proof is omitted since it was obtained with similar manipulations on $Z_a(t)$ and $Z^{\text{EV}}_a(t)$ as done on $T^\star(\mu, \sigma^2)$ and $T^\star_{\sigma^2}(\mu)$ in Appendix~\ref{app:ss_explicit_formulas_inequalities}.

%% file: sections/appendix_box_concentration.tex

\section{Box Concentration} \label{app:box_concentration}

In Appendix~\ref{app:ss_sub_exponential_concentration}, we derive time-uniform upper and lower tail concentrations for sub-exponential processes, as well as their fixed-time counterpart. In Appendix~\ref{app:ss_gaussian_concentration}, for Gaussian observations, we use those results to obtain similar concentrations for the empirical variance and derive time-uniform (and fixed-time) upper and lower tail concentrations for the empirical mean.

\subsection{Sub-Exponential Processes}
\label{app:ss_sub_exponential_concentration}

We prove time-uniform and fixed-time concentration results for $1$-sub-$\psi_{E,c}$ process with variance process $V_t = ct$ (Appendix~\ref{app:sss_sub_exponential_upper_tail_concentration}) and $1$-sub-$\psi_{E,-c}$ process with variance process $V_t = ct$ (Appendix~\ref{app:sss_sub_exponential_lower_tail_concentration}). The concept of sub-$\psi$ process (Definition~\ref{def:sub_phi_process}) was introduced in \citet{howard_2020_TimeuniformChernoffBounds}. This concept is particularly useful to derive time-uniform concentration results.

\begin{definition} \label{def:sub_phi_process}
	Let $(S_t)_{t \in \cT \cup \{0\}}$ and $(V_t)_{t \in \cT \cup \{0\}}$ be two real-valued processes adapted to an underlying filtration $(\cF_{t})_{ t \in \cT \cup \{0\}}$ with $S_0 = 0$ and $V_{0} = 0$ a.s. and $V_t \geq 0$ a.s. for all $t \in \cT$. For a function $\psi: [0, \lambda_{\max}) \mapsto \Real$ and a scalar $l_{0} \in [1, + \infty)$, we say that $(S_{t})$ is $l_0$-sub-$\psi$ with variance process $(V_t)$ if, for each $\lambda \in [0, \lambda_{\max})$, there exists a supermartingale $(L_{t}(\lambda))_{t \in \cT \cup \{0\}}$ with respect to $(\cF_{t})$ such that $L_{0}(\lambda)\leq l_{0}$ a.s. and
	\begin{align*}
		\exp \left\{ \lambda S_t - \psi(\lambda) V_{t}\right\} \leq L_{t}(\lambda)\quad \quad \text{a.s. for all }t\in \cT \: .
	\end{align*}
\end{definition}

\begin{lemma}[Ville's inequality] \label{lem:ville_inequality}
	Let $\probability_{0}[\cdot] = \probability_{0}[\cdot \mid \cF_{0}]$. Let $\cT \subseteq \Natural^{\star}$, such that $|\cT| = \infty$. If $(L_{t})_{t \in \cT \cup \{0\}}$ is a non-negative supermartingale with respect to the filtration $(\cF_{t})_{ t \in \cT \cup \{0\}}$, then
	\begin{align*}
	\forall a > 0, \quad \probability_{0}\left(\exists t \in \cT: L_t \geq a \right) \leq L_{0} / a \: .
	\end{align*}
\end{lemma}

Since we aim at deriving one-sided bounds on scalar martingales, we have $l_0 = 1$. Using Ville's inequality (Lemma~\ref{lem:ville_inequality}) on a sub-$\psi$ process yields time-uniform concentration results. Let $(S_{t})$ be a $1$-sub-$\psi$ with variance process $(V_t)$, then for all $\lambda \in [0,\lambda_{\max})$, with probability greater than $1-\delta$,
\begin{equation*}
		\forall t \in \cT, \quad \lambda S_t - \psi(\lambda) V_{t} < \log\left( 1/\delta\right)\: .
\end{equation*}

Let $\lambda \in [0,\lambda_{\max})$. Direct manipulations show the above result,
\begin{align*}
		\probability \left(\exists t\in \cT: \: \lambda S_t - \psi(\lambda) V_{t} \geq \log\left( 1/\delta\right) \right) &\leq \probability \left( \exists t\in \cT: \:  L_{t}(\lambda) \geq 1/\delta \right) \leq \delta \: .
\end{align*}

In the following, we are interested by $1$-sub-$\psi_{E,c}$ processes for $c \in \Real$, where $\psi_{E,c}$ is defined as
\begin{equation}\label{eq:sub_exp_process}
	\forall \lambda \in \left[ 0, 1/(c \lor 0)\right), \quad \psi_{E,c}(\lambda) = \frac{-\log(1-c\lambda) - c\lambda}{c^2} \: .
\end{equation}

The derived upper and lower tails concentrations involve the positive ($i=0$) and negative ($i=-1$) Lambert's branches $W_{i}$ solutions of $W(x) e^{W(x)} = x$. We refer the reader to Appendix~\ref{app:lambert_W_functions} for mode details and corresponding technical results.

\subsubsection{Upper Tail Concentration}
\label{app:sss_sub_exponential_upper_tail_concentration}

We derive time-uniform and fixed-time upper tail concentration for $1$-sub-$\psi_{E,c}$ process with variance process $V_t = ct$.
While the time-uniform result require using the peeling method, the proof of the fixed-time concentration is simpler.
To use the peeling method, we need to control the deviation of the process on slices of time (Lemma~\ref{lem:fixed_subExp_martingale_on_a_slice_upper_tail}).

\begin{lemma} \label{lem:fixed_subExp_martingale_on_a_slice_upper_tail}
	Let $c > 0$ and $S_{t}$ a $1$-sub-$\psi_{E,c}$ process with variance process $V_t = ct$. Let $N > 0$. For all $x > 1$, there exists $\lambda = \lambda(x)$ such that for all $t \ge N$,
	\begin{align*}
		\left\{ S_t + t \geq t x \right\}
		\subseteq \left\{ \lambda S_{t} - ct \psi_{E,c}(\lambda) \geq \frac{N}{c} \left( h\left( x \right)-1\right)\right\}
	\end{align*}
	where $\lambda(x) = \argmax_{\lambda \in [0,1/c)} \left(x\lambda + \frac{\ln (1 - c \lambda)}{c} \right)$ and $h(x) = x-\ln(x)$ for $x>1$.
\end{lemma}
\begin{proof}
	Defining $\psi_{U}(\lambda) = \lambda + c\psi_{E,c}(\lambda)  = - \frac{\ln (1 - c \lambda)}{c}$ and $\lambda(x) = \argmax_{\lambda \in [0,1/c)} x\lambda - \psi_{U}(\lambda)$, we have $x\lambda(x) - \psi_{U}(\lambda(x)) = \psi^*_{U}\left(x\right)$ where $\psi^*_{U}$ is the convex conjugate of $\psi_{U}$. Note that $\psi^*_{U}\left(x\right) \geq 0$ (see below), hence $t \psi^*_{U}\left(x\right) \geq N \psi^*_{U}\left(x\right)$ for $t \geq N$. Direct computations yield
	\begin{align*}
		S_t + t \geq t x
		&\iff \lambda  S_t - ct \psi_{E,c}(\lambda) \geq t x\lambda - t(\lambda + c\psi_{E,c}(\lambda)) \\
		& \implies \lambda  S_t - ct \psi_{E,c}(\lambda) \geq t\left(x\lambda - \psi_{U}(\lambda) \right)
		= t \psi^*_{U}\left(x\right) \\
		& \implies  \lambda  S_t - ct \psi_{E,c}(\lambda) \geq N_i \psi^*_{U}\left(x\right)
		= \frac{N}{c} \left( h\left( x \right)-1\right)
	\end{align*}

	Note that for $f(\lambda) = \lambda x + \frac{\ln(1-c\lambda)}{c}$, we have $f'(\lambda) = x -\frac{1}{1-c\lambda} = 0 \iff \lambda = \frac{1}{c}\left( 1- \frac{1}{x} \right)$ and $\frac{1}{c}\left( 1- \frac{1}{x} \right) \in [0,\frac{1}{c}) \iff x > 1$. Since $f''(\lambda) = - \frac{c}{(1-c\lambda)^2} \leq 0$, the function is concave hence this is a maximum. This yields that for all $x > 1$, $\psi^*_{U}(x) = f(\frac{1}{c}\left( 1- \frac{1}{x} \right)) = \frac{1}{c}\left( x-1-\ln(x)\right) = \frac{1}{c}(h(x)-1) \geq 0$ where $h(x) = x - \ln(x)$.
\end{proof}

Let $\eta>0$. Applying Lemma~\ref{lem:fixed_subExp_martingale_on_a_slice_upper_tail} on slices of time with geometric growth rate $(N_{i})_{i \in \Natural^{\star}}$ with $N_i = (1+\eta)^{i-1}$, we obtain Lemma~\ref{lem:uniform_time_subExp_upper_tail_concentration}.

\begin{lemma} \label{lem:uniform_time_subExp_upper_tail_concentration}
	Let $\overline{W}_{-1}(x) = -W_{-1}(-e^{-x})$ for $x\geq1$, $\delta \in (0,1)$, $\eta > 0$, $s > 1$, $c > 0$, and $\zeta$ be the Riemann $\zeta$ function. Let $S_{t}$ a $1$-sub-$\psi_{E,c}$ process with variance process $V_t = ct$. Then, with probability greater than $1 - \delta$, for all $t \in \Natural^{\star}$,
	\begin{align*}
		S_{t} + t  \leq t\overline{W}_{-1} \left(1 +  \frac{c(1 + \eta)}{t}\left(\ln\left( \frac{\zeta(s)}{\delta} \right) + s\ln \left( 1+ \frac{\ln(t)}{\ln(1+\eta)}\right) \right)\right) \: .
	\end{align*}
\end{lemma}
\begin{proof}
	Let $g(t,\delta)$ such that $g(t,\delta) \geq x_{i}(\delta)$ for $t \in [N_i, N_{i+1})$ and $x_{i}(\delta) > 1$. Using Lemma~\ref{lem:fixed_subExp_martingale_on_a_slice_upper_tail} with $x_i(\delta) > 1$ and $g(t,\delta) \geq x_{i}(\delta)$ on $[N_i, N_{i+1})$, we obtain
	\begin{align*}
		\probability \left( \exists t \in \Natural^\star: S_t + t \geq t g(t,\delta) \right)
		&\leq \sum_{i \in \Natural^{\star}}  \probability \left( \exists t \in [N_i, N_{i+1}): S_t + t \geq t x_{i}(\delta) \right)
		\\
		&\leq  \sum_{i \in \Natural^{\star}}  \probability \left( \exists t \in [N_i, N_{i+1}):  \lambda S_{t} - ct \psi_{E,c}(\lambda) \geq \frac{N_i}{c} \left( h\left( x_{i}(\delta)\right)-1\right) \right)
		\\
		&\leq \sum_{i \in \Natural^{\star}} e^{-\frac{N_i}{c} \left( h\left(x_{i}(\delta)\right)-1\right)} \: ,
	\end{align*}
	where the last inequality uses that $S_{t}$ a $1$-sub-$\psi_{E,c}$ process with variance process $V_t = ct$. Taking
	\[
	g(t,\delta) = \overline{W}_{-1} \left(1 + \frac{c(1 + \eta)}{t}\left(\ln\left( \frac{\zeta(s)}{\delta} \right) + s\ln \left( 1+ \frac{\ln(t)}{\ln(1+\eta)}\right) \right)\right)
	\]
	and $x_{i}(\delta) = \overline{W}_{-1} \left(1 + \frac{c}{N_{i}}\ln\left( \frac{i^s\zeta(s)}{\delta} \right) \right)$ satisfies the required properties. First, we have $x_{i}(\delta) > 1$ (Lemma~\ref{lem:lambert_branches_properties}). Second, since $\overline{W}_{-1}$ is increasing on $(1,+\infty)$ (Lemma~\ref{lem:lambert_branches_properties}), $t \in [N_i, N_{i+1})$ and $i = 1 + \frac{\ln(N_i)}{\ln(1+\eta)}$, we obtain
	\begin{align*}
		g(t,\delta) &\geq \overline{W}_{-1} \left(1 +  \frac{c\ln\left( \frac{\zeta(s)}{\delta} \right) + cs\ln \left( 1+ \frac{\ln(t)}{\ln(1+\eta)}\right)}{N_{i}}\right) \geq \overline{W}_{-1} \left(1 +  \frac{c}{N_{i}}\ln\left( \frac{i^s\zeta(s)}{\delta} \right)\right)
	\end{align*}

	Using Lemma~\ref{lem:lambert_branches_properties} for each $i \in \Natural^{\star}$ yields
	\begin{align*}
		\probability \left( \exists t \in \Natural^{\star}: S_t + t \geq t g(t,\delta) \right) \leq \sum_{i \in \Natural^{\star}} e^{-\frac{N_i}{c} \left( h\left(x_{i}(\delta)\right)-1\right)} \leq \frac{\delta}{\zeta(s)} \sum_{i \in \Natural^{\star}} \frac{1}{i^s} =\delta
	\end{align*}
\end{proof}

\paragraph{Fixed-time concentration} When the time is fixed and not random, there is no need to consider slices of time and we can directly control the deviation of the process (Lemma~\ref{lem:fixed_time_subExp_uppertail_concentration}).

\begin{lemma} \label{lem:fixed_time_subExp_uppertail_concentration}
	Let $h(x) = x-\ln(x)$ for $x>1$. Let $c > 0$ and $S_{t}$ a $1$-sub-$\psi_{E,c}$ process with variance process $V_t = ct$. Then,
	\begin{align*}
		&\forall t \in \Natural^{\star}, \: \forall x>1, \quad  \bP \left( S_t + t \geq t x \right) \leq \exp \left( -\frac{t}{c} \left( h\left( x \right)-1\right)\right) \:.
	\end{align*}
\end{lemma}
\begin{proof}
	With similar computations as in the proof of Lemma~\ref{lem:fixed_subExp_martingale_on_a_slice_upper_tail}, the fact that $S_{t}$ a $1$-sub-$\psi_{E,c}$ process with variance process $V_t = ct$ and the Chernoff inequality yield the first result.
\end{proof}

\subsubsection{Lower Tail Concentration}
\label{app:sss_sub_exponential_lower_tail_concentration}

We derive time-uniform and fixed-time lower tail concentration for $1$-sub-$\psi_{E,-c}$ process with variance process $V_t = ct$.
Likewise, we use the peeling method and control the deviation of the process on slices of time (Lemma~\ref{lem:fixed_subExp_martingale_on_a_slice_lower_tail}).

\begin{lemma} \label{lem:fixed_subExp_martingale_on_a_slice_lower_tail}
	Let $c > 0$ and $-S_{t}$ a $1$-sub-$\psi_{E,-c}$ process with variance process $V_t = ct$. Let $N>0$. For all $x \in (0, 1)$, there exists $\lambda = \lambda(x)$ such that for all $t \ge N$,
	\begin{align*}
		\left\{- S_t -t \geq -tx  \right\} \subseteq \left\{ \lambda (-S_{t}) - ct \psi_{E,-c}(\lambda) \geq \frac{N}{c} \left(h\left(x\right) -1\right) \right\}
	\end{align*}
	where $\lambda(x) = \argmax_{\lambda \in [0,+ \infty)} \left( -x \lambda + \frac{\ln (1 + c \lambda)}{c} \right)$ and $h(x) = x - \ln(x)$ for $x \in (0,1)$.
\end{lemma}
\begin{proof}
Defining $\psi_{L}(\lambda) = - \lambda + c\psi_{E,-c}(\lambda)  = - \frac{\ln (1 + c \lambda)}{c}$ and $\lambda(x) = \argmax_{\lambda \in [0,+ \infty)} -x\lambda - \psi_{L}(\lambda)$, we have $-x\lambda(x)-\psi_{L}(\lambda(x)) = \psi^{*}_{L}\left(-x\right) \geq 0$ (see below), hence $t \psi^{*}_{L}\left(-x\right) \geq N \psi^{*}_{L}\left(-x\right)$ for $t \geq N$. Direct computations yield
	\begin{align*}
		 - S_t - t \geq -tx &\iff \lambda  (- S_t) - ct \psi_{E,-c}(\lambda) \geq -t x\lambda - t(-\lambda + c\psi_{E,-c}(\lambda)) \\
		& \implies \lambda (- S_t) - ct \psi_{E,-c}(\lambda) \geq t\left(-x\lambda - \psi_{L}(\lambda) \right) = t \psi^{*}_{L}\left(-x\right) \\
		& \implies  \lambda (- S_t) - ct \psi_{E,-c}(\lambda) \geq N \psi^{*}_{L}\left(-x\right) = \frac{N}{c} \left(h\left(x\right) -1\right)
	\end{align*}

	Note that for $f(\lambda) = -\lambda x + \frac{\ln(1+c\lambda)}{c}$, we have $f'(\lambda) = -x +\frac{1}{1+c\lambda} = 0 \iff \lambda = \frac{1}{c}\left(  \frac{1}{x} -1\right)$ and $\frac{1}{c}\left(  \frac{1}{x} -1\right) \in [0,+ \infty) \iff x \in (0,1)$. Since $f''(\lambda) = - \frac{c}{(1+c\lambda)^2} \leq 0$, the function is concave hence this is a maximum. This yields that for all $x \in (0,1)$, $\psi^{*}_{L}(-x) = f(\frac{1}{c}\left(  \frac{1}{x} -1\right)) = \frac{1}{c}\left( x-1-\ln(x)\right) = \frac{1}{c}\left( h(x)-1\right)\geq 0$ where $h(x) = x - \ln(x)$ for $x \in (0,1)$.
\end{proof}

Let $\eta > 0$. Applying Lemma~\ref{lem:fixed_subExp_martingale_on_a_slice_lower_tail} on slices of time with geometric growth rate $(N_{i})_{i \in \Natural^{\star}}$ with $N_i = (1+\eta)^{i-1}$, we obtain Lemma~\ref{lem:uniform_time_subExp_lower_tail_concentration}.

\begin{lemma} \label{lem:uniform_time_subExp_lower_tail_concentration}
	Let $\overline{W}_{0}(x) = -W_{0}(-e^{-x})$ for $x\geq1$, $\delta \in (0,1)$, $\eta > 0$, $s > 1$, $c > 0$, and $\zeta$ be the Riemann $\zeta$ function. Let $S_{t}$ a $1$-sub-$\psi_{E,-c}$ process with variance process $V_t = ct$. Then, with probability greater than $1 - \delta$, for all $t \in \Natural^{\star}$,
	\begin{align*}
		S_{t} + t   \geq  t \overline{W}_{0} \left(1 +  \frac{c(1+\eta)}{t}\left(\ln\left( \frac{\zeta(s)}{\delta} \right) + s\ln \left( 1+ \frac{\ln(t)}{\ln(1+\eta)}\right) \right) \right)  \: .
	\end{align*}
\end{lemma}
\begin{proof}
	Let $g(t,\delta)$ positive such that $g(t,\delta) \leq x_{i}(\delta)$ for $t \in [N_i, N_{i+1})$ and $x_{i}(\delta) \in (0,1)$. Using Lemma~\ref{lem:fixed_subExp_martingale_on_a_slice_lower_tail} with $x_i(\delta) < 1$ and $g(t,\delta) \leq x_{i}(\delta)$ for $t \in [N_i, N_{i+1})$, we obtain
	\begin{align*}
		\probability \left( \exists t \in \Natural^\star: S_t + t \leq tg(t,\delta) \right) &= \probability \left( \exists t \in \Natural: -S_t - t \geq -tg(t,\delta) \right) \\
		&\leq \sum_{i \in \Natural^{\star}}  \probability \left( \exists t \in T_i: -S_t - t \geq -t x_{i}(\delta) \right) \\
		&\leq  \sum_{i \in \Natural^{\star}}  \probability \left( \exists t \in T_i:  \lambda (-S_{t}) - ct \psi_{E,-c}(\lambda) \geq \frac{N_i}{c} \left(h\left(x_{i}(\delta)\right) -1\right)  \right) \\
		&\leq \sum_{i \in \Natural^{\star}} e^{-\frac{N_i}{c} \left(h\left(x_{i}(\delta)\right) -1\right) }
	\end{align*}
	where the last inequality uses that $-S_{t}$ a $1$-sub-$\psi_{E,-c}$ process with variance process $V_t = ct$. Taking
	\[
	g(t,\delta) = \overline{W}_{0} \left(1 +  \frac{c(1+\eta)}{t}\left(\ln\left( \frac{\zeta(s)}{\delta} \right) + s\ln \left( 1+ \frac{\ln(t)}{\ln(1+\eta)}\right) \right) \right)
	\]
	and $x_{i}(\delta) = \overline{W}_{0} \left(1 +  \frac{c}{N_{i}}\ln\left( \frac{i^s\zeta(s)}{\delta} \right)\right)$ satisfies the required properties. First, we have $x_{i}(\delta) \in (0,1)$ and $g(t,\delta) > 0$ (Lemma~\ref{lem:lambert_branches_properties}). Second, since $\overline{W}_{0}$ is decreasing on $(1,+\infty)$ (Lemma~\ref{lem:lambert_branches_properties}), $t \in [N_i, N_{i+1})$ and $i = 1 + \frac{\ln(N_i)}{\ln(1+\eta)}$, we obtain
	\begin{align*}
		g(t,\delta) &\leq \overline{W}_{0} \left(1 +  \frac{c\ln\left( \frac{\zeta(s)}{\delta} \right) + cs\ln \left( 1+ \frac{\ln(t)}{\ln(1+\eta)}\right)}{N_{i}}\right) \leq \overline{W}_{0} \left(1 +  \frac{c}{N_{i}}\ln\left( \frac{i^s\zeta(s)}{\delta} \right)\right)
	\end{align*}

	Using Lemma~\ref{lem:lambert_branches_properties} for each $i \in \Natural^{\star}$ yields
	\begin{align*}
		\probability \left( \exists t \in \Natural^\star: S_t + t  \leq t g(t,\delta) \right) \leq \sum_{i \in \Natural^{\star}} e^{-\frac{N_i}{c} \left( h\left(x_{i}(\delta)\right)-1\right)} \leq \frac{\delta}{\zeta(s)} \sum_{i \in \Natural^{\star}} \frac{1}{i^s} =\delta \: .
	\end{align*}
\end{proof}

\paragraph{Fixed-time concentration} When the time is fixed and not random, there is no need to consider slices of time and we can directly control the deviation of the process (Lemma~\ref{lem:fixed_time_subExp_lowertail_concentration}).

\begin{lemma} \label{lem:fixed_time_subExp_lowertail_concentration}
	Let $h(x) = x-\ln(x)$ for $x\in(0,1)$. Let $c > 0$ and $-S_{t}$ a $1$-sub-$\psi_{E,-c}$ process with variance process $V_t = ct$. Then,
	\begin{align*}
		&\forall t \in \Natural^{\star}, \: \forall x \in (0,1), \quad  \bP \left( S_t + t \leq t x \right) \leq \exp \left( -\frac{t}{c} \left( h\left( x \right)-1\right)\right) \: .
	\end{align*}
\end{lemma}
\begin{proof}
	With similar computations as in the proof of Lemma~\ref{lem:fixed_subExp_martingale_on_a_slice_lower_tail}, the fact that $-S_{t}$ a $1$-sub-$\psi_{E,-c}$ process with variance process $V_t = ct$ and the Chernoff inequality yield the first result.
\end{proof}

\subsection{Univariate Gaussian}
\label{app:ss_gaussian_concentration}

We prove time-uniform and fixed-time upper and lower concentration results for the empirical variance (Appendix~\ref{app:sss_empirical_variance}) and empirical mean (Appendix~\ref{app:sss_empirical_mean}) of Gaussian observations.

\subsubsection{Empirical Variance}
\label{app:sss_empirical_variance}

We first prove Lemma~\ref{lem:variance_is_subExp2} which shows that the empirical variance is closely linked with a sub-exponential process.
This is obtained with manipulations derived in the Appendix H of \citet{howard_2021_TimeUniformNonParametric}, in which they consider martingales with $\chi^2$ increments.

\begin{lemma} \label{lem:variance_is_subExp2}
	 Let $\sigma^2_{t}$ be the empirical variance of $t$ i.i.d. samples from a Gaussian distribution with variance $\sigma^2$. Then, $\frac{\sigma^2_{t}}{\sigma^2} = \frac{S_{t-1}-1}{t} + 1$ with $S_{t-1} + t - 1= \sum_{i=1}^{t-1}Y_i^2$ where $(Y_i)$ are i.i.d. with distributions $\cN(0,1)$. In particular, $S_t$ is a $1$-sub-$\psi_{E,2}$ process and $- S_t$ is a $1$-sub-$\psi_{E,-2}$ process, both with variance process $V_t = 2t$.
\end{lemma}
\begin{proof}
	Let $(X_{i})_{i \in [n]}$ the $n$ samples from a Gaussian distribution with parameters $(\mu,\sigma^2)$. Let $\hat\mu_{n}$ and $\hat\sigma_{n}^2$ be the empirical mean and variance.
	Let $Z_i = \frac{X_i - \mu}{\sigma}$ for all $i \in [n]$, $\hat Z_{n} = \frac{1}{n}\sum_{i\in [n]}Z_i$ and $S_{n-1} = \sum_{i=1}^n (Z_i - \hat{Z}_n)^2 - (n-1)$.
	Then, $S_{0} = 0$ and for all $n \ge 2$
	\begin{align*}
	S_{n-1} = \frac{1}{\sigma^2}\sum_{i=1}^n (X_i - \hat{\mu}_n)^2 - (n-1)
	= n \frac{\sigma_n^2}{\sigma^2} - (n-1) \: .
	\end{align*}

Rewriting the increment of $S_n$, we obtain for all $n \ge 2$
	\begin{align*}
	S_{n-1} - S_{n-2}
	&= (Z_n - \hat{Z}_n)^2 + \sum_{i=1}^{n-1}((Z_i - \hat{Z}_n)^2 - (Z_i - \hat{Z}_{n-1})^2) - 1
	\\
	&= Z_n^2 + \hat{Z}_n^2 - 2 Z_n \hat{Z}_n + \sum_{i=1}^{n-1}(-2 Z_i \hat{Z}_n + 2 Z_i \hat{Z}_{n-1}) + (n-1)(\hat{Z}_n^2 - \hat{Z}_{n-1}^2) - 1
	\\
	&= Z_n^2 - n \hat{Z}_n^2 + (n-1) \hat{Z}_{n-1}^2 - 1 = \frac{n-1}{n}(Z_{n} - \hat{Z}_{n-1})^2 - 1 	\: .
	\end{align*}

Since $S_{n-1} = \sum_{s=1}^{n-1}(S_s - S_{s-1})$ and $S_0 = 0$ a.s., we obtain $S_{n-1} = \sum_{i=1}^{n-1} (Y_i^2 - 1)$ where $Y_{n-1} = \sqrt{\frac{n-1}{n}}(Z_{n} - \hat{Z}_{n-1})$. The $(Y_i)$ are iid with distribution $\mathcal N(0,1)$ and the CGF of $(Y_i^2 - 1)$ is
	\begin{align*}
	\log \mathbb{E}e^{\lambda (Y_i^2 - 1)} = - \frac{\log(1 - 2 \lambda)}{2} - \lambda = 2 \psi_{E,2}(\lambda)\quad \text{for } \lambda \in (-\infty, 1/2) \: .
	\end{align*}

	By Definition~\ref{def:sub_phi_process}, we have that $S_n$ is $1$-sub-$\psi_{E,2}$ and that $- S_n$ is $1$-sub-$\psi_{E,-2}$, both with variance process $V_n = 2(n-1)$.
\end{proof}

Thanks to Lemmas~\ref{lem:uniform_time_subExp_upper_tail_concentration}-\ref{lem:uniform_time_subExp_lower_tail_concentration}-\ref{lem:variance_is_subExp2}, Corollary~\ref{cor:uniform_time_upper_lower_tail_concentration_variance} gives time-uniform upper and lower tails concentrations on the empirical variance of Gaussian observation.

\begin{corollary} \label{cor:uniform_time_upper_lower_tail_concentration_variance}
	For $i \in \{0,-1\}$, let $\overline{W}_{i}(x) = -W_{i}(-e^{-x})$ for $x\geq1$, $\delta \in (0,1)$, $\eta_{0}, \eta_{1} > 0$, $s > 1$ and $\zeta$ be the Riemann $\zeta$ function. Let $\sigma^{2}_{t+1}$ be the empirical variance of $t+1$ i.i.d. samples from a Gaussian distribution with variance $\sigma^2$. Then, with probability greater than $1 - \delta$, for all $t \in \mathbb{N}^\star$,
	\begin{align*}
	\sigma^{2}_{t+1} \leq \sigma^2 \left(\overline{W}_{-1} \left(1 +  \frac{2(1 + \eta_{1})}{t}\left(\ln\left( \frac{\zeta(s)}{\delta} \right) + s\ln \left( 1+ \frac{\ln(t)}{\ln(1+\eta_{1})}\right) \right)\right) -\frac{1}{t}\right) \: .
	\end{align*}
	Moreover, with probability $1 - \delta$, for all $t \geq t_{0}(\delta)$,
	\begin{align*}
		\sigma^{2}_{t+1} \geq \sigma^2 \left( \overline{W}_{0} \left(1 +  \frac{2(1 + \eta_{0})}{t}\left(\ln\left( \frac{\zeta(s)}{\delta} \right) + s\ln \left( 1+ \frac{\ln(t)}{\ln(1+\eta_{0})}\right) \right) \right) -\frac{1}{t}\right) \: ,
	\end{align*}
	where the initial time condition, which ensures the lower bound is positive, is
	\[
	t_{0}(\delta) = \inf \left\{ t \mid t > e^{ 1 + W_{0} \left( \frac{2(1+\eta_{0})}{e}\left(\ln\left( \frac{\zeta(s)}{\delta} \right) + s\ln \left( 1+ \frac{\ln(t)}{\ln(1+\eta_{0})}\right) \right) -e^{-1}\right)} \right\} \: .
	\]
\end{corollary}
\begin{proof}
	Combining Lemmas~\ref{lem:uniform_time_subExp_upper_tail_concentration}-\ref{lem:uniform_time_subExp_lower_tail_concentration}-\ref{lem:variance_is_subExp2} yields the desired result. Using Lemma~\ref{lem:lambert_branches_properties}, we know that the upper bound is always positive. The intial time condition, after which the lower bound is positive, is obtained by Lemma~\ref{lem:lambert_branches_properties}
	\begin{align*}
		\overline{W}_{0} \left(1 +  \frac{2(1+\eta_{0})}{t}\left(\ln\left( \frac{\zeta(s)}{\delta} \right) + s\ln \left( 1+ \frac{\ln(t)}{\ln(1+\eta_{0})}\right) \right) \right) > \frac{1}{t} \quad \iff \quad t \geq t_{0}(\delta) \: .
		\end{align*}
\end{proof}

\paragraph{Fixed-time concentration}
To our knowledge, the first fixed-time upper tail concentration result for the empirical variance dates back to Lemma 3 of \citet{honda2014optimality} and to Lemma 7 of \citet{chan_2020_multi} for fixed-time lower tail concentration results.
Corollary~\ref{cor:fixed_time_upper_lower_concentration_variance} is obtained as a direct consequence of Lemmas~\ref{lem:fixed_time_subExp_uppertail_concentration}-\ref{lem:fixed_time_subExp_lowertail_concentration}-\ref{lem:variance_is_subExp2}, hence the proof is omitted.

\begin{corollary} \label{cor:fixed_time_upper_lower_concentration_variance}
	Let $h(x) = x -\ln(x)$ for $x>0$. Let $\sigma^{2}_{t+1}$ be the empirical variance of $t+1$ i.i.d. samples from a Gaussian distribution with variance $\sigma^2$. Then,
	\begin{align*}
		&\forall t \geq 1, \: \forall x > 1, \quad \bP \left( \sigma^2_{t+1} \geq \sigma^2 x \right) \leq \exp \left( -\frac{t}{2} \left( h\left( x + \frac{1}{t} \right)-1\right)\right) \\
		&\forall t \geq 1, \: \forall x \in \left(0,1-\frac{1}{t}\right), \quad \bP \left( \sigma^2_{t+1} \leq \sigma^2 x \right) \leq \exp \left( -\frac{t}{2} \left( h\left( x + \frac{1}{t} \right)-1\right)\right) \: .
	\end{align*}
\end{corollary}

\paragraph{On sub-Gaussian distributions with unknown variances}
While it is well-known that (time-uniform) concentration for Gaussian distributions with known $\sigma^2$ apply to $\sigma^2$-sub-Gaussian distributions, thus the same bandit algorithms can be used in both settings, we believe that there is no counterpart of this phenomenon when $\sigma^2$ is unknown.
For regret minimization, some papers have provided examples of sub-Gaussian arms with unknown $\sigma^2$ under which the regret can be linear.

Extending our algorithms to sub-Gaussian distributions would require time-uniform concentration results for the empirical variance of sub-Gaussian distributions with unknown variance.
Our concentration on the empirical variance relies on Lemma~\ref{lem:variance_is_subExp2}, which leverages the fact that the empirical mean and empirical variances of Gaussian distributions are independent, which does not extend to the sub-Gaussian case.

Moreover, if the focus is on asymptotically optimal algorithms, we note that it is difficult to express the characteristic time $T^\star$ for the non-parametric class of sub-Gaussian distributions with unknown variances.
Optimal BAI has however been studied under other interesting non-parametric assumptions (see, e.g., \citet{Agrawal20GeneBAI}).

\subsubsection{Empirical Mean}
\label{app:sss_empirical_mean}

While time-uniform concentration results for the empirical mean of Gaussian observations already exist in the literature (e.g. \citet{kaufmann_2018_MixtureMartingalesRevisited}), Lemma~\ref{lem:uniform_upper_lower_tails_concentration_mean} is proved for completeness and to present unified concentration results as it also involves $\overline{W}_{-1}$. The empirical mean of a Gaussian after a given number of observations is a sub-Gaussian random variable.
It is in fact exactly Gaussian, but the sub-Gaussian hypothesis will be easier to handle for a random number of samples.

\begin{lemma} \label{lem:uniform_upper_lower_tails_concentration_mean}
	Let $\overline{W}_{-1}(x) = -W_{-1}(-e^{-x})$ for $x\geq1$, $\delta \in (0,1)$, $s > 1$ and $\zeta$ be the Riemann $\zeta$ function. Let $\mu_t$ be the empirical mean of $t$ i.i.d. samples from a Gaussian distribution with parameter $(\mu,\sigma^2)$. Then, with probability greater than $1-\delta$, for all $t \in \Natural^\star$,
\begin{align*}
	|\mu_t - \mu| \leq \sqrt{\frac{\sigma^{2}}{t}\overline{W}_{-1} \left( 1 + 2\ln \left( \frac{1}{\delta}\right) + 2g(s) +  2s \ln\left(2s + \ln t\right) \right)}   \: ,
\end{align*}
where $g(s) = \ln(\zeta(s)) + s(1 -  \ln(2s))$.
\end{lemma}
\begin{proof}
	Let $(X_{s})_{s \in [t]}$ the observations from a standard normal distributions and denote $S_t = \sum_{s\in [t]} X_{s}$. We will derive a concentration result for $S_t = t \frac{\hat\mu_t - \mu}{\sigma}$, which implies a concentration result for $\hat\mu_t$.

	Let $\eta > 0$ and $s > 1$.
	For all $i \in \Natural^{\star}$, let $\gamma_i > 0$ and $N_{i} = (1+\eta)^{i-1}$.
	For all $i \in \Natural^{\star}$, we define the family of priors $f_{N_i, \gamma_{i}}(x) = \sqrt{\frac{\gamma_{i}  N_i}{2 \pi }} \exp\left( - \frac{x^2\gamma_{i} N_i}{2}\right)$ with weights $w_{i} = \frac{1}{i^{s}\zeta(s)}$ and process
		\begin{align*}
			\overline M(t) = \sum_{i \in \Natural^{\star}} w_i \int f_{N_i, \gamma_{i}}(x) \exp \left( x S_{t}- \frac{1}{2} x^2 t \right) \,dx \: ,
		\end{align*}
		which satisfies $\overline M(0) = 1$ since $\sum_{i \in \Natural^{\star}} w_i$ and $ \int f_{N_i, \gamma_{i}}(x)  \,dx = 1$. A test martingale is defined as a non-negative martingale of unit initial value.
		It is direct to see that $M(t) = \exp \left( x S_t- \frac{1}{2} x^2  t \right)$ is a test martingale, as $M(0)=1$ and $\expectedvalue [M(t)\mid \cF_{t-1}] = M(t-1) \expectedvalue_{Y \sim \cN(0,1)} [ e^{x Y - \frac{1}{2} x^2  }] = M(t-1)$ since $\expectedvalue_{Y \sim \cN(0,1)} [ e^{x Y}]= e^{\frac{1}{2} x^2}$.
		By Tonelli's theorem and using that $M(t)$ is a martingale
		\begin{align*}
			\expectedvalue [\overline M(t)\mid \cF_{t-1}] &=  \sum_{i \in \Natural^{\star}} w_i \int f_{N_i, \gamma_{i}}(x) M(t-1)\,dx  = \overline  M(t-1) \: .
		\end{align*}
		Therefore, $\overline M(t)$ is also a test martingale. Let $i \in \Natural^{\star}$ and consider $t\in [N_{i}, N_{i+1})$. For all $x$,
		\[
			f_{N_i, \gamma}(x) \geq \sqrt{\frac{N_i}{t}} f_{t, \gamma_{i}}(x) \geq \frac{1}{\sqrt{1+\eta}} f_{t, \gamma_{i}}(x)
		\]
		Direct computations shows that
		\begin{align*}
			\int f_{t, \gamma_{i}}(x) \exp \left( x S_{t}- \frac{1}{2} x^2 t \right)  \,dx = \frac{1}{\sqrt{1+\gamma_{i}^{-1}}} \exp \left( \frac{S_{t}^2}{2(1+\gamma_{i})t}\right) \: .
		\end{align*}
		Combining those results with the fact that $\overline M(t)\geq w_i \int f_{N_i, \gamma_{i}}(x) \exp \left( x S_{t}- \frac{1}{2} x^2 t \right) \,dx$, we obtain
		\begin{align*}
			\overline M(t) \geq \frac{1}{i^{s}\zeta(s)} \frac{1}{\sqrt{(1+\gamma_{i}^{-1})(1+\eta)}} \exp \left( \frac{S_{t}^2}{2(1+\gamma_{i})t}\right) \: ,
		\end{align*}
		Using Ville's maximal inequality, we have that with probability greater than $1-\delta$, $\ln\overline M(t) \leq \ln \left( \frac{1}{\delta}\right)$. Therefore, with probability greater than $1-\delta$, for all $i \in \Natural^{\star}$ and $t \in [N_{i}, N_{i+1})$,
		\begin{align*}
			\frac{|S_{t}|}{\sqrt{t}} &\leq \sqrt{(1+\gamma_{i}) \left( 2\ln \left( \frac{1}{\delta}\right) +  2\ln\left(i^s\zeta(s)\right) + \ln(1+\gamma_{i}^{-1}) + \ln (1+\eta) \right)} \: .
		\end{align*}
		Since this upper bound is independent of $t$, we can optimize it and choose $\gamma_{i}$ as in Lemma~\ref{lem:lemma_A_3_of_Remy} for all $i \in \Natural^{\star}$.
		Therefore, with probability greater than $1-\delta$, for all $i \in \Natural^{\star}$ and $t \in [N_{i}, N_{i+1})$,
		\begin{align*}
			\frac{S_{t}^2}{t} &\leq \overline{W}_{-1}\left( 1 + 2\ln \left( \frac{\zeta(s)}{\delta}\right) + 2s \ln\left(i\right) + \ln (1+\eta)\right) \\
			&\leq \overline{W}_{-1}\left( 1 + 2\ln \left( \frac{\zeta(s)}{\delta}\right) + 2s \ln\left(\ln(1+\eta) + \ln t\right)   - 2s \ln \ln (1+\eta)+ \ln (1+\eta)\right) \\
			&= \overline{W}_{-1}\left( 1 + 2\ln \left( \frac{1}{\delta}\right) + 2s \ln\left(2s + \ln t\right)   + 2g(s)\right)
		\end{align*}
		where $g(s) = \ln(\zeta(s)) + s(1 -  \ln(2s))$.	The second inequality is obtained since $i \leq 1+ \frac{\ln t}{\ln(1+\eta)}$ for $t \in [N_{i}, N_{i+1})$. The last equality is obtained for the choice $\eta^\star = e^{2s} - 1$, which minimizes $\eta \mapsto \ln (1+\eta) - 2s \ln(\ln(1+\eta))$. Since $\Natural^{\star} \subseteq \bigcup_{i\in \Natural^{\star}} [N_{i}, N_{i+1})$ and $\frac{\hat\mu_t - \mu}{\sigma} = \frac{S_t}{t}$ this yields the result.
\end{proof}

\paragraph{Fixed-time concentration} Lemma~\ref{lem:fixed_time_upper_lower_concentration_mean} is a known result for the deviation of the empirical mean of Gaussian observations, hence we omit the proof (see Ex. 2.2.23 in \citet{dembo_1998_LargeDeviationTechniques}).

\begin{lemma} \label{lem:fixed_time_upper_lower_concentration_mean}
	Let $\mu_{t}$ be the empirical mean of $t$ i.i.d. samples from a Gaussian distribution with parameter $(\mu,\sigma^2)$. Then, for all $t \in \Natural^\star$ and all $ x > 0$
	\begin{align*}
		\bP \left( \mu_{t} \geq \mu + x \right) \leq \exp \left( -\frac{tx^2}{2\sigma^2} \right) \quad \text{and} \quad \forall x > 0, \quad \bP \left( \mu_{t} \leq \mu - x \right) \leq \exp \left( -\frac{tx^2}{2\sigma^2} \right) \: .
	\end{align*}
\end{lemma}

%% file: sections/appendix_kl_concentration.tex

\section{Kullback-Leibler Concentration}
\label{app:kl_concentration}

We prove our time-uniform upper tail concentration for $d$-dimensional exponential families (Theorem~\ref{thm:uniform_upper_tail_concentration_kl_exp_fam}).
The two key novelties compared to previous work \cite{degenne_2019_ImpactStructureDesign} are that we consider: (1) a sum over $\cS \subseteq [K]$ arms and (2) $\Theta_{D} \subseteq \Real^d$.
To go from one arm to $\cS$ arms, it is enough to consider the product of the the mixture of martingales used for each arm.
Dealing with a support different from $\Real^d$ is the real challenge of the proof.
This is the reason why we start by showing the result for one arm in Appendix~\ref{app:ss_one_arm}, and then generalize it to $\cS$ arms in Appendix~\ref{app:ss_sum_over_arms}.
For Gaussian with unknown variance, which are a $2$-dimensional exponential family with support $\Theta_{D} = \Real \times \Real_{-}^{\star}$, we obtain Theorem~\ref{thm:uniform_upper_tail_concentration_kl_exp_fam_gaussian} (Appendix~\ref{app:ss_univariate_gaussian}).

In the following, each arm $a \in [K]$ has a parameter $\theta_a$ and distribution $\nu_{\theta_a}$ belonging to an exponential family with parameter space $\Theta_a \subseteq \mathbb{R}^d$ (we call $d$ the dimension of the family), sufficient statistic $F_a:\mathbb{R}\to \mathbb{R}^d$ and log-partition function $\phi_a : \Theta_a \to \mathbb{R}$. That is, there exists a distribution $\nu_0$ such that $\nu_{\theta_a}$ is defined by
$
\frac{d \nu_{\theta_a}}{d \nu_0}(X)
= \exp(\theta_a^\top F_a(X) - \phi_a(\theta_a))
\: .
$

We define the average statistic $F_{t,a} := \frac{1}{N_{t,a}}\sum_{s=1}^t \1\{a_s = a\} F_a(X_s) $.
We denote the maximum likelihood estimator (MLE) of $\theta_a$ by $\theta_{t,a}$. It is defined as $\theta_{t,a} := \argmax_{\lambda \in \Theta_a} \lambda^\top F_{t,a} - \phi_a(\lambda)$ (and may not exist).
When $F_{t,a} \in \nabla \phi_a (\Theta_a)$, we have $\theta_{t,a} = (\nabla \phi_a)^{-1}(F_{t,a})$.

Let $d_{\phi_a}(\theta_a, \lambda_a)$ denote the Bregman divergence of $\phi_a$ between parameters $\theta_a$ and $\lambda_a$. It is equal to the Kullback-Leibler divergence between the distributions with parameters $\lambda_a$ and $\theta_a$ (note the reversed order of the parameters). For a subset of arms $\cS \subseteq [K]$, we seek high probability bounds on $\sum_{a \in \cS} N_{t,a} d_{\phi_a}(\theta_a, \theta_{t,a})$.

\subsection{One Arm}
\label{app:ss_one_arm}

In this section we consider one arm $a \in [K]$. The first step of the proof consists in linking the KL divergence to a mixture of martingales in order to obtain a time-uniform upper tail concentration (Lemma~\ref{lem:kl_bound_with_prior}).

\begin{lemma}\label{lem:kl_bound_with_prior}
Let $\rho_{0,a}$ be a distribution supported on $\Theta_{D,a}$ and $\tau$ be an almost surely bounded stopping time. With probability $1 - \delta$, either $F_{N_{\tau,a},a} \notin \nabla \phi_a(\Theta_{D,a})$ or
\begin{align*}
N_{\tau,a} d_{\phi_a}(\theta_a, \theta_{N_{\tau,a},a})
\le - \ln \mathbb{E}_{y \sim \rho_{0,a}}\exp\left( - N_{\tau,a} d_{\phi_a}(y, \theta_{N_{\tau,a},a}) \right) + \log \frac{1}{\delta}
\: .
\end{align*}
\end{lemma}

\begin{proof}
We first remark that for all $y \in \Theta_{D,a}$, the log-likelihood ratio $\prod_{s\le t, a_s = a} \frac{d \nu_{y}}{d \nu_{\theta_a}} (X_s)$ is a martingale with expectation 1 under $\nu_{\theta_a}$. This is also true for $\mathbb{E}_{y \sim \rho_{0,a}}\left[\prod_{s\le t, a_s = a} \frac{d \nu_{y}}{d \nu_{\theta_a}} (X_s)\right]$. By the optional stopping theorem, its stopped version at $\tau$ also has expectation 1. Then by Markov's inequality, with probability $1 - \delta$,
\begin{align}\label{eq:llr_martingale}
\mathbb{E}_{y \sim \rho_{0,a}}\left[\prod_{s\le \tau, a_s = a} \frac{d \nu_{y}}{d \nu_{\theta_a}} (X_s)\right] \le \frac{1}{\delta} \: .
\end{align}
We now rewrite the product as $\exp \left(\sum_{s=1}^\tau \mathbb{I}\{a_s = a\} \log \frac{d \nu_{y}}{d \nu_{\theta_a}} (X_s)\right)$ and detail its value:
\begin{align*}
\sum_{s=1}^t \mathbb{I}\{a_s = a\}\log \frac{d \nu_{y}}{d \nu_{\theta_a}} (X_s)
&= \sum_{s=1}^t \mathbb{I}\{a_s = a\} (\phi_a(\theta_a) - \phi_a(y) - (\theta_a - y)^\top F_a(X_s))
\\
&= N_{t,a} (\phi_a(\theta_a) - \phi_a(y) - (\theta_a - y)^\top F_{N_{t,a},a})
\: .
\end{align*}
Now if $F_{N_{\tau,a},a} \in \nabla \phi_a(\Theta_{D,a})$, we have $F_{N_{t,a},a} = \nabla \phi_a(\theta_{N_{t,a},a})$ and
\begin{align*}
\sum_{s=1}^t \mathbb{I}\{a_s = a\}\log \frac{d \nu_{y}}{d \nu_{\theta_a}} (X_s)
&= N_{t,a} (\phi_a(\theta_a) - \phi_a(y) - (\theta_a - y)^\top  \nabla \phi_a(\theta_{N_{t,a},a}))
\\
&= N_{t,a} d_{\phi_a}(\theta_a, \theta_{N_{t,a},a}) - N_{t,a} d_{\phi_a}(y, \theta_{N_{t,a},a})
\: .
\end{align*}
We can write~\eqref{eq:llr_martingale} again with that expression to get
\begin{align*}
\mathbb{E}_{y \sim \rho_{0,a}}\exp\left(N_{\tau,a} d_{\phi_a}(\theta_a, \theta_{N_{\tau,a},a}) - N_{\tau,a} d_{\phi_a}(y, \theta_{N_{\tau,a},a})\right)
\le \frac{1}{\delta}
\: .
\end{align*}
The first term in the subtraction does not depend on $y$ and ca be brought outside of the expectation. Taking the logarithm then proves the lemma.
\end{proof}

Corollary~\ref{cor:kl_bound_with_prior_all_n} generalizes Lemma~\ref{lem:kl_bound_with_prior} to all times.

\begin{corollary}\label{cor:kl_bound_with_prior_all_n}
Let $\rho_{0,a}$ be a distribution supported on $\Theta_{D,a}$. With probability $1 - \delta$, for all $t \in \mathbb{N}$, either $F_{N_{t,a},a} \notin \nabla \phi_a(\Theta_{D,a})$ or
\begin{align*}
N_{t,a} d_{\phi_a}(\theta_a, \theta_{N_{t,a},a})
\le - \ln \mathbb{E}_{y \sim \rho_{0,a}}\exp\left( - N_{t,a} d_{\phi_a}(y, \theta_{N_{t,a},a}) \right) + \log \frac{1}{\delta}
\: .
\end{align*}
\end{corollary}
\begin{proof}
The extension of the result of Lemma~\ref{lem:kl_bound_with_prior} from a stopping time to all times is standard.
\end{proof}

We will now use Corollary~\ref{cor:kl_bound_with_prior_all_n} with a truncated Gaussian prior (Lemma~\ref{lem:concentration_with_gaussian_trunc_prior_matrix}). The reason for the truncation is that the domain $\Theta_{D,a}$ may not be $\mathbb{R}^d$ but only a subset. In that case, we have to restrict the prior. We also truncate to ensure a control the hessian $\nabla^2 \phi_a$ on the support with positive semi-definite upper and lower bound.

For $M$ a positive-definite matrix and a set $S$, let $V_{M}(S)$ be the volume of set $S$ according to measure $\mathcal N(0, M^{-1})$. Note that $V_M(S) = V_{I_d}(M^{1/2}S)$.

\begin{lemma}\label{lem:concentration_with_gaussian_trunc_prior_matrix}
Let $A$ be a convex set containing 0, $H$ and $G$ be upper bounds and lower bounds of $\nabla^2 \phi$ in the positive semi-definite (PSD) sense on $\theta_a + A$. Let $M_0$ be a positive definite matrix.
For $n \in \mathbb{N}$, let $\xi_{n,a} = ((n H)^{-1} + M_0^{-1})^{-1}(nH \theta_{n,a} + M_{0} \theta_a)$.
With probability $1 - \delta$, for all $t \in \mathbb{N}$, if $F_{N_{t,a},a} \in \nabla \phi_a(\Theta_{D,a})$ and $\theta_{N_{t,a},a} \in \theta_a + A$ then
\begin{align*}
N_{t,a} d_{\phi_a}(\theta_a, \theta_{N_{t,a},a})
&\le \frac{1}{2}\Vert \theta_a - \theta_{N_{t,a},a} \Vert_{((N_{t,a} H)^{-1} + M_0^{-1})^{-1}}^2
	+ \frac{1}{2} \log\det(1 + N_{t,a} H M_0^{-1})
	\\&\quad + \log \frac{1}{\delta}
	+ \log \frac{V_{M_0}(A)}{V_{N_{t,a} H + M_0}(\theta_a - \xi_{N_{t,a},a} + A)}
\: .
\end{align*}
\end{lemma}
\begin{proof}
We use Corollary~\ref{cor:kl_bound_with_prior_all_n} with prior $\rho_{0,a}$ equal to $\mathcal N(\theta_a, M_0^{-1})$ truncated to $A$ and rescaled to have mass 1.

Suppose that $F_{N_{t,a},a} \in \nabla \phi_a(\Theta_{D,a})$ and $\theta_{N_{t,a},a} \in \theta_a + A$. For all $y \in \theta_a + A$, since $\theta_{N_{t,a},a} \in \theta_a + A$ as well and $A$ is convex, we can conclude from the definition of $H$ that $d_{\phi_a}(y, \theta_{N_{t,a},a}) \le \frac{1}{2}\Vert y - \theta_{N_{t,a},a} \Vert_H^2$. We use this to compute an upper bound to the expectation in Corollary~\ref{cor:kl_bound_with_prior_all_n}.
\begin{align*}
\mathbb{E}_{y \sim \rho_{0,a}} \left[ \exp \left( - n d_{\phi_a} \left(y, \theta_{n,a}\right)\right) \right] &\ge \mathbb{E}_{y \sim \rho_{0,a}} \left[ \exp \left( -  \frac{1}{2} \|y-\theta_n\|_{nH}^2 \right) \right]
\\
= \frac{\sqrt{\det(M_0)}}{(2 \pi)^{d/2} V_{M_0}(A)} &\int_{y \in \Real^d} \1_{\theta + A} \exp\left(-\frac{1}{2}\|y-\theta\|_{M_0}^2 - \frac{1}{2} \|y-\theta_n\|_{nH}^2\right) \,dy
\end{align*}
Let $\xi_n = (n H + M_0)^{-1}(nH \theta_n + M_{0} \theta_a)$ and $C_{n} = \frac{1}{2} \| \theta_a - \theta_n\|_{((n H)^{-1} + M_0^{-1})^{-1}}^2$. Remark that
\begin{align*}
	-\frac{1}{2}\|y-\theta\|_{M_0}^2 - \frac{1}{2} \|y-\theta_n\|_{nH}^2 = -\frac{1}{2}\|y-\xi_{n}\|_{nH + M_0}^2 - C_{n} \: .
\end{align*}
We use this equality:
\begin{align*}
	\mathbb{E}_{y \sim \rho_{0,a}} \left[ \exp \left( - n d_{\phi} \left(y, \theta_n\right)\right) \right]
	&\ge \frac{\sqrt{\det(M_0)}}{(2 \pi)^{d/2} V_{M_0}(A)} e^{-C_n} \int_{y \in \Real^d} \1_{\theta_a + A} \exp\left( -\frac{1}{2}\|y-\xi_{n}\|_{nH + M_0}^2 \right) \,dy
	\\
	&= \frac{\sqrt{\det(M_0)}}{\sqrt{\det(nH + M_0)}} e^{-C_n} \frac{V_{nH+M_0}(\theta_a - \xi_n + A)}{V_{M_0}(A)}
	\: .
\end{align*}
The result of the lemma is obtained by taking the logarithm of this expression.
\end{proof}

Lemma~\ref{lem:concentration_with_gaussian_trunc_prior} starts by choosing for $H$ and $G$ multiples of the identity, hence we will compare the divergence $d_{\phi_a}$ with the euclidean distance $(x, y) \mapsto \Vert x - y \Vert^2/2$.
This choice is possible by using the maximal and minimal eigenvalues of $\nabla^2 \phi_a(\lambda)$ on $A$, which we denote by
\begin{align*}
	\lambda_{+,a} &= \max \left\{ \lambda_{+}\left( \nabla^2 \phi_a(\tilde{\theta}_a)\right) \mid \tilde{\theta}_a \in \theta_a + A \right\} \: , \: \lambda_{-,a} &= \min \left\{ \lambda_{-}\left( \nabla^2 \phi_a(\tilde{\theta}_a)\right) \mid \tilde{\theta}_a \in \theta_a + A \right\}
	\: .
\end{align*}
For $v>0$ and a set $S$, let $V_{v}(S)$ be the volume of set $S$ according to measure $\mathcal N(0, v^{-1}I_d)$.

\begin{lemma}\label{lem:concentration_with_gaussian_trunc_prior}
Let $A$ be a convex set containing 0, $\lambda_{+,a}$ and $\lambda_{-,a}$ be the maximum and minimum of the eigenvalues of $\nabla^2 \phi_a(\lambda)$ on $\theta_a + A$ and let $v_0> 0$.
For $n \in \mathbb{N}$, let $\kappa_n = \frac{v_0}{n \lambda_{+,a} + v_0}$.
With probability $1 - \delta$, for all $t \in \mathbb{N}$, if $F_{N_{t,a},a} \in \nabla \phi_a(\Theta_{D,a})$ and $\theta_{N_{t,a},a} \in \theta_a + A$ then
\begin{align*}
N_{t,a} d_{\phi_a}(\theta_a, \theta_{N_{t,a},a})
&\le \frac{1}{2}\Vert \theta_a - \theta_{N_{t,a},a} \Vert^2 \frac{1}{(N_{t,a} \lambda_{+,a})^{-1} + v_0^{-1}}
	+ \frac{d}{2} \log(1 + N_{t,a} \lambda_{+,a} v_0^{-1})
	\\&\quad + \log \frac{1}{\delta}
	+ \log \frac{V_{v_0}(A)}{V_{N_{t,a} \lambda_{+,a} + v_0}(\kappa_{N_{t,a}} A)}
\: .
\end{align*}
\end{lemma}

Note that since $\rho_{0,a}$ is a prior and has to ensure the martingale property used in the proof of Lemma~\ref{lem:kl_bound_with_prior}, it cannot depend on the observations and in particular it cannot depend on the random variable $N_{t,a}$.

\begin{proof}
We use Corollary~\ref{cor:kl_bound_with_prior_all_n} with prior $\rho_{0,a}$ equal to $\mathcal N(\theta_a, v_0^{-1} I_d)$ truncated to $A$ and rescaled to have mass 1.

Suppose that $F_{N_{t,a},a} \in \nabla \phi_a(\Theta_{D,a})$ and $\theta_{N_{t,a},a} \in \theta_a + A$. As in Lemma~\ref{lem:concentration_with_gaussian_trunc_prior_matrix}, we obtain
\begin{align*}
	\mathbb{E}_{y \sim \rho_{0,a}} \left[ e^{- n d_{\phi} \left(y, \theta_n\right)} \right]
	&\ge \frac{v_0^{d/2}}{(2 \pi)^{d/2} V_{v_0}(A)} e^{-C_n} \int_{y \in \Real^d} \1_{\theta_a + A} \exp\left(-\frac{n\lambda_{+,a} + v_{0}}{2}\|y-\xi_{n}\|^2\right) \,dy
	\: .
\end{align*}

We define $\kappa_n = \frac{v_0}{n\lambda_{+,a} + v_{0}}$ and show that $\xi_n + \kappa_n A \subseteq \theta_a + A$. We can rewrite $\xi_n = \kappa_n\theta_a + (1 - \kappa_n) \theta_n$.
Let then $y \in \xi_n + \kappa_n A$. We need to prove that $y - \theta_a \in A$, but $y - \theta_a = y - \xi_n + (1 - \kappa_n) (\theta_n - \theta_a)$. By hypothesis $y - \xi_n \in \kappa_n A$ and $(1 - \kappa_n) (\theta_n - \theta_a) \in (1 - \kappa_n) A$. By convexity of $A$, we conclude that $\kappa_n A + (1 - \kappa_n) A \subseteq A$ and $y \in \theta_a + A$.

We use that set inclusion to lower bound the integral over $\theta_a + A$ by the same integral over the subset $\xi_n + \kappa_n A$.
\begin{align*}
&\mathbb{E}_{y \sim \rho_{0,a}} \left[ e^{- n d_{\phi} \left(y, \theta_n\right)}\right] \ge \frac{v_0^{d/2}}{(2 \pi)^{d/2} V_{v_0}(A)} e^{-C_n} \int_{y \in \Real^d} \1_{\xi_n + \kappa_n A} \exp\left(-\frac{n\lambda_{+,a} + v_{0}}{2}\|y-\xi_{n}\|^2\right) \,dy
\\
&= \frac{v_0^{d/2} e^{-C_n}}{(n \lambda_{+,a} + v_0)^{d/2}V_{v_0}(A)} \frac{(n \lambda_{+,a} + v_0)^{d/2}}{(2 \pi)^{d/2}} \int_{y \in \Real^d} \1_{\xi_n + \kappa_n A} \exp\left(-\frac{n\lambda_{+,a} + v_{0}}{2}\|y-\xi_{n}\|^2\right) \,dy
\\
&= e^{-C_n}\frac{v_0^{d/2}}{(n \lambda_{+,a} + v_0)^{d/2}} \frac{V_{n \lambda_{+,a} + v_0}(\kappa_n A)}{V_{v_0}(A)}
\: .
\end{align*}
The result of the lemma is obtained by taking the logarithm of this expression.
\end{proof}

Lemma~\ref{lem:ratio_volumes} gives a bound on the ratio of volumes.

\begin{lemma}\label{lem:ratio_volumes}
Let $\lambda, n, v > 0$ and $\kappa_n = \frac{v}{n \lambda + v}$. Then
\begin{align*}
\frac{V_{v}(A)}{V_{n \lambda + v}(\kappa_n A)} \le \kappa_n^{-d/2}
\: .
\end{align*}
\end{lemma}
\begin{proof}
We first prove that for all $\alpha \ge 1$, $\frac{V_{1}(\alpha A)}{V_{1}(A)} \le \alpha^d$.
For a convex set $S \subseteq \mathbb{R}^d$ containing 0, let $\gamma_S(x) = \argmax\{r \ge 0 \mid r x \in S\}$. We have that for all $y \in \mathbb{R}^d$ and $\beta \in \Real^\star$, $\gamma_S(\beta y) \beta y  = \gamma_S(y) y$. Note also that $\gamma_S(x) \ge 1 \iff x \in S$.
\begin{align*}
\frac{V_{1}(\alpha A)}{V_{1}(A)}
= 1 + \frac{V_{1}(\alpha A \setminus A)}{V_{1}(A)}
&= 1 + \frac{\int_{x \in \Real^d} \1_{\alpha A \setminus A}(x) \exp\left(-\frac{1}{2}\|x\|^2_{2}\right) \,dx}{\int_{x \in \Real^d} \1_{A}(x) \exp\left(-\frac{1}{2}\|x\|^2_{2}\right) \,dx}
\\
&\le 1 + \frac{\int_{x \in \Real^d} \1_{\alpha A \setminus A}(x) \exp\left(-\frac{1}{2}\|\gamma_A(x)x\|^2_{2}\right) \,dx}{\int_{x \in \Real^d} \1_{A}(x) \exp\left(-\frac{1}{2}\|\gamma_A(x)x\|^2_{2}\right) \,dx}
\\
&= \frac{\int_{x \in \Real^d} \1_{\alpha A}(x) \exp\left(-\frac{1}{2}\|\gamma_A(x)x\|^2_{2}\right) \,dx}{\int_{x \in \Real^d} \1_{A}(x) \exp\left(-\frac{1}{2}\|\gamma_A(x)x\|^2_{2}\right) \,dx}
\\
&= \frac{\alpha^d \int_{y \in \Real^d} \1_{A}(y) \exp\left(-\frac{1}{2}\|\gamma_A(\alpha y)\alpha y\|^2_{2}\right) \,dy}{\int_{x \in \Real^d} \1_{A}(x) \exp\left(-\frac{1}{2}\|\gamma_A(x)x\|^2_{2}\right) \,dx}
\\
&= \frac{\alpha^d \int_{y \in \Real^d} \1_{A}(y) \exp\left(-\frac{1}{2}\|\gamma_A(y)y\|^2_{2}\right) \,dy}{\int_{x \in \Real^d} \1_{A}(x) \exp\left(-\frac{1}{2}\|\gamma_A(x)x\|^2_{2}\right) \,dx}
= \alpha^d
\: .
\end{align*}

Using that $\frac{\sqrt{v}}{\sqrt{n \lambda + v} \kappa_n} = \kappa_n^{-1/2} \ge 1$, we obtain the result
\begin{align*}
\frac{V_{v}(A)}{V_{n \lambda + v}(\kappa_n A)}
&= \frac{V_{1}(\sqrt{v}A)}{V_{1}(\sqrt{n \lambda + v}\kappa_n A)}
\le \left(\frac{\sqrt{v}}{\sqrt{n \lambda + v} \kappa_n}\right)^d
= \kappa_n^{-d/2}
\: .
\end{align*}
\end{proof}

Combining Lemma~\ref{lem:concentration_with_gaussian_trunc_prior} and Lemma~\ref{lem:ratio_volumes} yield Corollary~\ref{cor:concentration_with_gaussian_trunc_prior}.

\begin{corollary}\label{cor:concentration_with_gaussian_trunc_prior}
Let $A$ be a convex set containing 0, $\lambda_{+,a}$ and $\lambda_{-,a}$ be the maximum and minimum of the eigenvalues of $\nabla^2 \phi_a(\lambda)$ on $\theta_a + A$ and let $v_0> 0$.
For $n \in \mathbb{N}$, let $\kappa_n = \frac{v_0}{n \lambda_{+,a} + v_0}$.
With probability $1 - \delta$, for all $t \in \mathbb{N}$, if $F_{N_{t,a},a} \in \nabla \phi_a(\Theta_{D,a})$ and $\theta_{N_{t,a},a} \in \theta_a + A$ then
\begin{align*}
N_{t,a} d_{\phi_a}(\theta_a, \theta_{N_{t,a},a})
&\le \frac{1}{2}\Vert \theta_a - \theta_{N_{t,a},a} \Vert^2 \frac{1}{(N_{t,a} \lambda_{+,a})^{-1} + v_0^{-1}}
	+ \frac{d}{2} \log(1 + N_{t,a} \lambda_{+,a} v_0^{-1})
	\\&\quad + \log \frac{1}{\delta}
	+ \frac{d}{2}\log \kappa_{N_{t,a}}^{-1}
\: .
\end{align*}
\end{corollary}
\begin{proof}
This is simply Lemma~\ref{lem:concentration_with_gaussian_trunc_prior} combined with the bound on the ratio of volumes of Lemma~\ref{lem:ratio_volumes}.
\end{proof}

We will now choose $\gamma > 1$ and define values $n_i = \gamma^i$ for $i \in \mathbb{N}$. For each interval $[n_i, n_{i+1})$, we will choose a good value of the variance $v_0^{-1}$ to get a bound for $N_{t,a}$ in this interval. The reason we restrict to an interval is that the best choice for that variance scales with $N_{t,a}$, but we are not allowed to do that since the prior cannot depend on $N_{t,a}$. Thanks to this smart choice, we can control the KL divergence on this interval (Lemma~\ref{lem:concentration_slice_i_with_eta}).

The ratio of the eigenvalues represents the approximation error in considering that $d_{\phi_a}$ is locally quadratic. It yields that $\lambda_{-,a} I_d \preceq \nabla^2 \phi_a(\lambda) \preceq \lambda_{+,a} I_d$ on the set, which for $x, y \in \Theta_a$ translates to
\[
\lambda_{-,a}  \le d_{\phi_a}(x, y) / (\Vert x - y \Vert^2/2) \le \lambda_{+,a} \ : .
\]

\begin{lemma}\label{lem:concentration_slice_i_with_eta}
Let $A$ be a convex set containing 0, $\lambda_{+,a}$ and $\lambda_{-,a}$ be the maximum and minimum of the eigenvalues of $\nabla^2 \phi_a(\lambda)$ on $\theta_a + A$ and let $\eta > 0$, $i \in \mathbb{N}$.
With probability $1 - \delta$, for all $t \in \mathbb{N}$, if $F_{N_{t,a},a} \in \nabla \phi_a(\Theta_{D,a})$, $N_{t,a} \in [n_i, n_{i+1})$ and $\theta_{N_{t,a},a} \in \theta_a + A$ then
\begin{align*}
N_{t,a} d_{\phi_a}(\theta_a, \theta_{N_{t,a},a})
\le (1 + \eta)\left(d \log(1 + \frac{1}{\eta})
	+ d \log(\gamma \frac{\lambda_{+,a}}{\lambda_{-,a}})
	+ \log \frac{1}{\delta} \right)
\: .
\end{align*}
\end{lemma}

\begin{proof}
We assume in this proof that the three conditions of the lemma are true.

Choose $v_{0,i}^{-1} = \frac{1}{\lambda_{-,a} n_i}((1 + \frac{1}{\eta}) - \frac{\lambda_{-,a}}{\lambda_{+,a}})$. Then for $N_{t,a} \in [n_i, n_{i+1})$,
\begin{align*}
\frac{1}{(N_{t,a} \lambda_{+,a})^{-1} + v_{0,i}^{-1}}
&\le \frac{1}{(N_{t,a} \lambda_{+,a})^{-1} + \frac{1}{\lambda_{-,a} N_{t,a}}((1 + \frac{1}{\eta}) - \frac{\lambda_{-,a}}{\lambda_{+,a}})}
= N_{t,a} \lambda_{-,a}(1 - \frac{1}{1 + \eta})
\: , \\
\log(1 + N_{t,a} \lambda_{+,a} v_{0,i}^{-1})
&\le \log(1 + \gamma n_i \lambda_{+,a} v_{0,i}^{-1})
\\
&= \log(1  - \gamma + \gamma \frac{\lambda_{+,a}}{\lambda_{-,a}}(1 + \frac{1}{\eta}))
\le \log(1 + \frac{1}{\eta}) + \log(\gamma \frac{\lambda_{+,a}}{\lambda_{-,a}})
\: .
\end{align*}

With that value of $v_{0,i}$, we also have
\begin{align*}
\kappa_{N_{t,a}}^{-1}
&= 1 + N_{t,a} \lambda_{+,a} v_{0,i}^{-1}
\le 1 + \gamma n_i \lambda_{+,a} v_{0,i}^{-1}
= 1 - \gamma + \gamma \frac{\lambda_{+,a}}{\lambda_{-,a}}(1 + \frac{1}{\eta})
\\
&\le \gamma \frac{\lambda_{+,a}}{\lambda_{-,a}}(1 + \frac{1}{\eta})
\: .
\end{align*}

We use those inequalities in Corollary~\ref{cor:concentration_with_gaussian_trunc_prior} to get
\begin{align*}
&N_{t,a} d_{\phi_a}(\theta_a, \theta_{N_{t,a},a})
\\
&\le \frac{1}{2}\Vert \theta_a - \theta_{N_{t,a},a} \Vert^2 N_{t,a} \lambda_{-,a}(1 - \frac{1}{1 + \eta})
	+ d \log(1 + \frac{1}{\eta})
	+ d \log(\gamma \frac{\lambda_{+,a}}{\lambda_{-,a}})
	+ \log \frac{1}{\delta}
\: .
\end{align*}

Thanks to the fact that $\lambda_{-,a}$ is a lower bound on the eigenvalues of $\nabla^2 \phi_a$ on $\theta_a + A$, we have $ \frac{1}{2}\Vert \theta_a - \theta_{N_{t,a},a} \Vert^2 \lambda_{-,a} \le d_{\phi_a}(\theta_a, \theta_{N_{t,a},a})$. We group the divergences on the left of the inequality to finally obtain
\begin{align*}
N_{t,a} d_{\phi_a}(\theta_a, \theta_{N_{t,a},a})
\le (1 + \eta)\left(d \log(1 + \frac{1}{\eta})
	+ d \log(\gamma \frac{\lambda_{+,a}}{\lambda_{-,a}})
	+ \log \frac{1}{\delta} \right)
\: .
\end{align*}
\end{proof}

A good choice of $\eta$ yields Lemma~\ref{lem:concentration_slice_i}.

\begin{lemma}\label{lem:concentration_slice_i}
Let $A$ be a convex set containing 0, $\lambda_{+,a}$ and $\lambda_{-,a}$ be the maximum and minimum of the eigenvalues of $\nabla^2 \phi_a$ on $\theta_a + A$ and let $i \in \mathbb{N}$.
With probability $1 - \delta$, for all $t \in \mathbb{N}$, if $F_{N_{t,a},a} \in \nabla \phi_a(\Theta_{D,a})$, $N_{t,a} \in [n_i, n_{i+1})$ and $\theta_{N_{t,a},a} \in \theta_a + A$ then
\begin{align*}
N_{t,a} d_{\phi_a}(\theta_a, \theta_{N_{t,a},a})
&\le d \overline{W}_{-1}\left( 1 + \frac{1}{d} \log\frac{1}{\delta} + \log(\gamma \frac{\lambda_{+,a}}{\lambda_{-,a}}) \right)
\: .
\end{align*}
\end{lemma}
\begin{proof}
Optimize $\eta$ in Lemma~\ref{lem:concentration_slice_i_with_eta} by using Lemma~\ref{lem:lemma_A_3_of_Remy}.
\end{proof}

Considering a union bound over the intervals and applying Lemma~\ref{lem:concentration_slice_i}, we obtain Lemma~\ref{lem:time_uniform_kl_concentration}. Therefore, we have obtained a time-uniform upper tail concentration of the KL divergence of a $d$-dimensional exponential family with support $\Theta_{D} \subseteq \Real^d$.

\begin{lemma} \label{lem:time_uniform_kl_concentration}
Let $(A_a(n, \delta))_{n \in \mathbb{N}}$ be a sequence of non-increasing convex sets containing 0, $\lambda_{+,a}(n, \delta)$ and $\lambda_{-,a}(n, \delta)$ be the maximum and minimum of the eigenvalues of $\nabla^2 \phi_a(\lambda)$ on $\theta_a + A_a(n, \delta)$. Let $s > 1$ and let $\zeta$ be the Riemann $\zeta$ function. For all $t$, let $i(t) = \lfloor \log_\gamma N_{t,a} \rfloor$.
With probability $1 - \delta$, for all $t \in \mathbb{N}$, if $F_{N_{t,a},a} \in \nabla \phi_a(\Theta_{D,a})$ and $\theta_{N_{t,a},a} \in \theta_a + A_a(N_{t,a}, \delta)$, then
\begin{align*}
N_{t,a} d_{\phi_a}(\theta_a, \theta_{N_{t,a},a})
&\le d \overline{W}_{-1}\left( 1 + \log\frac{\zeta(s)}{\delta} + \frac{s}{d} \log(1 + \log_\gamma N_{t,a}) + \log\left(\gamma \frac{\lambda_{+,a}(n_{i(t)}, \delta)}{\lambda_{-,a}(n_{i(t)}, \delta)}\right) \right)
\: .
\end{align*}
\end{lemma}

\begin{proof}
Suppose that $F_{N_{t,a},a} \in \nabla \phi_a(\Theta_{D,a})$.

For all $i \in \mathbb{N}$, Lemma~\ref{lem:concentration_slice_i} gives that with probability $1 - \frac{\delta}{\zeta(s) (i+1)^s}$, if $N_{t,a} \in [n_i, n_{i+1})$ and $\theta_{N_{t,a},a} \in \theta_a + A_a(n_i, \delta)$, then
\begin{align*}
N_{t,a} d_{\phi_a}(\theta_a, \theta_{N_{t,a},a})
&\le d \overline{W}_{-1}\left( 1 + \frac{1}{d} \log\frac{\zeta(s) (i+1)^s}{\delta} + \log(\gamma \frac{\lambda_{+,a}(n_i, \delta)}{\lambda_{-,a}(n_i, \delta)}) \right)
\: .
\end{align*}
With probability $1 - \delta$, this is true for all $i \in \mathbb{N}$. In particular we have the inequality for $i(t)$: if $\theta_{N_{t,a},a} \in \theta_a + A_a(n_{i(t)}, \delta)$, then
\begin{align*}
N_{t,a} d_{\phi_a}(\theta_a, \theta_{N_{t,a},a})
&\le d \overline{W}_{-1}\left( 1 + \frac{1}{d} \log\frac{\zeta(s) (i(t)+1)^s}{\delta} + \log(\gamma \frac{\lambda_{+,a}(n_{i(t)}, \delta)}{\lambda_{-,a}(n_{i(t)}, \delta)}) \right)
\: .
\end{align*}
Since the sets $A_a(n, \delta)$ are non-increasing, a sufficient condition for $\theta_{N_{t,a},a} \in \theta_a + A_a(n_{i(t)}, \delta)$ is $\theta_{N_{t,a},a} \in \theta_a + A_a(N_{t,a}, \delta)$.
By definition of $i(t)$, it satisfies $i(t) \le \log_\gamma N_{t,a}$. We get that if $\theta_{N_{t,a},a} \in \theta_a + A_a(N_{t,a}, \delta)$, then
\begin{align*}
N_{t,a} d_{\phi_a}(\theta_a, \theta_{N_{t,a},a})
&\le d \overline{W}_{-1}\left( 1 + \frac{1}{d}\log\frac{\zeta(s)}{\delta} + \frac{s}{d} \log(1 + \log_\gamma N_{t,a}) + \log(\gamma \frac{\lambda_{+,a}(n_{i(t)}, \delta)}{\lambda_{-,a}(n_{i(t)}, \delta)}) \right)
\: .
\end{align*}
\end{proof}

To obtain Theorem~\ref{thm:uniform_upper_tail_concentration_kl_exp_fam}, it remains to extend this argument to a sum of $\cS\subseteq [K]$ arms.

\subsection{Sum Over Arms}
\label{app:ss_sum_over_arms}

In this section we consider s subset of arms $\cS \subseteq [K]$. Lemma~\ref{lem:kl_bound_with_prior_all_n_multi_arm} extends Corollary~\ref{cor:kl_bound_with_prior_all_n} to the sum over $\cS$ arms.

\begin{lemma} \label{lem:kl_bound_with_prior_all_n_multi_arm}
Let $(\rho_{0,a})_{a \in \cS}$ be distributions each supported on $\Theta_{D,a}$. With probability $1 - \delta$, for all $t \in \mathbb{N}$, either there exists $a \in \cS$ such that $F_{N_{t,a},a} \notin \nabla \phi_a(\Theta_{D,a})$ or
\begin{align*}
\sum_{a \in \cS} N_{t,a} d_{\phi_a}(\theta_a, \theta_{N_{t,a},a})
\le - \sum_{a \in \cS} \ln \mathbb{E}_{y \sim \rho_{0,a}}\exp\left( - N_{t,a} d_{\phi_a}(y, \theta_{N_{t,a},a}) \right) + \log \frac{1}{\delta}
\: .
\end{align*}
\end{lemma}
\begin{proof}
The proof is the same as the one of Lemma~\ref{lem:kl_bound_with_prior} and Corollary~\ref{cor:kl_bound_with_prior_all_n}, except that the likelihood ratio martingale considered contains observations of all arms in $\cS$.
\end{proof}

Let $\gamma > 1$. We define $n_i = \gamma^i$ for $i \in \mathbb{N}$. Lemma~\ref{lem:concentration_slice_i_with_eta_multi_arm} extends Lemma~\ref{lem:concentration_slice_i_with_eta} the sum over $\cS$ arms.

\begin{lemma}\label{lem:concentration_slice_i_with_eta_multi_arm}
Let $(A_a)_{a \in \cS}$ be convex sets containing 0, $\lambda_{+,a}$ and $\lambda_{-,a}$ be the maximum and minimum of the eigenvalues of $\nabla^2 \phi_a(\lambda)$ on $A_a$ and let $\eta > 0$, $(i^a)_{a \in \cS} \in \mathbb{N}^{|\cS|}$.
With probability $1 - \delta$, for all $t \in \mathbb{N}$, if for all $a \in \cS$, $F_{N_{t,a},a} \in \nabla \phi_a(\Theta_{D,a})$, $N_{t,a} \in [n_{i^a}, n_{i^a + 1})$ and $\theta_{N_{t,a},a} \in \theta_a + A_a$ then
\begin{align*}
\sum_{a \in \cS} N_{t,a} d_{\phi_a}(\theta_a, \theta_{N_{t,a},a})
\le (1 + \eta)\left( d |\cS| \log(1 + \frac{1}{\eta})
	+ \log \frac{1}{\delta}
	+ d \sum_{a \in \cS} \log(\gamma \frac{\lambda_{+,a}}{\lambda_{-,a}})
	\right)
\: .
\end{align*}
\end{lemma}
\begin{proof}
For each arm, we choose a truncated Gaussian prior, as in Lemma~\ref{lem:concentration_with_gaussian_trunc_prior_matrix}, and consider upper and lower bound multiples of the identity, as in Lemma~\ref{lem:concentration_with_gaussian_trunc_prior}. Then, we choose the variance of that prior according to the interval, as in Lemma~\ref{lem:concentration_slice_i_with_eta}.
\end{proof}

Lemma~\ref{lem:concentration_slice_i_multi_arm} extends Lemma~\ref{lem:concentration_slice_i} the sum over $\cS$ arms.

\begin{lemma}\label{lem:concentration_slice_i_multi_arm}
Let $(A_a)_{a \in \cS}$ be convex sets containing 0, $\lambda_{+,a}$ and $\lambda_{-,a}$ be the maximum and minimum of the eigenvalues of $\nabla^2 \phi_a(\lambda)$ on $\theta_a + A_a$ and let $(i^a)_{a \in \cS} \in \mathbb{N}^{|\cS|}$.
With probability $1 - \delta$, for all $t \in \mathbb{N}$, if for all $a \in \cS$, $F_{N_{t,a},a} \in \nabla \phi_a(\Theta_{D,a})$, $N_{t,a} \in [n_{i^a}, n_{i^a + 1})$ and $\theta_{N_{t,a},a} \in \theta_a + A_a$ then
\begin{align*}
\sum_{a \in \cS} N_{t,a} d_{\phi_a}(\theta_a, \theta_{N_{t,a},a})
&\le d |\cS| \overline{W}_{-1}\left( 1 + \frac{1}{d |\cS|} \log\frac{1}{\delta} + \frac{1}{|\cS|}\sum_{a \in \cS}\log(\gamma \frac{\lambda_{+,a}}{\lambda_{-,a}}) \right)
\: .
\end{align*}
\end{lemma}
\begin{proof}
Optimize $\eta$ in Lemma~\ref{lem:concentration_slice_i_with_eta_multi_arm} by using Lemma~\ref{lem:lemma_A_3_of_Remy}.
\end{proof}

Theorem~\ref{thm:uniform_upper_tail_concentration_kl_exp_fam} extends Lemma~\ref{lem:time_uniform_kl_concentration} to the sum of $\cS$ arms. Therefore, we have obtained a time-uniform upper tail concentration of the sum of $\cS$ KL divergence of $d$-dimensional exponential family with support $\Theta_{D,a} \subseteq \Real^d$.

\begin{theorem} \label{thm:uniform_upper_tail_concentration_kl_exp_fam}
For all $a \in \cS$, let $(A_a(n, \delta))_{n \in \mathbb{N}}$ be a sequence of non-increasing convex sets containing 0, $\lambda_{+,a}(n, \delta)$ and $\lambda_{-,a}(n, \delta)$ be the maximum and minimum of the eigenvalues of $\nabla^2 \phi_a(\lambda)$ on $\theta_a + A_a(n, \delta)$.
Let $s > 1$ and let $\zeta$ be the Riemann $\zeta$ function.
For all $t$, let $i^a(t) = \lfloor \log_\gamma N_{t,a} \rfloor$.
With probability $1 - \delta$, for all $t \in \mathbb{N}$, if for all $a \in \cS$, $F_{N_{t,a},a} \in \nabla \phi_a(\Theta_{D,a})$ and $\theta_{N_{t,a},a} \in \theta_a + A_a(N_{t,a}, \delta)$, then
\begin{align*}
&\sum_{a \in \cS} N_{t,a} d_{\phi_a}(\theta_a, \theta_{N_{t,a},a})
\\
&\le d |\cS|\overline{W}_{-1}\left( 1 + \frac{\log \left(\frac{\zeta(s)^{|\cS|}}{\delta} \right)}{d |\cS|}
	+ \frac{s}{d |\cS|} \sum_{a \in \cS} \log(1 + \log_\gamma N_{t,a})
	+ \frac{1}{|\cS|}\sum_{a \in \cS} \log(\gamma \frac{\lambda_{+,a}(n_{i^a(t)}, \delta)}{\lambda_{-,a}(n_{i^a(t)}, \delta)}) \right)
\: .
\end{align*}
\end{theorem}
\begin{proof}
Suppose that for all $a \in \cS$, $F_{N_{t,a},a} \in \nabla \phi_a(\Theta_{D,a})$. For all $(i^a)_{a \in \cS} \in \mathbb{N}^{|\cS|}$, Lemma~\ref{lem:concentration_slice_i_multi_arm} gives that with probability $1 - \frac{\delta}{\zeta(s)^{|\cS|} \prod_{a \in\cS} (i^a+1)^s}$, if for all $a \in \cS$,  $N_{t,a} \in [n_{i^a}, n_{i^a + 1})$ and $\theta_{N_{t,a},a} \in \theta_a + A_a(n_{i^a}, \delta)$, then
\begin{align*}
&\sum_{a \in \cS} N_{t,a} d_{\phi_a}(\theta_a, \theta_{N_{t,a},a})
\\
&\le d |\cS| \overline{W}_{-1}\left( 1 + \frac{1}{d|\cS|} \log\frac{\zeta(s)^{|\cS|} \prod_{a \in\cS} (i^a+1)^s}{\delta} + \frac{1}{|\cS|}\sum_{a \in \cS} \log(\gamma \frac{\lambda_{+,a}(n_{i^a}, \delta)}{\lambda_{-,a}(n_{i^a}, \delta)}) \right)
\: .
\end{align*}
With probability $1 - \delta$, this is true for all $(i^a)_{a\in \cS} \in \mathbb{N}^{|\cS|}$. In particular we have the inequality for $(i^a(t))_{a \in \cS}$: if for all $a \in \cS$, $\theta_{N_{t,a},a} \in \theta_a + A_a(n_{i^a(t)}, \delta)$, then
\begin{align*}
&\sum_{a \in \cS} N_{t,a} d_{\phi_a}(\theta_a, \theta_{N_{t,a},a})
\\
&\le d |\cS| \overline{W}_{-1}\left( 1 + \frac{1}{d|\cS|} \log\frac{\zeta(s)^{|\cS|} \prod_{a \in\cS} (i^a(t)+1)^s}{\delta} + \frac{1}{|\cS|}\sum_{a \in \cS} \log(\gamma \frac{\lambda_{+,a}(n_{i^a(t)}, \delta)}{\lambda_{-,a}(n_{i^a(t)}, \delta)}) \right)
\: .
\end{align*}
Since the sets $A_a(n, \delta)$ are non-increasing with respect to $n$, a sufficient condition for $\theta_{N_{t,a},a} \in \theta_a + A_a(n_{i^a(t)}, \delta)$ is $\theta_{N_{t,a},a} \in \theta_a + A_a(N_{t,a}, \delta)$.
By definition of $i^a(t)$, it satisfies $i^a(t) \le \log_\gamma N_{t,a}$. We get the inequality of the theorem.
\end{proof}

\subsection{Univariate Gaussian}
\label{app:ss_univariate_gaussian}

Theorem~\ref{thm:uniform_upper_tail_concentration_kl_exp_fam_gaussian} gives a time-uniform upper tail concentration of the sum of $\cS$ KL divergence of Gaussian with unknown variance. It is obtained by using Theorem~\ref{thm:uniform_upper_tail_concentration_kl_exp_fam} with adequate pre-concentration results.

\begin{theorem} \label{thm:uniform_upper_tail_concentration_kl_exp_fam_gaussian}
	For $i \in \{0,-1\}$, let $\overline{W}_{i}(x) = -W_{i}(-e^{-x})$ for $x\geq1$, $\delta \in (0,1)$, $\cS \in \subseteq [K]$, $\eta_{0}>0$, $\eta_{1}>0$, $\gamma >1$, $s>1$ and $\zeta$ be the Riemann $\zeta$ function. Let
	\begin{align*}
			&\varepsilon_\mu(t, \delta)  = \frac{1}{t}\overline{W}_{-1} \left( 1 + 2\ln \left( \frac{6 |\cS|\zeta(s)}{\delta}\right) + 2s +  2s \ln\left(1 + \frac{\ln t}{2s}\right) \right) \: , \\
			&1 + \varepsilon_{+,\sigma} (t, \delta)  = \overline{W}_{-1} \left(1 +  \frac{2(1+\eta_{1})}{t}\left(\ln\left( \frac{6 |\cS|\zeta(s)}{\delta} \right) + s\ln \left( 1+ \frac{\ln(t)}{\ln(1+\eta_{1})}\right) \right) \right) -\frac{1}{t} \:  , \\
			&1- \varepsilon_{-,\sigma} (t, \delta)  = \overline{W}_{0} \left(1 +  \frac{2(1+\eta_{0})}{t}\left(\ln\left( \frac{6 |\cS|\zeta(s)}{\delta} \right) + s\ln \left( 1+ \frac{\ln(t)}{\ln(1+\eta_{0})}\right) \right) \right) -\frac{1}{t} \: .
	\end{align*}
	For all $a \in [K]$ and $t \geq t_{a}(\delta)$, let $i_a(t) = \lfloor \log_\gamma N_{t,a} \rfloor $ and $ \bar{t}_a = \inf \left\{ t \mid N_{t,a} \geq i_a(t) \right\}$,
	\begin{align*}
	& \mu_{\pm,t,a} = \mu_{ \bar{t}_a ,a} \pm 2\sigma_{ \bar{t}_a,a}\sqrt{ \frac{\varepsilon_\mu(N_{\bar{t}_a,a}, \delta)}{1-\varepsilon_{-,\sigma} (N_{\bar{t}_a,a}-1, \delta)} } \quad \text{,} \quad  \sigma^2_{\pm,t,a} = \sigma_{ \bar{t}_a ,a}^2 \frac{1 \pm \varepsilon_{\pm,\sigma} (N_{\bar{t}_a,a}-1, \delta)}{1 \mp \varepsilon_{\mp,\sigma} (N_{\bar{t}_a,a}-1, \delta)} \: ,\\
	& \mu_{++,t,a}^2 = \max  \mu_{\pm,t,a}^2 \quad \text{,} \quad	R_{ t,a}(\delta) = \frac{\sigma^3_{+,t,a} f_{+}\left( g(\sigma^2_{+,t,a}, \mu_{++,t,a}^2)\right)}{\sigma^3_{-,t,a} f_{-}\left( g(\sigma^2_{-,t,a},\mu_{++,t,a}^2)\right)} \: ,
	\end{align*}
 where $f_{\pm}(x) = \frac{1 \pm \sqrt{1 - x}}{\sqrt{x}}$ and $g(x,y) = \frac{2x}{(x+2y+\frac{1}{2})^2}$. Then, with probability $1 - \delta$, for all $t\geq \max_{a \in \cS} t_{a}(\delta)$,
	\begin{align*}
	&\sum_{a \in \cS} N_{t,a} \KL((\mu_{t,a}, \sigma_{t,a}^2), (\mu_a, \sigma^2_a))
	\\
	&\le 2 |\cS|\overline{W}_{-1}\left( 1 + \frac{\log\frac{2\zeta(s)^{|\cS|}}{\delta}}{2 |\cS|}
		+ \frac{s}{2 |\cS|} \sum_{a \in \cS} \log(1 + \log_\gamma N_{t,a})
		+ \frac{1}{|\cS|}\sum_{a \in \cS} \log\left(\gamma R_{ t,a}(\delta) \right) \right)
	\: .
	\end{align*}
	For all $a \in \cS$, the stochastic initial time is defined as
	\begin{align}\label{eq:def_initial_time_kl_concentration_gaussian}
	t_a(\delta) = \inf \left\{ t \mid N_{t,a} > 1+ \max\left\{t_{\pm}(\delta), e^{ 1 + W_{0} \left( \frac{2(1+\eta_{0})}{e}\left(\ln\left( \frac{6 |\cS|\zeta(s)}{\delta} \right) + s\ln \left( 1+ \frac{\ln(N_{t,a}-1)}{\ln(1+\eta_{0})}\right) \right) -e^{-1}\right)} \right\} \right\} \: .
	\end{align}
	where $\ln (t_{-}(\delta)) = \frac{s}{ \ln\left( \frac{6|\cS|\zeta(s)}{\delta} \right) } - \ln(1+\eta_{0})$ and $\ln( t_{+}(\delta)) = \frac{s}{ \ln\left( \frac{6|\cS|\zeta(s)}{\delta} \right) - \frac{1}{2(1+\eta_1)}} - \ln(1+\eta_1)$.
\end{theorem}
As $W_{0}(x) \in [-1,+\infty)$, $t \geq t_{a}(\delta)$ implies that $N_{t,a} > 2$. Numerically, we always observed $\max t_{\pm}(\delta) < 2$.

The distribution family $\mathcal D$ of Gaussian distributions with positive variance is an exponential family, with natural parameter domain $\Theta_D = \mathbb{R}\times \mathbb{R}^\star_-$ and log-partition function $\phi(\theta) = - \frac{\theta_1^2}{4 \theta_2} - \frac{1}{2}\log(-2 \theta_2)$.
The distribution with mean $x$ and variance $v$ corresponds to natural parameters $(\frac{x}{v}, -\frac{1}{2v})$.
The distribution with natural parameter $\theta_a$ has mean $\mu_a = -\frac{\theta_{a,1}}{2 \theta_{a,2}}$ and variance $\sigma_a^2=-\frac{1}{2 \theta_{a,2}}$.
With those correspondences, we get $\phi(\theta_a) = \frac{\mu_a^2}{2 \sigma_a^2} + \log \sigma_a$ and $\nabla \phi(\theta_a) = (\mu_a, \mu_a^2 + \sigma_a^2)$.
The domain $\nabla \phi(\Theta_D)$ is the set $\{(x, y) \in \mathbb{R}^2 \mid y > x^2\}$.
Finally, the sufficient statistic for that exponential family is $F(X) = (X, X^2)$.
Indeed we have
\begin{align*}
d \nu_{\mu_a, \sigma_a^2}(x)
&= \frac{1}{\sqrt{2 \pi}}e^{- (x - \mu_a)^2/2 \sigma_a^2}
= \frac{1}{\sqrt{2 \pi}} e^{\theta_a^\top(x,x^2) - \phi(\theta_a)}
\: .
\end{align*}

We define the estimator $F_{t,a} = \frac{1}{N_{t,a}}\sum_{s=1}^t F(X_s) \1\{a_s = a\} = (\mu_{t,a}, \mu_{t,a}^2 + \sigma_{t,a}^2)$.
When $\sigma_{t,a}^2 > 0$, $F_{t,a}$ belongs to $\nabla \phi(\Theta_D)$ and the maximum likelihood estimator for $\theta_a$ in $\mathbb{R} \times \mathbb{R}^\star_-$ is well defined: it is $\theta_{t,a} \eqdef \nabla\phi^{-1}(F_{t,a})$. With probability 1, we have $\sigma_{t,a}^2 > 0$ if and only if $N_{t,a} > 1$. That maximum likelihood estimator (MLE) has mean $\mu_{t,a}$ and variance $\sigma^2_{t,a}$. Therefore, the condition $F_{N_{t,a},a} \in \nabla \phi_a(\Theta_{D,a})$ will be satisfied almost surely if and only if $N_{t,a} > 1$.

\paragraph{Preliminary concentration} We will need non-increasing sequence of convex sets $A_a(N_{t,a}, \delta)$ containing $0$ such that with probability $1 - \delta$, for all $t$ greater than an initial time $t_0(\delta)$, the MLE $\theta_{t,a}$ belongs to $\theta_a + A_a(N_{t,a}, \delta)$. Thanks to Lemma~\ref{lem:rectangle_to_trapeze}, the problem is reduced to the task of finding intervals for the mean and variance. The non-increasingness of the sequence of convex sets will be a direct consequence of the monotonicity of the bounds defining the rectangles $[\mu_{-,a}, \mu_{+,a}] \times [\sigma_{-,a}^2, \sigma_{+,a}^2]$ for all $a \in \cS$ (Lemma~\ref{lem:monotonicity_pre_concentration}).

\begin{lemma}\label{lem:rectangle_to_trapeze}
A rectangle in mean and variance space maps to a (convex) trapeze in natural parameter space. That is, for $\sigma_-^2 > 0$ the set $[\mu_{-}, \mu_{+}] \times [\sigma_{-}^2, \sigma_{+}^2]$ represents the same distributions as the natural parameters $\{ \tilde \theta \in \Theta_D \mid \tilde \theta_{2} \in [-\frac{1}{2\tilde \sigma_-^2}, -\frac{1}{2 \tilde \sigma_+^2}], \: \frac{\tilde\theta_{1}}{\tilde\theta_{2}} \in [-2 \tilde\mu_+, -2\tilde\mu_-] \}$. If $(\mu,\sigma^2) \in [\mu_{-}, \mu_{+}] \times [\sigma_{-}^2, \sigma_{+}^2]$, it rewrites as $\theta + A$, where $A$ is a (convex) trapeze containing $0_2$.
\end{lemma}
\begin{proof}
	Given $(\tilde{\mu}, \tilde{\sigma}^2)$, the associated natural parameter is $\tilde{\theta} = \left( \frac{\tilde{\mu}}{\tilde{\sigma}^2}, -\frac{1}{2\tilde{\sigma}^2}\right)$. In particular, for $\theta = (\mu, \sigma^2)$, we define $A$ to be the following trapeze in the natural parameter space
	\begin{align*}
		\theta + A &\eqdef \left\{ \tilde{\theta} \mid (\tilde{\mu},\tilde{\sigma}^2) \in [\mu_{-}, \mu_{+}] \times [\sigma^2_{-}, \sigma_{+}^2]\right\} \\
		&=\: \left\{ \tilde{\theta} \mid \tilde{\theta}_{2} \in \left[-\frac{1}{2\sigma_{-}^2}, -\frac{1}{2\sigma_{+}^2}\right], \: \tilde{\theta}_{1}
		 \in \left[-2\mu_{-} \tilde{\theta}_{2}, -2\mu_{+} \tilde{\theta}_{2}\right] \right\}.
	\end{align*}
	where the second equality uses that $\tilde{\theta}_{2} < 0$. Since convexity is preserved by translation, $A$ is convex if and only if $\theta + A$ is convex, which is direct by considering the last expression. Likewise, $A$ is a trapeze since $\theta + A$ and geometry is preserved by translation. By assumption, we have $(\mu,\sigma^2) \in [\mu_{-}, \mu_{+}] \times [\sigma_{-}^2, \sigma_{+}^2]$. Therefore, $\theta \in \theta + A$, in other words $A$ contains $0_2$. Note that $A \subset \left[\frac{\mu_{-}}{\sigma_{+}^2}, \frac{\mu_{+}}{\sigma_{-}^2}\right] \times \left[-\frac{1}{2\sigma_{-}^2}, -\frac{1}{2\sigma_{+}^2}\right]$.
\end{proof}

Combining the time-uniform concentration inequalities for the mean (Lemma~\ref{lem:uniform_upper_lower_tails_concentration_mean}) and the variance (Corollary~\ref{cor:uniform_time_upper_lower_tail_concentration_variance}) derived in Appendix~\ref{app:box_concentration}, a direct union bound yields that with probability at least $1-\frac{\delta}{2}$, for all $a \in \cS$ and all $t\geq t_{a}(\delta)$
\begin{align*}
	&\sigma_a^2 \left(1 - \varepsilon_{-,\sigma} (N_{t,a}-1, \delta) \right) \leq \sigma^{2}_{t,a} \leq \sigma_a^2 \left(1 + \varepsilon_{+,\sigma} (N_{t,a}-1, \delta)  \right) \: , \\
	&|\mu_{t,a} - \mu_a| \leq \sigma_a\sqrt{ \varepsilon_\mu(N_{t,a}, \delta) }   \: ,
\end{align*}
where, abusing notations compared to Lemma~\ref{lem:delta_correct_box_thresholds} (re-scaling of $\delta$),
\begin{align*}
		&t_{a}(\delta) = \inf \left\{ t \mid N_{t,a} > 1 + e^{ 1 + W_{0} \left( \frac{2(1+\eta_{0})}{e}\left(\ln\left( \frac{6 |\cS|\zeta(s)}{\delta} \right) + s\ln \left( 1+ \frac{\ln(N_{t,a}-1)}{\ln(1+\eta_{0})}\right) \right) -e^{-1}\right)} \right\} \: , \\
		&\varepsilon_\mu(t, \delta)  = \frac{1}{t}\overline{W}_{-1} \left( 1 + 2\ln \left( \frac{6 |\cS|\zeta(s)}{\delta}\right) + 2s +  2s \ln\left(1 + \frac{\ln t}{2s}\right) \right) \: , \\
		&1 + \varepsilon_{+,\sigma} (t, \delta)  = \overline{W}_{-1} \left(1 +  \frac{2(1+\eta_{1})}{t}\left(\ln\left( \frac{6 |\cS|\zeta(s)}{\delta} \right) + s\ln \left( 1+ \frac{\ln(t)}{\ln(1+\eta_{1})}\right) \right) \right) -\frac{1}{t} \:  , \\
		&1- \varepsilon_{-,\sigma} (t, \delta)  = \overline{W}_{0} \left(1 +  \frac{2(1+\eta_{0})}{t}\left(\ln\left( \frac{6 |\cS|\zeta(s)}{\delta} \right) + s\ln \left( 1+ \frac{\ln(t)}{\ln(1+\eta_{0})}\right) \right) \right) -\frac{1}{t} \:  .
\end{align*}
Note that $t \geq t_{a}(\delta)$ implies that $\sigma^{2}_{t,a} > 0$ and $N_{t,a} > 2$ hence $F_{N_{t,a},a} \in \nabla \phi_a(\Theta_{D,a})$. Defining
\begin{align*}
	\mu_{\pm,t,a} = \mu_a \pm \sigma_a\sqrt{ \varepsilon_\mu(N_{t,a}, \delta) } \quad \text{and} \quad \sigma^2_{\pm,t,a} = \sigma_a^2 \left(1 \pm \varepsilon_{\pm,\sigma} (N_{t,a}-1, \delta) \right) \: ,
\end{align*}
we have obtain non-decreasing sequences $\mu_{-,t,a}$ and $\sigma^2_{-,t,a}$, and non-increasing sequences $\mu_{+,t,a}$ and $\sigma^2_{+,t,a}$, such that with probability $1 - \frac{\delta}{2}$, for all $a \in \cS$ and all $t\geq t_{a}(\delta)$,
\[
(\mu_a, \sigma_a^2) \in [\mu_{-,t,a}, \mu_{+,t,a}] \times [\sigma^2_{-,t,a}, \sigma^2_{+,t,a}] \quad \text{and} \quad (\mu_{t,a}, \sigma_{t,a}^2) \in [\mu_{-,t,a}, \mu_{+,t,a}] \times [\sigma^2_{-,t,a}, \sigma^2_{+,t,a}]  \: .
\]

For all $a \in \cS$, applying Lemma~\ref{lem:rectangle_to_trapeze} on this non-increasing sequence of rectangles yield a sequence of non-increasing convex sets containing $0_2$, denoted by $(A_a(n, \delta))_{n \in \mathbb{N}}$, such that $\theta_a \in \theta_a + A_a(n, \delta)$. By construction, we have with probability $1 - \frac{\delta}{2}$, for all $a \in \cS$ and all $t\geq t_{a}(\delta)$,
\begin{align*}
	\theta_{t,a} \in \theta_a + A_a(N_{t,a}, \delta) \: .
\end{align*}

This concludes the construction of the preliminary concentration sets used to apply Theorem~\ref{thm:uniform_upper_tail_concentration_kl_exp_fam}. Doing a union bound (splitting $\delta$ in two) and restricting to $t\geq \max_{a \in \cS} t_{a}(\delta)$, we obtain that with probability $1 - \delta$, for all $t\geq \max_{a \in \cS} t_{a}(\delta)$,
\begin{align*}
&\sum_{a \in \cS} N_{t,a} \KL((\mu_{t,a}, \sigma_{t,a}^2), (\mu_a, \sigma^2_a))
\\
&\le 2 |\cS|\overline{W}_{-1}\left( 1 + \frac{\log\frac{2\zeta(s)^{|\cS|}}{\delta}}{2 |\cS|}
	+ \frac{s}{2 |\cS|} \sum_{a \in \cS} \log(1 + \log_\gamma N_{t,a})
	+ \frac{1}{|\cS|}\sum_{a \in \cS} \log\left(\gamma \frac{\lambda_{+,a}(n_{i^a(t)}, \delta)}{\lambda_{-,a}(n_{i^a(t)}, \delta)}\right) \right)
\: .
\end{align*}
where we used that $d_{\phi_a}(\theta_a, \theta_{N_{t,a},a}) = \KL((\mu_{t,a}, \sigma_{t,a}^2), (\mu_a, \sigma^2_a))$.

\paragraph{Controlling the eigenvalues} While the above is enough to obtain a concentration result, the ratio $\frac{\lambda_{+,a}(n_{i^a(t)}, \delta)}{\lambda_{-,a}(n_{i^a(t)}, \delta)}$ cannot be computed since it depends on $(\mu,\sigma^2)$. To circumvent this issue, we derive an upper bound on this unknown quantity, which will yield a valid concentration inequality.

Upper bounding this quantity might be done by considering a larger rectangle containing the one used to apply Theorem~\ref{thm:uniform_upper_tail_concentration_kl_exp_fam}, which can be done by chaining the concentration to replace $(\mu_{a},\sigma^2_{a})$ by $(\mu_{t,a},\sigma^2_{t,a})$. Defining
\begin{align*}
	\bar \mu_{\pm,t,a} = \mu_{t,a} \pm 2\sigma_{t,a}\sqrt{ \frac{\varepsilon_\mu(N_{t,a}, \delta)}{1-\varepsilon_{-,\sigma} (N_{t,a}-1, \delta)} } \quad \text{and} \quad \bar \sigma^2_{\pm,t,a} = \sigma_{t,a}^2 \frac{1 \pm \varepsilon_{\pm,\sigma} (N_{t,a}-1, \delta)}{1 \mp \varepsilon_{\mp,\sigma} (N_{t,a}-1, \delta)} \: ,
\end{align*}
we have by direct manipulations that
\begin{align*}
	[\mu_{-,t,a}, \mu_{+,t,a}] \times [\sigma^2_{-,t,a}, \sigma^2_{+,t,a}] \subset [\bar \mu_{-,t,a}, \bar \mu_{+,t,a}] \times [ \bar \sigma^2_{-,t,a}, \bar \sigma^2_{+,t,a}] \: .
\end{align*}

By Lemma~\ref{lem:rectangle_to_trapeze}, there is also inclusion of the associated trapeze in the natural space parameters, hence implying the following ordering of the eigenvalues
\begin{align*}
		\lambda_{-,a}(n_{i^a(t)}, \delta) \geq \bar \lambda_{-,a}(n_{i^a(t)}, \delta) \quad \text{and} \quad \lambda_{+,a}(n_{i^a(t)}, \delta) \leq \bar \lambda_{+,a}(n_{i^a(t)}, \delta) \: .
\end{align*}

Lemma~\ref{lem:eigenvalues_based_on_box_concentration} provides a control on the eigenvalues $\lambda_{+}$ and $\lambda_{-}$ in the case of Gaussian distributions based on the trapeze obtained in Lemma~\ref{lem:rectangle_to_trapeze}.

\begin{lemma} \label{lem:eigenvalues_based_on_box_concentration}
	Let $\mu_{-}$ and $\mu_{+}$ be real values such that $\mu_{-} < \mu_{+}$.
	Let $\sigma_{-}^2$ and $\sigma_{+}^2$ be real values such that $0 < \sigma_{-}^2 < \sigma_{+}^2$.
	Let $\Theta$ be the trapeze corresponding to that set in natural parameter space (see Lemma~\ref{lem:rectangle_to_trapeze}).
	The minimal and maximal eigenvalues of $\nabla^2 \phi(\lambda)$ on $\Theta$ are
	\begin{align*}
	\lambda_{-}
	&\ge \sqrt{2}\sigma_{-}^3 f_{-}\left( g(\sigma_{-}^2, \mu_{++}^2)\right)
	\: ,&
	\lambda_{+}
	&\le \sqrt{2}\sigma_{+}^3 f_{+}\left( g(\sigma_{+}^2, \mu_{++}^2)\right)
	\: ,
	\end{align*}
	where $\mu_{++}^2 = \max\{\mu_{-}^2, \mu_{+}^2\}$, $f_{\pm}(x) = \frac{1 \pm \sqrt{1 - x}}{\sqrt{x}}$ and $g(x,y) = \frac{2x}{(x+2y+\frac{1}{2})^2}$.
\end{lemma}
\begin{proof}
	Recall that $\phi(\theta) = - \frac{\theta_1^2}{4 \theta_2} - \frac{1}{2}\log(-2 \theta_2)$, therefore we obtain
		\begin{align*}
			\nabla \phi (\theta) = \begin{bmatrix} -\frac{\theta_{1}}{2 \theta_{2}}  \\ \left( \frac{\theta_{1}}{2 \theta_{2}}\right)^2  - \frac{1}{2 \theta_{2}}\end{bmatrix} \quad \text{and} \quad \nabla^2 \phi (\theta) &=  \frac{1}{2\theta_{2}^2} \begin{bmatrix} -\theta_{2}  & \theta_{1} \\ \theta_{1} & 1 -  \frac{\theta_{1}^2}{ \theta_{2}} \end{bmatrix} \: .
		\end{align*}
		 Computing the eigenvalues of $\nabla \phi (\theta)$, we have $\lambda_{-}(\theta) I_2 \preccurlyeq \nabla^2 \phi (\theta) \preccurlyeq \lambda_{+}(\theta) I_2$ for all $ \theta \in \Theta_{D}$, where the values of the eigenvalues expressed with the mean parameters are
		\begin{align*}
			\lambda_{\pm}(\theta) &= \sigma^2 \left( \sigma^2 + 2 \mu^2 +\frac{1}{2}  \pm \sqrt{\left(\sigma^2 + 2 \mu^2 +\frac{1}{2}\right)^2 - 2\sigma^2} \right) \: .
		\end{align*}
	Note that $\left(\sigma^2 + 2 \mu^2 +\frac{1}{2}\right)^2 - 2\sigma^2 \geq 0 \iff (\sigma^{2})^2 + (4\mu^2-1)\sigma^2+ (2 \mu^2 +\frac{1}{2})^2$. Since $(4\mu^2-1)^2 -4(2 \mu^2 +\frac{1}{2})^2 = -16\mu^2 \leq 0$, we have $\left(\sigma^2 + 2 \mu^2 +\frac{1}{2}\right)^2 - 2\sigma^2 \geq 0$.

	Defining $f_{\pm}(x) = \frac{1 \pm \sqrt{1 - x}}{\sqrt{x}}$ and $g(x,y) = \frac{2x}{(x+2y+\frac{1}{2})^2}$, we have, for all $\tilde \theta \in \theta + A$,
		\begin{align*}
			\lambda_{\pm}(\tilde{\theta}) &=  \tilde{\sigma}^2 \left( \tilde{\sigma}^2 + 2 \tilde{\mu}^2 +\frac{1}{2}  \pm \sqrt{\left(\tilde{\sigma}^2 + 2 \tilde{\mu}^2 +\frac{1}{2}\right)^2 - 2\tilde{\sigma}^2} \right) = \sqrt{2}\tilde{\sigma}^3 f_{\pm}\left( g(\tilde{\sigma}^2, \tilde{\mu}^2)\right) \: .
		\end{align*}

		Direct computations yield, for all $x \in (0,1)$
		\begin{align*}
			f_{+}'(x) = - \frac{1+\sqrt{1-x}}{2x\sqrt{x(1-x)}} < 0 \quad \text{and}\quad
			f_{-}'(x) = \frac{1-\sqrt{1-x}}{2x\sqrt{x(1-x)}} > 0 \: ,
		\end{align*}
		hence $f_{+}$ is decreasing on $(0,1)$ and $f_{-}$ is increasing on $(0,1)$. For all $(x,y) \in \Real^{+} \times \Real^{+}$
		\begin{align*}
		\partial_{y} g(x,y) = -\frac{8x}{(x+2y+\frac{1}{2})^3} \leq 0 \: ,
		\end{align*}
		hence $y \mapsto g(x,y)$ is decreasing for all $x \in \Real^{+}$. Let $\mu_{++}^2 = \max\{\mu_-^2, \mu_+^2\}$. By composition rule and $\lambda_{\pm}(\tilde{\theta}) = \sqrt{2}\tilde{\sigma}^3 f_{\pm}\left( g(\tilde{\sigma}^2, \tilde{\mu}^2)\right)$, we obtain that: for all $\tilde{\theta} \in \theta + A(n, \delta)$
		\begin{align*}
			\lambda_{+}(\tilde{\theta}) \leq \sqrt{2}\tilde{\sigma}^3 f_{+}\left( g(\tilde{\sigma}^2, \mu_{++}^2)\right) \quad \text{and} \quad
			\lambda_{-}(\tilde{\theta}) \geq \sqrt{2}\tilde{\sigma}^3 f_{-}\left( g(\tilde{\sigma}^2, \mu_{++}^2)\right) \: .
		\end{align*}

		To show that $\tilde{\sigma}^2 \mapsto \tilde{\sigma}^3 f_{+}\left( g(\tilde{\sigma}^2, \mu_{++}^2)\right)$ and $\tilde{\sigma}^2 \mapsto \tilde{\sigma}^3 f_{-}\left( g(\tilde{\sigma}^2, \mu_{++}^2)\right)$ are increasing on $\Real^{+}$, we use $\tilde{\sigma}^3 f_{\pm}\left( g(\tilde{\sigma}^2, \mu_{++}^2)\right) = \tilde{\sigma}^2 f_{2 \mu_{++}^2 +\frac{1}{2},\pm}(\tilde{\sigma}^2) = h_{a,\pm}(\tilde{\sigma}^2)$ where $f_{a,\pm}(x) = x + a  \pm \sqrt{\left(x + a\right)^2 - 2x} \geq 0$ and $h_{a,\pm}(x) = x f_{a,\pm}(x)$. Since $h_{a,\pm}'(x) = f_{a,\pm}(x) + xf_{a,\pm}'(x)$ and $f_{a,\pm}(x)\geq 0$, having $f_{a,\pm}'(x) \geq 0$ on $\Real^{+}$ is sufficient to conclude that $h_{a,\pm}'(x) \geq 0$ on $\Real^{+}$. Therefore, a sufficient condition to conclude is to show that $f_{a,\pm}$ is increasing on $\Real^{+}$.
		\begin{align*}
			f_{a,\pm}'(x) = 1 \pm \frac{x+a-1}{\sqrt{(x+a)^2-2x}}
		\end{align*}

		When $x+a-1 \geq 0$, we have directly $f_{a,+}'(x) \geq 0$. Moreover,
		\begin{align*}
			f_{a,-}'(x) \geq 0 \iff  1 \geq \frac{x+a-1}{\sqrt{(x+a)^2-2x}}  \iff 0 \geq 1-2a
		\end{align*}
		Since $1-2a = -2\mu_{++}^2$, we can conclude that $f_{a,-}'(x) \geq 0$.

		When $x+a-1 < 0$, we have directly $f_{a,-}'(x) \geq 0$. Moreover,
		\begin{align*}
			f_{a,+}'(x) \geq 0 \iff  1 \geq - \frac{x+a-1}{\sqrt{(x+a)^2-2x}}  \iff  0 \geq 1-2a
		\end{align*}
		Since $1-2a = -2\mu_{++}^2$, we can conclude that $f_{a,-}'(x) \geq 0$.
\end{proof}

Lemma~\ref{lem:eigenvalues_based_on_box_concentration} and using the above arguments concludes the proof of Theorem~\ref{thm:uniform_upper_tail_concentration_kl_exp_fam_gaussian}.

\paragraph{Monotonicity of the preliminary concentration}
The non-increasingness of the sequence of convex sets is obtained by monotonicity of the bounds of the rectangles $[\mu_{-,a}, \mu_{+,a}] \times [\sigma_{-,a}^2, \sigma_{+,a}^2]$ for all $a \in \cS$.
Given their definitions, it is sufficient to show that the functions $t \mapsto \varepsilon_\mu(t, \delta)$ and $t \mapsto \varepsilon_{\pm,\sigma}(t, \delta)$ are decreasing.
This is shown in Lemma~\ref{lem:monotonicity_pre_concentration}. The conditions on the initial time to have monotonicity are quite mild since $t_{\pm}(\delta) < 1$ after a relatively mild condition on $\delta$. Depending on the choice of $(|\cS|, s,\eta_0,\eta_1)$, it can even hold for all $\delta \in (0,1]$. Numerically, for any practical choices of parameters, we always had $t_{\pm}(\delta) < 2$. Since $t \geq t_{a}(\delta)$ implies that $N_{t,a} \geq 2$, the condition $N_{t,a} \geq t_{\pm}(\delta)$ is milder.

\begin{lemma} \label{lem:monotonicity_pre_concentration}
	Let $\delta \in (0,1]$. The function $t \mapsto \varepsilon_\mu(t, \delta)$ is decreasing on $[1,+\infty)$. There exists $t_{+}(\delta)$ such that the function $t \mapsto \varepsilon_{\pm,\sigma}(t, \delta)$ is decreasing on $[t_{+}(\delta),+\infty)$. In particular, $t_{-}(\delta) \leq 1 $ if and only if $\delta \leq 6|\cS|\zeta(s)e^{-\frac{s}{\ln(1+\eta_{0})}}$ and $t_{+}(\delta) \leq 1 $ if and only if $\delta \leq 6|\cS|\zeta(s)e^{-\frac{s}{\ln(1+\eta_{1})} - \frac{1}{2(1+\eta_1)}}$.
\end{lemma}
\begin{proof}
	Let $\delta \in (0,1]$. Recall that $\overline{W}_{i}'(x) = \frac{1}{1-\frac{1}{\overline{W}_{i}(x)}}$ for all $x>1$ and $i \in \{-1,0\}$ (Lemma~\ref{lem:lambert_branches_properties}).

\noindent \textbf{Decreasing $\varepsilon_\mu$.} Let $f(t) = 2s \ln\left(1 + \frac{\ln t}{2s}\right) $ and $c = 1 + 2\ln \left( \frac{6|\cS|\zeta(s)}{\delta}\right) + 2s$. Since $s >1$ and $\frac{6|\cS|\zeta(s)}{\delta} >1$, we have $c > 3$. Directly, we have $f'(t)=\frac{1}{t}\frac{1}{1 + \frac{\ln t}{2s}}$. Then, by composition of the derivatives, we obtain
	\begin{align*}
		t^2 \frac{\partial \varepsilon_\mu(t, \delta)}{\partial t} &=  \frac{1 - \left(1 + \frac{\ln t}{2s} \right)\left(\overline{W}_{-1}\left( c+f(t)\right) -1\right)}{\left(1 + \frac{\ln t}{2s} \right) \left( 1-\frac{1}{\overline{W}_{-1}\left( c+f(t)\right)}\right)} \: .
	\end{align*}
	Since $t\geq 1$ and $\overline{W}_{-1}\left( c+f(t)\right) > 1$, we have
	\begin{align*}
		\frac{\partial \varepsilon_\mu(t, \delta)}{\partial t} < 0 &\iff   \left(1 + \frac{\ln t}{2s} \right)\left(\overline{W}_{-1}\left( c+f(t)\right) -1\right) > 1 \\
		&\iff \frac{\ln t}{2s} > \frac{1}{\overline{W}_{-1}\left( c+f(t)\right) -1} - 1 \: .
	\end{align*}

	Using that $f(t) \geq 0$ for all $t\geq 1$, $c > 3$ and $\overline{W}_{-1}(x) > x$, we obtain
	\begin{align*}
		\frac{1}{\overline{W}_{-1}\left( c+f(t)\right) -1} - 1 \leq -\frac{1}{2} < 0 \leq  \frac{\ln t}{2s} \: .
	\end{align*}
	Therefore, we have shown that $t \mapsto \varepsilon_\mu(t, \delta)$ is decreasing on $[1,+\infty)$ for all $\delta \in (0,1]$.

	\noindent \textbf{Decreasing $\varepsilon_{\pm,\sigma}$.} Let $(i_{+}, i_{-}) = (-1,0)$ and $(\eta_{+}, \eta_{-}) = (\eta_1,\eta_0)$. Let $a_{\pm}=2(1+\eta_{\pm})$, $b = \ln\left( \frac{6|\cS|\zeta(s)}{\delta} \right)$, $f_{\pm}(t) = s \ln\left(1 + \frac{\ln t}{\ln(1+\eta_{\pm})}\right) $ and $g_{\pm}(t) = \frac{b+f_{\pm}(t)}{t}$. Directly, we have
	\[
	f_{\pm}'(t)=\frac{1}{t}\frac{s}{\ln(1+\eta_{\pm}) + \ln t} \quad \text{and} \quad g_{\pm}'(t) = \frac{tf_{\pm}'(t) - (b+f_{\pm}(t))}{t^2} \: .
	\]
	Then, by composition of the derivatives, we obtain
	\begin{align*}
		\pm t^2 \frac{\partial \varepsilon_{\pm,\sigma}(t,\delta)}{\partial t} &=   \frac{\overline{W}_{i_{\pm}}\left( 1+a_{\pm}g_{\pm}(t) \right)\left(1-a_{\pm} \left(b+f_{\pm}(t) - tf_{\pm}'(t) \right)  \right) - 1}{1 - \overline{W}_{i_{\pm}}\left( 1+a_{\pm}g_{\pm}(t) \right) } \: .
	\end{align*}
	Since $t\geq 1$, $\overline{W}_{-1}\left( 1+a_{+}g_{+}(t)\right) > 1$ and $\overline{W}_{0}\left( 1+a_{-}g_{-}(t)\right) < 1$, we have
	\begin{align*}
			\frac{\partial \varepsilon_{\pm,\sigma}(t,\delta)}{\partial t} < 0 & \iff   a_{\pm} \left(b+f_{\pm}(t) - tf_{\pm}'(t) \right) > 1 - \frac{1}{\overline{W}_{i_{\pm}}\left( 1+a_{\pm}g_{\pm}(t) \right)} \\
			&\impliedby a_{\pm} \left(b - \frac{s}{\ln(1+\eta_{\pm}) + \ln t} \right) \geq 1 - \frac{1}{\overline{W}_{i_{\pm}}\left( 1+a_{\pm}g_{\pm}(t) \right)} \: ,
	\end{align*}
	where the sufficient condition is obtained by noting that $f_{\pm}(t) > 0$ for $t > 1$. Using that $\overline{W}_{0}(1+x) \in (0,1)$ and that $\overline{W}_{-1}(1+x) > 1$, we can obtain further sufficient conditions
	\begin{align*}
			\frac{\partial \varepsilon_{+,\sigma}(t,\delta)}{\partial t} < 0 &  \impliedby  \ln\left( \frac{6|\cS|\zeta(s)}{\delta} \right)  \geq  \frac{1}{2(1+\eta_1)} + \frac{s}{\ln(1+\eta_1) + \ln t} \\
			&\iff t \geq t_{+}(\delta) = \exp \left( \frac{s}{ \ln\left( \frac{6|\cS|\zeta(s)}{\delta} \right) - \frac{1}{2(1+\eta_1)}} - \ln(1+\eta_1) \right)  \: , \\
			\frac{\partial \varepsilon_{-,\sigma}(t,\delta)}{\partial t} < 0 &  \impliedby t \geq t_{-}(\delta) = \exp \left( \frac{s}{ \ln\left( \frac{6|\cS|\zeta(s)}{\delta} \right) } - \ln(1+\eta_{0}) \right)  \: .
	\end{align*}
\end{proof}

%% file: sections/appendix_thresholds.tex

\section{Thresholds}
\label{app:thresholds}

After leveraging the link between the GLR and the EV-GLR statistics (Lemma~\ref{lem:glrt_evglrt_stopping_threshold_relationships}) and showing how to calibrate $\delta$-correct thresholds (Lemma~\ref{lem:GLR_delta_correct}), we derive several $\delta$-correct family of thresholds with increasing complexities both theoretically and numerically: Student thresholds (Appendix~\ref{app:ss_student_thresholds}), box thresholds (Appendix~\ref{app:ss_box_thresholds}), KL thresholds (Appendix~\ref{app:ss_kl_thresholds}) and BoB thresholds (Appendix~\ref{app:ss_Kinf_thresholds}).
In Appendix~\ref{app:ss_asymptotically_tight_thresholds}, we study whether the derived family of thresholds is asymptotically tight.

\paragraph{Thresholds Relationship}
Lemma~\ref{lem:glrt_evglrt_stopping_threshold_relationships} shows that the relationship between the GLR and the EV-GLR statistics (Lemma~\ref{lem:glrt_evglrt_inequalities}) allows to obtain $\delta$-correct thresholds for the EV-GLR stopping rule by using the ones obtained for GLR stopping rule, and vice-versa.

\begin{lemma} \label{lem:glrt_evglrt_stopping_threshold_relationships}
	Let $(c_{a,b})_{(a,b) \in [K]^2}$ be a family of thresholds.

	If $(c_{a,b})_{(a,b) \in [K]^2}$ ensures $\delta$-correctness of the EV-GLR stopping rule, then it ensures $\delta$-correctness of the GLR stopping rule.

	Let $C_{b}(\mu_{t}, \sigma^2_{t}) = \frac{(\mu_{t, \hat a_t} - \mu_{t, b})^2}{\min\{\sigma_{t, \hat a_t}^2,\sigma_{t,b}^2\}}$. If $(c_{a,b})_{(a,b) \in [K]^2}$ ensures $\delta$-correctness of the GLR stopping rule, then $(\tilde c_{a,b})_{(a,b) \in [K]^2}$ ensures $\delta$-correctness of the EV-GLR stopping rule, where $\tilde c_{a,b}(N_t,\delta) = \frac{C_{b}(\mu_{t}, \sigma^2_{t})}{\ln \left( 1+ C_{b}(\mu_{t}, \sigma^2_{t})\right)} c_{a,b}(N_t,\delta)$.
\end{lemma}
\begin{proof}
	Using Lemma~\ref{lem:glrt_evglrt_inequalities}, we have the following inequalities between the statistics involved in the GLR and EV-GLR stopping rules
	\begin{equation*}
		Z^{\text{EV}}_a(t) \geq Z_a(t) \geq  \frac{\ln \left( 1+ C_a(\mu_{t}, \sigma^2_{t})\right)}{C_a(\mu_{t}, \sigma^2_{t})}  Z^{\text{EV}}_a(t)  \: .
	\end{equation*}

	Let $(c_{a,b})_{(a,b)\in [K]^2}$ a family of thresholds ensuring $\delta$-correctness of $\taud^{\text{EV}}$. We show by inclusion of event that $(c_{a,b})_{(a,b)\in [K]^2}$ is a family of thresholds ensuring $\delta$-correctness of $\taud$.
	\begin{align*}
		\left\{\taud < + \infty , \hat{a}_{\taud} \neq a^\star\right\} &=\left\{ \exists t \in \Natural, \: \forall a \neq \hat{a}_t, \:   Z_a(t) > c_{\hat a_t, a}(N_t,\delta) , \hat{a}_t \neq a^\star \right\} 	\\
		&\subseteq \left\{ \exists t \in \Natural,\: \forall a \neq \hat{a}_t, \:   \: Z^{\text{EV}}_a(t) > c_{\hat a_t, a}(N_t,\delta) , \hat{a}_t \neq a^\star \right\}
	\end{align*}
	The exact same argument can be used to show the second statement.
\end{proof}

\paragraph{Calibration by concentration}
Lemma~\ref{lem:GLR_delta_correct} gives the terms to concentrate to ensure $\delta$-correctness.
Due to the structure of the identification problem which doesn't involve the variance, two distance metrics can be used as starting point for the GLR stopping rule.
They both rely on a weighted sum of the per-arm KL divergences between the current estimator and an unknown parameter $(\mu, \tilde \sigma_{t}^2)$.
While the choice $\tilde \sigma_{t}^2 = \sigma^2$ seems natural in (\ref{eq:kl_based_concentration}), it doesn't fully leverage the BAI structure.
This can be done in (\ref{eq:log_based_concentration}) by choosing $\tilde \sigma_{t}^2 = \sigma_{t}^2 + (\mu_t - \mu)^2$, which yields smaller thresholds.

\begin{lemma} \label{lem:GLR_delta_correct}
If with probability $1 - \delta$, for all $t \in \mathbb{N}$ and for all $a \ne a^\star(\mu)$,
\begin{equation*}
\sum_{b \in \{a, a^\star(\mu)\}}\frac{N_{t,b}}{2} \ln \left( 1 + \frac{(\mu_{t,b} - \mu_{b})^2}{\sigma_{t,b}^2}\right)
\le c_{a, a^\star(\mu)}(N_t,\delta)
\: ,
\end{equation*}
then the GLR stopping rule using the family of thresholds $(c_{b,a})_{(b,a) \in [K]^2}$ is $\delta$-correct on $\mathcal D^K$.

If with probability $1 - \delta$, for all $t \in \mathbb{N}$ and for all $a \ne a^\star(\mu)$,
\begin{equation*}
\sum_{b \in \{a, a^\star(\mu)\}}N_{t,b}\KL((\mu_{t,b}, \sigma_{t,b}^2), (\mu_b, \sigma^2_b))
\le c_{a, a^\star(\mu)}(N_t,\delta)
\: ,
\end{equation*}
then the GLR stopping rule using the family of thresholds $(c_{b,a})_{(b,a) \in [K]^2}$ is $\delta$-correct on $\mathcal D^K$.

If with probability $1 - \delta$, for all $t \in \mathbb{N}$ and for all $a \ne a^\star(\mu)$,
\begin{equation} \label{eq:sum_squared_concentration}
\sum_{b \in \{a, a^\star(\mu)\}}N_{t,b}\frac{(\mu_{t,b} - \mu_{b})^2}{2\sigma_{t,b}^2}
\le c_{a, a^\star(\mu)}(N_t,\delta)
\: ,
\end{equation}
then the EV-GLR stopping rule using the family of thresholds $(c_{b,a})_{(b,a) \in [K]^2}$ is $\delta$-correct on $\mathcal D^K$.
\end{lemma}

\begin{proof}
Let $\hat{a}_t = \hat a_t$ and $a^\star = a^\star(\mu)$. Using Lemma~\ref{lem:glrt_formula_gaussian}, the stopping time (\ref{eq:def_stopping_rule_glrt}) for the GLR stopping rule involves the statistics
\begin{align*}
		Z_a(t) &= \inf_{ (\lambda,\kappa^2) : \lambda_{a} \geq \lambda_{\hat{a}_t}}  \sum_{b \in \{a, \hat{a}_t\}} N_{t,b}\KL((\mu_{t,b}, \sigma_{t,b}^2), (\lambda_b, \kappa_{b}^2)) \\
		&=  \inf_{ (\lambda,\kappa^2) : \lambda_{a} \geq \lambda_{\hat{a}_t}}  \sum_{b \in \{a, \hat{a}_t\}} \frac{N_{t,b}}{2} \ln \left( 1 + \frac{(\mu_{t,b} - \lambda_{b})^2}{\sigma_{t,b}^2}\right)  \: .
\end{align*}

Then, by definition,
\begin{align*}
	\bP\left(\taud < + \infty , \hat{a}_{\taud} \neq a^\star \right)
	\leq \bP\left(\exists t \in \Natural, \: \exists a \neq a^\star ,\: a=\hat{a}_t, \: \forall c \neq a, \: Z_{c}(t)  > c_{a, c}(N_t,\delta)\right) \: .
\end{align*}

Since a valid choice is $c = a^\star$ and $\lambda = \mu$, we obtain
\begin{align*}
&\bP\left(\taud < + \infty , \hat{a}_{\taud} \neq a^\star \right)\\
&\leq \begin{cases}
\bP\left(\exists t \in \Natural, \: \exists a \neq a^\star ,\:  \sum_{b \in \{a^\star,a\}} N_{t,b}\KL((\mu_{t,b}, \sigma_{t,b}^2), (\mu_b, \sigma_{b}^2)) > c_{a, a^\star}(N_t,\delta) \right)  \\
  \bP\left(\exists t \in \Natural, \: \exists a \neq a^\star ,\:  \sum_{b \in \{a^\star,a\}}  \frac{N_{t,b}}{2} \ln \left( 1 + \frac{(\mu_{t,b} - \mu_{b})^2}{\sigma_{t,b}^2}\right)  > c_{a, a^\star}(N_t,\delta) \right)
\end{cases} \: .
\end{align*}

The concentration assumptions yield $\delta$-correctness of the first two family of thresholds.
The proof for the EV-GLR stopping rule is identical to the one for the GLR stopping rule except that it uses the statistics $Z^{\text{EV}}_a(t)$.
\end{proof}

\subsection{Student Thresholds}
\label{app:ss_student_thresholds}

Lemma~\ref{lem:delta_correct_student_thresholds} gives the family of Student thresholds.
\begin{lemma} \label{lem:delta_correct_student_thresholds}
	Let $s> 1$ and $\zeta$ be the Riemann $\zeta$ function. Let a family of thresholds $c_{a,b}(N_t, \delta)$ with value $+ \infty$ if $t < \max_{c \in \{a, b \}} t_c^{\text{S}}(\delta)$ and otherwise $c_{a,b}^{\text{S}}(N_t, \delta) = \max \left\{ \beta^{\text{S}}(N_{t,a}, \delta), \beta^{\text{S}}(N_{t,b}, \delta)\right\}$. Taking
	\begin{equation}  \label{eq:def_student_threshold_glrt}
		\beta^{\text{S}}(t, \delta) = t \ln \left( 1+ \frac{1}{t-1} Q\left(1 - \frac{\delta}{4(K-1)\zeta(s)t^{s}} ; \cT_{t-1}\right)^2\right)
	\end{equation}
	yields a $\delta$-correct family of thresholds for the GLR stopping rule. The stochastic initial times are
\begin{equation} \label{eq:def_initial_time_student_thresholds}
	\forall a \in [K], \quad t_{a}^{\text{S}}(\delta) \eqdef \inf\left\{ t \in \Natural \mid N_{t,a} \geq \max\left\{2,\left(\frac{\delta}{4(K-1)\zeta(s)} \right)^{1/s}\right\}\right\} \: .
\end{equation}
\end{lemma}

\begin{proof}
	Let $s> 1$ and $\zeta$ be the Riemann $\zeta$ function. Let $\cT_{n}$ denotes the Student distribution with $n$ degrees of freedom and $Q$ its quantile function. We define a threshold $c_{a,b}(N_t, \delta)$ with value $+ \infty$ if $t < \max_{c \in \{a, b \}} t_c(\delta)$ and otherwise $c_{a,b}(N_t, \delta) = \max \left\{ c(N_{t,a}, \delta), c(N_{t,b}, \delta)\right\}$.

	Using Lemma~\ref{lem:GLR_delta_correct}, $\delta$-correctness of the family of thresholds can be obtained directly by upper bounding (\ref{eq:log_based_concentration}). An initial time condition $t \geq \max_{c \in \{a, b \}} t_c(\delta)$ is necessary for the threshold $(c_{a,b})_{a,b}$ to be defined since they involve the quantiles of Student distribution. We obtain
	\begin{equation*}
		t_{a}(\delta) \eqdef \inf\left\{ t \in \Natural \mid N_{t,a} \geq \max\left\{2,\left(\frac{\delta}{4(K-1)\zeta(s)} \right)^{1/s}\right\}\right\} \: .
	\end{equation*}

	Let $a^\star = a^\star(\mu)$. A simple approach to control the sum of two terms is to control each term individually. Each individual term is a function of $\frac{\mu_{t,a} - \mu_{a}}{\sqrt{\tilde \sigma_{t,a}^2/N_{t,a}}}$ which has a Student distribution, where $\tilde \sigma_{t,a}^2= \frac{N_{t,a}}{N_{t,a}-1} \sigma_{t,a}^2$ is the unbiased variance. Let $\hat{\mu}_t$ be the empirical mean of $t$ standard Gaussian, $\hat{\sigma}_t^2 = \frac{1}{t} \sum_{s=1}^{t} \left(X_s - \hat{\mu}_t\right)^2$ is the empirical variance and $\tilde{\sigma}_t^2 = \frac{t}{t-1} \hat{\sigma}_t^2$ is its unbaised version. Using a union bound and the fact that $\tfrac{\hat{\mu}_t}{\sqrt{\tilde{\sigma}_t^2/t}} \sim \cT_{t-1}$ we obtain
	\begin{align*}
		\bP\left(\exists t \geq \tilde t_{0}(\delta)  : \frac{\hat{\mu}_t}{\sqrt{\tilde{\sigma}_t^2/t}} > \tilde c(t,\delta) \right) &\leq \sum_{t \geq t_{0}(\delta)} \bP\left(\frac{\hat{\mu}_t}{\sqrt{\tilde{\sigma}_t^2/t}} > \tilde c(t,\delta) \right) \leq \frac{\delta}{4(K-1)}
	\end{align*}
	where $\tilde c(t,\delta) = Q\left(1 - \frac{\delta}{4(K-1)\zeta(s)t^{s}} ; \cT_{t-1}\right)$. Using this result, direct computations yield that
	\begin{align*}
		&\bP \left( \exists t \geq t_{0}(\delta): \: \exists a \neq a^\star, \: \sum_{b \in \{a, a^\star\}}\frac{N_{t,b}}{2} \ln \left( 1 + \frac{(\mu_{t,b} - \mu_{b})^2}{\sigma_{t,b}^2}\right) >  \max_{b \in \{a, a^\star\} } c(N_{t,b}, \delta) \right) \\
		&\leq \bP \left( \exists t \geq t_{0}(\delta): \: \exists a \neq a^\star, \: \exists b \in \{a, a^\star\}, \: N_{t,b} \ln \left( 1 + \frac{(\mu_{t,b} - \mu_{b})^2}{\sigma_{t,b}^2}\right) > c(N_{t,b}, \delta) 	\right) \\
		&\leq 2(K-1) \bP \left( \exists t \geq \tilde t_{0}(\delta): \: t \ln \left( 1 + \frac{t}{t-1} \frac{\hat \mu_{t}^2}{\tilde \sigma_{t}^2}\right) > c(t, \delta) 	\right) \\
		&\leq 4(K-1) \bP\left(\exists t \geq \tilde t_{0}(\delta)  : \frac{\hat{\mu}_t}{\sqrt{\tilde{\sigma}_t^2/t}} >  \sqrt{(t-1)\left(\exp \left(\frac{c(t,\delta)}{t}\right)-1\right)} \right) \leq \delta
	\end{align*}
	where the last equation is obtained by choice of the stopping threshold
		\begin{equation}
			c(t, \delta) = t \ln \left( 1+ \frac{1}{t-1} Q\left(1 - \frac{\delta}{4(K-1)\zeta(s)t^{s}} ; \cT_{t-1}\right)^2\right) \: .
		\end{equation}
	Since it satisfies the hypothesis of Lemma~\ref{lem:GLR_delta_correct}, this yields the desired result.
\end{proof}

\paragraph{EV-GLR stopping rule}
Up to a log-transform the same arguments yield a family of thresholds for the EV-GLR stopping rule.
The proof of Lemma~\ref{lem:delta_correct_student_thresholds_ev} is omitted since it is almost identical to the above.

	\begin{lemma} \label{lem:delta_correct_student_thresholds_ev}
		Let $s> 1$ and $\zeta$ be the Riemann $\zeta$ function. Let $(t_a)_{a \in [K]}$ as in (\ref{eq:def_initial_time_student_thresholds}). We define a threshold $c_{a,b}(N_t, \delta)$ with value $+ \infty$ if $t < \max_{c \in \{a, b \}} t_c(\delta)$ and otherwise $c_{a,b}(N_t, \delta) = \max \left\{ c(N_{t,a}, \delta), c(N_{t,b}, \delta)\right\}$. Taking
			\begin{equation}  \label{eq:def_student_threshold_evglrt}
				c(t, \delta) = \frac{t}{t-1} Q\left(1 - \frac{\delta}{4(K-1)\zeta(s)t^{s}} ; \cT_{t-1}\right)^2
			\end{equation}
		yields a $\delta$-correct family of thresholds for the EV-GLR stopping rule.
	\end{lemma}

\subsection{Box Thresholds}
\label{app:ss_box_thresholds}

Before presenting the counterparts for the EV-GLR stopping rule, we first present the proof of Lemma~\ref{lem:delta_correct_box_thresholds}.

\begin{proof}
	Let $a\neq a^\star = a^\star(\mu)$.
Using Lemma~\ref{lem:GLR_delta_correct}, we only need to exhibit threshold ensuring the required concentration behavior.
One way of obtaining such an upper bound is to maximize the above quantities under constraints obtained by our concentration results.
The form of the optimization is independent of the considered pair of arms $\{a, a^\star(\mu)\}$ and of the time $t$ (omitted in the following).
As we will see it only depends on $y_b = \frac{(\mu_{t,b} - \mu_{b})^2}{\sigma_{t,b}^2}$ and $x_b = \frac{\sigma_{t,b}^2}{\sigma_{b}^2}$.

We can show that the family of thresholds in (\ref{eq:def_log_box_threshold_glrt}) is the solutions of an optimization problem.
Let $C, D \in (\mathbb{R}^\star_+)^2$ and $N \in (\Natural^{\star})^2$,
\begin{align*}
\text{maximize } & \sum_{b \in \{1, 2\}} \frac{N_{b}}{2} \ln \left( 1 + y_b\right)
\\
\text{such that }\: & \forall b \in \{1, 2\}, \quad y_b \geq 0, \: x_b y_b \leq C_b, \: x_b \geq D_b \: .
\end{align*}

Since $y \mapsto \ln \left( 1 + y\right)$ is concave and increasing, $y \mapsto \sum_{b \in \{1, 2\}} N_{b} \ln \left( 1 + y_b\right)$ is concave and increasing in each of its coordinates. Since the constraints and the objective are separate between each coordinate, the maximum is achieved at $\frac{C_b}{D_b}$ and as value
\[
\sum_{b \in \{1, 2\}} \frac{N_{b}}{2}  \ln \left( 1 + \frac{C_b}{D_b}\right) \: .
\]
To obtain the family of thresholds as in (\ref{eq:def_log_box_threshold_glrt}), we simply need to use concentration results to specify the constraints $C_b$ and $D_b$, which differ depending on the considered pair of arms $\{a, a^\star(\mu)\}$ and of the time $t$.
This can be done by combining the lower tail concentration on the empirical variance (Corollary~\ref{cor:uniform_time_upper_lower_tail_concentration_variance}) and the upper and lower tail concentration of the empirical mean (Lemma~\ref{lem:uniform_upper_lower_tails_concentration_mean}). By direct union bound, we have with probability greater than $1 - \frac{\delta}{K-1}$, for all $b \in \{a, a^\star\}$ and all $t \geq \max_{b\in \{a, a^\star\}} t_{b}(\delta)$,
	\begin{align*}
	(\mu_{t,b} - \mu_b)^2
	&\le \sigma_b^2 \varepsilon_\mu(N_{t,b}, \delta)
	\: , \\
	\sigma_{t,b}^2
	&\ge \sigma_b^2 (1 - \varepsilon_{-,\sigma} (N_{t,b}-1, \delta))
	\: .
	\end{align*}
	where $\varepsilon_\mu$, $\varepsilon_{-,\sigma}$ and $t_{b}$ are defined as in Lemma~\ref{lem:delta_correct_box_thresholds}.
	Using Lemma~\ref{lem:lambert_branches_properties}, this initial time condition ensures that $1- \varepsilon_{-,\sigma} (N_{t,b}-1, \delta) > 0$. Since $\overline{W}_{0}$ has values in $(0,1)$, we obtain that $\varepsilon_{-,\sigma} (t, \delta) \in (0,1)$.
	Taking a union bound over $a \neq a^\star$ concludes the proof.
\end{proof}

\paragraph{EV-GLR stopping rule} Up to a log-transform the same arguments yield a family of thresholds for the EV-GLR stopping rule (Lemma~\ref{lem:delta_correct_student_thresholds_ev}).
The proof of Lemma~\ref{lem:delta_correct_box_thresholds_ev} is omitted since it is almost identical to the above.

\begin{lemma} \label{lem:delta_correct_box_thresholds_ev}
	Let $\varepsilon_\mu$, $\varepsilon_{-,\sigma}$ and $(t_{a})_{a \in [K]}$ as in Lemma~\ref{lem:delta_correct_box_thresholds}. We define a threshold $c_{a,b}(N_t, \delta)$ with value $+ \infty$ if $t < \max_{c \in \{a, b \}} t_c(\delta)$ and otherwise
		\begin{equation}  \label{eq:def_box_threshold_evglrt}
			c_{a,b}(N_t, \delta) = \sum_{c \in \{a, b \}}  \frac{N_{t,c}\varepsilon_\mu(N_{t,c}, \delta)}{2(1 - \varepsilon_{-,\sigma} (N_{t,c}-1, \delta))} \: .
		\end{equation}
	This yields a $\delta$-correct family of thresholds for the EV-GLR stopping rule.
\end{lemma}

\subsection{KL Thresholds}
\label{app:ss_kl_thresholds}

Thanks to Lemma~\ref{lem:GLR_delta_correct}, it is sufficient to concentrate the summed KL divergence. Let $a^\star = a^\star(\mu)$. This can be done by Theorem~\ref{thm:uniform_upper_tail_concentration_kl_exp_fam_gaussian} with $\cS = \{a, a^\star\}$ for all $a\neq a^\star$ (and taking a bound over those $K-1$ terms). For all $a \in [K]$, the additional initial time condition to obtain monotonicity rewrites
\[
t_{a}(\delta) = \inf \left\{ t \mid  N_{t,a} > 1 + \max \left\{ \frac{e^{s / \ln\left( \frac{12(K-1)\zeta(s)}{\delta} \right)}  }{1+ \eta_0},  \frac{e^{s /\left( \ln\left( \frac{12(K-1)\zeta(s)}{\delta} \right) - \frac{1}{2(1+\eta_1)} \right)}  }{1+\eta_1}\right\}  \right\} \: .
\]
Numerically, we always observed that this was satisfied after initialization. Therefore, those terms have no impact.

\subsection{BoB Thresholds}
\label{app:ss_Kinf_thresholds}

In Appendix~\ref{app:ss_box_thresholds}, we saw how to calibrate stopping threshold based on an optimization problem using concentration constraints.
Modifying the optimization problem therein, we can leverage Theorem~\ref{thm:delta_correct_complex_threshold_glrt} by adding a constraint.
While this argument requires doing the union bound over two concentration results (hence considering $\delta/2$), the added constraint can result in a smaller stopping threshold (hence faster stopping).

With our notations, for $b\in \{1,2\}$, the KL divergence rewrites as $\KL((\mu_{t,b}, \sigma_{t,b}^2), (\mu_b, \sigma^2_b)) = \frac{1}{2}f(x_b, y_b)$ with $f(x,y) = (1+y)x-1-\ln(x)$. Since $\nabla^2 f(x,y) = \begin{bmatrix}
	x^{-2} & 1 \\
	1 & 0
\end{bmatrix}$ is a positive semi-definite matrix on $(\mathbb{R}^\star_+)^2$, $f$ is a convex function of $(\mathbb{R}^\star_+)^2$. Let $E > 0$. The KL constraint from Theorem~\ref{thm:delta_correct_complex_threshold_glrt} is convex and can be expressed as
\[
\sum_{b \in \{1, 2\}} \frac{N_{b}}{2} f \left(x_b, y_b\right) \leq E \: .
\]

Since this constraint mixes the two coordinates, we can't conclude by using the separation arguments.
To our knowledge, there is no closed form solution for the resulting optimization problem
\begin{align*}
\text{maximize } & \sum_{b \in \{1, 2\}} \frac{N_{b}}{2} \ln \left( 1 + y_b\right)
\\
\text{such that }\: & \forall b \in \{1, 2\}, \quad y_b \geq 0, \: x_b y_b \leq C_b, \: x_b \geq D_b \: ,\\
&\text{and} \quad \sum_{b \in \{1, 2\}} \frac{N_{b}}{2} f \left(x_b, y_b\right) \leq E \: .
\end{align*}
However, as a maximization of a concave function under linear and convex inequalities, we can solve it numerically.
Corollary~\ref{cor:delta_correct_Kinf_thresholds} is a direct consequence of the above manipulations and an union bound over concentration result.

\paragraph{EV-GLR stopping rule} The same ideas lead to Corollary~\ref{cor:delta_correct_Kinf_thresholds_ev}, which optimizes a different function under the same constraints. The resulting optimization problem is computationally faster since the objective is linear.
The proof of Corollary~\ref{cor:delta_correct_Kinf_thresholds_ev} is omitted since it is almost identical to the above.

\begin{corollary} \label{cor:delta_correct_Kinf_thresholds_ev}
	Let $f(x,y) = (1+y)x-1-\ln(x)$ for all $(x,y) \in (\Real_{+}^{\star})^2$. Let $(t_{a})_{a\in [K]}$ in (\ref{eq:def_initial_time_box_thresholds}), $\epsilon_{\mu},\epsilon_{-,\sigma}$ as in Lemma~\ref{lem:delta_correct_box_thresholds} and $(c_{b,a}^{\KL})_{b,a \in [K]}$ in (\ref{eq:def_complex_threshold_glrt}). The family of thresholds $c_{a,b}(N_t, \delta)$ with value $+ \infty$ if $t < \max_{c \in \{a,b\}} t_{c}(\delta/6)$ and otherwise solution of the optimization problem
 \begin{align*}
 \text{maximize } &  \frac{1}{2} \sum_{c \in \{a, b\}} N_{t,c} y_c
 \\
 \text{such that }\: & \forall c \in \{a, b\}, \quad y_c \geq 0, \: x_c y_c \leq \epsilon_{\mu}(N_{t,c}, \delta/2), \: x_c \geq 1 - \epsilon_{-,\sigma}(N_{t,c}-1, \delta/2) \: ,\\
 &\text{and} \quad  \frac{1}{2} \sum_{c \in \{a, b\}} N_{t,c} f \left(x_c, y_c\right) \leq c_{b,a}^{\KL}(N_t,\delta/2)
 \end{align*}
	yields a $\delta$-correct family of thresholds for the EV-GLR stopping rule.
\end{corollary}

\subsection{Asymptotically Tight Thresholds}
\label{app:ss_asymptotically_tight_thresholds}

Among the class of $\delta$-correct family of thresholds for the GLR stopping rule, Theorem~\ref{thm:upper_bound_sample_complexity_algorithm} suggests we should select the asymptotically tight ones since they yield an asymptotically optimal algorithm.

As observed empirically in Figure~\ref{fig:stopping_thresholds_evolutions_glrt}(a), the Student and the box thresholds are not asymptotically tight since their slope in $\ln \left( \frac{1}{\delta}\right)$ is higher than the ones of the KL and BoB thresholds. Theoretically, the arguments used in the proofs also justify why they couldn't reach the $\ln \left( \frac{1}{\delta}\right)$. For the Student threshold, it comes from the fact that we consider the inequality $\{(I) + (II) > 2\beta\} \subset \{(I) > \beta\} \cup  \{(II) > \beta\}$, hence we could at most reach $2 \ln \left( \frac{1}{\delta}\right)$. For the box threshold, it comes from the fact that we concentrate the term of each arm individually as well, hence we could also reach at most $2 \ln \left( \frac{1}{\delta}\right)$ (higher in practice).

\paragraph{EV-GLR stopping rule} Similar arguments could be used to study the $\ln \left( \frac{1}{\delta}\right)$-dependency of the families of thresholds derived for the EV-GLR stopping rule. However, since they are $\delta$-correct, they can't be asymptotically tight. Otherwise, the Theorem~\ref{thm:upper_bound_sample_complexity_algorithm} will contradict the asymptotic lower bound on the expected sample complexity as $T^\star_{\sigma^2}(\mu) < T^\star(\mu, \sigma^2)$ (Lemma~\ref{lem:complexity_inequalities}).

\paragraph{KL and BoB thresholds} In order to hope to achieve the $\ln \left( \frac{1}{\delta}\right)$ dependency one should concentrate the whole sum, as done for the KL threshold. Since $\ln \left( 1 + y\right) = \inf_{x>0} f \left(x, y\right)$, the KL constraint is also an upper bound on the solution of the optimization problem, i.e. the BoB thresholds are smaller than the KL thresholds $c_{b,a}^{\KL}(N_t,\delta/2)$, hence they inherit their $\ln \left( \frac{1}{\delta}\right)$ dependency. Therefore, to show the asymptotically tightness of both the KL and the BoB thresholds, we only need to study the KL family of thresholds. This result is established in Lemma~\ref{lem:kl_Kinf_thresholds_asymptotically_tight}.

\begin{lemma} \label{lem:kl_Kinf_thresholds_asymptotically_tight}
		The KL thresholds defined in (\ref{eq:def_complex_threshold_glrt}) is an asymptotically tight family of thresholds.
\end{lemma}
\begin{proof}
We analyze our threshold under a sampling rule that starts by pulling all arms $t(\delta) = \max\{t_0(\delta),t_1(\delta)\}$ times, where
\begin{align*}
t_1(\delta)
&\eqdef \min \left\{ t\in \Natural \mid
	t > e^{1 + W_{0} \left( \frac{2(1+\eta_{0})}{e}\left(\ln\left( \frac{12(K-1)\zeta(s)}{\delta} \right) + s\ln \left( 1+ \frac{\ln(t)}{\ln(1+\eta_{0})}\right) \right)  - \frac{1}{e}  \right)}\right\}
\: , \\
t_0(\delta) &\eqdef \max \left\{ \frac{e^{s / \ln\left( \frac{12(K-1)\zeta(s)}{\delta} \right)}  }{1+ \eta_0},  \frac{e^{s /\left( \ln\left( \frac{12(K-1)\zeta(s)}{\delta} \right) - \frac{1}{2(1+\eta_1)} \right)}  }{1+\eta_1}\right\} \: .
\end{align*}

After $K t(\delta)$ pulls, our threshold is finite for all arms.
We have $t_0(\delta) \to_{\delta \to 0} \frac{1}{1 + \min\{\eta_0, \eta_1\}}$.
Using that $W_0(x) \approx \ln(x)- \ln\ln(x)$ (Appendix~\ref{app:lambert_W_functions}), $t_1(\delta)$ is asymptotically equivalent to $\frac{2(1+\eta_{0})\ln (1/\delta)}{\ln \ln (1/\delta)}$.
Therefore, this initial pulling count $t(\delta)$ satisfies $\limsup_{\delta \to 0} t(\delta)/\log(1/\delta) = 0$.

Our threshold is in fact a family of threshold functions $(c_{a,b}(N_t, \delta))_{a,b \in [K]}$, which after that initialization, are defined by
\begin{align*}
	c_{a,b}(N_t, \delta) = 4\overline{W}_{-1}\left( 1 + \frac{\log\frac{2\zeta(s)^{2}}{\delta}}{4}
	+ \frac{s}{4} \sum_{c \in \{a,b\}} \log(1 + \log_\gamma N_{t,c})
	+ \frac{1}{2}\sum_{c \in \{a,b\}} \log\left(\gamma R_{t, c}(\delta) \right) \right) \: .
\end{align*}

Using concavity of $x \mapsto \ln(1 + \ln_{\gamma}(x))$ and $\sum_{c \in \{a,b\}} N_{t,c} \leq t$, we have
\begin{align*}
		c_{a,b}(N_t, \delta) \leq 4\overline{W}_{-1}\left( 1 + \frac{\log\frac{2\zeta(s)^{2}}{\delta}}{4}
		+ \frac{s}{2} \log\left(1 + \frac{\ln t/2}{\ln \gamma}\right)
		+ \ln \left( \gamma R(t,\delta) \right)  \right)\: ,
\end{align*}
where $R(t,\delta) = \max_{a \in [K]} R_{t, a}(\delta) $. Using the above, the concavity of $\overline{W}_{-1}$ (Lemma~\ref{lem:lambert_branches_properties}) yields for $t \ge K t_0(\delta)$
\begin{align*}
	c_{a,b}(N_t, \delta) &\leq  4 \overline{W}_{-1}\left(1+\frac{1}{4}\ln\left(\frac{2\zeta(s)^{2}}{\delta}\right)\right) \\
	&\quad + 4\left(\frac{s}{2} \log\left(1 + \frac{\ln t/2}{\ln \gamma}\right)
	+ \ln \left( \gamma R(t,\delta) \right)\right)\overline{W}_{-1}'\left(1+\frac{1}{4}\ln\left(\frac{2\zeta(s)^{2}}{\delta}\right)\right)
	\: .
\end{align*}
For all $x$ for which it is defined, $\overline{W}_{-1}'(x) = \frac{1}{1 - 1/\overline{W}_{-1}(x)}$. Let $c>1$, we have $\overline{W}_{-1}'(x) \leq c \iff \overline{W}_{-1}(x) \geq \frac{1}{1-\frac{1}{c}}$. Taking $x_{c}$ such that $\overline{W}_{-1}(x_{c}) = \frac{1}{1-\frac{1}{c}}$, which is possible since it is a strictly increasing function with values in $[1,+\infty)$, Lemma~\ref{lem:lambert_branches_properties} yields that $x_{c}\leq \frac{1}{1-\frac{1}{c}} + \ln\left(1-\frac{1}{c}\right)$. For $c = 2$, we obtain $\overline{W}_{-1}'(x) \leq 2$ for all $x \geq 2 -\ln(2)$, hence $\overline{W}_{-1}'\left(1+\frac{1}{4}\ln\left(\frac{2\zeta(s)^{2}}{\delta}\right)\right) \leq 2$ for all $\delta < \min\{1, \delta_{0}\}$ where $\delta_{0} = 2\zeta(s)^{2} e^{-4}$. Then, for all $\delta < \min\{1, \delta_{0}\}$,
\begin{align*}
c_{a,b}(N_t, \delta)
&\le 4 \overline{W}_{-1}\left(1+\frac{1}{4}\ln\left(\frac{2\zeta(s)^{2}}{\delta}\right)\right)+ 8\ln \left( \gamma R(t,\delta) \right) + 4s \log\left(1 + \frac{\ln t/2}{\ln \gamma}\right)
\: .
\end{align*}
We know that $4s \ln\left(1+\frac{\ln t}{\ln \gamma}\right) = O(t^\alpha)$ for some $\alpha \in (0,1)$ and, using $\overline{W}_{-1}(x) =_{+\infty} x+ \ln(x) + o(1)$ (Lemma~\ref{lem:lambert_branches_properties}), that $4 \overline{W}_{-1}\left(1+\frac{1}{4}\ln\left(\frac{2\zeta(s)^{2}}{\delta}\right)\right) = \ln \left( \frac{1}{\delta}\right) + o\left(\ln \left( \frac{1}{\delta}\right) \right)$. To conclude the proof, we need to show that $\limsup_{\delta \to 0} \frac{8\ln( R(t,\delta))}{\log(1/\delta)} \le 1$.

The term $R(t,\delta)$ is a data-dependent term quantifying the goodness of the local quadratic approximation for the KL (as a function of the natural parameters). Taking the notations from Theorem~\ref{thm:delta_correct_complex_threshold_glrt} and setting $\mu_{++,t,a}^2 = \max \mu_{\pm,t,a}^2$, we have
\begin{align*}
		R(t,\delta) = \max_{a \in [K]}\frac{\sigma_{-,t,a}^3 f_{-}\left( g(\sigma_{-,t,a}^2, \mu_{++,t,a}^2)\right)}{\sigma_{+,t,a}^3 f_{+}\left( g(\sigma_{+,t,a}^2, \mu_{++,t,a}^2)\right)}
\end{align*}
where $f_{\pm}(x) = \frac{1 \pm \sqrt{1 - x}}{\sqrt{x}}$ and $g(x,y) = \frac{2x}{(x+2y+\frac{1}{2})^2}$. Using the definition of $\varepsilon_{\mu}$ and $\varepsilon_{\pm,\sigma}$,we obtain directly that$ \lim_{t \to + \infty} \varepsilon_{\mu}(t,\delta) = 0$ and $\lim_{t \to + \infty}  \varepsilon_{\pm,\sigma} (t,\delta)= 0$.

Let $a \in [K]$. For all $t$, define $i_{t,a} = \lfloor \log_\gamma N_{t,a} \rfloor$, $n_{t,a} = \gamma^{i_{t,a}}$, $ \bar{t}_a = \inf \left\{ t \mid N_{t,a} = n_{t,a} \right\}$. Using the law of large number, we also have that $\lim_{N_{t,a} \to + \infty} \mu_{\bar t_a, a} = \mu_a$ and $\lim_{N_{t,a} \to + \infty} \sigma^2_{\bar t_a,a} = \sigma^2_{a}$. Chaining the limits, we obtain $\lim_{N_{t,a} \to + \infty} \mu_{++,t,a}^2 = \mu_a^2$ and $\lim_{N_{t,a} \to + \infty} \sigma_{\pm,t,a} = \sigma_{a}$. Since the functions $f_{\pm}$ and $g$ are continuous, we have shown for all $a \in [K]$
\begin{align*}
	\lim_{N_{t,a} \to + \infty}  \frac{\sigma_{-,t,a}^3 f_{-}\left( g(\sigma_{-,t,a}^2, \mu_{++,t,a}^2)\right)}{\sigma_{+,t,a}^3 f_{+}\left( g(\sigma_{+,t,a}^2, \mu_{++,t,a}^2)\right)} = \frac{f_{-}\left( g(\sigma_{a}^2, \mu_{a}^2)\right)}{f_{+}\left( g(\sigma_{a}^2, \mu_{a}^2)\right)}
\end{align*}

When $\delta \to 0$, the initialization yields $t_{0}(\delta) \to \infty$. Therefore, by using the above convergence and the fact that $N_{t,a}\geq t_{0}(\delta)$, there exists $\delta_1 \in (0, \min\{1, \delta_{0}\})$ such that for all $\delta \leq \delta_1$ and all $t \geq K t_{0}(\delta)$,
\begin{align*}
		R(t,\delta) \leq  2 \max_{a \in [K]} \frac{f_{-}\left( g(\sigma_{a}^2, \mu_{a}^2)\right)}{f_{+}\left( g(\sigma_{a}^2, \mu_{a}^2)\right)} \: ,
\end{align*}
which is a constant independent of $\delta$, hence $\limsup_{\delta \to 0} \frac{8\ln( R(t,\delta))}{\log(1/\delta)} \le 1$.
\end{proof}

%% file: sections/appendix_tas_optimality.tex

\section{Expected Sample Complexity} \label{app:expected_sample_complexity}

Theorem~\ref{thm:upper_bound_sample_complexity_algorithm} gives asymptotic upper bound on the expected sample complexity of the algorithms using the GLR stopping rule.
Theorem~\ref{thm:impossibility_result} gives an impossibility results for the algorithms using the EV-GLR stopping rule, which is based on first deriving an upper bound on the expected sample complexity.
In Appendix~\ref{app:expected_sample_complexity_tas}, we prove the upper bound for TaS with the GLR stopping rule and for EV-TaS with the EV-GLR stopping rule.
In Appendix~\ref{app:expected_sample_complexity_top_two}, we prove it for $\beta$-EB-TCI with the GLR stopping rule and for $\beta$-EB-EVTCI with the EV-GLR stopping rule.

\subsection{Wrapped Track-and-Stop} \label{app:expected_sample_complexity_tas}

Showing an asymptotic upper bound on the expected sample complexity of TaS and EV-TaS can be done with similar asymptotic arguments as when the variance is known \citep{garivier_2016_OptimalBestArm}.
The proof relies on two main ingredients.
First, a concentration result for $(\mu_{t}, \sigma^2_{t})$ is obtained thanks to the forced exploration (Lemma~\ref{lem:lemma_19_garivier_2016_OptimalBestArm} proved in Appendix~\ref{proof:lemma_19_garivier_2016_OptimalBestArm}).
Second, we use the Lemma 20 in \citet{garivier_2016_OptimalBestArm}, which ensures that the empirical allocation $\frac{N(t)}{t}$ converges towards the optimal allocation being targeted.
While it corresponds to $w^\star(\mu,\sigma^2)$ for TaS, it is $w^\star_{\sigma^2}(\mu)$ for EV-TaS (where the $\sigma^2$ dependency is hidden by notation).

Without loss of generality and for the sake of simpler notations, we assume that the Gaussian bandit model with parameter $(\mu, \sigma^2) \in (\Real \times \Real^\star_+)^{K}$ is such that $a^\star(\mu)=1$.

\paragraph{Track-and-Stop}
Since the proofs share the same structure as in \citet{garivier_2016_OptimalBestArm}, we detail the one for TaS and highlight the differences for EV-TaS later.

\begin{proof}
	Let $\epsilon>0$.
	From the continuity of $(\mu, \sigma^2) \mapsto w^\star(\mu, \sigma^2)$ \citep{degenne_2019_PureExplorationMultiple}, there exists $\xi_1 = \xi_1(\epsilon) \leq \frac{\min_{a \neq 1}(\mu_1 - \mu_a)}{4}$, $\xi_2 = \xi_2(\epsilon) > 1$ and $\xi_3 = \xi_3(\epsilon) \in (0, 1)$ such that
	$$
	\mathcal{I}_{\epsilon} \eqdef \left(\bigtimes_{a \in [K]} \left[\mu_{a}-\xi_1, \mu_{a}+\xi_1\right] \right) \times \left(\bigtimes_{a \in [K]}\left[\sigma^2_{a}\xi_3, \sigma^2_{a} \xi_2\right] \right)
	$$
	satisfies $\max _{a}\left|w^\star_a\left(\tilde{\mu}, \tilde{\sigma}^2\right)-w^\star_a(\mu, \sigma^2)\right| \leq \epsilon
	$ for all $(\tilde{\mu}, \tilde{\sigma}^2) \in \mathcal{I}_{\epsilon}$.
	Since $a^\star(\tilde\mu)=1$ for all $(\tilde \mu, \tilde \sigma^2) \in \mathcal{I}_{\epsilon}$, the empirical best arm is $\hat{a}_{t}=1$ whenever $(\mu_{t}, \sigma^2_{t}) \in \mathcal{I}_{\epsilon}$.
	Let $T \in \mathbb{N}$, $h(T) \eqdef T^{1 / 4}$ and define the concentration event $\mathcal{E}_{T}=\bigcap_{t=h(T)}^{T}\left\{(\mu_{t}, \sigma^2_{t}) \in \mathcal{I}_{\epsilon}\right\}$.

The forced exploration ensures that each arm is drawn at least of order $\sqrt{t}$ times at round $t$.
Thanks to concentration results on both $(\mu_{t}, \sigma^2_{t})$, Lemma~\ref{lem:lemma_19_garivier_2016_OptimalBestArm} upper bounds $\mathbb{P}_{\nu}\left(\mathcal{E}_{T}^{\complement}\right)$.

	\begin{lemma} \label{lem:lemma_19_garivier_2016_OptimalBestArm}
	Let $T$ such that $h(T)> \left( K+\frac{1}{1-\xi_{3}}\right)^2$. There exist two constants $B, C$ (that depend on $(\mu, \sigma^2)$ and $\epsilon$) such that
	$	\mathbb{P}_{\nu}\left(\mathcal{E}_{T}^{\complement}\right) \leq B T \exp \left(-C T^{1 / 8}\right)	$.
	\end{lemma}

	Lemma~\ref{lem:tas_lemma_20_garivier_2016_OptimalBestArm} is exactly Lemma 20 in \citet{garivier_2016_OptimalBestArm}, hence the proof is omitted.

	\begin{lemma}[Lemma 20 in \citet{garivier_2016_OptimalBestArm}] \label{lem:tas_lemma_20_garivier_2016_OptimalBestArm}
		There exists a constant $T_{\epsilon}$ such that for $T \geq T_{\epsilon}$, it holds that on $\mathcal{E}_{T}$, for either $C$ Tracking or D-Tracking,
	$$
	\forall t \geq \sqrt{T}, \quad \max_{a\in [K]}\left|\frac{N_{t,a}}{t}-w^\star_a(\mu, \sigma^2)\right| \leq 3(K-1) \epsilon
	$$
	\end{lemma}

	On the event $\mathcal{E}_{T}$, it holds for $t \geq h(T)$ that $\hat{a}_{t}=1$ and the GLR rewrites
	\begin{align*}
	\min _{a \neq 1} Z_{a}(t)&=\min _{a \neq 1} \inf_{\lambda \in [\mu_{t,a},\mu_{t,1}]} \sum_{b \in \{1,a\}} \frac{N_{t,b}}{2} \ln \left( 1+ \frac{\left(\mu_{t,b} - \lambda\right)^2}{\sigma_{t,b}^2}\right) =t g\left(\mu_{t}, \sigma^2_{t},\frac{N_{t}}{t}\right)
	\end{align*}
	where, for $(\tilde{\mu}, \tilde{\sigma}^2) \in (\Real \times \Real^\star_+)^{K}$ such that $a^{\star}(\tilde{\mu}) = 1$ and $\tilde w \in \simplex$, we introduced the function
	\begin{align*}
		g(\tilde{\mu}, \tilde{\sigma}^2, \tilde w) = \min_{a \neq 1} \inf_{\lambda \in [\tilde{\mu}^{a},\tilde{\mu}^{1}]} \sum_{b \in \{1,a\}} \frac{\tilde w_b}{2} \ln \left( 1+ \frac{\left(\tilde\mu_{b} - \lambda\right)^2}{\tilde\sigma_{b}^2}\right)
	\end{align*}

	Recall that $a^\star(\tilde \mu)=1$ on $\mathcal{I}_{\epsilon}$. We introduce $C_\epsilon^\star(\mu, \sigma^2) = \underset{(\tilde{\mu}, \tilde{\sigma}^2,\tilde w ) \in \cH_{\epsilon}(\mu, \sigma^2) }{\inf } g\left(\tilde{\mu}, \tilde{\sigma}^2,\tilde w\right)$, where
	\[
	\cH_{\epsilon}(\mu, \sigma^2) = \mathcal{I}_{\epsilon} \times \left\{ \tilde w \in \simplex \mid \left\|\tilde w -w^\star(\mu, \sigma^2)\right\|_{\infty} \leq 2(K-1) \epsilon\right\} \: .
	\]
	Using Lemma~\ref{lem:tas_lemma_20_garivier_2016_OptimalBestArm}, for $T \geq T_{\epsilon}$,
	on the event $\mathcal{E}_{T}$ it holds that for every $t \geq \sqrt{T}$,	$\min _{a \neq 1} Z_{a}(t) \geq t C_\epsilon^\star(\mu, \sigma^2)$.

Let $\alpha \in [0,1)$, $\delta_0 \in (0,1]$, functions $f,\bar{T} : (0,1] \to \mathbb{R}_+$ and $C$ as in the definition of an asymptotically tight family of thresholds.
In the following, we consider $\delta \leq \delta_{0}$ and $T \geq \left\{\bar{T}(\delta), T_{\epsilon}\right\}$. Using that for all $a\neq 1$, $Z_{a}(t) \geq \min_{a \neq 1} Z_{a}(t) \geq t C_\epsilon^\star(\mu, \sigma^2)$ and $c_{1,a}(N_t, \delta) \leq  f(\delta) + C t^\alpha \leq  f(\delta) + C T^\alpha$, we obtain on $\mathcal{E}_{T}$,
	\begin{align*}
	\min \left\{\taud, T\right\} & \leq \sqrt{T}+\sum_{t=\sqrt{T}}^{T} \1\{\tau_{\delta}>t\}  \leq \sqrt{T}+\sum_{t=\sqrt{T}}^{T} \1\{\exists a \neq 1, \: Z_{a}(t) \leq c_{1,a}(N_t, \delta)\} \\
	&\leq \sqrt{T}+\sum_{t=\sqrt{T}}^{T} \1\{t C_\epsilon^\star(\mu, \sigma^2) \leq  f(\delta) + C T^\alpha\} \leq \sqrt{T}+\frac{ f(\delta) + C T^\alpha}{C_\epsilon^\star(\mu, \sigma^2)}
	\end{align*}

	Introducing
	$$
	T_{0}(\delta)=\inf \left\{T \ge \bar{T}(\delta) : \sqrt{T}+\frac{f(\delta) + C T^\alpha}{C_\epsilon^\star(\mu, \sigma^2)} \leq T\right\},
	$$
	for every $T \geq \max \left\{T_{0}(\delta), \bar{T}(\delta), T_{\epsilon}\right\}$, one has $\mathcal{E}_{T} \subseteq \left\{\taud \leq T\right\}$, therefore, by using Lemma~\ref{lem:lemma_19_garivier_2016_OptimalBestArm},
	$$
	\mathbb{P}_{\nu}\left(\taud>T\right) \leq \mathbb{P}_{\nu}\left(\mathcal{E}_{T}^{c}\right) \leq B T \exp \left(-C T^{1 / 8}\right)
	$$
	and
	$$
	\mathbb{E}_{\nu}\left[\taud\right] \leq T_{0}(\delta)+\bar{T}(\delta)+T_{\epsilon}+\sum_{T=1}^{\infty} B T \exp \left(-C T^{1 / 8}\right) .
	$$
	We now provide an upper bound on $T_{0}(\delta)$. Letting $\eta>0$ and introducing the constant
	\begin{align*}
	D(\eta)
	&\eqdef \inf \{T \in \mathbb{N}: T-\sqrt{T} \geq T /(1+\eta)\}
	\le 1 + \left(1 + \frac{1}{\eta} \right)^2
	\end{align*}
	one has
	\begin{align*}
	T_{0}(\delta)
	& \le D(\eta)+\inf \left\{T \ge \bar{T}(\delta) \mid f(\delta) + C T^\alpha \le T \frac{C_\epsilon^\star(\mu, \sigma^2)}{1+\eta}\right\}
	\: .
	\end{align*}
	For all $\gamma > 0$, there exists $T_{\alpha, \gamma}$ (depending on $\mu, \sigma^2$) such that for all $T \ge T_{\alpha, \gamma}$,
	\[
	T \frac{C_\epsilon^\star(\mu, \sigma^2)}{1+\eta} - C T^\alpha \ge T \frac{C_\epsilon^\star(\mu, \sigma^2)}{(1+\eta)(1+\gamma)} \: .
	\]
	Then,
	\begin{align*}
	T_{0}(\delta)
	& \le D(\eta)+ \bar{T}(\delta) + T_{\alpha, \gamma} + \inf \left\{T \mid f(\delta)\le T \frac{C_\epsilon^\star(\mu, \sigma^2)}{(1+\eta)(1+\gamma)}\right\}
	\\
	&\le 1 + D(\eta)+ \bar{T}(\delta) + T_{\alpha, \gamma} + \frac{(1+\eta)(1+\gamma)}{C_\epsilon^\star(\mu, \sigma^2)} f(\delta)
	\: .
	\end{align*}
	Dividing by $\log(1/\delta)$ and taking limits, we get
	\begin{align*}
	\limsup_{\delta \to 0}\frac{T_{0}(\delta)}{\log(1/\delta)}
	&\le \frac{(1+\eta)(1+\gamma)}{C_\epsilon^\star(\mu, \sigma^2)}
	\: .
	\end{align*}

	We obtain that for every $\eta, \epsilon, \gamma >0$,
	\begin{align*}
		\limsup _{\delta \to 0} \frac{\mathbb{E}_{\nu}\left[\tau_{\delta}\right]}{\log (1 / \delta)}
		\leq \limsup_{\delta \to 0} \frac{T_{0}(\delta)}{\log (1 / \delta)}
		\le \frac{(1+\eta)(1+\gamma)}{C_\epsilon^\star(\mu, \sigma^2)}
		\: .
	\end{align*}
By continuity of $g$ and by definition of $w^\star(\mu, \sigma^2)$, we obtain $\lim_{\epsilon \to 0}C_\epsilon^\star(\mu, \sigma^2) = T^\star(\mu, \sigma^2)^{-1}$.
Letting $\eta$ and $\epsilon$ go to zero yields
	\begin{align*}
		\limsup _{\delta \to 0} \frac{\mathbb{E}_{\nu}\left[\taud\right]}{\log (1 / \delta)} \leq T^\star(\mu, \sigma^2)
		\: .
	\end{align*}
\end{proof}

\paragraph{Empirical Variance Track-and-Stop}
We only highlight some differences since the proofs are similar.
We emphasize here that the notations $w^\star_{\sigma^2}(\mu)$, $T^\star_{\sigma^2}(\mu)$ and all the ones we will use below hide the dependency in $\sigma^2$ to distinguish this from the quantity defined when the variance is assumed to be unknown.
While this might be unfortunate, we believe it eases greatly the notations and the highlight the difference between the two complexities.

\begin{lemma} \label{lem:upper_bound_ev_tas}
	Using the EV-GLR stopping rule with an asymptotically tight family of thresholds, EV-TaS satisfies that, for all $\nu$ with $|a^\star(\mu)|=1$,
	\[
	 \limsup_{\delta \rightarrow 0} \frac{\mathbb{E}_{\nu}\left[\taud^{\text{EV}}\right]}{\log (1 / \delta)} \le T^\star_{\sigma^2}(\mu) \: .
	\]
\end{lemma}
\begin{proof}
	Let $\epsilon>0$.
	Similarly, there exists $\xi_{1} \le \frac{\min_{a \neq 1}(\mu_{1}-\mu^{a})}{4}$, $\xi_2 > 1$, $\xi_3 \in (0, 1)$ such that $\mathcal{I}_{\epsilon}$ defined as above satisfies $\max _{a}\left|w^\star_a\left(\tilde{\mu}\right)-w^\star_a(\mu)\right| \leq \epsilon$ for all $(\tilde{\mu}, \tilde{\sigma}^2) \in \mathcal{I}_{\epsilon}$.
	Let $T \in \mathbb{N}$, $h(T):=T^{1 / 4}$ and define the concentration event $\mathcal{E}_{T}$ as above.
	Since Lemma~\ref{lem:lemma_19_garivier_2016_OptimalBestArm} relies solely on forced exploration, its result still hold in that case.
	A tracking result similar to Lemma~\ref{lem:tas_lemma_20_garivier_2016_OptimalBestArm} gives the existence of $T_{\epsilon}$ such that for $T \geq T_{\epsilon}$, it holds that on $\mathcal{E}_{T}$,
$$
\forall t \geq \sqrt{T}, \quad \max_{a\in [K]}\left|\frac{N_{t,a}}{t}-w^\star_a(\mu)\right| \leq 3(K-1) \epsilon
$$

	On the event $\mathcal{E}_{T}$, it holds for $t \geq h(T)$ that $\hat{a}_{t}=1$ and the EV-GLR rewrites
	\begin{align*}
	\min _{a \neq 1} Z^{\text{EV}}_{a}(t)&=\min _{a \neq 1} \inf_{\lambda \in [\mu_{t,a},\mu_{t,1}]} \sum_{b \in \{1,a\}} N_{t,b} \frac{\left(\mu_{t,b} - \lambda\right)^2}{2\sigma_{t,b}^2} =t g\left(\mu_{t}, \sigma^2_{t},\frac{N_{t}}{t}\right)
	\end{align*}
	where, for $(\tilde{\mu}, \tilde{\sigma}^2) \in (\Real \times \Real^\star_+)^{K}$ such that $a^{\star}(\tilde{\mu}) = 1$ and $\tilde w \in \simplex$, we introduced the function
	\begin{align*}
		g(\tilde{\mu}, \tilde{\sigma}^2, \tilde w) = \min_{a \neq 1} \inf_{\lambda \in [\tilde{\mu}^{a},\tilde{\mu}^{1}]} \sum_{b \in \{1,a\}} \tilde w_b\frac{\left(\tilde\mu_{b} - \lambda\right)^2}{2\tilde\sigma_{b}^2}
	\end{align*}

	We introduce $C_\epsilon^\star(\mu, \sigma^2) = \underset{(\tilde{\mu}, \tilde{\sigma}^2,\tilde w ) \in \cH_{\epsilon}(\mu, \sigma^2) }{\inf } g\left(\tilde{\mu}, \tilde{\sigma}^2,\tilde w\right)$, where
	\[
	\cH_{\epsilon}(\mu, \sigma^2) = \mathcal{I}_{\epsilon} \times \left\{ \tilde w \in \simplex \mid \left\|\tilde w -w^\star_{\sigma^2}(\mu)\right\|_{\infty} \leq 2(K-1) \epsilon\right\} \: .
	\]

	By tracking, for $T \geq T_{\epsilon}$, on the event $\mathcal{E}_{T}$ it holds that for every $t \geq \sqrt{T}$,
	$	\min _{a \neq 1} Z^{\text{EV}}_{a}(t) \geq t C_\epsilon^\star(\mu, \sigma^2)$.
	Let $T \geq \max \left\{ \bar{T}(\delta),T_{\epsilon} \right\}$.
	As above, on $\mathcal{E}_{T}$, we have $\min \left\{\taud^{\text{EV}}, T\right\}  \leq \sqrt{T}+\frac{f(\delta) + C T^\alpha}{C_\epsilon^\star(\mu, \sigma^2)}$ and we introduce $T_{0}(\delta)=\inf \left\{T \ge \bar{T}(\delta) : \sqrt{T}+\frac{f(\delta) + C T^\alpha}{C_\epsilon^\star(\mu, \sigma^2)} \leq T\right\}$.
	For every $T \geq \left\{ T_{0}(\delta), \bar{T}(\delta), T_{\epsilon} \right\}$, one has $\mathcal{E}_{T} \subseteq \left\{\taud^{\text{EV}} \leq T\right\}$, therefore, by using Lemma~\ref{lem:lemma_19_garivier_2016_OptimalBestArm},
	$$
	\mathbb{E}_{\nu}\left[\taud^{\text{EV}}\right] \leq T_{0}(\delta)+\bar{T}(\delta)+T_{\epsilon}+\sum_{T=1}^{\infty} B T \exp \left(-C T^{1 / 8}\right) \: .
	$$
	Manipulations similar as above yield an upper bound on $T_0(\delta)$.
	By continuity of $g$ and by definition of $w^\star_{\sigma^2}(\mu)$, we obtain $\lim_{\epsilon \rightarrow 0}C_\epsilon^\star(\mu, \sigma^2) = T^\star_{\sigma^2}(\mu)^{-1}$.
	Letting $\eta$ and $\epsilon$ go to zero yields
	\begin{align*}
		\limsup_{\delta \rightarrow 0} \frac{\mathbb{E}_{\nu}\left[\taud^{\text{EV}}\right]}{\log (1 / \delta)} \leq T^\star_{\sigma^2}(\mu)
		\: .
	\end{align*}
\end{proof}

We are now ready to prove the impossibility result for EV-TaS (Theorem~\ref{thm:impossibility_result}).
\begin{proof}
	Let $(c_{a,b})_{(a,b) \in [K]^2}$ be an asymptotically tight family of thresholds and a problem independent constant $c_0 > 0$.
	Combining EV-TaS with the EV-GLR stopping rule using $(c_0 c_{a,b})_{(a,b) \in [K]^2}$ yields an algorithm such that, for all $\nu$ with $|a^\star(\mu)|=1$,
	\[
	 \limsup_{\delta \rightarrow 0} \frac{\mathbb{E}_{\nu}\left[\taud^{\text{EV}}\right]}{\log (1 / \delta)} \le c_0 T^\star_{\sigma^2}(\mu) \: .
	\]
	Lemma~\ref{lem:upper_bound_ev_tas} shows the above result for $c_0 = 1$, and generalizing to $c_0 > 0$ is direct.
	Suppose towards contradiction that the obtained algorithm is $\delta$-correct.
	Therefore, using Lemmas~\ref{lem:lower_bound_sample_complexity} and~\ref{lem:complexity_inequalities}, we have shown that
	\[
		 T^\star(\mu, \sigma^2) \le \liminf _{\delta \rightarrow 0} \frac{\mathbb{E}_{\nu}\left[\taud^{\text{EV}}\right]}{\log (1 / \delta)}  \le \limsup_{\delta \rightarrow 0} \frac{\mathbb{E}_{\nu}\left[\taud^{\text{EV}}\right]}{\log (1 / \delta)} \leq c_0 T^\star_{\sigma^2}(\mu) < c_0 T^\star_{\beta}(\mu, \sigma^2) \: .
	\]
 For $c_0 \in (0,1)$, the contradiction is direct.
 For $c_0 \ge 1$, we have shown that there is a problem independent constant $c_0 > 0$ such that $T^\star(\mu, \sigma^2) / T^\star_{\sigma^2}(\mu) \le c_0$.
 This is a direct contradiction with Lemma~\ref{lem:ratio_characteristic_time_large} showing that there exists a sequence of instances $(\nu_n)_{n \in \Natural}$ with $|a^\star(\nu_n)| = 1$ such that $\lim_{n \to + \infty} T^\star(\mu_n, \sigma_n^2) / T^\star_{\sigma_n^2}(\mu_n) = + \infty$.
 Therefore, the obtained algorithm is not $\delta$-correct.
\end{proof}

\subsubsection{Proof of Lemma~\ref{lem:lemma_19_garivier_2016_OptimalBestArm}} \label{proof:lemma_19_garivier_2016_OptimalBestArm}

Since $T$ is such that $h(T)> \left( K+\frac{1}{1-\xi_{3}}\right)^2$, we have $h(T) \geq K^{2}$ and $\xi_{3} < 1-\frac{1}{\sqrt{h(T)}- K}$. In particular, for all $t \in \{h(T),\cdots, T\}$ and all $s \in \{\sqrt{t} - K, \cdots, t\}$, $\xi_{3} + \frac{1}{s} < 1$. Then,
\begin{align*}
	\mathbb{P}\left(\mathcal{E}_{T}^{c}\right) &\leq \sum_{t=h(T)}^{T} \sum_{a=1}^{K}\left[\mathbb{P}\left(|\mu_{t, a} -  \mu_{a}| \ge \xi_{1}\right) + \mathbb{P}\left(\sigma^{2}_{a,t} \geq \sigma^2_{a} \xi_{2}\right) + \mathbb{P}\left(\sigma^{2}_{a,t} \leq \frac{\sigma^2_{a}}{ \xi_{2}} \right)\right]
\end{align*}

By forced exploration, for $t \geq h(T)$ one has $N_{t,a} \geq(\sqrt{t}-K / 2)_{+}-1 \geq \sqrt{t}-K$ for every arm $a$. In the proof of Lemma 19 in \citet{garivier_2016_OptimalBestArm} it was shown using a union bound over time and Lemma~\ref{lem:fixed_time_upper_lower_concentration_mean} that, for all $t \in [h(T), T]$,
\begin{align*}
	\mathbb{P}\left(|\mu_{t, a} -  \mu_{a}| \ge \xi_{1}\right) &\leq \frac{2e^{-(\sqrt{t}-K) \frac{\xi_{1}^2}{2\sigma_{a}^2}}}{1-e^{-\frac{\xi_{1}^2}{2\sigma_{a}^2}}} \: .
\end{align*}
Since we use the same method to show our result on the variance, the proof for the mean is omitted.

Let $\hat{\sigma}^{2}_{a,s}$ be the empirical variance of the first $s$ reward from arm $a$ (such that $\hat{\sigma}^{2}_{a,N_{t,a}} = \sigma^{2}_{a,t}$). We adopt their proof strategy to derive the equivalent upper bound on the concentration of the variance.
\begin{align*}
	\mathbb{P}\left(\sigma^{2}_{a,t} \geq \sigma^2_{a} \xi_{2}\right)  &=  \mathbb{P}\left(\sigma^{2}_{a,t} \geq \sigma^2_{a} \xi_{2}, N_{t,a} \geq \sqrt{t} - K\right) \\
	&\leq \sum_{s = \sqrt{t}-K-1}^{t-1} \mathbb{P}\left(\hat{\sigma}^{2}_{a,s+1} \geq \sigma^2_{a} \xi_{2}\right)\leq \sum_{s = \sqrt{t}-K-1}^{t-1} \exp \left( -\frac{s}{2} \left( h\left(\xi_{2}+ \frac{1}{s}\right) - 1\right)\right)
\end{align*}
The first inequality is obtained by taking a union bound over the values of $N_{t,a} \in [\sqrt{t}-K, t]$. The second inequality is obtained by Corollary~\ref{cor:fixed_time_upper_lower_concentration_variance} with $x = \xi_{2} + \frac{1}{s} > 1$. Using that $h(x) = x- \ln(x)$ and $\ln(1+x)\leq x$, direct computations yield that
\begin{align*}
	&sh\left(\xi_{2}+ \frac{1}{s}\right) = s \xi_{2} + 1 - s \ln(\xi_{2}+ \frac{1}{s}) = sh(\xi_{2}) + 1 -s \ln(1+\frac{1}{s\xi_{2}}) \geq sh(\xi_{2}) + 1 -\frac{1}{\xi_{2}}  \: ,\\
	& \sum_{s = \sqrt{t}-K-1}^{t-1} \left(e^{-\frac{1}{2}\left(h(\xi_{2}) -1 \right)} \right)^s\leq \frac{1}{1-e^{-\frac{1}{2}\left(h(\xi_{2}) -1 \right)}} e^{-\frac{(\sqrt{t}-K-1)}{2}\left(h(\xi_{2}) -1 \right)} \: .
\end{align*}
Putting those together, we obtain, for all $t \in [h(T), T]$,
\begin{align*}
	\mathbb{P}\left(\sigma^{2}_{a,t} \geq \sigma^2_{a} \xi_{2}\right)  & \leq \frac{e^{-\frac{1}{2}\left(1 -\frac{1}{\xi_{2}} \right)}}{1-e^{-\frac{1}{2}\left(h(\xi_{2}) -1 \right)}} e^{-\frac{(\sqrt{t}-K-1)}{2}\left(h(\xi_{2}) -1 \right)} \: .
\end{align*}

The same manipulations using Corollary~\ref{cor:fixed_time_upper_lower_concentration_variance} with $x = \xi_{3} + \frac{1}{s} \in (0,1)$ for all $s\in [\sqrt{t}-K, t]$ and all $t \in [h(T), T]$ (see above by choice of $T$), yield that,  for all $t \in [h(T), T]$,
\begin{align*}
	\mathbb{P}\left(\sigma^{2}_{a,t} \leq \sigma^2_{a} \xi_{3} \right)  & \leq \frac{e^{-\frac{1}{2}\left(1 -\frac{1}{\xi_{3}} \right)}}{1-e^{-\frac{1}{2}\left(h(\xi_{3}) -1 \right)}} e^{-\frac{(\sqrt{t}-K-1)}{2}\left(h(\xi_{3}) -1 \right)} \: .
\end{align*}

Finally, letting
\begin{align*}
	C &= \frac{1}{2} \min \left\{ \frac{\xi_{1}^2}{\max_{a \in [K]}\sigma_{a}^2}, h\left(\xi_{2}\right) -1, h\left(\xi_{3}\right) -1 \right\} \quad \text{and}\\
	B &= \sum_{a=1}^{K}\left(\frac{2e^{K\frac{\xi_{1}^2}{2\sigma_{a}^2}}}{1-e^{-\frac{\xi_{1}^2}{2\sigma_{a}^2}}} + \frac{e^{\frac{K-1}{2}\left(h(\xi_{2}) -1 \right)-\frac{1}{2}\left(1 -\frac{1}{\xi_{2}} \right)}}{1-e^{-\frac{1}{2}\left(h(\xi_{2}) -1 \right)}} + \frac{e^{\frac{K-1}{2}\left(h(\xi_{3}) -1 \right)-\frac{1}{2}\left(1 -\frac{1}{\xi_{3}} \right)}}{1-e^{-\frac{1}{2}\left(h(\xi_{3}) -1 \right)}} \right) \: ,
\end{align*}
one obtains
$$
\mathbb{P}\left(\mathcal{E}_{T}^{c}\right) \leq \sum_{t=h(T)}^{T} B \exp (-\sqrt{t} C) \leq B T \exp (-\sqrt{h(T)} C)=B T \exp \left(-C T^{1 / 8}\right) \: .
$$

\subsection{Wrapped $\beta$-EB-TCI} \label{app:expected_sample_complexity_top_two}

Showing an asymptotic upper bound on the expected sample complexity of $\beta$-EB-TCI and $\beta$-EB-EVTCI can be done with similar asymptotic arguments as when the variance is known \citep{Shang20TTTS}.
We will use the unified analysis of Top Two algorithms introduced in \citet{jourdan_2022_TopTwoAlgorithms}, which highlights simple properties that the leader and the challenger should satisfy to obtain the desired upper bound.
While they introduced it for single-parameter exponential families and bounded distributions, it also allows to cope for our two-parameters setting.
The proof is composed of three steps: showing sufficient exploration of all arms, proving that the expectation of the convergence time towards the $\beta$-optimal allocation is finite and then concluding on the asymptotic upper bound.

To ensure sufficient exploration, we assume that $\min_{a \neq b}|\mu_a - \mu_b| > 0$, i.e. all the arms have distinct means.
To our knowledge, this assumption is shared by all Top Two algorithms analysis.

\subsubsection{$\beta$-EB-TCI}

Let $a^\star = a^\star(\mu)$. Since the proofs share the same structure as in \citet{jourdan_2022_TopTwoAlgorithms}, we detail the one for $\beta$-EB-TCI and highlight the differences for $\beta$-EB-EVTCI.
The $\beta$-optimal allocation for Gaussian with unknown variance is defined as
\begin{align*}
	&w_{\beta}^\star(\mu, \sigma^2)
	= \argmax_{w \in \triangle_K, w_{a^\star} = \beta} \min_{a \neq a^\star} \inf_{ u \in \Real } \sum_{b \in \{a^\star,a\}} w_b \log \left( 1 + \frac{(\mu_b - u )^2}{\sigma_b^2}\right)
	\: .
\end{align*}
In \citet{Russo2016TTTS} and \citet{jourdan_2022_TopTwoAlgorithms}, they show the $\beta$-optimal allocation was unique for any single-parameter exponential families and for bounded distributions.
Since the proof only relies on strict convexity already shown in Appendix~\ref{app:ss_optimal_allocation_oracles}, it is straightforward to see that $w_{\beta}^\star(\mu, \sigma^2)$ is a singleton (Property 1 in \citet{jourdan_2022_TopTwoAlgorithms}), denoted by $\{w^{\beta}\}$, such that $\min_{a \in [K]} w^{\beta}_{a} > 0$.

Before delving into the specifics of the proof, Lemma~\ref{lem:W_concentration_gaussian} gathers concentration results on which the subsequent analysis heavily relies on.
\begin{lemma}\label{lem:W_concentration_gaussian}
There exists a sub-Gaussian random variable $W_{\mu}$ and a sub-exponential random variable $W_{\sigma}$, which are independent, such that almost surely for all $a \in [K]$ and all $t$ such that $N_{t,a} \ge 2$,
\[
N_{t,a} |\mu_{t,a} - \mu_a| \le W_\mu \log (e + N_{t,a})  \quad \text{and} \quad
\left|N_{t,a} \left(\frac{\sigma^2_{t,a}}{\sigma^2_{a}} -1\right) + 1\right| \le W_{\sigma} \log (e + N_{t,a}) \: .
\]
There exists a random variable $W_{0} = \max_{a \in [K]} W_a^{-1}$, independent of $W_{\mu}$, where $W_a /\sigma_{a}^2 \sim \chi^2(n_0 -1)$ such that almost surely, for all $a \in [K]$ and all $t$ such that $N_{t,a} \ge n_0$, $N_{t,a} \sigma^2_{t,a} \ge W_0^{-1}$.
In particular, $W_0$ admits a finite mean for $n_0 \ge 4$ and a finite variance for $n_0 \ge 6$, and any random variable which is polynomial in both $W_\sigma$ and $W_{\mu}$ has a finite expectation.
\end{lemma}
\begin{proof}
	The existence of $W_{\mu}$ is given by Lemma 73 in \citet{jourdan_2022_TopTwoAlgorithms}.

	Using Lemma~\ref{lem:variance_is_subExp2}, we have for all $a \in [K]$ and $t$ such that $N_{t,a}\ge 2$, $N_{t,a} \left(\frac{\sigma^2_{t,a}}{\sigma_{a}^2} - 1 \right) + 1 = S_{N_{t,a}-1,a}$,	where $S_{n,a}$ is sub-exponential with $c=2$ and $-S_{n,a}$ is sub-exponential with $c=-2$ and $S_{n-1,a} + n - 1 =  \sum_{s \in [n-1]} Y_{s,a}^2$ is such that $Y_{s,a}$ are i.i.d. $\cN(0,1)$.
	Using Lemma 72 in \citet{jourdan_2022_TopTwoAlgorithms}, $\sup_{N_{t,a} \ge 2 } \frac{|S_{N_{t,a}-1,a}|}{\ln(e+N_{t,a})}$ is also sub-exponential.
	By defining $W_{\sigma}$ as the maximum over the finitely many arms of the above quantity, we obtain that $W_{\sigma}$ is sub-exponential.

	Since the increments $Y_{s,a}^2$ are positive, we have for all $a \in [K]$ and $t$ such that $N_{t,a}\ge n_0$
	\[
	\frac{\sigma^2_{t,a}}{\sigma^2_a}  = \frac{S_{N_{t,a}-1,a} + N_{t,a} - 1 }{N_{t,a}}   = \frac{\sum_{s \in [N_{t,a}-1]} Y_{s,a}^2}{N_{t,a}} \ge \frac{Z_a}{N_{t,a}} \: ,
	\]
	where $Z_a = \sum_{s \in [n_0-1]} Y_{s,a}^2 \sim \chi^2(n_0 - 1)$. Taking $W_a = \sigma^2_a Z_a$ and the maximum over $[K]$ yields the result.
	As maximum of finitely many inverse-chi-squared distributions, $W_0$ admits a finite mean for $n_0 \ge 4$ and a finite variance for $n_0 \ge 6$.

	For the independence property between $W_\sigma$ and $W_{\mu}$ and between $W_{0}$ and $W_{\mu}$, this is a direct consequence of the fact that the empirical mean and empirical variance are independent.
	Therefore, any random variable which is polynomial in both $W_\sigma$ and $W_{\mu}$ has a finite expectation.
\end{proof}

Using Lemma~\ref{lem:W_concentration_gaussian}, it is direct to obtain that, for all $\epsilon > 0$, there exists $N_{\epsilon}$ with $\mathbb{E}_{\nu}[N_{\epsilon}] < + \infty$ such that for all $a \in [K]$ and $t$ such that $N_{t,a} \ge N_{\epsilon}$, $	| \mu_{t,a} - \mu_a | \le \epsilon$ and $\left| \frac{\sigma^2_{t,a}}{\sigma^2_{a}} - 1\right| \le \epsilon$.

As in \cite{Qin2017TTEI,Shang20TTTS,jourdan_2022_TopTwoAlgorithms}, our goal is to upper bound the expectation of the \emph{convergence time}.
For $\varepsilon> 0$, the random variable $T^{\varepsilon}_{\beta}$ quantifies the number of samples required for the empirical allocations $\frac{N_{t}}{t}$ to be $\epsilon$-close to $w^{\beta}$:
\begin{equation} \label{eq:rv_T_eps_beta}
	T^{\epsilon}_{\beta} \eqdef \inf \left\{ T \ge 1 \mid \forall t \geq T, \: \left\| \frac{N_{t}}{t} -w^{\beta} \right\|_{\infty} \leq \epsilon \right\}  \: .
\end{equation}

Lemma~\ref{lem:asymptotic_optimality_top_two_algorithms} shows that a sufficient condition for asymptotic $\beta$-optimality is to show $\mathbb{E}_{\nu}[T^{\epsilon}_{\beta}] < + \infty$ for all $\epsilon$ small enough.
\begin{lemma} \label{lem:asymptotic_optimality_top_two_algorithms}
	Let $(\delta, \beta) \in (0,1)^2$.
	Assume that there exists $\epsilon_1(\mu, \sigma^2)> 0$ such that for all $\epsilon \in (0,\epsilon_1(\mu, \sigma^2)]$, $\mathbb{E}_{\nu}[T^{\epsilon}_{\beta}] < + \infty$.
	Combining the GLR stopping rule \eqref{eq:def_stopping_rule_glrt} with an asymptotically tight family of threshold
	yields an algorithm such that, for all $(\mu, \sigma^2) \in \Real^K \times (\Real^\star_{+})^K$ with $|a^\star(\mu)|=1$,
	\[
		\limsup_{\delta \to 0} \frac{\mathbb{E}_{\nu}[\tau_{\delta}]}{\log \left(1/\delta \right)} \leq T_{\beta}^\star(\mu, \sigma^2 ) \: .
	\]
\end{lemma}
\begin{proof}
	Let $\epsilon_1 = \epsilon_1(\mu, \sigma^2)$.
	Let $c_{\beta} = \frac{1}{2}\min_{a \in [K]}w^{\beta}_a > 0$ and $\Delta = \min_{a \neq a^\star}|\mu_{a^\star} - \mu_a| > 0$.
	Let $\zeta > 0$. By continuity of
	\[
	(\mu, \sigma^2, w) \mapsto  \inf_{x \in \Real} \sum_{c \in \{a,b\}} w_c \ln \left( 1 + \frac{(\mu_{c} - x)^2}{\sigma^2_{c}} \right)  \1\{\mu_{a} > \mu_{b}\}
	\]
	on $\Real^K \times (\Real^\star_{+})^K \times \simplex$, there exists $\epsilon_2 > 0$ such that
\begin{align*}
	&\max_{a \in [K]} \left| \frac{N_{t,a}}{t} -w^{\beta}_a \right| \leq \epsilon_2 \quad \text{,} \quad \max_{a \in [K]} \left| \mu_{t,a} - \mu_a \right| \leq \epsilon_2 \quad \text{and} \quad \max_{a \in [K]} \left| \frac{\sigma^2_{t,a}}{\sigma^2_{a}} - 1\right| \le \epsilon_2 \\
	\implies \quad &  \max_{a \in [K]}|\mu_{t,a} - \mu_{a}| \le \frac{\Delta}{4} \quad  \text{and} \quad \frac{1}{t}\min_{a \neq a^\star} C_{t}(a^\star, a) \geq \frac{1 - \zeta}{T_{\beta}^\star(\mu, \sigma^2 )} \: .
\end{align*}
	Choosing such a $\epsilon_2$, we take $\epsilon \in (0, \min\{\epsilon_1, \epsilon_2, c_{\beta}\})$.
	By assumption, we have $\mathbb{E}_{\nu}[T^{\epsilon}_{\beta}] < + \infty$, hence $\frac{N_{t,a}}{t} \ge w^{\beta}_{a} - \epsilon \ge c_{\beta}$ for all $a \in [K]$.

	Let $N_{\epsilon}$ as described above (obtained with Lemma~\ref{lem:W_concentration_gaussian}). For all $t \ge  c_{\beta}^{-1}N_{\epsilon}$, we have $ N_{t,a} \ge N_{\epsilon}$ for all $a \in [K]$, hence $\max_{a \in [K]}|  \mu_{t,a} - \mu_a  | \le \epsilon \le \epsilon_2$ and $ \max_{a \in [K]} \left| \frac{\sigma^2_{t,a}}{\sigma^2_{a}} - 1\right| \le \epsilon \le \epsilon_2$.
	Therefore, we have $\hat a_t \in \argmax_{a \in [K]} \mu_{t,a} = \argmax_{a \in [K]} \mu_a = a^\star$ as $\max_{a \in [K]}|\mu_{t,a} - \mu_{a}| \le \frac{\Delta}{4}$.

	Let $\alpha \in [0,1)$, $\delta_0 \in (0,1]$, functions $f,\bar{T} : (0,1] \to \mathbb{R}_+$ and $C$ as in the definition of an asymptotically tight family of thresholds.
	In the following, we consider $\delta \leq \delta_{0}$.
	Let $\kappa > 0 $. Let $T \ge \frac{1}{\kappa}\max\{T^{\epsilon}_{\beta}, c_{\beta}^{-1}N_{\epsilon}, \bar{T}(\delta) \}$.
	Using the definition of the GLR stopping rule \eqref{eq:def_stopping_rule_glrt} with a family of asymptotically tight threshold, we have
	\begin{align*}
	\min \left\{\tau_{\delta}, T\right\} \leq \kappa T + \sum_{t = \kappa T}^{T} \1\{\tau_{\delta}>t\}  &\leq \kappa T  + \sum_{n = \kappa T }^{T} \1\{ \exists a \neq a^\star , \: C_{t}(a^\star, a) \leq c_{a^\star, a}(N_t,\delta)\}\\&
	\leq \kappa T  + \sum_{n=\kappa T }^{T} \1\{ t\frac{1 - \zeta}{T_{\beta}^\star(\mu, \sigma^2 )} \leq f(\delta) + CT^\alpha\}\\
	&\leq \kappa T  + \frac{T_{\beta}^\star(\mu, \sigma^2 )}{1 - \zeta}  \left( f(\delta) + CT^\alpha\right) \: .
	\end{align*}

	Let $T_{\zeta}(\delta)$ as
	\[
	T_{\zeta}(\delta) \eqdef \inf \left\{ T \geq 1 \mid \frac{T_{\beta}^\star(\mu, \sigma^2 )}{(1 - \zeta)(1-\kappa)}  \left( f(\delta) + CT^\alpha\right) \leq  T \right\} \: .
	\]
	 For every $T \ge \max \{T_{\zeta}(\delta) , \frac{1}{\kappa} \max \{T^{\epsilon}_{\beta}, c_{\beta}^{-1}N_{\epsilon},\bar{T}(\delta) \}$, we have $\tau_{\delta} \le T$, hence
\[
	\mathbb{E}_{\nu}[\tau_{\delta}] \leq \frac{1}{\kappa} \mathbb{E}_{\nu}[T^{\epsilon}_{\beta}] + \frac{1}{\kappa c_{\beta}} \mathbb{E}_{\nu}[N_{\epsilon}] + \frac{1}{\kappa} \bar{T}(\delta)  + T_{\zeta}(\delta) \: .
\]
	As $\mathbb{E}_{\nu}[T^{\epsilon}_{\beta}] + c_{\beta}^{-1} \mathbb{E}_{\nu}[N_{\epsilon}] < + \infty$ and $\lim_{\delta \to 0} \frac{\bar{T}(\delta)}{\ln (1 / \delta)} = 0 $, we obtain for all $\zeta, \kappa > 0$
	 \begin{align*}
		 \limsup _{\delta \to 0} \frac{\mathbb{E}_{\nu}\left[\tau_{\delta}\right]}{\log (1 / \delta)}
		 \leq \limsup_{\delta \to 0} \frac{T_{\zeta}(\delta)}{\log (1 / \delta)}
		 \le \frac{T_{\beta}^\star(\mu, \sigma^2 )}{(1 - \zeta)(1-\kappa)}
		 \: ,
	 \end{align*}
	 where the last inequality uses an inversion result.
	 Letting $\zeta$ and $\kappa$ go to zero yields the result.
\end{proof}

To upper bound the expected convergence time, as prior work we first establish sufficient exploration.
Given an arbitrary threshold $L \in \Real_{+}^{*}$, we define the sampled enough set and its arms with highest mean (when not empty) as
\begin{equation} \label{eq:def_sampled_enough_sets}
	S_{t}^{L} \eqdef \{a \in [K] \mid N_{t,a} \ge L \} \quad \text{and} \quad \mathcal I_t^\star \eqdef \argmax_{ a \in S_{t}^{L}} \mu_{a} \: .
\end{equation}

We define the highly and the mildly under-sampled sets
\begin{equation} \label{eq:def_undersampled_sets}
	U_{t}^L \eqdef \{a \in [K]\mid N_{t,a} < L^{1/6} \} \quad \text{and} \quad V_{t}^L \eqdef \{a \in [K] \mid N_{t,a} < L^{5/6}\} \: ,
\end{equation}
where we used $L^{1/6}$ and $L^{5/6}$ instead of $\sqrt{L}$ and $L^{3/4}$ as done in \citet{Shang20TTTS}.
However, \citet{jourdan_2022_TopTwoAlgorithms} noted the results would hold with $L^{\alpha_1}$ and $L^{\alpha_2}$ with $0 < \alpha_1 < \alpha_2 < 1$.

Lemma~\ref{lem:tc_rate_lower_bound} shows that the transportation cost is strictly positive and increases linearly.
\begin{lemma} \label{lem:tc_rate_lower_bound}
Let $\mu \in \Real^K$ with $\mu_b < \mu_a$.
There exists $L$ with $\mathbb{E}[|L|^\alpha] < +\infty$ for all $\alpha>0$ and $D_{\nu} > 0$ such that for $N_{t,a} \ge L$ and $N_{t,b} \ge L$, $C_{t}(a,b) > L D_{\nu}$, where $ D_{\nu} > 0$ is a problem dependent constant.
\end{lemma}
\begin{proof}
Suppose that $N_{t,a} \ge L$ and $N_{t,b} \ge L$, for some $L$ to be determined. First we get
\begin{align*}
	C_{t}(a,b) &\geq \frac{L}{2} \1\{\mu_{t,a} > \mu_{t,b}\} \inf_{x \in \Real} \sum_{c \in \{a,b\}} \ln \left( 1 + \frac{(\mu_{t,c} - x)^2}{\sigma^2_{t,c}} \right)  \: .
\end{align*}

For any compact interval $\mathcal I_C \subseteq \Real$, the function defined by
\[
(\mu, \sigma^2) \mapsto \1\{\mu_{a} > \mu_{b}\} \inf_{x \in \mathcal I_C} \sum_{c \in \{a,b\}}  \ln \left( 1 + \frac{(\mu_{c} - x)^2}{\sigma^2_{c}} \right)
\]
is continuous on $\Real^K \times (\Real^\star_{+})^K$.
For $L$ greater than some $L_1$ with finite moments, we have $\mu_{t,a} > \mu_{t,b}$, $[\mu_b - \epsilon, \mu_a + \epsilon] \subseteq \mathcal I_{C}$ and $(\mu_{t,a}, \sigma^2_{t,a})$ is $\varepsilon$-close to $(\mu_a, \sigma^2_a)$ (and same thing for $(\mu_b, \sigma^2_b)$).
The continuity then gives that there exists $L$ with finite moments such that
\begin{align*}
\inf_{x \in \Real} \sum_{c \in \{a,b\}} \ln \left( 1 + \frac{(\mu_{t,c} - x)^2}{\sigma^2_{t,c}} \right)
\ge \frac{1}{2} \inf_{x \in \mathcal I_C} \sum_{c \in \{a,b\}}  \ln \left( 1 + \frac{(\mu_{c} - x)^2}{\sigma^2_{c}} \right) \: .
\end{align*}
This is strictly positive since $\mu_b < \mu_a$ due to the continuity and strict convexity properties holding for Gaussian with unknown variance.
\end{proof}

\begin{lemma} \label{lem:fast_rate_emp_tc}
	Let $ D_{\nu} > 0$ as in Lemma~\ref{lem:tc_rate_lower_bound}.
Let $S_{t}^{L}$ and $\mathcal I_t^\star$ as in (\ref{eq:def_sampled_enough_sets}).
There exists $L_4$ with $\mathbb E_{\nu}[(L_4)^{\alpha}] < +\infty$ for all $\alpha > 0$ such that if $L \ge L_4$, for all $t$ such that $S_{t}^{L} \neq \emptyset$, for all $(a,b) \in \mathcal I_t^\star \times \left(S_{t}^{L} \setminus  \mathcal I_t^\star \right)$, we have $C_{t}(a, b) \geq  L D_{\nu}$.
\end{lemma}
\begin{proof}
	Applying Lemma~\ref{lem:tc_rate_lower_bound} yields the result directly.
\end{proof}

Lemma~\ref{lem:small_tc_undersampled_arms} gives an upper bound on the transportation costs between a sampled enough arm and an under-sampled one.
\begin{lemma} \label{lem:small_tc_undersampled_arms}
	Let $S_{t}^{L}$ as in (\ref{eq:def_sampled_enough_sets}). There exists $L_5$ with $\mathbb E_{\nu}[(L_5)^{\alpha}] < +\infty$ for all $\alpha > 0$ such that for all $L \geq L_5$ and all $t \in \Natural$,
	\[
	\forall  (a,b) \in  S_{t}^{L} \times \overline{S_{t}^{L}} , \quad 	C_{t}(a,b) \leq \frac{L}{2} \ln \left( 1 + L W_{0} \left( D_1 + 4 W_\mu  \right)^2 \right) \: ,
	\]
where $D_1 > 0$ is a problem dependent constant and $W_{0}, W_\mu$ are the random variables defined in Lemma~\ref{lem:W_concentration_gaussian}.
\end{lemma}
\begin{proof}
Let $(a,b) \in  S_{t}^{L} \times \overline{S_{t}^{L}}$ ($a$ is sampled more than $L$ times, $b$ is not). Taking $x = \mu_{t,a}$ yields
\begin{align*}
	C_{t}(a,b) &\leq \frac{1}{2} N_{t,b} \ln \left( 1 + \frac{(\mu_{t,b} - \mu_{t,a})^2}{\sigma^2_{t,b}} \right) \leq \frac{L}{2} \ln \left( 1 + \frac{(\mu_{t,b} - \mu_{t,a})^2}{\sigma^2_{t,b}} \right)	\: ,
\end{align*}
where we used that $b \in \overline{S_{t}^{L}}$ and that $C_{t}(a,b) = 0$ when $\mu_{t,b} \ge \mu_{t,a}$.

By definition of $W_{0}$ and $W_\mu$, we have for all $k \in [K]$
\begin{align*}
\mu_{t,k} \le \mu_k + W_\mu \log(e+N_{t,k})/N_{t,k} \qquad \text{and} \qquad  \sigma^2_{t,k} &\ge  W_{k} /N_{t,k} \: .
\end{align*}
Then, we obtain
\begin{align*}
	\frac{(\mu_{t,b} - \mu_{t,a})^2}{\sigma^2_{t,b}} &\le \frac{N_{t,b}}{W_{b}} \left( |\mu_{a} - \mu_{b}| + W_\mu \left( \frac{\log(e+N_{t,a})}{N_{t,a}} + \frac{\log(e+N_{t,b})}{N_{t,b}} \right) \right)^2 \: .
\end{align*}
Since $x \mapsto \frac{\ln(e+x)}{x}$ is decreasing on $\Real^{\star}_{+}$, we have $\frac{\log(e+N_{t,k})}{N_{t,k}} \le 2$ for $N_{t,k} \ge 1$. Then, since $W_0 = \max_{a \in [K]}W_a^{-1}$, we have
\begin{align*}
	\frac{(\mu_{t,b} - \mu_{t,a})^2}{\sigma^2_{t,b}} &\le L W_{0} \left( |\mu_{a} - \mu_{b}| + 4 W_\mu  \right)^2 \: ,
\end{align*}
which yields the result by taking $D_1 = \max_{a \neq b} |\mu_{a} - \mu_{b}|$.
\end{proof}

Lemma~\ref{lem:EB_ensures_suff_explo} shows that the EB leader satisfies the required condition on a leader to obtain sufficient exploration, i.e. Property 2 in \citet{jourdan_2022_TopTwoAlgorithms}.
\begin{lemma}[Lemma 17 in \citet{jourdan_2022_TopTwoAlgorithms}] \label{lem:EB_ensures_suff_explo}
	Let $S_{t}^{L}$ and $\mathcal I_t^\star$ as in (\ref{eq:def_sampled_enough_sets}).
	Let $L_4$ in Lemma~\ref{lem:fast_rate_emp_tc}.
	Then, for all $L \ge L_4$, for all $n$ such that $S_{t}^{L} \neq \emptyset$, $B_{t+1}^{\text{EB}} \in S_{t}^{L}$ implies $ B_{t+1}^{\text{EB}} \in \mathcal I_t^\star$.
\end{lemma}

Lemma~\ref{lem:TCI_ensures_suff_explo} shows that the TCI challenger satisfies the required condition on a challenger to obtain sufficient exploration, i.e. Property 3 in \citet{jourdan_2022_TopTwoAlgorithms}.
The proof of Lemma~\ref{lem:TCI_ensures_suff_explo} resembles the proof of Lemma 21 in \citet{jourdan_2022_TopTwoAlgorithms}.
The sole technical difference lies in the necessities of having an finite mean for $W_{0}$, which explains why we consider $n_0 = 4$ for $\beta$-EB-TCI.
\begin{lemma} \label{lem:TCI_ensures_suff_explo}
Let $U_t^L$ and $V_t^L$ as in (\ref{eq:def_undersampled_sets}) and $\mathcal J_t^\star = \argmax_{ a \in \overline{V_{t}^{L}}} \mu_{a}$.
There exists $L_6$ with $\mathbb E_{\nu}[L_6] < + \infty$ such that if $L \ge L_6$, for all $n$ such that $U_t^L \neq \emptyset$, $B_{t+1}^{\text{EB}} \notin V_{t}^{L}$ implies $ A_{t+1}^{\text{TCI}} \in V_{t}^{L} \cup \left( \mathcal J_t^\star \setminus \left\{ B_{t+1}^{\text{EB}} \right\} \right)$.
\end{lemma}
\begin{proof}
Using Lemmas~\ref{lem:fast_rate_emp_tc}, ~\ref{lem:small_tc_undersampled_arms} and~\ref{lem:EB_ensures_suff_explo}, for all $L$ larger than a random variable $L_7$ with finite expectation, $B^{\text{EB}}_{t+1} \in \mathcal J_t^\star$ and
	\begin{align*}
		&\forall (a,b) \in \mathcal J_t^\star \times \left(\overline{V_t^L} \setminus \mathcal J_t^\star\right), \quad  C_{t}(a, b) + \ln N_{t,b} \geq  L^{5/6} D_{\nu} + \frac{5}{6} \ln L  \: , \\
		&\forall (a,b) \in \overline{U_t^L} \times U_t^L, \quad 	 C_{t}(a, b) + \ln N_{t,b} \leq \frac{L^{1/6}}{2} \ln \left( 1 + L^{1/6} W_{0} \left( D_1 + 4 W_\mu  \right)^2 \right) + \frac{1}{6} \ln L \: .
	\end{align*}
Using that $\ln(1+x) \le \sqrt{x}$, we obtain
\begin{align*}
\frac{L^{1/6}}{2} \ln \left( 1 + L^{1/6} W_{0} \left( D_1 + 4 W_\mu  \right)^2 \right) \le \frac{L^{1/4}}{2} \sqrt{W_{0}} \left( D_1 + 4 W_\mu  \right)
\end{align*}
Therefore, for $L \ge L_8 := W_{0}^{6/7}\left( \frac{D_1 + 4 W_\mu}{2D_{\nu}} \right)^{12/7} + 1$, we have that
\[
L^{5/6} D_{\nu} + \frac{5}{6} \ln L > \frac{L^{1/6}}{2} \ln \left( 1 + L^{1/6} W_{0} \left( D_1 + 4 W_\mu  \right)^2 \right) + \frac{1}{6} \ln L \: .
\]
Let $L_6 = \max \{L_7, L_8\}$
Therefore, at least one under-sampled arm has transportation cost lower than all the ones that are much sampled.
This implies that $ A_{t+1}^{\text{TCI}} \in V_{t}^{L} \cup \left( \mathcal J_t^\star \setminus \left\{ B_{t+1}^{\text{EB}} \right\} \right)$.
Then, we have
\[
\mathbb E_{\nu}[L_6] \le \mathbb E_{\nu}[L_7] + 1 + \mathbb E_{\nu}\left[W_{0}^{6/7}\right]\mathbb E_{\nu}\left[\left( \frac{D_1 + 4 W_\mu}{2D_{\nu}} \right)^{12/7}\right] < + \infty \: ,
\]
The first inequality and the last strict inequality are obtained by Lemma~\ref{lem:W_concentration_gaussian}, thanks to the independence of $W_{0}$ and $W_\mu$, the fact that $W_{0}$ has finite mean since $n_0 \ge 4$ and that polynomial of $W_\mu$ have finite expectation.
\end{proof}

Since Properties 2 and 3 of \citet{jourdan_2022_TopTwoAlgorithms} are satisfied (Lemmas~\ref{lem:EB_ensures_suff_explo} and~\ref{lem:TCI_ensures_suff_explo}), Lemma~\ref{lem:suff_exploration} holds for $\beta$-EB-TCI on instances such that $\min_{a \neq b}|\mu_a - \mu_b| > 0$.
\begin{lemma}[Lemma 7 in \citet{jourdan_2022_TopTwoAlgorithms}] \label{lem:suff_exploration}
	There exist $N_0$ with $\mathbb{E}_{\nu}[N_0] < + \infty$ such that for all $ t \geq N_0$ and all $a \in [K]$, $N_{t,a} \geq \left(t / K \right)^{1/6}$.
\end{lemma}

Now that we have proved sufficient exploration (Lemma~\ref{lem:suff_exploration}), we will show convergence towards the $\beta$-optimal allocation.
Lemma~\ref{lem:EB_ensures_convergence} shows that the EB leader satisfies the required condition on a leader to obtain convergence, i.e. Property 5 in \citet{jourdan_2022_TopTwoAlgorithms}.
\begin{lemma}[Lemma 18 in \citet{jourdan_2022_TopTwoAlgorithms}] \label{lem:EB_ensures_convergence}
		There exists $N_6$ with $\mathbb{E}_{\nu}[N_6] < + \infty$ such that for all $t \ge N_6$, $B_{t+1}^{\text{EB}} = a^\star$.
\end{lemma}

For all $a \in [K]$ and all $t > t_0$, let $\psi_{t,a} \eqdef \bP_{\mid (t-1)}[a_t = a]$ be the probability of sampling arm $a$ at time $t$ and $\Psi_{t,a} \eqdef \sum_{s \in [t]} \psi_{s,a}$ its cumulative sum.

Using the EB leader, Lemma~\ref{lem:TCI_ensures_convergence} shows that the TCI challenger satisfies the required condition on a challenger to obtain convergence towards the $\beta$-optimal allocation, i.e. Property 6 in \citet{jourdan_2022_TopTwoAlgorithms}.
The proof of Lemma~\ref{lem:TCI_ensures_convergence} is very similar to the proof of Lemma 22 in \citet{jourdan_2022_TopTwoAlgorithms}, hence we omit certain details.
\begin{lemma} \label{lem:TCI_ensures_convergence}
	Let $\epsilon> 0$. There exists $ N_7$ with $\mathbb E_{\nu}[ N_7] < + \infty$ such that for all $n \geq  N_7$ and all $a \neq a^\star(\mu)$, $\Psi_{t,a}/t \geq w_{a}^{\beta} + \epsilon$ implies that $ A_{t+1}^{\text{TCI}} \neq a$.
\end{lemma}
\begin{proof}
	Using Lemma~\ref{lem:EB_ensures_convergence}, we know that $B_{t+1}^{\text{EB}} = a^\star$ for all $t \ge N_6$.
	By definition of $A_{t+1}^{\text{TCI}}$ and algebraic manipulation, $ A_{t+1}^{\text{TCI}} \neq a$ is implied by
\begin{align*}
	\frac{1}{t}\left(C_{t}(a^\star,a) - \min_{b\neq a^\star} C_{t}(a^\star,b) \right) > \frac{\ln(tK)}{2t} \: .
\end{align*}
In Appendix~\ref{app:ss_optimal_allocation_oracles}, we obtain the desired regularity properties: joint continuity, strict convexity, equality at the equilibrium.
Therefore, by using the same proof as the one of Lemma 20 in \citet{jourdan_2022_TopTwoAlgorithms}, there exists $C_{\nu} > 0$ and $N_7$ with $\mathbb E_{\nu}[N_7] < + \infty$ such that for all $t \geq N_7$ and all $a \neq a^\star$,
\[
\frac{\Psi_{t,a}}{t} \geq w_{a}^{\beta} + \epsilon  \quad \implies \quad \frac{1}{t}\left(C_{t}(a^\star,a) - \min_{b\neq a^\star} C_{t}(a^\star,b) \right) \ge C_{\nu} \: .
\]
Since $\frac{\ln(tK)}{2t} \to_{\infty} 0$, there exists a deterministic $N_8$ such that for all $n \ge  N_8$, $\frac{\ln(tK)}{2t} < C_{\nu}$. Therefore, for all $n \geq \tilde N_7 \eqdef \max \{ N_8, N_7\}$ and all $a \neq a^\star$, $\Psi_{t,a}/t \geq w_{a}^{\beta} + \epsilon$ implies that $ A_{t+1}^{\text{TCI}} \neq a$.
Since $\mathbb E_{\nu}[\tilde N_7] = N_8 + \mathbb E_{\nu}[ N_7] < + \infty$, this concludes the proof.
\end{proof}

Since there is sufficient exploration (Lemma~\ref{lem:suff_exploration}) and Properties 5 and 6 of \citet{jourdan_2022_TopTwoAlgorithms} are satisfied (Lemmas~\ref{lem:EB_ensures_convergence} and~\ref{lem:TCI_ensures_convergence}), Lemma~\ref{lem:finite_mean_time_eps_convergence_beta_opti_alloc} holds for $\beta$-EB-TCI on instances such that $\min_{a \neq b}|\mu_a - \mu_b| > 0$.

\begin{lemma}[Lemma 10 in \citet{jourdan_2022_TopTwoAlgorithms}] \label{lem:finite_mean_time_eps_convergence_beta_opti_alloc}
	Let $\epsilon > 0$ and $T_{\beta}^{\epsilon}$ as in \eqref{eq:rv_T_eps_beta}. Then, $\beta$-EB-TCI satisfies $\mathbb{E}_{\nu}[T_{\beta}^{\epsilon}] < +\infty$.
\end{lemma}

Combining Lemmas~\ref{lem:finite_mean_time_eps_convergence_beta_opti_alloc} and~\ref{lem:asymptotic_optimality_top_two_algorithms} concludes the proof for $\beta$-EB-TCI.

\subsubsection{$\beta$-EB-EVTCI}

Let $a^\star = a^\star(\mu)$. Similarly, the $\beta$-optimal allocation for Gaussian with known variance is defined as
\begin{align*}
	&w^\star_{\sigma^2, \beta}(\mu)
	= \argmax_{w \in \triangle_K, w_{a^\star} = \beta} \min_{a \neq a^\star} \inf_{ u \in \Real } \sum_{b \in \{a^\star,a\}} w_b \frac{(\mu_b - u )^2}{\sigma_b^2}
	\: ,
\end{align*}
is a singleton, denoted by $\{w^{\beta}\}$, such that $\min_{a \in [K]} w^{\beta}_{a} > 0$.

\begin{lemma} \label{lem:upper_bound_ev_toptwo}
	Using the EV-GLR stopping rule with an asymptotically tight family of thresholds, $\beta$-EB-EVTCI) satisfies that, for all $\nu$ with $\min_{a \neq b}|\mu_a - \mu_b| > 0 $,
	\[
	 \limsup_{\delta \rightarrow 0} \frac{\mathbb{E}_{\nu}\left[\taud^{\text{EV}}\right]}{\log (1 / \delta)} \le  T^\star_{\beta}(\mu, \sigma^2) \: .
	\]
\end{lemma}

For $\varepsilon> 0$, the random variable $\bar T^{\varepsilon}_{\beta}$ quantifies the number of samples required for the empirical allocations $\frac{N_{t}}{t}$ to be $\epsilon$-close to $w^{\beta}$:
\begin{equation} \label{eq:rv_T_eps_beta_bis}
	\bar T^{\epsilon}_{\beta} \eqdef \inf \left\{ T \ge 1 \mid \forall t \geq T, \: \left\| \frac{N_{t}}{t} -w^{\beta} \right\|_{\infty} \leq \epsilon \right\}  \: .
\end{equation}

Lemma~\ref{lem:asymptotic_optimality_top_two_algorithms_bis} shows that a sufficient condition to obtain an upper bound on the asymptotic expected sample complexity is to show $\mathbb{E}_{\nu}[\bar T^{\epsilon}_{\beta}] < + \infty$ for $\epsilon$ small enough.
The proof of Lemma~\ref{lem:asymptotic_optimality_top_two_algorithms_bis} is omitted since it is almost identical to the one of Lemma~\ref{lem:asymptotic_optimality_top_two_algorithms}.
\begin{lemma} \label{lem:asymptotic_optimality_top_two_algorithms_bis}
	Let $\delta, \beta \in (0,1)$.
	Assume that there exists $\epsilon_1> 0$ such that for all $\epsilon \in (0,\epsilon_1]$, $\mathbb{E}_{\nu}[\bar T^{\epsilon}_{\beta}] < + \infty$.
	Combining the EV-GLR stopping rule \eqref{eq:def_stopping_rule_evglrt} with an asymptotically tight family of threshold
	yields an algorithm such that, for all $(\mu, \sigma^2) \in \Real^K \times (\Real^\star_{+})^K$ with $|a^\star(\mu)|=1$,
	\[
		\limsup_{\delta \to 0} \frac{\mathbb{E}_{\nu}[\tau^{\text{EV}}_{\delta}]}{\log \left(1/\delta \right)} \leq T^\star_{\sigma^2, \beta}(\mu) \: .
	\]
\end{lemma}

Lemma~\ref{lem:fast_rate_emp_ev_tc} is obtained similarly as Lemma~\ref{lem:fast_rate_emp_tc}, hence we omit the proof.
\begin{lemma} \label{lem:fast_rate_emp_ev_tc}
Let $S_{t}^{L}$ and $\mathcal I_t^\star$ as in (\ref{eq:def_sampled_enough_sets}).
There exists $L_4$ with $\mathbb E_{\nu}[(L_4)^{\alpha}] < +\infty$ for all $\alpha > 0$ such that if $L \ge L_4$, for all $t$ such that $S_{t}^{L} \neq \emptyset$, for all $(a,b) \in \mathcal I_t^\star \times \left(S_{t}^{L} \setminus  \mathcal I_t^\star \right)$, we have $C^{\text{EV}}_{t}(a, b) \geq  L D_{\mu}$, where $ D_{\nu} > 0$ is a problem dependent constant.
\end{lemma}

Lemma~\ref{lem:small_ev_tc_undersampled_arms} is obtained similarly as Lemma~\ref{lem:small_tc_undersampled_arms}, hence we omit the proof.
\begin{lemma} \label{lem:small_ev_tc_undersampled_arms}
	Let $S_{t}^{L}$ as in (\ref{eq:def_sampled_enough_sets}). There exists $L_5$ with $\mathbb E_{\nu}[(L_5)^{\alpha}] < +\infty$ for all $\alpha > 0$ such that for all $L \geq L_5$ and all $t \in \Natural$,
	\[
	\forall  (a,b) \in  S_{t}^{L} \times \overline{S_{t}^{L}} , \quad 	C^{\text{EV}}_{t}(a,b) \leq  \frac{L^2}{2}
   W_0 \left( D_1 + 4 W_\mu  \right)^2 \: ,
	\]
where $D_1 > 0$ is a problem dependent constant and $W_{0}, W_\mu$ are the random variables defined in Lemma~\ref{lem:W_concentration_gaussian}.
\end{lemma}

Lemma~\ref{lem:EV_TCI_ensures_suff_explo} is obtained similarly as Lemma~\ref{lem:TCI_ensures_suff_explo}.
The sole technical difference lies in the necessities of having an finite variance for $W_0$, which explains why we consider $n_0 = 6$ for $\beta$-EB-EVTCI.
\begin{lemma} \label{lem:EV_TCI_ensures_suff_explo}
Let $U_t^L$ and $V_t^L$ as in (\ref{eq:def_undersampled_sets}) and $\mathcal J_t^\star = \argmax_{ a \in \overline{V_{t}^{L}}} \mu_{a}$.
There exists $L_6$ with $\mathbb E_{\nu}[L_6] < + \infty$ such that if $L \ge L_6$, for all $n$ such that $U_t^L \neq \emptyset$, $B_{t+1}^{\text{EB}} \notin V_{t}^{L}$ implies $ A_{t+1}^{\text{EVTCI}} \in V_{t}^{L} \cup \left( \mathcal J_t^\star \setminus \left\{ B_{t+1}^{\text{EB}} \right\} \right)$.
\end{lemma}
\begin{proof}
Using Lemmas~\ref{lem:fast_rate_emp_ev_tc}, ~\ref{lem:small_ev_tc_undersampled_arms} and~\ref{lem:EB_ensures_suff_explo}, for all $L$ larger than a random variable $L_7$ with finite expectation, $B^{\text{EB}}_{t+1} \in \mathcal J_t^\star$ and
	\begin{align*}
		\forall (a,b) \in \mathcal J_t^\star \times \left(\overline{V_t^L} \setminus \mathcal J_t^\star\right), \quad & C_{t}(a, b) + \ln N_{t,b} \geq  L^{5/6} D_{\mu} + \frac{5}{6} \ln L  \: , \\
		\forall (a,b) \in \overline{U_t^L} \times U_t^L, \quad 	& C_{t}(a, b) + \ln N_{t,b} \leq \frac{L^{1/3}}{2}
	   W_0 \left( D_1 + 4 W_\mu  \right)^2 + \frac{1}{6} \ln L \: .
	\end{align*}
Therefore, for $L \ge L_8 := W_{0}^{2}\left( \frac{D_1 + 4 W_\mu}{2D_{\nu}} \right)^{4} + 1$, we have that
\[
L^{5/6} D_{\mu} + \frac{5}{6} \ln L  > \frac{L^{1/3}}{2}
W_0 \left( D_1 + 4 W_\mu  \right)^2 + \frac{1}{6} \ln L   \: .
\]
Let $L_6 = \max \{L_7, L_8\}$.
Therefore, at least one under-sampled arm has transportation cost lower than all the ones that are much sampled.
This implies that $ A_{t+1}^{\text{EVTCI}} \in V_{t}^{L} \cup \left( \mathcal J_t^\star \setminus \left\{ B_{t+1}^{\text{EB}} \right\} \right)$.
Then, we have
\[
\mathbb E_{\nu}[L_6] \le \mathbb E_{\nu}[L_7] + 1 + \mathbb E_{\nu}\left[W_{0}^{2}\right]\mathbb E_{\nu}\left[\left( \frac{D_1 + 4 W_\mu}{2D_{\nu}} \right)^{4}\right] < + \infty \: .
\]
The first inequality and the last strict inequality are obtained by Lemma~\ref{lem:W_concentration_gaussian}, thanks to the independence of $W_{0}$ and $W_\mu$, the fact that $W_{0}$ has finite variance since $n_0 \ge 6$ and that polynomial of $W_\mu$ have finite expectation.
\end{proof}

Since Properties 2 and 3 of \citet{jourdan_2022_TopTwoAlgorithms} are satisfied (Lemmas~\ref{lem:EB_ensures_suff_explo} and~\ref{lem:EV_TCI_ensures_suff_explo}), Lemma~\ref{lem:suff_exploration} holds for $\beta$-EB-EVTCI on instances such that $\min_{a \neq b} |\mu_a - \mu_b| > 0$.

Lemma~\ref{lem:EV_TCI_ensures_convergence} is obtained similarly as Lemma~\ref{lem:TCI_ensures_convergence}, hence we omit the proof.
\begin{lemma} \label{lem:EV_TCI_ensures_convergence}
	Let $\epsilon> 0$. There exists $ N_7$ with $\mathbb E_{\nu}[ N_7] < + \infty$ such that for all $n \geq  N_7$ and all $a \neq a^\star(\mu)$, $\Psi_{t,a}/t \geq w_{a}^{\beta} + \epsilon$ implies that $ A_{t+1}^{\text{EVTCI}} \neq a$.
\end{lemma}

Since there is sufficient exploration (Lemma~\ref{lem:suff_exploration}) and Properties 5 and 6 of \citet{jourdan_2022_TopTwoAlgorithms} are satisfied (Lemmas~\ref{lem:EB_ensures_convergence} and~\ref{lem:EV_TCI_ensures_convergence}), Lemma~\ref{lem:finite_mean_time_eps_convergence_beta_opti_alloc_bis} holds for $\beta$-EB-EVTCI on instances such that $\min_{a \neq b}|\mu_a - \mu_b| > 0$.

\begin{lemma}[Lemma 10 in \citet{jourdan_2022_TopTwoAlgorithms}] \label{lem:finite_mean_time_eps_convergence_beta_opti_alloc_bis}
	Let $\epsilon > 0$ and $\bar T_{\beta}^{\epsilon}$ as in \eqref{eq:rv_T_eps_beta_bis}. Then, $\beta$-EB-EVTCI satisfies $\mathbb{E}_{\nu}[\bar T_{\beta}^{\epsilon}] < +\infty$.
\end{lemma}

Combining Lemmas~\ref{lem:finite_mean_time_eps_convergence_beta_opti_alloc_bis} and~\ref{lem:asymptotic_optimality_top_two_algorithms_bis} concludes the proof of the asymptotic upper bound on the expected sample complexity of $\beta$-EB-EVTCI, i.e. Lemma~\ref{lem:upper_bound_ev_toptwo}.

In Appendix~\ref{app:ss_explicit_formulas_inequalities}, we explained that $T^\star_{\sigma^2, \beta}(\mu)$ and $T^\star_{\beta}(\mu, \sigma^2)$ satisfied inequalities like the ones between $T^\star_{\sigma^2}(\mu)$ and $T^\star(\mu, \sigma^2)$ in~\eqref{eq:characteristic_times_inequalities}.
Similar to Lemma~\ref{lem:ratio_characteristic_time_large}, we can show that there exists a sequence of instances $(\nu_n)_{n \in \Natural}$ with $\min_{b\neq a} |\mu_{n,a} - \mu_{n,b}| > 0$ such that $\lim_{n \to + \infty} T^\star_{\beta}(\mu_n, \sigma_n^2) / T^\star_{\sigma_n^2, \beta}(\mu_n) = + \infty$.
Therefore, we can conclude the impossibility result for $\beta$-EB-EVTCI (Theorem~\ref{thm:impossibility_result}) with the same arguments as for EV-TaS in Appendix~\ref{app:expected_sample_complexity_tas}.

%% file: sections/appendix_lambert_W.tex

\section{The Lambert $W$ Function} \label{app:lambert_W_functions}

The Lambert $W$ function is implicitly defined by the equation $W(x) e^{W(x)}=x .$ It defines two main branches $W_{-1}$ (negative) and $W_{0}$ (positive).
\begin{itemize}
	\item $W_{-1}$, defined on $[-e^{-1}, 0)$, is decreasing and $W_{-1}(-e^{-1}) = -1$.
	\item $W_{0}$, defined on $[-e^{-1},+\infty)$, is increasing and $W_{0}(-e^{-1}) = -1$.
\end{itemize}
The function $W_{0}$ satisfies for all $x \ge e$, $W_{0}\left(e^{x}\right) \leq x$ and
\[
\frac{\log \log (x)}{2 \log (x)} \leq W_{0}(x)-(\log (x)-\log \log (x)) \leq \frac{e}{e-1} \frac{\log \log (x)}{\log (x)} \: .
\]

Lambert's branches are involved in the inversion of $h(x) = x -\ln(x)$. When $x\geq1$, it involves the negative branch. When $x\leq 1$, it involves the negative part of the positive branch. To make the notations clearer, we define for all $x \geq 1$
\begin{equation} \label{eq:lambert_transforms}
	\overline{W}_{-1}(x) = - W_{-1}(-e^{-x}) \quad \text{and}\quad  \overline{W}_{0}(x) = - W_{0}(-e^{-x}) \: .
\end{equation}

Lemma~\ref{lem:lambert_branches_properties} gather useful properties on $\overline{W}_{-1}$ and $\overline{W}_{0}$ that we will use.

\begin{lemma} \label{lem:lambert_branches_properties}
	\textbf{(1)}	For $x \geq 1$, let $h(x) = x -\ln(x)$. Then,
	\begin{align*}
		&\forall y \ge 1, \qquad y \leq h(x) \: \iff \: \begin{cases}
			\overline{W}_{-1}\left( y\right) \leq x & \text{if } x \ge 1 \\
			\overline{W}_{0}\left( y\right) \geq x & \text{if } x \in (0,1]
		\end{cases} \: , \\
		&\forall \delta > 0, \forall c > 0, \qquad e^{- c\left( h(x)-1\right)} \leq \delta \: \iff \: \begin{cases}
			\overline{W}_{-1}\left( 1+\frac{1}{c}\ln\frac{1}{\delta}\right) \leq x & \text{if } x > 1 \\
			\overline{W}_{0}\left( 1+\frac{1}{c}\ln\frac{1}{\delta}\right) \geq x & \text{if } x \in (0,1)
		\end{cases} \: , \\
		& \forall x > 1, \qquad \exp \left(-x + e^{-x}\right) \leq \overline{W}_{0}(x) \leq \exp \left(-x + e^{1-x}\right) \: , \\
		& \forall x > 1, \qquad x  + \ln(x) \leq \overline{W}_{-1}(x) \leq x + \ln(x) + \min\left\{\frac{1}{2}, \frac{1}{\sqrt{x}}\right\} \: , \\
		& \forall u > 1, \forall t>1, \qquad \overline{W}_{0}\left(1+ \frac{u}{t}\right) \geq \frac{1}{t} \: \iff \: t \geq \exp \left( 1 + W_{0} \left( \frac{u -1}{e}  \right) \right) \: .
	\end{align*}
	\textbf{(2)}	The function $\overline{W}_{-1}$ is increasing and strictly concave on $(1,+\infty)$. The function $\overline{W}_{0}$ is decreasing and strictly convex on $(1,+\infty)$. In particular,
		\[
		\forall x > 1, \qquad (\overline{W}_0)'(x) = \left( 1 - \frac{1}{\overline{W}_{0}(x)} \right)^{-1} \: \text{ and } \: (\overline{W}_{-1})'(x) = \left( 1 - \frac{1}{\overline{W}_{-1}(x)} \right)^{-1} \: .
		\]
\end{lemma}
\begin{proof}
	\textbf{(1)} Let $y \geq 1$ and $x \in (0, 1]$. We obtain
	\begin{align*}
		\overline{W}_{0}(y) \geq x \iff  W_{0}(-e^{-y}) \leq -x \iff -e^{-y} \leq -xe^{-x} \iff y \leq x - \ln(x)
	\end{align*}
	where the second equivalence uses that $-e^{-y} = W_{0}(-e^{-y})e^{W_{0}(-e^{-y})}$,  $y \mapsto y e^{y}$ is increasing on $[-1,+\infty)$ and $W_{0}(x)$ has values on $[-1,0)$ for $x \in [-e^{-1},0)$. Let $x \in (0,1)$, $\delta, c > 0$. Then,
	\begin{align*}
		\overline{W}_{0}\left( 1+\frac{1}{c}\ln\frac{1}{\delta}\right) \geq x &\iff  1+\frac{1}{c}\ln\frac{1}{\delta} \leq h(x) \iff \exp \left( - c\left( h(x)-1\right) \right) \leq \delta
	\end{align*}
	Let $x > 1$ and $f(x) \in (0,1)$. Then, we obtain
		\begin{align*}
			\overline{W}_{0}\left( x\right) \geq f(x) &\iff  x \leq f(x) - \ln(f(x))
		\end{align*}
		For $f(x) = e^{-x + e^{-x}}$, we have $x \leq f(x) - \ln(f(x)) \iff e^{-x} \geq 0$, hence this condition holds and $\overline{W}_{0}\left( x\right) \geq f(x)$.
		For $f(x) = e^{-x + e^{1-x}}$, we have $x \leq f(x) - \ln(f(x)) \iff x \leq 1$, hence this condition doesn't hold for $x>1$, hence $\overline{W}_{0}\left( x\right) \leq f(x)$.

	For $\overline{W}_{-1}(y)$, the same arguments yield the three results, which were first proven in Lemma A.1 and A.2 of \citet{degenne_2019_ImpactStructureDesign}).

	We denote $v = \frac{u-1}{t} > 0$. Since $t> 1$, direct manipulations show that
 \begin{align*}
	 &\overline{W}_{0}\left(1+\frac{u}{t} \right) \geq \frac{1}{t}
	 \iff  1+\frac{u}{t} \leq \frac{1}{t} - \ln \left(\frac{1}{t} \right)
	 \iff v + \ln(v) \leq \ln \left( \frac{u -1}{e}\right) \\
	 &\iff  ve^{v} \leq  \frac{u -1}{e}
	 \iff  v \leq W_{0} \left( \frac{u -1}{e}  \right)
	 \iff  t \geq \frac{u - 1}{W_{0} \left( \frac{u -1}{e}  \right)}  = e^{1 + W_{0} \left( \frac{u -1}{e}  \right)}\
 \end{align*}
 The equivalence introducing $W_{0}$ uses that for $\alpha = \frac{u -1}{e} > 0$, $W_{0}(\alpha) e^{W_{0}(\alpha)}= \alpha$, $y \mapsto y e^{y}$ is increasing on $[-1,+\infty)$ and $v > 0$. The last equality uses that $e^{W_{0}(x)} = \frac{x}{W_{0}(x)}$.

	\textbf{(2)} Let $W$ denote $W_{0}$ or $W_{-1}$ and $\overline{W}(x) = - W(-e^{-x})$. It is known (by implicit derivation) that $W'(z) = \frac{1}{z+ e^{W(z)}}$ for $z \neq -e^{-1}$. Using that $e^{W(z)} = \frac{z}{W(z)}$, this yields that $z W'(z) = \left(1+ \frac{e^{W(z)}}{z}\right)^{-1} = \left(1+\frac{1}{W(z)}\right)^{-1}$. For $x \neq 1$, using the above with $z = -e^{-x}$, we obtain
	\begin{align*}
		\overline{W}'(x) = - \frac{\mathrm{d}}{\mathrm{d}x} \left( W(-e^{-x})\right) = -e^{-x} W'(-e^{-x}) = \left(1+\frac{1}{W(-e^{-x})}\right)^{-1} = \left(1-\frac{1}{\overline{W}(x)}\right)^{-1}
	\end{align*}
	Since $W_{0}(-e^{-x}) \in (-1,0)$ for all $x >1$ (positive branch on $(-e^{-1},0)$), we have $\overline{W}_{0}(x) \in (0,1)$, hence $\overline{W}_{0}'(x) < 0 $ for $x > 1$.
	Therefore, $\overline{W}_{0}$ is decreasing on $(1, + \infty]$.
	Using that $\overline{W}'(x) = \left(1-\frac{1}{\overline{W}(x)}\right)^{-1}$ for $x\neq 1$, we obtain that $\overline{W}'_{0}$ is increasing on $(1, + \infty]$, hence strictly convex.
	The same arguments yield that $\overline{W}_{-1}$ is increasing and strictly concave on $(1, + \infty]$.
\end{proof}

Lemma~\ref{lem:lemma_A_3_of_Remy} was proven in \citet{degenne_2019_ImpactStructureDesign}. It is needed when using the peeling method.

\begin{lemma}[Lemma A.3 in \citet{degenne_2019_ImpactStructureDesign}] \label{lem:lemma_A_3_of_Remy}
	For $a,b\geq 1$, the minimal value of $f(\eta)=(1+\eta)(a+\ln(b+\frac{1}{\eta}))$ is attained at $\eta^\star$ such that $f(\eta^\star) \leq 1-b+\overline{W}_{-1}(a+b)$.	If $b=1$, then there is equality.
\end{lemma}

%% file: sections/appendix_additional_experiments.tex

\section{Implementation Details and Additional Experiments}
\label{app:additional_experiments}

Implementations details are given in Appendix~\ref{app:ss_implementation_details}.
Supplementary experiments are provided in Appendix~\ref{app:ss_supplementary_experiments}.

\subsection{Implementation Details}
\label{app:ss_implementation_details}

We detail below the most relevant implementation details, both regarding the sampling rules and the considered family of thresholds.

\paragraph{Stopping thresholds}
The implementation of the different families of thresholds require to specify the hyper-parameters $(s, \eta_{0}, \eta_{1}, \gamma)$.
Based on the obtained formulas, we see that there is a trade-off in their choice.
Smaller values of $(\eta_{0}, \eta_{1})$ yield close to ideal dependency in $\delta$, at the cost of increasing the dependency in $t$.
Smaller values of $s$ yield close to ideal dependency in $t$, at the cost of increasing the constant term.
Smaller values of $\gamma$ yield more frequent update of the approximation term error (i.e. the ratio of eigenvalues), at the cost of increasing the dependency in $t$.
Those updates are responsible for stair-step shaped curves of the KL and (EV-)BoB thresholds.
For reasonable choices of the parameters, the relative performance of the thresholds are not changed.
We conduct our experiments with $s=2$, $\gamma=1.2$, $\eta_0=\eta_1 = \ln \left( 1/\delta\right)^{-1}$ (as $\delta < 1$).

\paragraph{Optimal allocation oracles}
To compute the optimal allocation in TaS, we perform nested binary searches to solve the optimization problem described in Theorem~\ref{thm:equivalent_optimization_problem}.
The outer binary search is done on $y \in [0,\min_{a\neq a^\star}d_{a^\star}(\mu_a))$.
The inner binary searches are done to compute $x_{a}(y)$ for all $a \in [K]$.
To obtain $\lambda(x)$, we compute the (at most three) real solutions of the third order polynomial equation, and then return the one minimizing the function of interest. For EV-TaS, the procedure is similar as the one implemented in \citet{garivier_2016_OptimalBestArm}, and it is significantly faster as $\lambda(x)$ has a closed-form solution.

\paragraph{Sampling rules}
As recommended in \citet{jourdan_2022_TopTwoAlgorithms}, we implemented the $\beta$-EB-TCI algorithm with $\beta=0.5$.
This is a deterministic instance of Top Two algorithm with an efficient implementation since it can re-use computations done in the stopping and recommendation rules.

Based on \citet{degenne_2019_NonAsymptoticPureExploration}, we consider DKM with one learner on $\simplex$ instead of $K$ learners.
For unknown variances, the wrapped DKM algorithm is obtained by considering adapted transportation costs (and the corresponding alternatives).
To cope for an additional uncertainty due to the unknown variance, the bonus $\frac{\log(t)}{N_{t,a}}$ is multiplied by $1 + 2 \frac{\log(t)}{N_{t,a}}$.
This heuristic is legitimated by our box concentration.
The implementation of EV-DKM is direct by plugging in the empirical variance.
Based on \citet{wang_2021_FastPureExploration}, the wrapped algorithms FWS and EV-FWS are obtained similarly.

FHN2 refers to the procedure 2 in \citet{Fan16RSUV}, whose pseudo-code (see Algorithm~\ref{algo:fhn2}) was given in \citet{hong_2021_ReviewRankingSelection}.
FHN$2$ is an elimination strategy which repeatedly samples all arms until only one arm is left.
For all $a \in [K]$, let $\bar X_{n,a} = \frac{1}{n} \sum_{t \in [n]} X_{t,a}$ where $(X_{t,a})_{t \in [n]}$ are i.i.d. observations from $\nu_{\mu_{a}, \sigma^2_{a}}$.
Its elimination mechanism compares the pairwise statistics $t_{b,a}(n)(\bar X_{n,b} - \bar X_{n,a})$ to an elimination threshold $g_{b,a}(t_{b,a}(n), \delta)$, where the effective time $t_{b,a}(n) = n/S_{b,a}(n)$ is defined with the empirical variances of pairwise comparison between arms $b$ and $a$
\[
	S_{b,a}(n) = \frac{1}{n-1} \sum_{t \in [n]} \left(X_{t,b} - X_{t,a} - \bar X_{n,b} + \bar X_{n,a}\right)^2 \: .
\]
The calibration of the elimination threshold, which is done by simulations arguments and continuous-time approximations, yields
\[
	g_{b,a}(x,\delta) = \sqrt{(x+1)\left(2\ln\left(\frac{K-1}{2\delta}\right) + \log(x + 1)\right)}\: .
\]
The choice of the initial time $n_0$ impacts the empirical performance of the algorithm, yet few practical guidelines were given.
In our experiments we took $n_0 = \max\{2, 10 \log(1 / \delta)\}$.

\begin{algorithm}[ht]
 \caption{FHN$2$}
 \label{algo:fhn2}
 	Let $I = [K]$ and $n=n_0$ (user-specified) \\
 	Sample $n_0$ observations from each arm \\
         \While{$|I|>1$}{
				Compute $S_{b,a}(n)$ and set $t_{b,a}(n) = n/S_{b,a}(n)$ for all $b \neq a$\\
				Update active arms, $I \leftarrow I \setminus \left\{ b \mid \forall a \neq b, \: t_{b,a}(n)(\bar X_{n,b} - \bar X_{n,a}) < -g_{b,a}(t_{b,a}(n), \delta)\right\}$ \\
            Pull each arm once, and set $n \leftarrow n+1$
        }
		Return the unique arm in $I$
\end{algorithm}

In \citet{Fan16RSUV}, they also propose the procedure $1$, which we refer to as FHN$1$.
The main difference is that FHN$1$ uses the empirical variances obtained after initialization, meaning it uses $t_{b,a}(n) = n/S_{b,a}(n_0)$ instead of $t_{b,a}(n) = n/S_{b,a}(n)$.
While one can argue that this wastes precious information, FHN$1$ enjoys better theoretical guaranty.
The impact of then choice of $n_0$ is more important for FHN$1$ since $S_{b,a}(n_0)$ is used for elimination.

\paragraph{Reproducibility}
Our code is implemented in \texttt{Julia 1.7.2}, and the plots are generated with the \texttt{StatsPlots.jl} package.
Optimizations are performed based on the \texttt{JuMP.jl} optimization package.
In particular, we use the Ipopt non-linear solver \citep{Wachter_2006_IpOpt}, available in \texttt{Ipopt.jl}, and the \texttt{LambertW.jl} package to compute the Lambert $W$ function.
Other dependencies are listed in the \texttt{Readme.md}.
The \texttt{Readme.md} file also provides detailed julia instructions to reproduce our experiments, as well as a \texttt{script.sh} to run them all at once.
The general structure of the code (and some functions) is taken from the \href{https://bitbucket.org/wmkoolen/tidnabbil}{tidnabbil} library.\footnote{This library was created by \cite{degenne_2019_NonAsymptoticPureExploration}, see https://bitbucket.org/wmkoolen/tidnabbil. No license were available on the repository, but we obtained the authorization from the authors.}

\subsection{Supplementary Experiments}
\label{app:ss_supplementary_experiments}

In Appendix~\ref{app:sss_exp_characteristic_times}, we perform simulations on the characteristic times $T^\star(\mu, \sigma^2)$ and $T^\star_{\sigma^2}(\mu)$.
In Appendix~\ref{app:sss_exp_thresholds}, we compare family of stopping thresholds for the EV-GLR stopping rule.

\subsubsection{Characteristic Times}
\label{app:sss_exp_characteristic_times}

We perform numerical simulations to compare $T^\star(\mu, \sigma^2)$ and $T^\star_{\sigma^2}(\mu)$.
Based on the inequalities (\ref{eq:characteristic_times_inequalities}), we will compare $\frac{T^\star(\mu, \sigma^2)}{T^\star_{\sigma^2}(\mu)}$ and its upper bound $\frac{d(\mu, \sigma^2)}{\ln(1+d(\mu, \sigma^2))}$ on $\frac{T^\star(\mu, \sigma^2)}{T^\star_{\sigma^2}(\mu)}$, where $d(\mu, \sigma^2) = \max_{a \neq a^\star(\mu)}\frac{(\mu_{a^\star(\mu)} - \mu_{a})^2}{\min\{\sigma_{a}^2,\sigma_{a^\star(\mu)}^2\}}$.
Taking $K=2$, we consider the canonical instance $\mu = (0,-\Delta)$ and $\sigma^2 = (1,r)$, hence $d(\mu, \sigma^2)=\frac{\Delta^2}{\min\{1,r\}}$, and instantiate it by default with $\Delta=0.2$ and $r=0.5$.
We perform $10000$ simulations for varying $\Delta$ and $r$.

\begin{figure}[ht]
	\centering
	\includegraphics[width=0.49\linewidth]{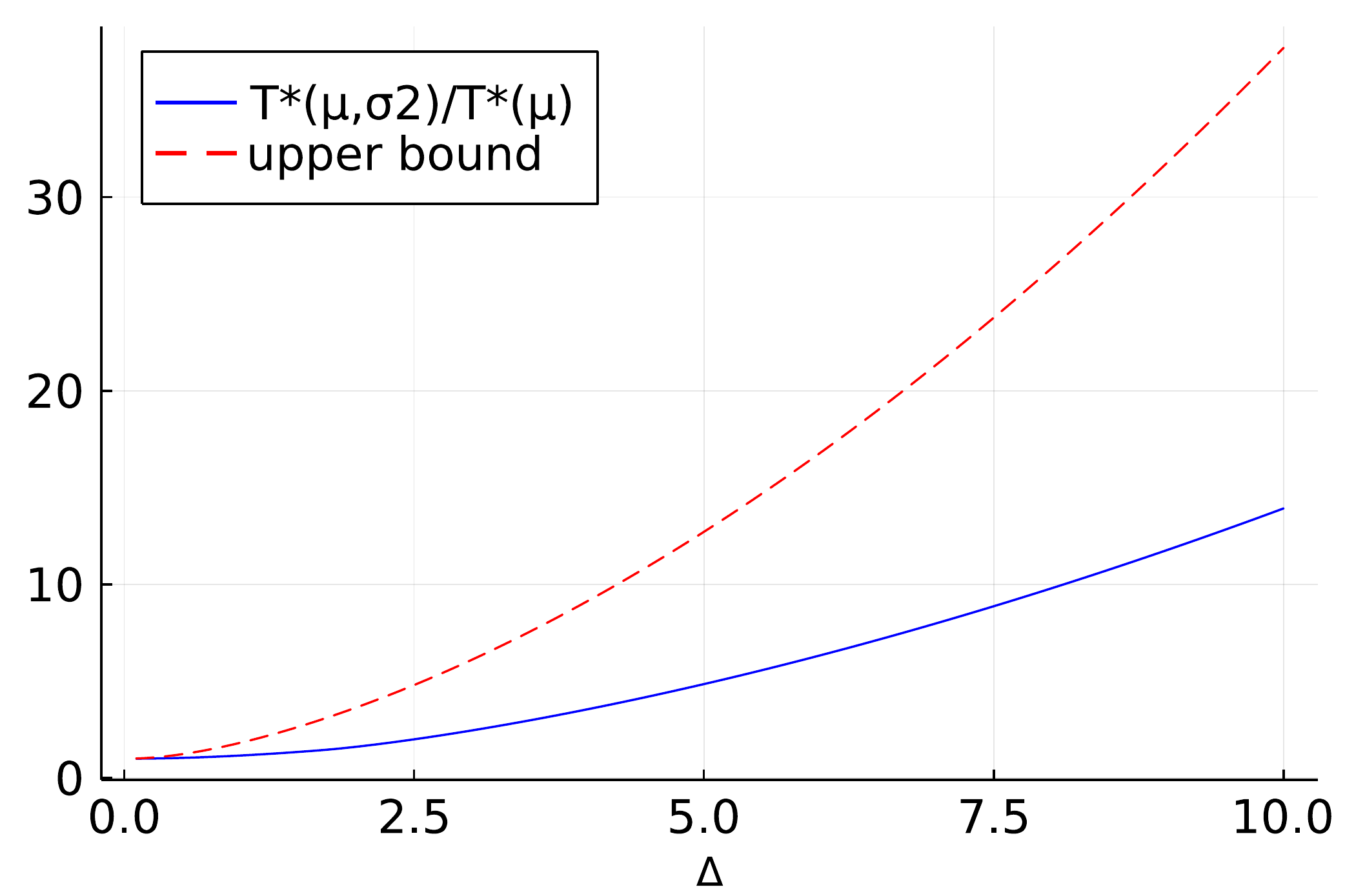}
	\includegraphics[width=0.49\linewidth]{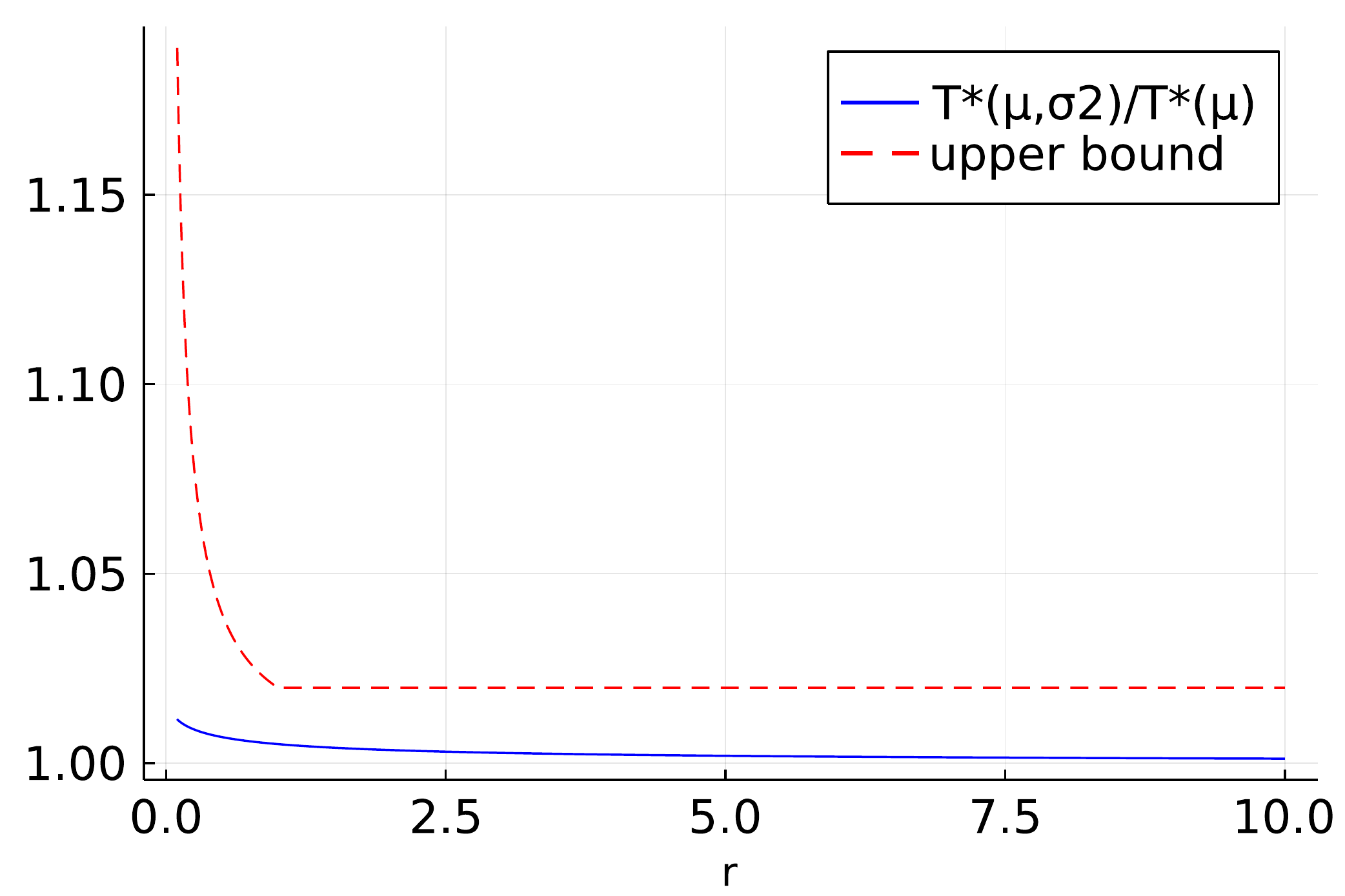}
	\caption{Evolution of $\frac{T^\star(\mu, \sigma^2)}{T^\star_{\sigma^2}(\mu)}$ as a function of (a) $\Delta$ for $r = 0.5$ and (b) $r$ for $\Delta = 0.2$.}
	\label{fig:evolution_ratios_characteristic_times}
\end{figure}

Figure~\ref{fig:evolution_ratios_characteristic_times} empirically confirms that $\frac{d(\mu, \sigma^2)}{\ln(1+d(\mu, \sigma^2))}$ is an upper bound on $\frac{T^\star(\mu, \sigma^2)}{T^\star_{\sigma^2}(\mu)}$, which is always above $1$.
It shows that $\frac{T^\star(\mu, \sigma^2)}{T^\star_{\sigma^2}(\mu)}$ is close to one as long as the gap is not too large.
Moreover, the upper bound $\frac{d(\mu, \sigma^2)}{\ln\left(1+d(\mu, \sigma^2)\right)}$ has the same behavior, except when the variance of the optimal arm is significantly larger than the one of the sub-optimal arm.
In this particular regime, our upper bound appears to be loose.

\subsubsection{Thresholds}
\label{app:sss_exp_thresholds}

Using the same experimental setup as in Section~\ref{ssec:simulations_stopping_thresholds}, we perform numerical simulations to compare the family of thresholds derived for the EV-GLR stopping rule: EV-Student (Lemma~\ref{lem:delta_correct_student_thresholds_ev}), EV-Box (Lemma~\ref{lem:delta_correct_box_thresholds_ev}) and EV-BoB (Corollary~\ref{cor:delta_correct_Kinf_thresholds_ev}).

\begin{figure}[h]
	\centering
	\includegraphics[width=0.48\linewidth]{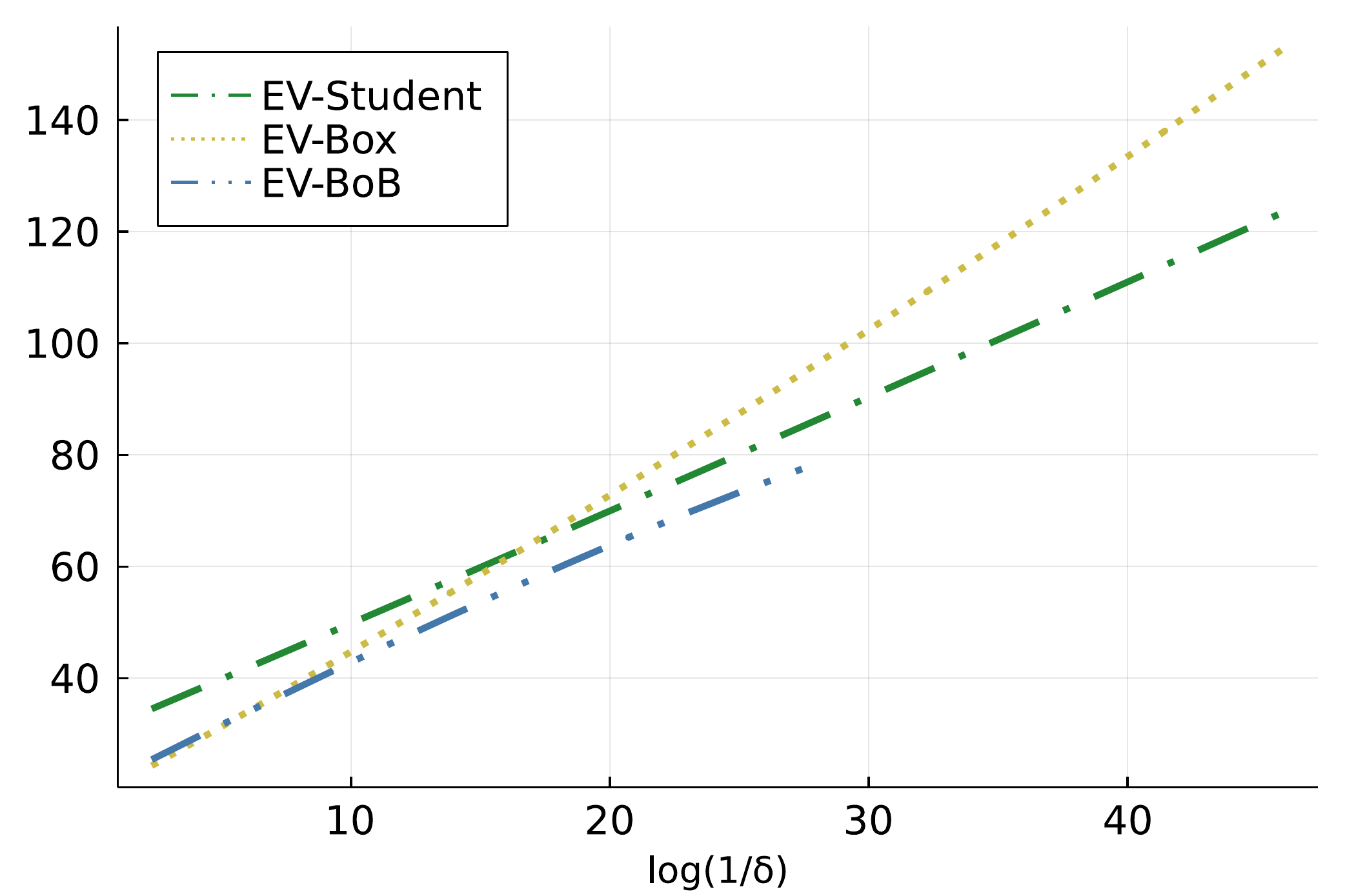}
	\includegraphics[width=0.48\linewidth]{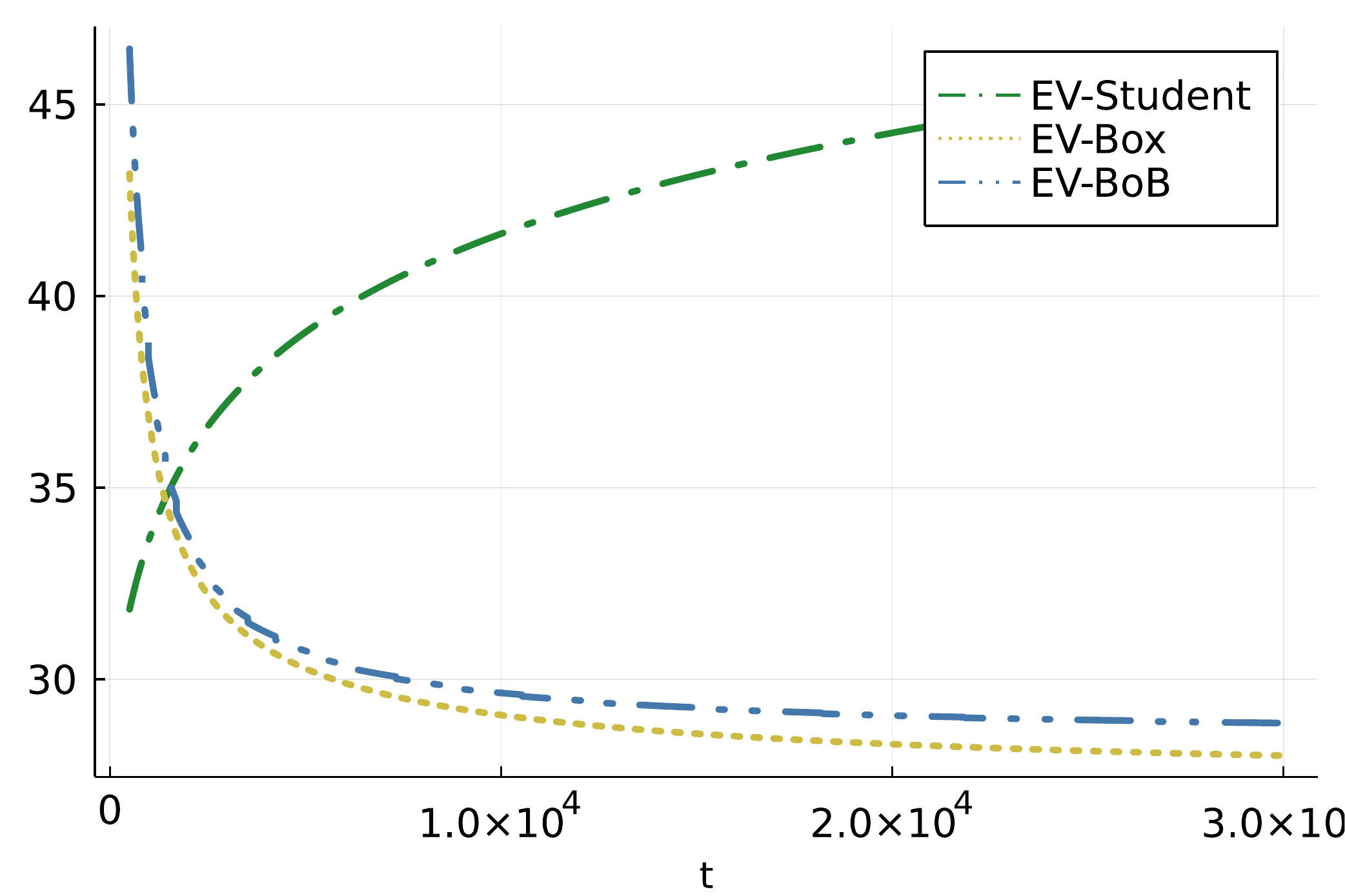}
	\caption{Threshold for (\ref{eq:def_stopping_rule_evglrt}) as a function of (a) $\ln\left(1/\delta\right)$ for $t = 5000$ and (b) $t$ for $\delta = 0.01$.}
	\label{fig:stopping_thresholds_evolutions_evglrt}
\end{figure}

Figure~\ref{fig:stopping_thresholds_evolutions_evglrt} is the equivalent of Figure~\ref{fig:stopping_thresholds_evolutions_glrt} and similar messages can be extracted from it.
However, for the EV-GLR stopping rule, the $\delta$-correct family of thresholds can be asymptotically tight, otherwise Theorem~\ref{thm:upper_bound_sample_complexity_algorithm} would contradict the asymptotic lower bound.
The Student threshold in (\ref{eq:def_student_threshold_evglrt}) has poor performance in $t$.
While the box threshold in (\ref{eq:def_box_threshold_evglrt}) yields better performance in $t$, it comes at a worse dependency in $\ln \left(1/\delta\right)$ (higher slope).
The BoB threshold in Corollary~\ref{cor:delta_correct_Kinf_thresholds_ev} has the good performance in $t$ of the box threshold and the same dependency in $\ln \left(1/\delta\right)$ as the Student threshold.

\begin{figure}[h]
	\centering
	\includegraphics[width=0.48\linewidth]{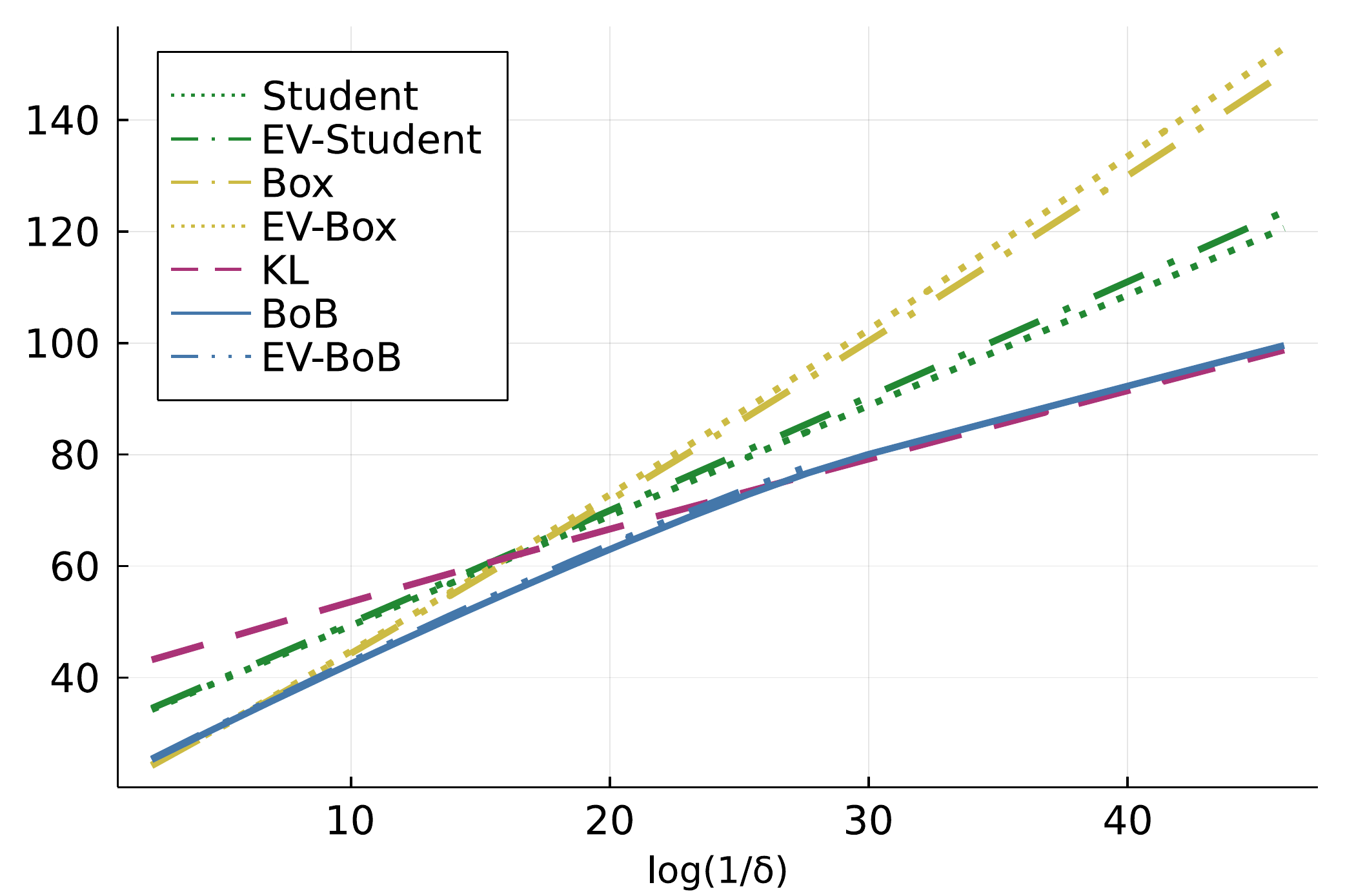}
	\includegraphics[width=0.48\linewidth]{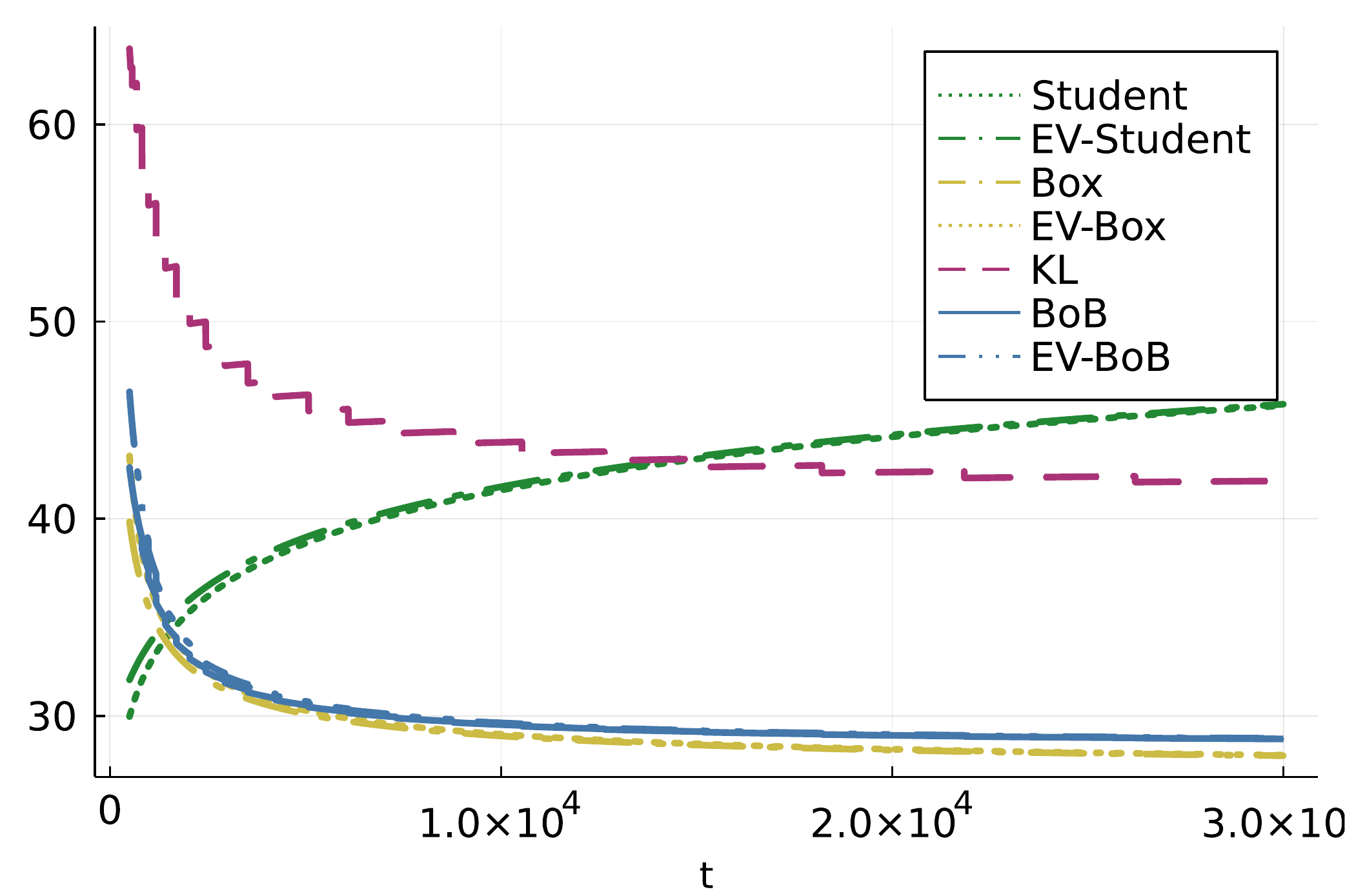}
	\caption{Thresholds as a function of (a) $\ln\left(1/\delta\right)$ for $t = 5000$ and (b) $t$ for $\delta = 0.01$.}
	\label{fig:stopping_thresholds_evolutions_glrt_evglrt}
\end{figure}

Figure~\ref{fig:stopping_thresholds_evolutions_glrt_evglrt} groups Figures~\ref{fig:stopping_thresholds_evolutions_glrt} and~\ref{fig:stopping_thresholds_evolutions_evglrt} together.
By comparing the thresholds obtained by using the same method, all the thresholds for the EV-GLR stopping rule are above the ones obtained for the GLR stopping rule.
Given the ordering of the statistic (Lemma~\ref{lem:glrt_evglrt_inequalities}), this was expected.